\documentclass[letterpaper, 10pt]{article}

\usepackage[dvipsnames]{xcolor}
\usepackage[hypertexnames=false]{hyperref}
\usepackage{
  amsmath, amsthm, amssymb, mathtools, dsfont, units,       
  graphicx, wrapfig, subfig, float,                         
  listings, color, inconsolata, pythonhighlight,            
  fancyhdr, enumitem, framed ,comment, multirow, 
  appendix, authblk, refcount
}
\usepackage[ruled,vlined]{algorithm2e}

\usepackage{newpxtext, newpxmath, inconsolata}
\usepackage[cal=cm,frak=pxtx,bb=dsfontserif,scr=boondoxo]{mathalpha}
\usepackage[utf8]{inputenc}

\usepackage[
    backend=biber,
    style=numeric,
    giveninits=true,
    maxnames = 999
]{biblatex}

\DeclareFieldFormat{journaltitlecase}{#1}
\DeclareFieldFormat[article]{journaltitle}{#1}

\DeclareFieldFormat[article]{volume}{\mkbibbold{#1}}

\DeclareFieldFormat[article]{title}{\mkbibquote{\mkbibemph{\MakeSentenceCase{#1}}}}
\DeclareFieldFormat[inproceedings]{title}{\mkbibquote{\mkbibemph{\MakeSentenceCase{#1}}}}

\DeclareFieldFormat[inproceedings]{booktitle}{#1}

\renewbibmacro*{in:}{}

\AtEveryBibitem{
  \clearfield{eprint}
  \clearfield{primaryclass}
  \clearfield{issn}
  \clearfield{number}
}

\renewbibmacro*{in:}{} 
\AtEveryBibitem{
    \clearfield{eprint} 
    \clearfield{primaryclass} 
    \clearfield{issn} 
    \clearfield{number} 
} 

\usepackage[left=1in, right=1in, top=1.0in, bottom=.9in, headsep=.2in, footskip=0.35in]{geometry}

\usepackage[bottom]{footmisc}

\allowdisplaybreaks

\usepackage[font={it,footnotesize}]{caption}

\hypersetup{colorlinks=true, linkcolor=RoyalBlue, citecolor=RedOrange, urlcolor=RoyalBlue}

\usepackage{titlesec}
\titleformat{\section}{\large\bfseries\selectfont}{\thesection\;\;\;}{0em}{}
\titleformat{\subsection}{\normalsize\bfseries\selectfont}{\thesubsection\;\;\;}{0em}{}

\setlist[itemize]{wide=0pt, leftmargin=16pt, labelwidth=10pt, align=left}

\usepackage[subfigure]{tocloft}

\theoremstyle{plain}
\newtheorem{theorem}{Theorem}[section]
\newtheorem{proposition}[theorem]{Proposition}
\newtheorem{lemma}[theorem]{Lemma}
\newtheorem{corollary}[theorem]{Corollary}

\theoremstyle{definition}
\newtheorem{definition}[theorem]{Definition}
\newtheorem{remark}{Remark}[section]
\newtheorem{assumption}{Assumption}
\newtheorem*{notations*}{Notations}

\makeatletter
\newenvironment{assumptionprime}[2]
{
    \renewcommand{\theassumption}{\getrefnumber{#1}$'$}
    \@ifundefined{theHassumption}{}{%
        \renewcommand{\theHassumption}{\getrefnumber{#1}prime}%
    }
    \begin{assumption}[#2]
}
{
    \end{assumption}
    \addtocounter{assumption}{-1}
}
\makeatother

\numberwithin{equation}{section}

\newcommand{\bbr}{\mathbb{R}}

\newcommand{\bbc}{\mathbb{C}}

\newcommand{\bbp}{\mathbb{P}}

\newcommand{\bbj}{\mathbb{J}}
\newcommand{\bbh}{\mathbb{H}}
\newcommand{\bbe}{\mathbb{E}}
\newcommand{\bbn}{\mathbb{N}}
\newcommand{\bbone}{\mathbb{1}}
\newcommand{\bbb}{\mathbb{B}}

\newcommand{\bfa}{\mathbf{a}}
\newcommand{\bfh}{\mathbf{h}}
\newcommand{\bfx}{\mathbf{x}}
\newcommand{\bfe}{\mathbf{e}}
\newcommand{\bfy}{\mathbf{y}}
\newcommand{\bfs}{\mathbf{s}}
\newcommand{\bfz}{\mathbf{z}}
\newcommand{\bfu}{\mathbf{u}}
\newcommand{\bfv}{\mathbf{v}}
\newcommand{\bfb}{\mathbf{b}}
\newcommand{\bfw}{\mathbf{w}}
\newcommand{\bfone}{\mathbf{1}}
\newcommand{\bfzero}{\mathbf{0}}

\newcommand{\bfS}{\mathbf{S}}
\newcommand{\bfH}{\mathbf{H}}
\newcommand{\bfK}{\mathbf{K}}
\newcommand{\bfD}{\mathbf{D}}
\newcommand{\bfL}{\mathbf{L}}
\newcommand{\bfI}{\mathbf{I}}
\newcommand{\bfT}{\mathbf{T}}

\newcommand{\bfQ}{\mathbf{Q}}
\newcommand{\bfA}{\mathbf{A}}

\newcommand{\bfB}{\mathbf{B}}

\newcommand{\bfF}{\mathbf{F}}
\newcommand{\bfR}{\mathbf{R}}
\newcommand{\bfU}{\mathbf{U}}
\newcommand{\bfV}{\mathbf{V}}

\newcommand{\bfW}{\mathbf{W}}
\newcommand{\bfM}{\mathbf{M}}

\newcommand{\barK}{\bar{\mathbf{K}}}

\newcommand{\barU}{\bar{\mathbf{U}}}
\newcommand{\baru}{\bar{\mathbf{u}}}
\newcommand{\barV}{\bar{\mathbf{V}}}
\newcommand{\barv}{\bar{\mathbf{v}}}
\newcommand{\barLamb}{\bar{\boldsymbol{\Lambda}}}
\newcommand{\barlamb}{\bar{\lambda}}
\newcommand{\barGamma}{\bar{\boldsymbol{\Gamma}}}
\newcommand{\bargamma}{\bar{\gamma}}

\newcommand{\barDelta}{\bar{\Delta}}
\newcommand{\barL}{\bar{\mathbf{L}}}
\newcommand{\barD}{\bar{\mathbf{D}}}

\newcommand{\ddotK}{\ddot{\mathbf{K}}}
\newcommand{\ddotu}{\ddot{\mathbf{u}}}
\newcommand{\ddotU}{\ddot{\mathbf{U}}}
\newcommand{\ddotv}{\ddot{\mathbf{v}}}
\newcommand{\ddotV}{\ddot{\mathbf{V}}}
\newcommand{\ddotLamb}{\ddot{\boldsymbol{\Lambda}}}
\newcommand{\ddotlamb}{\ddot{\lambda}}
\newcommand{\ddotGamma}{\ddot{\boldsymbol{\Gamma}}}
\newcommand{\ddotgamma}{\ddot{\gamma}}

\newcommand{\ddotG}{\ddot{\mathbf{G}}}
\newcommand{\ddotfkQ}{\ddot{\mathfrak{Q}}}
\newcommand{\ddotfkS}{\ddot{\mathfrak{S}}}
\newcommand{\ddotL}{\ddot{\mathbf{L}}}
\newcommand{\ddotD}{\ddot{\mathbf{D}}}

\newcommand{\rmd}{\mathrm{d}}
\newcommand{\rms}{\mathrm{s}}
\newcommand{\rmi}{\mathrm{i}}

\newcommand{\caL}{\mathcal{L}}
\newcommand{\caH}{\mathcal{H}}
\newcommand{\caS}{\mathcal{S}}
\newcommand{\caC}{\mathcal{C}}
\newcommand{\caM}{\mathcal{M}}

\newcommand{\bfmu}{\boldsymbol{\mu}}

\newcommand{\bfSigma}{\boldsymbol{\Sigma}}
\newcommand{\bfLamb}{\boldsymbol{\Lambda}}
\newcommand{\bfPhi}{\boldsymbol{\Phi}}

\newcommand{\bfzeta}{\boldsymbol{\zeta}}
\newcommand{\bfUps}{\boldsymbol{\Upsilon}}
\newcommand{\bfups}{\boldsymbol{\upsilon}}
\newcommand{\bfPi}{\boldsymbol{\Pi}}
\newcommand{\bfOmega}{\boldsymbol{\Omega}}
\newcommand{\bfGamma}{\boldsymbol{\Gamma}}
\newcommand{\bfXi}{\boldsymbol{\Xi}}

\newcommand{\fkx}{\mathfrak{x}}

\newcommand{\fkq}{\mathfrak{q}}
\newcommand{\fka}{\mathfrak{a}}
\newcommand{\fks}{\mathfrak{s}}
\newcommand{\fkn}{\mathfrak{n}}

\newcommand{\fkU}{\mathfrak{U}}
\newcommand{\fkV}{\mathfrak{V}}
\newcommand{\fkK}{\mathfrak{K}}
\newcommand{\fkA}{\mathfrak{A}}
\newcommand{\fkB}{\mathfrak{B}}
\newcommand{\fkX}{\mathfrak{X}}
\newcommand{\fkL}{\mathfrak{L}}
\newcommand{\fkD}{\mathfrak{D}}
\newcommand{\fkQ}{\mathfrak{Q}}
\newcommand{\fkS}{\mathfrak{S}}

\newcommand{\cov}{\operatorname{Cov}}
\newcommand{\Tr}{\operatorname{Tr}}
\newcommand{\dist}{\operatorname{dist}}
\newcommand{\diag}{\operatorname{diag}}
\newcommand{\rank}{\operatorname{rank}}
\newcommand{\sgn}{\operatorname{sgn}}

\newcommand{\Fnorm}{{\mathrm{F}}}
\newcommand{\argmin}{\mathop{\mathrm{argmin}}}

\newcommand{\qand}{\quad \text{ and } \quad}
\newcommand{\qwhere}{\quad \text{ where } \quad}
\newcommand{\qwith}{\quad \text{ with } \quad}
\newcommand{\qfor}{\quad \text{ for } \quad}

\DeclarePairedDelimiter\pars{\lparen}{\rparen}
\DeclarePairedDelimiter\braks{\lbrack}{\rbrack}
\DeclarePairedDelimiter\dbraks{\llbracket}{\rrbracket}
\DeclarePairedDelimiter\curls{\lbrace}{\rbrace}
\DeclarePairedDelimiter\angles{\langle}{\rangle}
\DeclarePairedDelimiter\abs{\lvert}{\rvert}
\DeclarePairedDelimiter\norm{\lVert}{\rVert}

\NewDocumentCommand{\oprec}{o m}{%
  O_{\prec}\IfNoValueTF{#1}
    {\pars{#2}}        
    {\pars[#1]{#2}}    
}
\NewDocumentCommand{\bigO}{o m}{%
  O\IfNoValueTF{#1}
    {\pars{#2}}        
    {\pars[#1]{#2}}    
}

\newcommand{\shortpara}[1]{\noindent\underline{\emph{#1}}}

\begin{document}


\title{Multi-kernel spectral clustering: Entrywise eigenvector perturbation bounds and exact recovery}


\author[1]{Zeqin Lin\textsuperscript{a,}}
\author[1]{Guangming Pan\textsuperscript{b,}}
\author[2]{Zhixiang Zhang\textsuperscript{c,}}
\author[3]{Yinbing Zhou\textsuperscript{d,}}

\affil[1]{Nanyang Technological University
\textsuperscript{a}\texttt{\href{mailto:zeqin.lin@ntu.edu.sg}{zeqin.lin@ntu.edu.sg}}; 
\textsuperscript{b}\texttt{\href{mailto:gmpan@ntu.edu.sg}{gmpan@ntu.edu.sg}}}
\affil[2]{University of Macau
\textsuperscript{c}\texttt{\href{mailto:zhixzhang@um.edu.mo}{zhixzhang@um.edu.mo}}}
\affil[3]{Dalian University of Technology
\textsuperscript{d}\texttt{\href{mailto:zhouyinbing_@outlook.com}{zhouyinbing\_@outlook.com}}}

\date{\vspace{-2.5em}}

\renewcommand\Authfont{\normalsize}
\renewcommand\Affilfont{\footnotesize}



\maketitle
\begin{abstract}
Kernel spectral clustering with a single bandwidth can be inadequate for data exhibiting multiple characteristic pairwise-distance scales, a problem particularly prevalent in the high-dimensional regime. We address this issue through a multi-kernel formulation that aggregates kernels with different bandwidths. The bandwidths are selected as prescribed empirical quantiles of the pairwise squared distances, thereby capturing the relevant distance scales without requiring prior population-scale information. 

We develop a rigorous theoretical analysis of the resulting method under a general high-dimensional, multi-scale mixture model with heterogeneous cluster centers and covariance geometries. We construct a blockwise constant, low-rank informative approximation to the empirical multi-kernel matrix and establish row-wise $\ell_{2,\infty}$ perturbation bounds for its leading spectral components, as well as for the associated normalized Laplacian matrix. These bounds yield observation-level control of the spectral embedding, which is more informative than conventional global eigenspace perturbation estimates. Under suitable eigen-gap and cluster-separation conditions, we show that approximate $K$-means applied to the multi-kernel spectral embedding achieves exact recovery with high probability.
\end{abstract}

\tableofcontents \label{sec:contents}


\section{Introduction}

Modern data sets are often high-dimensional while exhibiting intrinsic nonlinear structure. Such nonlinear geometric features arise in many real-world applications, yet they are not adequately captured by linear methods. This observation motivates the use of kernel-based approaches \cite{scholkopfLearningKernelsSupport2002}, which provide a flexible framework for encoding nonlinearity through implicit feature mappings. In parallel, spectral methods \cite{chenSpectralMethodsData2021a} play a central role in high-dimensional data analysis by producing low-dimensional embeddings that are computationally tractable while preserving essential structural information.

By combining these two ideas, kernel spectral clustering (KSC) \cite{ngSpectralClusteringAnalysis2001,shiNormalizedCutsImage2000} provides a flexible framework for identifying cluster structure in complex, high-dimensional data. A widely used implementation of KSC consists of three steps. First, a kernel similarity matrix and its associated normalized Laplacian matrix are constructed. Specifically, given data points $\{ \bfx_i \}_{i=1}^n \subset \bbr^p$, the kernel similarity matrix is given by
\begin{equation}
    \bfK = [K_{ij}]_{i,j=1}^n, 
    \quad \text{ with } \quad 
    K_{ij} = f (\bfx_i, \bfx_j; h).
    \label{def:single-Kernel-general}
\end{equation}
where $f: \bbr^p \times \bbr^p \to \bbr$ is a symmetric kernel function depending on a bandwidth parameter $h > 0$. The associated normalized Laplacian matrix is then defined as
\begin{equation*}
    \bfL = n \bfD^{-1/2} \bfK \bfD^{-1/2},
    \qwhere
    \bfD = \diag (D_{11}, \cdots, D_{nn})
    \quad \text{ with } \quad
    D_{ii} = \sum\nolimits_{j = 1}^n K_{ij}
\end{equation*}
Second, the leading eigenvectors of $\bfK$ or $\bfL$ are computed to obtain a low-dimensional embedding of the data. Finally, the $K$-means algorithm is applied to the embedded dataset for the final clustering. In contrast to standard spectral clustering methods based on linear or directly observed similarities, kernel functions introduce a nonlinear transformation of pairwise distances, enabling the representation of complex, nonlinearly separable structures.  The spectral embedding step then extracts the dominant components of $\bfK$ or $\bfL$, yielding a computationally tractable low-dimensional representation that retains the essential information for clustering. When this embedding faithfully reflects the underlying cluster structure, the clusters can be effectively recovered using a simple algorithm such as $K$-means. For theoretical analyses of KSC in classical asymptotic and nonparametric settings, see, for example, \cite{luxburgConsistencySpectralClustering2008,schiebingerGeometryKernelizedSpectral2015} and the references therein.

A critical aspect of kernel methods is the choice of the bandwidth parameter. Although numerous practical heuristics have been proposed, its theoretical understanding remains limited. The issue is particularly challenging for data with multi-scale structure, where different components of the data operate at distinct characteristic scales. In this setting, a single global bandwidth entails an inherent trade-off. A bandwidth calibrated to large-scale components over-smooths the small-scale ones, making their within- and between-cluster kernel similarities insufficiently distinguishable. Conversely, a bandwidth calibrated to small-scale components yields negligible similarities among observations from large-scale components, fragmenting the induced similarity graph and failing to preserve within-cluster connectivity. Such difficulties are especially pronounced in high-dimensional mixture models, where substantial heterogeneity in pairwise-distance scales can arise naturally from the underlying data-generating mechanism.

To address this limitation, we incorporate the framework of multiple kernel learning (MKL) \cite{gonenMultipleKernelLearning2011} into KSC, which provides a flexible mechanism for balancing information across multiple scales. Instead of relying on a single global bandwidth, MKL constructs a weighted sum of multiple kernel matrices, where the bandwidth parameters are allowed to vary across kernels. Specifically, we replace \eqref{def:single-Kernel-general} by
\begin{equation}
    \bfK = [K_{ij}]_{i,j=1}^n,
    \quad \text{ with } \quad
    K_{ij} = \sum\nolimits_{t=1}^T \alpha_t f_t(\bfx_i, \bfx_j; h_t),
    \label{def:multi-Kernel-general}
\end{equation}
where $T \in \bbn$ is the number of kernel components, $\alpha_t > 0$ are kernel weights satisfying $\sum_{t=1}^T \alpha_t = 1$, and $h_t > 0$ are bandwidth parameters. We refer to the resulting method as multi-kernel spectral clustering (multi-KSC) to distinguish it from the standard single-bandwidth formulation. By allowing bandwidths at different scales, this approach effectively mitigates the over-smoothing and under-connection issues inherent in single-bandwidth methods.

The main contributions of this manuscript are threefold. First, building on the multi-kernel construction \eqref{def:multi-Kernel-general}, we develop a concrete multi-KSC procedure equipped with a simple quantile-based rule for bandwidth selection. The rule depends only on the empirical distribution of pairwise distances and requires neither cluster labels nor prior knowledge of the underlying characteristic scales. Under suitable conditions, we show that the selected bandwidths capture the relevant distance scales with high probability. 

Second, we establish high-probability spectral perturbation bounds for the resulting multi-kernel matrix and its associated normalized Laplacian. In addition to eigenvalue bounds, we derive row-wise $\ell_{2,\infty}$ perturbation bounds for their leading eigenspaces. To our knowledge, comparable row-wise perturbation guarantees for the informative eigenspaces of high-dimensional empirical kernel matrices are not currently available, even in the single-kernel setting. Our analysis explicitly tracks the interaction between the intrinsic distance scales and the kernel bandwidths and employs a two-stage comparison argument to exploit the stochastic structure of the perturbations. These results provide pointwise control of the spectral embeddings that cannot be obtained directly by combining spectral-norm bounds with the Davis--Kahan theorem.

Third, we apply these perturbation bounds to analyze the proposed multi-KSC procedure with data-driven bandwidths. Under suitable conditions, we show that approximate $K$-means applied to the multi-kernel spectral embedding achieves exact recovery with high probability.

\subsection{Organization}

The remainder of the manuscript is organized as follows. Section \ref{subsec:related-works} reviews several lines of research related to this work. Section \ref{sec:main-res} motivates our multi-kernel construction and establishes eigenvalue and row-wise eigenspace perturbation bounds for the kernel matrix and its normalized Laplacian. Section \ref{sec:KSC-algorithm} presents the multi-KSC procedure with quantile-based bandwidth selection and proves exact recovery under suitable conditions. All technical proofs are deferred to Sections \ref{sec:proof-kernel}--\ref{sec:proof-misclass}.

\subsection{Related works}
\label{subsec:related-works}

Motivated by the widespread use of kernel methods, a substantial literature has investigated the spectral properties of random kernel matrices. Early work, such as \cite{koltchinskiiRandomMatrixApproximation2000}, considered samples from a fixed manifold and related the empirical eigenvalues to those of an associated integral operator. The seminal works of El Karoui \cite{karouiInformationNoiseKernel2010,karouiSpectrumKernelRandom2010} initiated the corresponding high-dimensional theory, in which the ambient dimension $p$ grows jointly with the sample size $n$. Subsequent research has developed along two main directions.

The first direction concerns the empirical spectral distribution of random kernel matrices. Initial results focused on the proportional regime $p \asymp n$ \cite{chengSpectrumRandomInnerproduct2013,dingSpectralPropertyKernelbased2021,doSpectrumRandomKernel2013,karouiSpectrumKernelRandom2010}, while more recent work has considered the general polynomial regimes $p \asymp n^\alpha$ \cite{dubovaUniversalityGlobalSpectrum2023,luEquivalencePrincipleSpectrum2025}. In this line of research, the feature vectors $\bfx_i$ are typically modeled as high-dimensional noise. Related works analyze extreme eigenvalues at the spectral edge \cite{fanSpectralNormRandom2019,koganExtremalEigenvaluesRandom2024}.

The second direction focuses on extracting the informative structure underlying the empirical kernel matrix, typically manifested through spiked eigenvalues separated from the bulk spectrum and their associated eigenvectors \cite{aminiConcentrationKernelMatrices2021,couilletKernelSpectralClustering2016,dingLearningLowdimensionalNonlinear2023,dingImpactSignaltonoiseRatio2023,karouiInformationNoiseKernel2010}. Particularly relevant to the present work are \cite{aminiConcentrationKernelMatrices2021,couilletKernelSpectralClustering2016}, both of which are motivated by spectral clustering. In the proportional regime $p \asymp n$, Couillet and Benaych-Georges \cite{couilletKernelSpectralClustering2016} characterized the outlier eigenvalues and associated eigenvectors of the kernel normalized Laplacian for Gaussian mixture models (GMMs), thereby revealing how clustering information is encoded in its spectral structure. From a complementary nonasymptotic perspective, Amini and Razaee \cite{aminiConcentrationKernelMatrices2021} established concentration inequalities for Lipschitz kernel matrices around their expectations. These results yield $\ell_2$ control of the spectral embedding and, consequently, misclassification guarantees for KSC under certain mixture models. Our work is most closely aligned with this second direction and develops nonasymptotic perturbation bounds for empirical eigenvectors that encode cluster-membership information. It differs from the existing literature in two main respects.

First, for eigenvectors of random kernel matrices corresponding to isolated eigenvalues, existing works such as \cite{aminiConcentrationKernelMatrices2021,dingLearningLowdimensionalNonlinear2023} primarily establish $\ell_2$ perturbation bounds. These bounds are typically obtained by controlling the spectral-norm deviation of the empirical kernel matrix from its informative counterpart, followed by an application of the classical Davis--Kahan theorem \cite{davisRotationEigenvectorsPerturbation1970}. Since spectral clustering is inherently a pointwise problem rather than a purely global one, we instead establish $\ell_\infty$ bounds for the spiked eigenvectors. Such entrywise bounds provide finer control of the spectral embedding and generally yield sharper misclassification guarantees. Second, our analysis explicitly addresses multi-scale mixture data, whereas existing theories \cite{couilletKernelSpectralClustering2016,aminiConcentrationKernelMatrices2021} primarily concern single-kernel constructions and do not explicitly account for the interaction between intrinsic distance scales and kernel bandwidths. By tracking this interaction, we obtain perturbation bounds tailored to the multi-scale structure. See Remark \ref{remark:compare-Amini} for more details.

Recently, the study of $\ell_\infty$ eigenvector perturbation bounds and, more generally, $\ell_{2,\infty}$ eigenspace perturbation bounds for large random signal-plus-noise matrices has attracted considerable attention. Unlike the classical Davis--Kahan--Wedin framework \cite{davisRotationEigenvectorsPerturbation1970,wedinPerturbationBoundsConnection1972}, which provides deterministic control of global subspace deviations, this theory controls row-wise errors and exploits the stochastic structure of the perturbation to obtain sharper high-probability bounds. A representative early contribution is \cite{fanell_inftyEigenvectorPerturbation2018}, which establishes $\ell_\infty$ eigenvector perturbation bounds for incoherent low-rank signal matrices using leave-one-out arguments. Subsequent work has developed sharper entrywise and row-wise bounds across a broad range of random matrix models; see, for example, \cite{abbeLpTheoryPCA2022,abbeEntrywiseEigenvectorAnalysis2020,agterbergEntrywiseEstimationSingular2022,bhardwajMatrixPerturbationDavisKahan2024,capeSignalplusnoiseMatrixModels2019,capeTwotoinfinityNormSingular2019,eldridgeUnperturbedSpectralAnalysis2018,leiUnifiedell_2rightarrowinftyEigenspace2020}. Of particular relevance here, such bounds have proved useful for establishing exact-recovery and misclassification guarantees in community detection and clustering \cite{abbeLpTheoryPCA2022,abbeEntrywiseEigenvectorAnalysis2020,maoEstimatingMixedMemberships2021,leiUnifiedell_2rightarrowinftyEigenspace2020}. For a broader overview and additional references, we refer to \cite{wangAnalysisSingularSubspaces2026}.

Despite these remarkable developments, entrywise eigenvector analysis for empirical kernel matrices in high-dimensional regimes does not appear to be directly available in the existing literature, even for kernel matrices involving a single bandwidth parameter. The main difficulty lies in the intricate dependence structure induced by the nonlinear kernel function. In this manuscript, we pursue this direction by combining high-dimensional kernel matrix approximation arguments with techniques from entrywise eigenvector perturbation theory.

We conclude this section with a brief review of MKL, a principled framework for combining candidate kernels rather than committing to a single one. These kernels may encode different notions of similarity, feature representations, data modalities, or choices of kernel parameters. The central idea is to learn or prescribe a combination of base kernels that is better adapted to the task at hand. Early studies \cite{lanckrietLearningKernelMatrix2004,bachMultipleKernelLearning2004} laid the optimization foundations of MKL, while subsequent research \cite{kloftLpNormMultipleKernel2011,rakotomamonjySimpleMKL2008,sonnenburgLargeScaleMultiple2006} developed computationally scalable methods and flexible regularization schemes; see \cite{gonenMultipleKernelLearning2011} for a comprehensive survey.

Although MKL was not originally designed to address the multi-scale phenomenon considered here, its organizing principle is naturally suited to our setting. Our base kernels differ in bandwidth and therefore capture similarity information at different distance scales. Combining them allows this information to be retained simultaneously, without requiring a single bandwidth to be appropriate for all cluster pairs. Multi-KSC can thus be viewed as a scale-adaptive specialization of MKL, for which we develop a rigorous theoretical analysis in a high-dimensional multi-scale regime.

\subsection{Notations}
\label{subsec:notations}

For vectors $\bfa, \bfb \in \bbr^q$, we use $\angles{\bfa,\bfb}$ and $\norm{\bfa}$ to denote their Euclidean inner product and Euclidean norm, respectively. We use $\bfone = (1,\ldots,1)^\top \in \bbr^n$ for the all-ones vector. For a matrix $\bfA$, $\norm{\bfA}$ and $\norm{\bfA}_{\Fnorm}$ denote its operator norm and Frobenius norm, respectively. For $\bfA\in\bbr^{q\times r}$, write $\bfA^\top = [\bfa_1,\ldots,\bfa_q]$, so that $\bfa_i^\top$ is the $i$-th row of $\bfA$. Its row-wise $\ell_{2,\infty}$ norm is defined by $\norm{\bfA}_{2,\infty} := \max\nolimits_{i\in\dbraks{q}} \norm{\bfa_i}$.

For $q \in \bbn$, we write $\dbraks{q} := \{1,\cdots,q\}$. We regard $n$ as the fundamental asymptotic parameter. Unless explicitly stated otherwise, all quantities may depend on $n$, although we frequently suppress this dependence in the notation. The notation $a_n \lesssim b_n$ means that $a_n \leq C b_n$ for some positive constant $C$ independent of $n$. We write $a_n \gtrsim b_n$ if $b_n \lesssim a_n$, and $a_n \asymp b_n$ if both $a_n \lesssim b_n$ and $a_n \gtrsim b_n$ hold. Constants implicit in $\lesssim$, $\gtrsim$, and $\asymp$ may change from line to line. We write $a_n \ll b_n$ if $a_n/b_n \to 0$ as $n\to\infty$. We introduce the following notion of high probability bounds to systematize statements of the form ``$X$ is bounded with high probability by $Y$ up to small powers of $n$''. Some of its basic properties are summarized in Lemma \ref{lemma:basic-property-prec}.

\begin{definition}[stochastic domination]
\label{def:stochastic-domination}
Let $X \equiv X_n(t)$ and $Y \equiv Y_n(t)$ be two families of
nonnegative random variables parameterized by $n\in\bbn$ and
$t\in\mathcal{T}_n$. We say that $X$ is \emph{stochastically dominated}
by $Y$, uniformly in $t\in\mathcal{T}_n$, and write $X\prec Y$, if, for
every $\varepsilon>0$ and $C>0$,
\begin{equation*}
    \sup\nolimits_{t\in\mathcal{T}_n}
    \bbp\{X_n(t)>n^\varepsilon Y_n(t)\}
    \leq n^{-C},
    \qquad
    \forall n\geq n_0(\varepsilon,C).
\end{equation*}
For a possibly signed or complex-valued random variable $X$, we write
$X=\oprec{Y}$ if $\abs{X}\prec Y$.
\end{definition}
\section{Spectral perturbation bounds}
\label{sec:main-res}

In this section, we establish row-wise eigenspace perturbation bounds for the multi-kernel matrix
\begin{equation}
    \bfK = [K_{ij}]_{i,j=1}^n,
    \quad \text{with} \quad
    K_{ij}
    := (1 - \delta_{ij}) \sum\nolimits_{t=1}^T \alpha_t
    \exp \pars[\big]{ - \norm{\bfx_i-\bfx_j}^2 / h_t }.
    \label{def:multi-Kernel}
\end{equation}
Here, the weights $\alpha_t \in (0,1)$ satisfy $\sum_{t=1}^T \alpha_t = 1$. The matrix $\bfK$ in \eqref{def:multi-Kernel} is the radial basis function (RBF) specialization of the general formulation in \eqref{def:multi-Kernel-general}. We focus on RBF kernels because this specialization substantially streamlines the analysis while still covering a broad range of practically relevant settings. Extensions to more general nonlinear kernel profiles are discussed in Remark \ref{remark:general-distance}. We set the diagonal entries of $\bfK$ to zero so that it represents a weighted graph without self-loops. Since the usual RBF diagonal equals $\sum_{t=1}^T\alpha_t=1$, retaining it would add $\bfI_n$ to $\bfK$, shifting its eigenvalues without changing its eigenvectors.

We now introduce the basic setup for the dataset $\{ \bfx_i \}_{i=1}^n \subset \bbr^p$. Suppose that the cluster membership is encoded by a partition $\{ \caC_k \}_{k=1}^m$ of $\dbraks{n}$. Specifically, $i \in \caC_k$ if and only if the $i$-th observation belongs to the $k$-th cluster. Let $n_k := \abs{\caC_k}$. We assume that the observations $\{ \bfx_i \}_{i \in \caC_k}$ are i.i.d. random vectors with
\begin{equation}
    \bfx_i = \bfmu_k + \bfA_k \bfz_i,
    \qfor i \in \caC_k.
    \label{eqn:data-model}
\end{equation}
Here the deterministic vector $\bfmu_k \in \bbr^p$ represents the center of the $k$-th cluster, and the deterministic matrix $\bfA_k \in \bbr^{p \times d_k}$ describes the local geometry of the cluster. The isotropic random vectors $\bfz_i \in \bbr^{d_k}$ capture the random variation of the observations and satisfy $\bbe \bfz_i = \bfzero$ and $\cov (\bfz_i) = \bfI$. Consequently,
\begin{equation}
    \bbe \bfx_i = \bfmu_k
    \qand
    \cov (\bfx_i) = \bfA_k \bfA_k^{\top}
    =: \bfSigma_k,
    \qfor i \in \caC_k.
\end{equation}
In our framework, we do not require the latent dimension $d_k$ to be fixed, nor do we assume that the entries of $\bfz_i$ are independent. This generality allows the model to capture data generated from nonlinear manifold structures with noise, where the latent variables may exhibit complex dependence.

Denote the spectral decomposition of the multi-kernel matrix $\bfK$ by
\begin{equation}
    \bfK  = \bfU \bfLamb \bfU^\top
    = \sum\nolimits_{k=1}^n \lambda_k \bfu_k \bfu_k^\top,
    \qwhere
    \bfLamb = \diag (\lambda_1, \cdots, \lambda_n)
    \qand
    \bfU = [\bfu_1, \cdots, \bfu_n].
    \label{eqn:spec-decomp-bfK}
\end{equation}
We arrange the eigenvalues in non-increasing order, $\lambda_1 \geq \lambda_2 \geq \cdots \geq \lambda_n$, and write $\bfu_k(i)$ for the $i$-th entry of $\bfu_k$. Given an embedding dimension $r \in \bbn$, a practical implementation of KSC maps each observation $\bfx_i$ to $\braks{\sqrt{\lambda_1} \bfu_1 (i), \cdots, \sqrt{\lambda_r} \bfu_r (i)} \in \bbr^r$, where $r$ is typically much smaller than the ambient dimension $p$. This construction admits a natural graph-theoretic interpretation. Viewing $\bfK$ as the weighted adjacency matrix of a similarity graph, observations from the same cluster should have similar connectivity profiles, while those from different clusters should have distinguishable profiles. Accordingly, $\bfK$ is expected to be close to a blockwise-structured matrix induced by the underlying cluster partition. Its leading eigenspace, and hence the associated spectral embedding, should then approximately encode the cluster memberships.

This intuition is formalized by the spectral perturbation viewpoint adopted throughout our analysis. Specifically, we approximate the empirical kernel matrix $\bfK$ by a low-rank informative matrix $\barK$ whose blockwise structure encodes the cluster memberships, and regard the residual $\bfK-\barK$ as a perturbation. To illustrate this viewpoint and the construction of $\barK$, consider first the conventional single-kernel setting
\begin{equation*}
    \bfK^\rms = [K_{ij}^\rms]_{i,j=1}^n,
    \quad \text{ with } \quad
    K_{ij}^\rms
    := (1 - \delta_{ij}) \exp \pars[\big]{ - {\norm{\bfx_i-\bfx_j}^2} / {h} }.
\end{equation*}
Suppose that the observations $\{\bfx_i\}_{i=1}^n$ follow the model \eqref{eqn:data-model}. In the high-dimensional regime where $d_k\to\infty$ for each $k\in\dbraks{m}$, and under suitable concentration assumptions on the isotropic random vectors $\bfz_i$, one expects that, for $i\in\caC_k$ and $j\in\caC_{\ell}$ with $i\neq j$, the squared distance $\norm{\bfx_i-\bfx_j}^2$ concentrates around its mean
\begin{equation}
    \theta_{k \ell} 
    := \bbe \norm{\bfx_i-\bfx_j}^2
    = \norm{\bfmu_k-\bfmu_{\ell}}^2
    + \Tr \bfSigma_k + \Tr \bfSigma_{\ell}.
    \label{def:theta-kl}
\end{equation}
Consequently, $\bfK^\rms$ is expected to be well approximated by the deterministic blockwise-constant matrix\footnote{Here, the diagonal of $\barK^\rms$ is filled with the corresponding block values to ensure the blockwise-constant structure; its discrepancy from the zero diagonal of $\bfK^\rms$ is a bounded diagonal perturbation.}
\begin{equation*}
    \barK^\rms = [\bar{K}_{ij}^\rms]_{i,j=1}^n,
    \qwith
    \bar{K}_{ij}^\rms = \exp (- \theta_{k \ell} / h),
    \qfor
    i \in \caC_k, ~ j \in \caC_{\ell}.
\end{equation*}
The spectral structure of $\bfK^\rms$ may therefore be studied as a perturbation of that of $\barK^\rms$. Since $\barK^\rms$ is constant on each block induced by the cluster partition, one easily sees that $\rank(\barK^\rms)\leq m$, and every eigenvector associated with a nonzero eigenvalue lies in the linear span of the cluster indicator vectors
$\bfone_{\caC_k}=[\bbone\{i\in\caC_k\}]_{i=1}^n$. Consequently, the corresponding population spectral embedding is constant within each cluster. From this viewpoint, successful clustering requires the cluster-specific rows of the population embedding to be sufficiently separated and the row-wise perturbation induced by $\bfK^\rms-\barK^\rms$ to be sufficiently small to preserve this separation in the empirical embedding.

\begin{figure}[htb]
    \centering
    \includegraphics[width=6in]{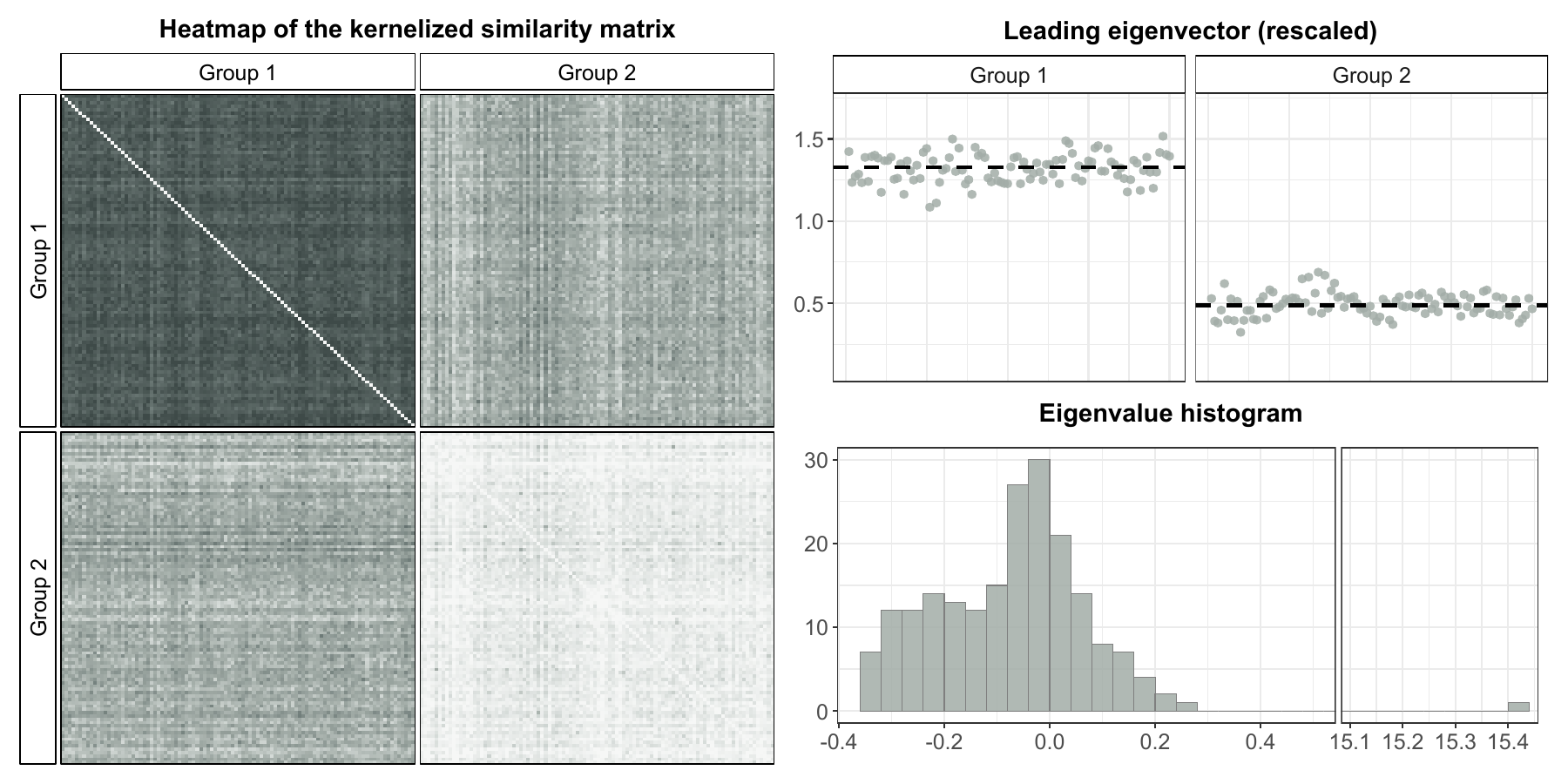}
    \caption{Single-kernel matrix in the two-cluster example.}
    \label{fig:single-kernel-single-scale}
\end{figure}

We illustrate this perspective via a simple two-cluster example. Suppose that $\bfx_i \sim \mathrm{N} (\mathrm{0}, \bfI)$ for $i \in \caC_1$, whereas $\bfx_i \sim \mathrm{N} (\mathrm{0}, 2 \bfI)$ for $i \in \caC_2$. In this case, we have $\theta_{11} = 2p$, $\theta_{12} = 3p$ and $\theta_{22} = 4p$, so all characteristic squared-distance scales are of order $p$. It is therefore natural to choose a bandwidth of the same order, and we set $h=p$ to avoid degeneracy of the kernel matrix. Figure \ref{fig:single-kernel-single-scale} displays one realization of the resulting kernel matrix $\bfK^\rms$ for $n_1=n_2=100$ and $p=400$. The left panel shows a heatmap of $\bfK^\rms$, with darker colors representing larger entries. The visible blockwise structure reflects the underlying cluster partition. The bottom-right panel displays the empirical eigenvalue distribution of $\bfK^\rms$: most eigenvalues are concentrated near zero, while a single dominant eigenvalue is clearly separated from the bulk. From the perturbation viewpoint, this spike originates from the low-rank informative component $\barK^\rms$. The top-right panel compares the leading eigenvector of $\bfK^\rms$ (scatter points) with its population counterpart from $\barK^\rms$ (dashed horizontal lines). The empirical eigenvector fluctuates around the two population cluster levels, with within-cluster variation substantially smaller than the separation between them.

In the two-cluster example above, the bandwidth $h$ must be chosen to avoid a degenerate kernel matrix. For multi-scale data, however, this choice is no longer straightforward, because the parameters $\theta_{k \ell}$ need not be of the same order. As a result, any single bandwidth risks being poorly adapted to part of the data, which may in turn attenuate the between-cluster separation in the resulting spectral embedding. To illustrate this issue, consider a simple four-cluster multi-scale example in which $\bfx_i \sim \mathrm{N}(\bfzero,\bfSigma_k)$ for $i \in \caC_k$, with $\bfSigma_1=\bfI$, $\bfSigma_2=2\bfI$, $\bfSigma_3=\bfI/p$, and $\bfSigma_4=2\bfI/p$. In this case,
\begin{align*}
    \begin{bmatrix}
        \theta_{11} & \theta_{12} & \theta_{13} & \theta_{14}\\
        &\theta_{22} & \theta_{23} & \theta_{24}\\
        & & \theta_{33} & \theta_{34}\\
        & & & \theta_{44}
    \end{bmatrix}
    =
    \begin{bmatrix}
        2p & 3p & p+1 & p+2\\
        & 4p & 2p+1 & 2p+2\\
        & & 2 & 3\\
        & & & 4
    \end{bmatrix}.
\end{align*}
Thus, the quantities $\theta_{k\ell}$ lie on two distinct scales, reflecting the multi-scale structure of the data.

\begin{figure}[htbp]
    \centering
    \includegraphics[width=6in]{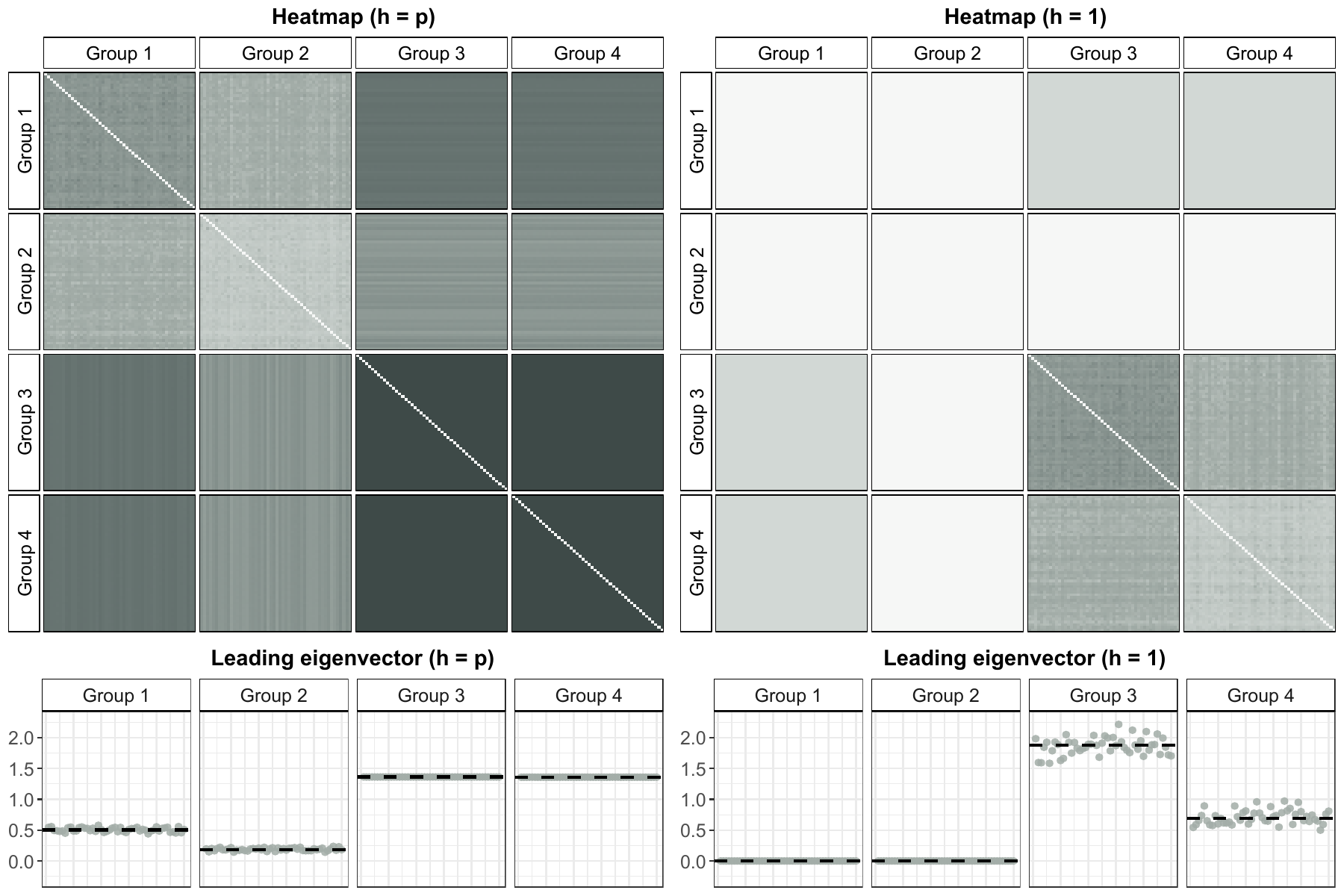}
    \caption{Single-kernel matrix in the multi-scale four-cluster example.}
    \label{fig:single-kernel-multi-scale}
\end{figure}

Figure \ref{fig:single-kernel-multi-scale} shows one realization of the corresponding single-kernel matrix for $p=400$ and $n_k=50$. We consider two bandwidth choices: $h=p$ (left) and $h=1$ (right). When $h=p$, the bandwidth is too large for $\caC_3$ and $\caC_4$, whose pairwise distances are typically of constant order. As a result, the corresponding kernel similarities are overly compressed toward $1$, making the two clusters difficult to distinguish in the heatmap. Consistently, the spectral embedding exhibits almost no separation between $\caC_3$ and $\caC_4$. In graph-theoretic language, this is the \emph{over-smoothing} effect. By contrast, the choice $h=1$ is too small for $\caC_1$ and $\caC_2$, whose pairwise distances are typically of order $p$. In this case, the corresponding kernel similarities become too weak, and the leading eigenvector carries little information for distinguishing these two clusters. This is the familiar \emph{under-connection} effect. This example highlights a basic limitation of the single-kernel approach for multi-scale data: a single bandwidth may fail to preserve both large-scale connectivity and small-scale discrimination, leaving some clusters poorly represented in the resulting spectral embedding.

This scale mismatch motivates the multi-kernel construction. By incorporating bandwidths comparable to the distinct characteristic distance scales, the construction \eqref{def:multi-Kernel} ensures that, at each relevant scale, at least one kernel component avoids both similarity saturation and extinction. The weighted combination can therefore preserve informative cluster contrasts across multiple scales within a single spectral embedding. To illustrate this effect, we return to the four-cluster multi-scale example above. We take $T=2$ in \eqref{def:multi-Kernel}, with bandwidths $h_1=p$ and $h_2=1$, and weights $\alpha_1=1/6$ and $\alpha_2=5/6$. The bandwidth $h_1=p$ captures the coarse-scale distinction between $\caC_1$ and $\caC_2$, while $h_2=1$ resolves the finer-scale distinction between $\caC_3$ and $\caC_4$. The weights are chosen primarily for visualization. In this example, $\alpha_1=\alpha_2=1/2$ already achieves a low misclassification rate. As shown in Figure \ref{fig:multi-kernel-multi-scale}, the combined kernel yields a spectral embedding that successfully separates all four clusters, unlike either single-bandwidth construction. This illustrates how the multi-kernel approach can accommodate heterogeneous distance scales within the same dataset.

\begin{figure}[htbp]
    \centering
    \includegraphics[width=4.5in]{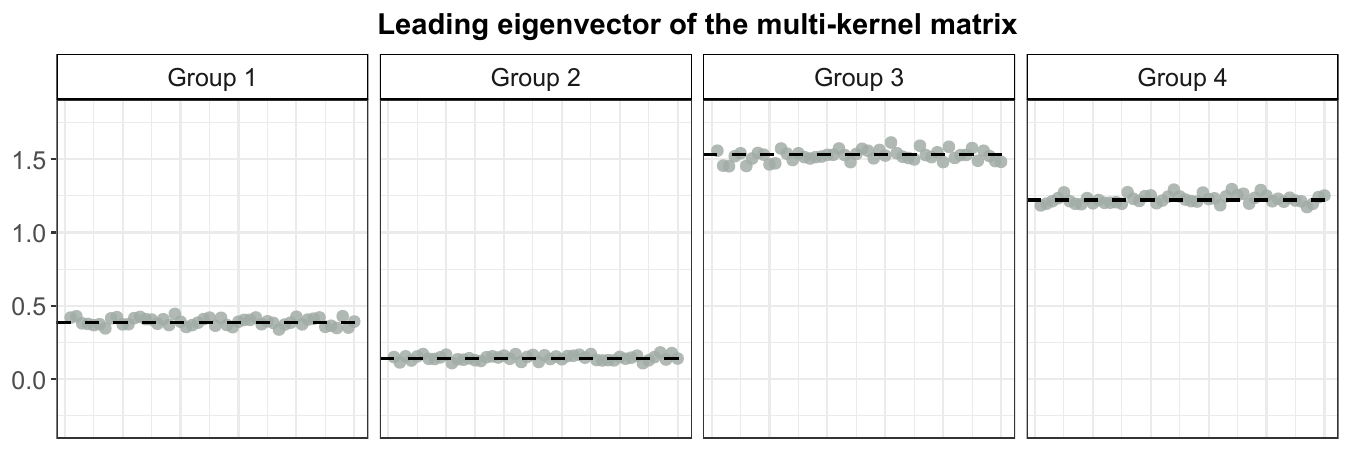}
    \caption{Spectral embedding produced by the multi-kernel matrix in the multi-scale four-cluster example.}
    \label{fig:multi-kernel-multi-scale}
\end{figure}

\subsection{Perturbation bounds for multi-kernel matrices}
\label{subsec:ell-inf-kernel}

In this section, we establish theoretical results for the spectral embeddings induced by the multi-kernel matrix \eqref{def:multi-Kernel}. We begin with an assumption that formalizes the multi-scale structure of the data.

\begin{assumption}[multi-scale data]
\label{assump:multi-scale}
There exist a (small) constant $c_{\ref{assump:multi-scale}} > 0$ and a finite collection of $n$-dependent deterministic parameters $0 < \vartheta_1 < \cdots < \vartheta_{T_{\ref{assump:multi-scale}}}$ such that the following hold for all sufficiently large $n$.
\begin{enumerate}[label = (\roman*)]
\begin{subequations}
     \item The parameters $\{ \theta_{k\ell} \}_{k,\ell = 1}^m$ are captured by the collection $\{ \vartheta_s \}$ in the sense that
    \begin{equation}
        \theta_{k\ell} \in \bigcup\nolimits_{s=1}^{T_{\ref{assump:multi-scale}}} 
        [c_{\ref{assump:multi-scale}} \vartheta_s, 
        \vartheta_s / c_{\ref{assump:multi-scale}}],
        \qfor
        k,\ell \in \dbraks{m}.
        \label{eqn:theta-captured-hbar}
    \end{equation}
    \item The scales $\vartheta_s$ are well separated in the sense that
    \begin{equation}
        \vartheta_{s+1} / \vartheta_{s} \geq n^{c_{\ref{assump:multi-scale}}},
        \qquad 1 \leq s \leq T_{\ref{assump:multi-scale}} - 1.
        \label{eqn:hbar-separation}
    \end{equation}
\end{subequations}
\end{enumerate}
\end{assumption}

Roughly speaking, Assumption \ref{assump:multi-scale} states that the parameters $\theta_{k\ell}$ fall into finitely many distinct asymptotic scale levels, represented by $\vartheta_1 < \cdots < \vartheta_{T_{\ref{assump:multi-scale}}}$, and that consecutive levels are separated by at least a small power of $n$. We refer to the collection $\{ \vartheta_s \}_{s=1} ^{T_{\ref{assump:multi-scale}}}$ as the collection of \emph{characteristic scales}. The separation by a small power of $n$ in \eqref{eqn:hbar-separation} is imposed mainly to make the scale separation compatible with the high-probability bounds in Definition \ref{def:stochastic-domination}. We expect that, under stronger concentration assumptions on the isotropic random vectors $\bfz_i$, this condition can be relaxed to a polylogarithmic separation. We retain the present formulation for simplicity. To streamline the analysis, we first consider an idealized setting in which the bandwidths in \eqref{def:multi-Kernel} are deterministic and have the same orders as the characteristic scales. In practice, the bandwidths $h_t$ should be chosen in a data-dependent manner; we return to this issue in Section \ref{sec:KSC-algorithm}.

\begin{definition}[admissible bandwidths]
\label{def:admissible-bandwidth}
Let $\{ \hbar_t \}_{t=1}^T$ be a finite collection of $n$-dependent deterministic parameters. We say that $\{ \hbar_t \}_{t=1}^T$ forms a collection of admissible bandwidths if
\begin{equation}
    \hbar_t \in \bigcup\nolimits_{s=1}^{T_{\ref{assump:multi-scale}}} 
    [c_{\ref{assump:multi-scale}} \vartheta_s / 2, 
    2 \vartheta_s / c_{\ref{assump:multi-scale}}],
    \quad 
    \forall t \in \dbraks{T} 
    \qand
    \vartheta_s \in \bigcup\nolimits_{t=1}^{T} 
    [c_{\ref{assump:multi-scale}} \hbar_t / 2, 
    2 \hbar_t / c_{\ref{assump:multi-scale}}],
    \quad
    \forall s \in \dbraks{T_{\ref{assump:multi-scale}}}.
    \label{cond:admissible-bandwidth}
\end{equation}
\end{definition}

The condition \eqref{cond:admissible-bandwidth} means that every characteristic scale $\vartheta_s$ is represented by at least one bandwidth $\hbar_t$ of the same order. Conversely, every bandwidth $\hbar_t$ is required to be comparable to at least one characteristic scale. However, we do not require this correspondence to be one-to-one; in particular, several bandwidths may correspond to the same characteristic scale. Consequently, any collection of admissible bandwidths must satisfy $T \geq T_{\ref{assump:multi-scale}}$. Moreover, the collection of characteristic scales $\{ \vartheta_s \}_{s=1} ^{T_{\ref{assump:multi-scale}}}$ itself is admissible.

For what follows, we fix a collection of admissible bandwidths $\{ \hbar_t \}_{t=1}^T$ and consider the empirical multi-kernel matrix $\bfK$ defined in \eqref{def:multi-Kernel} with $h_t = \hbar_t$. For each $k,\ell \in \dbraks{m}$, define the index sets
\begin{equation}
    \caS_{k\ell}(0)
    := \{t \in \dbraks{T}: \hbar_t \ge 
    c_{\ref{assump:multi-scale}}^2 \theta_{k\ell} / 2 \}
    \qand
    \caS_{k\ell}(1)
    := \{t \in \dbraks{T}: \hbar_t \le 
    2 \theta_{k\ell} / (c_{\ref{assump:multi-scale}}^2 
    n^{c_{\ref{assump:multi-scale}}}) \}.
    \label{def:caS}
\end{equation}
Thus, the index set $\caS_{k\ell}(0)$ corresponds to the kernel components whose bandwidths are comparable to or larger than the signal scale $\theta_{k\ell}$, whereas $\caS_{k\ell}(1)$ consists of those whose bandwidths are much smaller than $\theta_{k\ell}$. By Assumption \ref{assump:multi-scale} and the admissibility of $\{ \hbar_t \}_{t=1}^T$, we can easily check $\caS_{k\ell}(0) \cup \caS_{k\ell}(1) = \dbraks{T}$. 

As in the single-kernel case discussed above, we identify the low-rank informative component of the multi-kernel matrix $\bfK$ by expanding the kernelized similarities $\exp(-\norm{\bfx_i-\bfx_j}^2 / \hbar_t)$ round their deterministic counterparts $\exp(-\theta_{k\ell}/ \hbar_t)$ when $i \in \caC_k$, $j \in \caC_\ell$. The main difference is that no expansion is needed for the components with $\hbar_t \ll \theta_{k\ell}$, since in this regime $\exp(-\theta_{k\ell}/\hbar_t)$ is already sufficiently small. More precisely,
\begin{equation}
    \barK = [\bar{K}_{ij}]_{i,j=1}^n,
    \qwith
    \bar{K}_{ij}
    = \sum\nolimits_{t \in \caS_{k\ell}(0)}
    \alpha_t \exp (- \theta_{k\ell}/\hbar_t),
    \qfor
    i \in \caC_k, ~ j \in \caC_{\ell}.
    \label{def:bar-Kij}
\end{equation}

Note that, unlike the empirical kernel matrix $\bfK$, we do not zero out the diagonal entries of the informative component $\barK$. We adopt this convention because it makes $\barK$ exactly blockwise constant, with blocks determined by the cluster memberships. Consequently, $\barK$ has rank at most $m$, and its eigenspace is naturally tied to the underlying cluster structure. To make this observation precise, define the semi-orthogonal matrix
\begin{equation}
    \bfPhi := \braks[\big]{
    \bfone_{\caC_1} / \sqrt{n_1}, 
    \cdots, 
    \bfone_{\caC_m} / \sqrt{n_m} } \in \bbr^{n \times m},
    \qwhere
    \bfone_{\caC_k} := \braks{ \bbone\{i \in \caC_k \}}_{i=1}^n.
    \label{def:phi-mat-vec}
\end{equation}
Thus, the columns of $\bfPhi$ are the normalized cluster indicators. We further define the block-level matrix
\begin{equation}
    \fkK = [\fkK_{k\ell}]_{k,\ell=1}^m
    \qwhere
    \fkK_{k\ell}
    := \sqrt{n_k n_{\ell}}
    \sum\nolimits_{t \in \caS_{k\ell}(0)}
    \alpha_t \exp (-\theta_{k\ell}/\hbar_t).
    \label{eqn:blockwise-K}
\end{equation}
Then, by construction, $\barK = \bfPhi \fkK \bfPhi^\top$. Suppose the spectral decomposition of $\fkK$ is given by
\begin{equation*}
    \fkK = \fkU \barLamb \fkU^\top,
    \qwhere \barLamb = \diag (\barlamb_1, \cdots, \barlamb_m)
    \qand
    \fkU = [\fkU_{k \ell}]_{k,\ell=1}^m.
\end{equation*}
Here we arrange the eigenvalues in a non-increasing order, $\barlamb_1 \geq \barlamb_2 \geq \cdots \geq \barlamb_m$. Since $\fkU$ is orthogonal and $\bfPhi^\top \bfPhi = \bfI_m$, the nonzero eigenstructure of $\barK$ is obtained by lifting that of $\fkK$ through $\bfPhi$. More precisely, $\barK$ admits the spectral decomposition
\begin{equation}
    \barK
    =
    \barU \barLamb \barU^\top
    =
    \sum\nolimits_{\ell=1}^m
    \barlamb_{\ell} \baru_{\ell} \baru_{\ell}^\top,
    \qwhere
    \barU := \bfPhi \fkU.
    \label{eqn:spec-decomp-barK}
\end{equation}
In particular, each eigenvector $\baru_{\ell}$ is piecewise constant on the clusters. Indeed, for each $\ell \in \dbraks{m}$,
\begin{equation}
    \baru_{\ell} = \sum\nolimits_{k = 1}^m 
    \frac{1}{\sqrt{n_k}} \fkU_{k \ell} \bfone_{\caC_k}
    \quad \text{ and hence } \quad
    \baru_{\ell} (i) = \frac{1}{\sqrt{n_k}} \fkU_{k \ell},
    \quad \text{ for } \quad
    i \in \caC_k.
    \label{eqn:baru-piecewise}
\end{equation}

We may regard $\baru_{\ell}$ as the informative component of the empirical eigenvector $\bfu_{\ell}$. From this viewpoint, the success of KSC algorithms hinges on two conditions. First, the population levels $\fkU_{k\ell}$, $k \in \dbraks{m}$, must be sufficiently separated. Second, the fluctuation $\bfu_{\ell}-\baru_{\ell}$ induced by the residual matrix $\bfK-\barK$ must be small relative to the signal encoded by $\baru_{\ell}$. To ensure the latter, we impose several technical assumptions. We begin with a standard concentration assumption for linear and quadratic forms of $\bfz_i$.

\begin{assumption}[concentration]
\label{assump:concentration}
Let $\{ \bfz_i \}_{i=1}^n$ be independent random vectors with $\bbe \bfz_i = \bfzero$ and $\cov(\bfz_i) = \bfI$. Suppose that the following concentration bounds hold uniformly over all $k \in \dbraks{m}$ and $i \in \caC_k$,
\begin{subequations} \label{bound:concentration}
\begin{alignat}{2}
    \abs{\angles{\bfa, \bfz_i}}
    &\prec \norm{\bfa},
    & \qquad & \text{ for all deterministic } \qquad \bfa \in \bbr^{d_k}, 
    \label{eqn:bernstein} \\
    \abs{\angles{\bfz_i, \bfB \bfz_i} - \Tr \bfB}
    & \prec \norm{\bfB}_{\Fnorm},
    & \qquad & \text{ for all deterministic } \qquad \bfB \in \bbr^{d_k \times d_k}.
    \label{eqn:hanson-wright}
\end{alignat}
\end{subequations}
\end{assumption}

This assumption is widely used in the study of high-dimensional random covariance and Gram-type matrices, and it plays an important role in many standard results in this area. We refer interested readers to the recent work \cite{fanAnisotropicLocalLaw2026} for a comprehensive discussion. With a slight abuse of terminology, we refer to \eqref{eqn:bernstein} as a \emph{Bernstein-type estimate} and to \eqref{eqn:hanson-wright} as a \emph{Hanson--Wright-type estimate}, although their standard formulations usually involve exponential probability tails. Here we only require these bounds to hold in the sense of stochastic domination, and this weaker tail requirement allows Assumption \ref{assump:concentration} to accommodate a broad class of isotropic random vectors. In particular, provided that $\min_{k \in \dbraks{m}} d_k \geq n^c$ for some constant $c>0$, the bounds in \eqref{bound:concentration} are satisfied by the following two important classes of random vectors:
\begin{enumerate}[label = (\roman*)]
    \item The entries $\bfz_{i} (\alpha)$ are independent and have uniformly bounded moments: there exist $n$-independent constants $C_\ell > 0$ such that $\max_{\alpha, i} \bbe \abs{\bfz_{i} (\alpha)}^\ell \leq C_\ell$ for all $\ell \in \bbn_+$;
    \label{example1:independent}
    \item The random vectors $\bfz_i$ are log-concave.
    \label{example2:log-concave}
\end{enumerate}
It is well known that distributions in class \ref{example1:independent} satisfy the concentration bounds \eqref{bound:concentration}; see, for instance, \cite[Theorem 7.7]{erdosDynamicalApproachRandom2017}. On the other hand, the implication from isotropic log-concavity to the concentration bounds in \eqref{bound:concentration} was established in \cite[Lemma 3.3]{baoExtremeEigenvaluesLogconcave2025}, building on recent breakthroughs related to the KLS conjecture \cite{chenAlmostConstantLower2021,klartagBourgainsSlicingProblem2022}. We note that both classes \ref{example1:independent} and \ref{example2:log-concave} include the standard Gaussian distribution, and hence the model \eqref{eqn:data-model} in particular encompasses GMMs. 


We next introduce two sets of deterministic control parameters that will be used to state our perturbation bounds. The first set quantifies how well the empirical kernel matrix $\bfK$ is approximated by its informative component $\barK$. Under Assumptions \ref{assump:multi-scale} and \ref{assump:concentration}, the squared distances $\norm{\bfx_i-\bfx_j}^2$ concentrate entrywise around $\theta_{k\ell}$ for $i \in \caC_k$, $j \in \caC_\ell$, and $i \neq j$. The size of the corresponding fluctuation is controlled by
\begin{equation}
    \psi_{k \ell} := \norm{\bfSigma_k}_\Fnorm
    + \norm{\bfSigma_{\ell}}_\Fnorm
    + \norm{\bfA_k^{ \top} (\bfmu_k - \bfmu_{\ell})}
    + \norm{\bfA_{\ell}^{\top} (\bfmu_k - \bfmu_{\ell})}.
    \label{def:psi}
\end{equation}
See Lemma \ref{lemma:xi-control} below for the precise statement. For $\barK$ to provide an effective approximation to $\bfK$, we need suitable control of the signal-to-noise ratios $\theta_{k\ell} / \psi_{k\ell}$. To this end, we introduce
\begin{equation}
    \rho = \max_{k,\ell \in \dbraks{m}} 
    \frac{\psi_{k\ell}}{\theta_{k \ell}}
    \qand
    \phi = \max_{k,\ell \in \dbraks{m}} 
    \frac{\norm{\bfOmega_{k \ell}}_\Fnorm + \sqrt{n} \norm{\bfOmega_{k \ell}}}{\theta_{k \ell}},
    \label{def:ctrl-para-rho-phi}
\end{equation}
where $\bfOmega_{k \ell} := \bfA_k^\top \bfA_{\ell} \in \bbr^{d_k \times d_{\ell}}$. In particular, the parameter $\rho$ measures the relative size of the fluctuations in the pairwise distances, while $\phi$ is introduced mainly to sharpen several bounds involving cross-covariance terms. Since $\norm{\bfOmega_{k\ell}} \leq \norm{\bfOmega_{k\ell}}_\Fnorm$, we have the trivial bound $\phi \lesssim \sqrt{n} \rho$. In many regimes, however, $\phi$ yields substantially sharper control than this crude estimate.

The second set of parameters concerns the eigengaps of the low-rank informative component $\barK$. Eigengaps are central to eigenvector perturbation theory, as already illustrated by the classical Davis--Kahan $\sin\Theta$ theorem \cite{davisRotationEigenvectorsPerturbation1970}. Suppose that we are interested in the first $r$ principal components, where $r \in \dbraks{m}$ is some pre-specified integer. Recall the spectral decompositions of $\bfK$ and $\barK$ in \eqref{eqn:spec-decomp-bfK} and \eqref{eqn:spec-decomp-barK}. Let us write
\begin{equation}
    \bfLamb_r := \diag(\lambda_{1}, \cdots, \lambda_{r}), 
    \qquad
    \bfU_{r} := [\bfu_{1},\cdots,\bfu_{r}],
    \qquad
    \barLamb_r := \diag(\barlamb_{1}, \cdots, \barlamb_{r}),
    \qquad
    \barU_r := [\baru_{1},\cdots,\baru_{r}].
    \label{def:Lamb-U-r}
\end{equation}
To measure the separation of the leading $r$ eigenvalues of $\barK$ from the remaining spectrum, we introduce the $r$-th \emph{effective eigengap}
\begin{equation}
    \Delta_r (\barLamb)
    := (\barlamb_{r}-\barlamb_{r+1}) \wedge \barlamb_{r},
    \label{def:eigengap-kernel}
\end{equation}
where $\barlamb_k=0$ for $k>m$. If $\barK$ is nonnegative definite, then $\Delta_r(\barLamb)$ reduces to the conventional eigengap $\barlamb_r-\barlamb_{r+1}$. The additional truncation by $\barlamb_r$ is useful because several of our arguments involve the inverse $\barLamb_r^{-1}$ and thus require the leading informative eigenvalues to be bounded away from zero. 

Since the entries of $\barK$ are uniformly bounded and $\hbar_t \gtrsim \theta_{k\ell}$ for $t \in \caS_{k\ell}(0)$, one can verify that $\norm{\barK}_{\Fnorm} \asymp n$. Also recall that $\rank (\barK) \leq m$ by the blockwise constant structure. Hence, if $m$ is fixed, we have $\max_{k \in \dbraks{m}} \abs{\barlamb_k} \asymp n$, and therefore $\Delta_r(\barLamb) \lesssim n$. With an appropriate choice of the embedding dimension $r$, one may expect the optimal regime $\Delta_r(\barLamb) \asymp n$. Nevertheless, we first state our results under the following general eigengap assumption, and later illustrate the consequences when the optimal eigengap regime is available.

\begin{assumption}[eigengap]
\label{assump:signal}
There exists a (small) constant $c_{\ref{assump:signal}} > 0$ such that
\begin{equation}
    n^3 \rho^2 / \Delta_r (\barLamb)^3 
    \leq n^{-2 c_{\ref{assump:signal}}}
    \qand
    \Delta_r (\barLamb) 
    \geq n^{c_{\ref{assump:signal}}}.
    \label{cond:eigengap}
\end{equation}
\end{assumption}

Note that the first constraint in \eqref{cond:eigengap} is not simply an upper bound on $\rho$ or a lower bound on $\Delta_r(\barLamb)$. Rather, it is a combined condition involving both the approximation error $\rho$ and the effective eigengap $\Delta_r(\barLamb)$. Since $\Delta_r(\barLamb) \lesssim n$, this condition implies in particular that $\rho \leq n^{-c_{\ref{assump:signal}}}$ so that the signal-to-noise ratio grows with $n$. The second constraint in \eqref{cond:eigengap} is imposed mainly for technical reasons. In many regimes, the effective lower bound on $\Delta_r(\barLamb)$ is already provided by the first constraint. 


Finally, we impose the following auxiliary assumption to ensure that the kernel weights and cluster sizes are balanced, that is, $\alpha_t \asymp 1$ and $n_k \asymp n$. We also include several mild growth conditions on the model parameters, which are used only for technical purposes and are not essential to the main phenomena.

\begin{assumption}
\label{assump:tech-const}
There exists a (small) constant $c_{\ref{assump:tech-const}} > 0$ such that 
\begin{equation*}
    \min\nolimits_{t \in \dbraks{T}} \alpha_t \geq c_{\ref{assump:tech-const}}
    \qand
    \min\nolimits_{k \in \dbraks{m}} (n_k / n) \geq c_{\ref{assump:tech-const}}.
\end{equation*}
In addition, assume there exists a (large) constant $C_{\ref{assump:tech-const}} > 0$ such that 
\begin{equation*}
    p + \sum\nolimits_{k=1}^m (d_{k} + \norm{\bfmu_k} + \norm{\bfA_k}) 
    \leq n^{C_{\ref{assump:tech-const}}},
    \qquad
    \rho \wedge \phi \geq n^{- C_{\ref{assump:tech-const}}},
    \qquad
    n^{- C_{\ref{assump:tech-const}}} \leq \vartheta_1 
    < \vartheta_{T_{\ref{assump:multi-scale}}}
    \leq n^{C_{\ref{assump:tech-const}}}.
\end{equation*}
\end{assumption}

We are now ready to state our theoretical results for the principal components of the empirical multi-kernel matrix \eqref{def:multi-Kernel}. We begin with eigenvalue perturbation bounds.

\begin{theorem}[eigenvalue perturbation]
\label{thm:eval-bfK-barK}
Let $\bfK$ be defined as in \eqref{def:multi-Kernel} with $h_t=\hbar_t$, where $\{ \hbar_t \}_{t=1}^T$ is a collection of admissible bandwidths. Let $\barK$ be defined as in \eqref{def:bar-Kij}. Under Assumptions \ref{assump:multi-scale}--\ref{assump:tech-const}, 
\begin{equation}
    \norm{\bfK - \barK}
    \prec 1 + n \rho
    \qand
    \norm{\bfLamb_r - \barLamb_r}    
    \prec 1 + \sqrt{n} \phi
    + \frac{n^2 \rho^2}{\Delta_r (\barLamb)}.    
    \label{eqn:eigenvalue-perturb-K}
\end{equation}
\end{theorem}

\begin{remark}
As mentioned above, the eigenvalues lying in the bulk spectrum of $\bfK$ are not analyzed in detail in this manuscript. For this part of the spectrum, the crude operator-norm estimate on $\norm{\bfK-\barK}$ is sufficient to control the overall size of the bulk eigenvalues. By contrast, for the first $r$ spiked eigenvalues, we carry out a more refined analysis and establish the sharper perturbation bound in the second estimate of \eqref{eqn:eigenvalue-perturb-K}. Indeed, by \eqref{def:ctrl-para-rho-phi}, we have the crude comparison $\phi \lesssim \sqrt{n}\rho$. Moreover, the first constraint in \eqref{cond:eigengap}, together with $\Delta_r(\barLamb) \lesssim n$, implies that $n \rho / \Delta_r (\barLamb) \leq n^{-c_{\ref{assump:signal}}}$. Therefore, the second estimate in \eqref{eqn:eigenvalue-perturb-K} on $\norm{\bfLamb_r - \barLamb_r}$ is considerably sharper than the bound obtained by combining $\norm{\bfK-\barK}\prec 1+n\rho$ with Weyl's inequality. This refinement is achieved by exploiting the delocalization of the eigenvectors of $\barK$, which allows one to average the entrywise fluctuations of $\bfK-\barK$ and thereby obtain stronger concentration for the spiked eigenvalues. We refer the reader to the proof in Section \ref{sec:tech-lemma} for the details of this averaging mechanism.
\end{remark}

We next present the $\ell_{2,\infty}$ eigenvector perturbation bounds. Since eigenspaces are identifiable only up to an orthogonal transformation, we align $\bfU_r$ with its informative counterpart $\barU_r$ using the orthogonal Procrustes factor $\sgn(\bfU_r^\top\barU_r)\in\bbr^{r\times r}$. We also establish corresponding perturbation bounds for the eigenvalue-weighted embedding $\bfU_r\bfLamb_r^{1/2}$, whose rows are commonly used as the spectral embedding for clustering.

\begin{theorem}[$\ell_{2, \infty}$ eigenvector perturbation]
\label{thm:Linfty-evec-bfK-barK}
Let $\fkS_r^{\bfu} := \sgn(\bfU_r^\top \barU_r)$. Under the same setup as Theorem \ref{thm:eval-bfK-barK}, 
\begin{subequations}
\begin{align}
    \norm{\bfU_r \fkS_r^{\bfu} - \barU_r}_{2,\infty}
    & \prec 
    \frac{\sqrt{n} \rho}{\Delta_r (\barLamb)}
    + \frac{\sqrt{n} + n \phi}{\Delta_r (\barLamb)^2}
    + \frac{n^{5/2} \rho^2}{\Delta_r (\barLamb)^3}, 
    \label{bound:Kev-without-eigs} \\
    \norm{\bfU_r \bfLamb_r^{1/2} \fkS_r^{\bfu} 
    - \barU_r\barLamb_r^{1/2}}_{2,\infty}
    & \prec
    \frac{\sqrt{n} \rho}{\Delta_r (\barLamb)^{1/2}}
    + \frac{n + n^{3/2} \phi}{\Delta_r (\barLamb)^2}
    + \frac{n^3 \rho^2}{\Delta_r (\barLamb)^3}
    +\frac{n^{3/2}}{\Delta_r (\barLamb)^{9/2}}.
    \label{bound:Kev-with-eigs}
\end{align}
\end{subequations}
\end{theorem}


\begin{remark}[$\ell_2$ eigenvector bounds]
\label{remark:ell-2}
Combining the operator-norm control of $\bfK-\barK$ in \eqref{eqn:eigenvalue-perturb-K} with the classical Davis--Kahan $\sin\Theta$ theorem \cite{davisRotationEigenvectorsPerturbation1970}, one can readily obtain the following $\ell_2$ control of the eigenspace,
\begin{equation*}
    \norm{\bfU_r \fkS_r^{\bfu} - \barU_r}_{\Fnorm}
    \asymp \norm{\bfU_r \fkS_r^{\bfu} - \barU_r} 
    \prec \frac{1 + n \rho}{\Delta_r (\barLamb)}.
\end{equation*}
See also \eqref{tmp:fkS-fkH-diff} and \eqref{tmp:bfu-to-popu}. Here the first comparison follows from the fact that $r$ is fixed. If this $\ell_2$ control is sharp, then the optimal $\ell_{\infty}$ control one could expect is $(1/\sqrt{n} + \sqrt{n} \rho) / \Delta_r (\barLamb)$, which corresponds to the case where the fluctuation $\bfU_r \fkS_r^{\bfu} - \barU_r$ is evenly spread across the rows. Roughly speaking, our estimate \eqref{bound:Kev-without-eigs} is consistent with this heuristic, since in many regimes the term $\sqrt{n}\rho / \Delta_r(\barLamb)$ dominates, while the remaining terms are subleading. It remains unclear, however, whether these subleading terms can be removed under the present framework so as to attain the expected optimal $\ell_{\infty}$ rate. We leave this refinement for future work.
\end{remark}

\begin{remark}[eigengap]
\label{remark:optimal-eigengap}
By the piecewise constant structure of $\baru_{\ell}$ illustrated in \eqref{eqn:baru-piecewise}, together with the balanced cluster-size assumption $n_k \asymp n$ in Assumption \ref{assump:tech-const}, we have $\norm{\barU_r}_{2, \infty} \asymp n^{-1/2}$. Thus, the first constraint in our eigengap assumption  \eqref{cond:eigengap} essentially requires the third term on the r.h.s. of \eqref{bound:Kev-without-eigs} to be subleading relative to the natural size of $\barU_r$. Indeed, apart from the term $\sqrt{n}/\Delta_r(\barLamb)^2$, the first constraint in \eqref{cond:eigengap} ensures that the remaining terms on the r.h.s. of \eqref{bound:Kev-without-eigs} are bounded by $n^{-1/2-c_{\ref{assump:signal}}}$. Therefore, the perturbation bound \eqref{bound:Kev-without-eigs} typically yields effective entrywise control whenever $\Delta_r(\barLamb) \gg \sqrt{n}$.

As discussed above, the optimal scaling one may expect for the effective eigengap is $\Delta_r (\barLamb) \asymp n$ and this scaling is indeed available in many cases with an appropriate choice of the embedding dimension $r$. Under this optimal scaling, the perturbation bounds in Theorem \ref{thm:Linfty-evec-bfK-barK} simplify to
\begin{equation*}
    \norm{\bfU_r \fkS_r^{\bfu} 
    - \barU_r}_{2,\infty}
    \prec 
    \rho / \sqrt{n} + 1 / n^{3/2}
    \qand
    \norm{\bfU_r \bfLamb_r^{1/2} \fkS_r^{\bfu} 
    - \barU_r\barLamb_r^{1/2}}_{2,\infty}
    \prec 
    \rho + 1 / n,
\end{equation*}
where in the simplification we have also used $\phi \lesssim \sqrt{n}\rho$ and $\rho \leq n^{-c_{\ref{assump:signal}}}$. In particular, in this regime, the parameter $\rho$, originally defined to quantify the relative approximation error in the pairwise distances, propagates to the eigenvector level. More precisely, up to an additional term which is subleading in many regimes, the parameter $\rho$ controls the relative entrywise perturbation $\norm{\bfU_r \fkS_r^{\bfu} - \barU_r}_{2,\infty} / \norm{\barU_r}_{2,\infty}$.
\end{remark}

\begin{remark}[general distance-based kernels]
\label{remark:general-distance}
The main results, Theorems \ref{thm:eval-bfK-barK} and \ref{thm:Linfty-evec-bfK-barK}, are stated for multi-kernel matrices constructed from the RBF kernel. However, our proof strategy only uses several basic analytic properties of the exponential profile, and hence can be extended to a broader class of distance-based kernels of the form $f_t (\norm{\bfx-\bfy}^2/h)$, where $f_t: [0,\infty)\to\bbr$. More precisely, one may replace the entries of the empirical multi-kernel matrix $\bfK$ in \eqref{def:multi-Kernel} and the informative counterpart $\barK$ in \eqref{def:bar-Kij} by
\begin{equation}
    K_{ij}
    = (1 - \delta_{ij}) \sum\nolimits_{t=1}^T \alpha_t
    f_t \pars[\big]{ \norm{\bfx_i-\bfx_j}^2 / \hbar_t }
    \qand
    \bar{K}_{ij} = \sum\nolimits_{t \in \caS_{k\ell}(0)}
    \alpha_t f_t (\theta_{k\ell} / \hbar_t),
    \label{def:general-distance}
\end{equation}
where $i \in \caC_k$ and $j \in \caC_\ell$. In particular, the original definitions \eqref{def:multi-Kernel} and \eqref{def:bar-Kij} correspond to the special case $f_t(x)=\exp(-x)$ for all $t \in \dbraks{T}$. A straightforward inspection of the proof shows that Theorems \ref{thm:eval-bfK-barK} and \ref{thm:Linfty-evec-bfK-barK} continue to hold for the general distance-based kernel matrices in \eqref{def:general-distance}, provided that the nonlinear profiles $f_t$ satisfy the following regularity and tail conditions:
\begin{equation}
    \norm{f_t}_{\infty} + \norm{f_t'}_{\infty} + \norm{f_t''}_{\infty} \lesssim 1
    \qand
    \sup \curls[\big]{\abs{f_t (x)} : 
    x \geq n^{c_{\ref{assump:multi-scale}} / 2}, 
    t \in \dbraks{T}}
    \lesssim \rho^2.
    \label{cond:general-func}
\end{equation}
The first condition is used to control the errors produced by expanding the kernel entries around their deterministic counterparts. The second condition is a decay requirement at large arguments. It is needed to control the contribution of kernels corresponding to $t \in \caS_{k\ell}(1)$, for which $\theta_{k\ell} \gg \hbar_t$ and hence the ratio $\norm{\bfx_i-\bfx_j}^2/\hbar_t$ is typically large. For these kernels, no expansion is performed; instead, the corresponding kernel entries are controlled directly through their magnitude. See Section \ref{subsec:norm-bound-residual} for the relevant estimates. 
\end{remark}


\begin{remark}[Mahalanobis distance]
A common extension of distance-based kernel methods is to replace the Euclidean distance $\norm{\bfx_i - \bfx_j}$ in the construction of kernel matrix \eqref{def:multi-Kernel} by a Mahalanobis distance
\begin{equation*}
    \norm{\bfx_i - \bfx_j}_{\bfM} := \sqrt{\angles{\bfx_i - \bfx_j, \bfM (\bfx_i - \bfx_j)}},
\end{equation*}
where $\bfM \in \bbr^{p \times p}$ is a deterministic positive-definite matrix.; see, for example, \cite{jainMetricKernelLearning2012,weinbergerMetricLearningKernel2007}. Our analysis framework extends directly to this setting. Indeed, letting $\tilde{\bfx}_i = \bfM^{1/2}\bfx_i$, we have $\norm{\bfx_i-\bfx_j}_{\bfM} = \norm{\tilde{\bfx}_i-\tilde{\bfx}_j}$ and the transformed observations retain the structure in \eqref{eqn:data-model}. Specifically, for $i \in \caC_k$,
\begin{equation*}
    \tilde{\bfx}_i = \tilde{\bfmu}_k + \tilde{\bfA}_k \bfz_i
    \qwhere
    \tilde{\bfmu}_k := \bfM^{1/2} \bfmu_k
    \qand
    \tilde{\bfA}_k := \bfM^{1/2} \bfA_k.
\end{equation*}
Hence, the perturbation bounds in Theorems \ref{thm:eval-bfK-barK} and \ref{thm:Linfty-evec-bfK-barK} apply to the transformed observations whenever the corresponding assumptions hold for the transformed model. We particularly highlight that the pairwise scale parameters of the transformed model are given by
\begin{align*}
    \tilde{\theta}_{k\ell}
    & =
    \norm{\bfmu_k-\bfmu_\ell}_{\bfM}^2
    + \Tr \bfM \bfSigma_k
    + \Tr \bfM \bfSigma_\ell ,\\
    \tilde{\psi}_{k\ell}
    & =
    \norm{\bfM^{1/2}\bfSigma_k\bfM^{1/2}}_{\Fnorm}
    + \norm{\bfM^{1/2}\bfSigma_\ell\bfM^{1/2}}_{\Fnorm}
    + \norm{\bfA_k^{\top} \bfM (\bfmu_k - \bfmu_{\ell})}
    + \norm{\bfA_\ell^{\top} \bfM (\bfmu_k - \bfmu_{\ell})}.    
\end{align*}
\end{remark}

\begin{remark}[Comparison with \cite{aminiConcentrationKernelMatrices2021}]
\label{remark:compare-Amini}
Amini and Razaee \cite{aminiConcentrationKernelMatrices2021} establish concentration bounds for general empirical kernel matrices $\bfK=[f(\bfx_i,\bfx_j)]_{i,j=1}^n$, where the kernel function $f:\bbr^p\times\bbr^p\to\bbr$ is symmetric and $L$-Lipschitz. To facilitate comparison, we specialize their result to the GMM setting, i.e.,  $\bfx_i\sim\mathrm{N}(\bfmu_k,\bfSigma_k)$ for $i\in\caC_k$. In this setting, \cite[Theorem 1]{aminiConcentrationKernelMatrices2021} yields
\begin{equation}
    \bbp \curls[\big]{
    \norm{\bfK - \bbe \bfK} 
    \leq C L \sigma_{\infty} n }
    \geq 1 - \exp (- c n),
    \label{eqn:amini-result}
\end{equation}
where $\sigma_{\infty}^2 = \max_{k \in \dbraks{m}} \norm{\bfSigma_k}$, and $c,C>0$ are universal constants. The two results use different deterministic reference matrices: \eqref{eqn:amini-result} controls the fluctuation of $\bfK$ around its expectation $\bbe\bfK$, whereas our analysis compares $\bfK$ with the blockwise-structured informative matrix $\barK$. Nevertheless, the two estimates play analogous roles at the operator-norm level, as both control the deviation of the empirical kernel matrix from an appropriate deterministic counterpart. Our analysis then exploits the specific spectral structure of $\barK$. Under suitable eigengap assumptions, we derive refined perturbation bounds for the spiked eigenvalues of $\bfK$ and row-wise perturbation bounds for the associated eigenspaces, thereby controlling the empirical spectral embedding beyond what follows directly from a global operator-norm estimate.

A further distinction arises in multi-scale settings. The result of \cite{aminiConcentrationKernelMatrices2021}  can be applied to the multi-kernel matrix \eqref{def:multi-Kernel} by taking $f (\bfx, \bfy) = \sum_{t=1}^T \alpha_t \exp (-\norm{\bfx - \bfy}^2 / \hbar_t)$. For this kernel, the Lipschitz constant is controlled by the smallest bandwidth scale. More precisely,
\begin{equation*}
    L^2 \asymp \max\nolimits_{t \in \dbraks{T}} (1 / \hbar_t)
    \asymp {1 / (\min\nolimits_{k, \ell \in \dbraks{m}} \theta_{k \ell})},
\end{equation*}
where the second comparison relation follows from Assumption \ref{assump:multi-scale} and Definition \ref{def:admissible-bandwidth}. Consequently, for the multi-kernel matrix \eqref{def:multi-Kernel}, the concentration bound \eqref{eqn:amini-result} from \cite{aminiConcentrationKernelMatrices2021} and our bound \eqref{eqn:eigenvalue-perturb-K} respectively give
\begin{equation}
    \frac{1}{n} \norm{\bfK - \bbe \bfK} 
    \lesssim \sqrt{\frac{\max_{k \in \dbraks{m}} \norm{\bfSigma_k}}
    {\min_{k, \ell \in \dbraks{m}} \theta_{k \ell}}}
    \qand
    \frac{1}{n} \norm{\bfK - \barK} 
    \lesssim n^{\varepsilon} \pars*{ \frac{1}{n} + 
    \max_{k, \ell \in \dbraks{m}} 
    \frac{\psi_{k \ell}}{\theta_{k \ell}} }
    \label{eqn:compare-us-amini}
\end{equation}
with high probability. Here the constant $\varepsilon>0$ in our bound can be chosen arbitrarily small. The distinction between these two rates is especially pronounced when the data exhibit genuinely different characteristic scales. The first estimate in \eqref{eqn:compare-us-amini} involves a cross-scale worst-case ratio: the largest covariance scale is compared with the smallest characteristic squared-distance scale. By contrast, our estimate first compares each fluctuation scale $\psi_{k\ell}$ with its corresponding distance scale $\theta_{k\ell}$ and only then takes the maximum over $(k,\ell)$. This scale-matched relative-error structure can yield substantially sharper perturbative control when the quantities $\theta_{k\ell}$ vary over several asymptotic orders.

For example, consider the four-cluster multi-scale model used in Figures \ref{fig:single-kernel-multi-scale} and \ref{fig:multi-kernel-multi-scale}, for which one admissible choice is $\hbar_1\asymp p$ and $\hbar_2\asymp1$. The two estimates in \eqref{eqn:compare-us-amini} then give
\begin{equation*}
    \frac{1}{n} \norm{\bfK - \bbe \bfK} 
    \lesssim 1
    \qand
    \frac{1}{n} \norm{\bfK - \barK} 
    \lesssim n^{\varepsilon} \pars*{ \frac{1}{n} + 
    \frac{1}{\sqrt{p}} }.
\end{equation*}
Thus, relative to their respective deterministic counterparts, the general Lipschitz bound remains of constant order, whereas our perturbation bound is asymptotically small whenever $n^\varepsilon \ll \sqrt{p}$. This improvement stems from the scale-sensitive expansion of $\bfK$, which preserves the pairwise correspondence between the fluctuation and distance scales. The general theory of \cite{aminiConcentrationKernelMatrices2021}, by design, applies to a broader class of Lipschitz kernel matrices and does not exploit this additional multi-scale structure.
\end{remark}

\subsection{Perturbation bounds for normalized Laplacian matrices}
\label{subsec:ell-inf-Laplacian}

In this section, we establish the corresponding perturbation results for the normalized Laplacian matrix
\begin{equation}
    \bfL = n \bfD^{-1/2} \bfK \bfD^{-1/2},
    \qwhere
    \bfD = \diag (D_{11}, \cdots, D_{nn})
    \quad \text{ with } \quad
    D_{ii} = \sum\nolimits_{j = 1}^n K_{ij}.
    \label{def:laplacian}
\end{equation}
Here, $\bfK$ is the multi-kernel matrix defined in \eqref{def:multi-Kernel}. We note that the conventional symmetric normalized Laplacian is $\bfI - \bfD^{-1/2}\bfK\bfD^{-1/2}$. Here we instead work with the rescaled normalized affinity matrix $\bfL$ in \eqref{def:laplacian}. The two matrices share the same eigenvectors, and their eigenvalues are related through a simple affine transformation induced by the rescaling and shift inherent in the conventional definition. Thus, the leading eigenvectors of $\bfL$ correspond to the eigenvectors associated with the smallest eigenvalues of the conventional normalized Laplacian. Moreover, the prefactor $n$ places the informative eigenvalues of $\bfL$ on the same scale as those of $\bfK$, facilitating a direct comparison of their perturbation bounds. With a slight abuse of terminology, we continue to refer to $\bfL$ as the normalized Laplacian. Denote the spectral decomposition of $\bfL$ by
\begin{equation}
    \bfL
    = \bfV \bfGamma \bfV^\top
    = \sum\nolimits_{i=1}^n \gamma_i \bfv_i \bfv_i^\top,
    \qwhere 
    \bfGamma=\diag(\gamma_1,\dots,\gamma_n),
    \label{eqn:spectral-decomp-L}
\end{equation}
with the eigenvalues arranged in nonincreasing order. In parallel with the analysis of the multi-kernel matrix $\bfK$ in Section \ref{subsec:ell-inf-kernel}, we introduce the informative component of $\bfL$ through the normalized Laplacian associated with the informative kernel component $\barK$. Specifically, let
\begin{equation}
    \barL
    := n \barD^{-1/2} \barK \barD^{-1/2},
    \qwhere
    \barD = \diag (\bar{D}_{1}, \cdots, \bar{D}_{n})
    \qwith
    \bar{D}_{i}
    = \sum\nolimits_{j=1}^n \bar{K}_{ij}.
    \label{def:barL-barD}
\end{equation}
Recall the blockwise constant structure of $\barK$ described in \eqref{eqn:blockwise-K}. Its Laplacian counterpart $\barL$ inherits the same structure and is therefore also low-rank, with $\rank(\barL) \leq m$. To make this explicit, define
\begin{equation*}
    \fkD = \diag (\fkD_1, \cdots, \fkD_m)
    \qwhere \fkD_k = \sum\nolimits_{\ell = 1}^m 
    \sqrt{n_{\ell} / n_k} \, \fkK_{k \ell}.
\end{equation*}
Then, for every $i \in \caC_k$, we have $\bar{D}_{i} = \fkD_k$. Hence the blockwise constant structure of $\barL$ can be written as 
\begin{equation*}
    \barL = \bfPhi \fkL \bfPhi^\top,
    \qwhere
    \fkL = n \fkD^{-1/2} \fkK \fkD^{-1/2},
\end{equation*}
and $\bfPhi$ is defined in \eqref{def:phi-mat-vec}. Denote the spectral decomposition of $\fkL \in \bbr^{m \times m}$ by
\begin{equation}
    \fkL = \fkV \barGamma \fkV^\top,
    \qwhere \barGamma = \diag (\bargamma_1, \cdots, \bargamma_m)
    \qand
    \fkV = [\fkV_{k \ell}]_{k,\ell=1}^m.
\end{equation}
It follows that the informative component $\barL$ admits the spectral decomposition
\begin{equation}
    \barL
    = \barV \barGamma \barV^\top
    = \sum\nolimits_{\ell=1}^m
    \bargamma_{\ell} \barv_{\ell} \barv_{\ell}^\top,
    \qwhere
    \barv_{\ell} (i) = \frac{1}{\sqrt{n_k}} \fkV_{k \ell},
    \qfor i \in \caC_k.
    \label{eqn:barv-piecewise}
\end{equation}
This representation also makes explicit the piecewise constant structure of the eigenvectors $\barv_{\ell}$.

We now present perturbation results by viewing $\bfL$ as a stochastic perturbation of its informative component $\barL$. As before, we focus on the leading $r$ principal components,
\begin{equation*}
    \bfGamma_r = \diag(\gamma_{1}, \cdots, \gamma_{r}),
    \qquad
    \bfV_{r} = \braks{\bfv_{1}, \cdots, \bfv_{r}},
    \qquad
    \barGamma_r = \diag(\bargamma_{1}, \cdots, \bargamma_{r})
    \qquad   
    \barV_{r} = \braks{\barv_{1}, \cdots, \barv_{r}},.
\end{equation*}
Most of the technical assumptions required for the Laplacian perturbation results are identical to those imposed for the kernel counterpart in Section \ref{subsec:ell-inf-kernel}. The only difference concerns Assumption \ref{assump:signal}: instead of requiring a lower bound on the eigengap of $\barK$, we impose the corresponding condition on the eigengap of $\barL$. Thus, in accordance with \eqref{def:eigengap-kernel}, we define the $r$-th effective eigengap of $\barL$ by
\begin{equation}
    \Delta_r (\barGamma)
    := (\bargamma_{r} - \bargamma_{r+1}) \wedge \bargamma_{r}.
    \label{def:eigengap-Laplacian}
\end{equation}
We then assume the following Laplacian counterpart of Assumption \ref{assump:signal}.

\begin{assumptionprime}{assump:signal}{eigengap}
\label{assump:eigengap-Laplacian}
There exists a (small) constant $c_{\ref{assump:signal}} > 0$ such that
\begin{equation}
    n^3 \rho^2 / \Delta_r (\barGamma)^3 
    \leq n^{-2 c_{\ref{assump:signal}}}
    \qand
    \Delta_r (\barGamma) 
    \geq n^{2 c_{\ref{assump:signal}}}.
\end{equation}
\end{assumptionprime}


Under this assumption, we obtain the following perturbation results for the principal components of the normalized Laplacian $\bfL$. These results are parallel to Theorems \ref{thm:eval-bfK-barK} and \ref{thm:Linfty-evec-bfK-barK}. As in those theorems, we work with the admissible bandwidths $\{ \hbar_t \}_{t = 1}^T$. The resulting rates are of the same form as their kernel counterparts, with the eigengap $\Delta_r(\barLamb)$ replaced by $\Delta_r(\barGamma)$. The discussions in Remarks \ref{remark:ell-2}, \ref{remark:optimal-eigengap}, and \ref{remark:general-distance} transfer directly to the Laplacian setting, and hence are not repeated here.

\begin{theorem}[eigenvalue perturbation]
\label{thm:eval-bfL-barL}
Suppose that $\{ \hbar_t \}_{t=1}^T$ is a collection of admissible bandwidths. Let $\bfL$ be defined as in \eqref{def:laplacian}, where $\bfK$ is given by \eqref{def:multi-Kernel} with $h_t=\hbar_t$. Let $\barL$ be defined as in \eqref{def:barL-barD}, where $\barK$ is given by \eqref{def:bar-Kij}. Then, under Assumptions \ref{assump:multi-scale}, \ref{assump:concentration}, \ref{assump:eigengap-Laplacian}, and \ref{assump:tech-const}, we have
\begin{equation}
    \norm{\bfL - \barL}
    \prec 1 + n \rho
    \qand
    \norm{\bfGamma_r - \barGamma_r}
    \prec 1 + \sqrt{n} \phi
    + \frac{n^2 \rho^2}{\Delta_r (\barGamma)}. 
\end{equation}
\end{theorem}

\begin{theorem}[$\ell_{2, \infty}$ eigenvector perturbation]
\label{thm:Linfty-evec-bfL-barL}
Let $\fkS_r^{\bfv} := \sgn(\bfV_r^\top \barV_r)$. Under the same setup as in Theorem \ref{thm:eval-bfL-barL},
\begin{subequations}
\begin{align}
    \norm{\bfV_r \fkS_r^{\bfv} - \barV_r}_{2,\infty}
    & \prec 
    \frac{\sqrt{n} \rho}{\Delta_r (\barGamma)} 
    + \frac{\sqrt{n} + n \phi}{\Delta_r (\barGamma)^2} 
    + \frac{n^{5/2} \rho^2}{\Delta_r (\barGamma)^3},
    \label{eqn:rowwise-L-without-eval} \\
    \norm{\bfV_r \bfGamma_r^{1/2} \fkS_r^{\bfv}
    - \barV_r \barGamma_r^{1/2} }_{2,\infty}
    & \prec
    \frac{\sqrt{n} \rho}{\Delta_r (\barGamma)^{1/2}} 
    + \frac{n + n^{3/2} \phi}{\Delta_r (\barGamma)^2} 
    + \frac{n^3 \rho^2}{\Delta_r (\barGamma)^3}
    + \frac{n^{3/2}}{\Delta_r (\barGamma)^{9/2}}.
    \label{eqn:rowwise-L-with-eval}
\end{align}    
\end{subequations}
\end{theorem}

\section{Multi-kernel spectral clustering}
\label{sec:KSC-algorithm}

In this section, we discuss how the perturbation results established in Section \ref{sec:main-res} can be used to guide theory-informed implementations of the multi-KSC algorithm. When spectral embeddings are constructed from the multi-kernel matrix $\bfK$ in \eqref{def:multi-Kernel} or its associated normalized Laplacian $\bfL$ in \eqref{def:laplacian}, three practical choices arise: the bandwidths $h_t$, the kernel weights $\alpha_t$, and the embedding dimension $r$.

The first issue, and our main motivation for incorporating the MKL framework, is the choice of the bandwidth parameters $h_t$. Our results in Section \ref{sec:main-res} suggest that the selected bandwidths should satisfy the admissibility condition introduced in Definition \ref{def:admissible-bandwidth}, which ensures that the information associated with each characteristic scale $\vartheta_s$ is effectively utilized. In practice, however, these characteristic scales are unknown and must be inferred from the data. We therefore seek data-driven bandwidths that behave like admissible bandwidths while preserving the applicability of the perturbation results in Section \ref{sec:main-res} despite their data dependence.

A natural starting point is the common practice in kernel methods of choosing the bandwidth as an empirical quantile of the pairwise distances. The median heuristic is the best-known example; see \cite{garreauLargeSampleAnalysis2018} for a broader discussion of this heuristic and its use in the kernel-method literature. For kernel spectral embeddings of high-dimensional noisy data, Ding and Ma \cite{dingLearningLowdimensionalNonlinear2023} provided a theoretical justification for such a quantile-based choice in the setting of a single underlying distribution and a single data-adaptive bandwidth. We extend this principle to mixture distributions with potentially multiple characteristic scales. Rather than selecting one global bandwidth, we use several empirical quantiles of the pairwise squared distances to construct a collection of bandwidths; see \eqref{eqn:bandwidth-selection}. The goal is to ensure that every characteristic scale is captured by at least one kernel component. We then show that these data-adaptive bandwidths satisfy a suitable analogue of the admissibility condition with high probability, thereby allowing the perturbation results in Section \ref{sec:main-res} to remain applicable.

The second issue is the choice of the kernel weights $\alpha_t$. These weights modify the population-level affinities encoded by the informative matrix $\barK$ and may therefore affect cluster separation in the spectral embedding. In this work, we do not optimize $\alpha_t$. Instead, both the algorithm and the perturbation theory are developed for general prescribed non-degenerate weights, including simple practical choices such as uniform weighting. Data-adaptive weight selection has been studied extensively in the MKL literature \cite{bachMultipleKernelLearning2004,cortesAlgorithmsLearningKernels2012,kloftLpNormMultipleKernel2011,rakotomamonjySimpleMKL2008}. Extending the present framework to optimized, data-dependent weights is left for future work.

Finally, we consider the embedding dimension $r$. The natural population-level choice is $r=m$, which retains all informative directions of $\barK$. More generally, the most suitable value of $r$ for the clustering task may depend on the spectral structure of the constructed kernel matrix. Since the primary focus of this work is bandwidth selection for multi-scale data, we treat $r$ as a prespecified input rather than incorporating a data-adaptive dimension-selection step into the proposed procedure. Nevertheless, when a data-adaptive choice is desired, one may select $r$ according to the empirical spectral structure, for example through the commonly used maximum-gap or maximum-ratio rules
\begin{equation*}
    r = 
    \arg \max\nolimits_{1 \leq \ell \leq m}
    (\lambda_\ell-\lambda_{\ell+1})
    \quad \text{ or } \quad
    r = 
    \arg \max\nolimits_{1 \leq \ell \leq m}
    (\lambda_\ell / \lambda_{\ell+1}).
\end{equation*}
A systematic analysis of such dimension-selection procedures is beyond the scope of this work. To summarize, the multi-KSC procedure analyzed in this manuscript is given in Algorithm \ref{alg:multi-ksc} below.

\begin{algorithm}[htb]
\caption{Multi-kernel spectral clustering}
\label{alg:multi-ksc}
\KwIn{Dataset $\{ \bfx_i \}_{i=1}^n$; kernel weights $\{ \alpha_t \}_{t=1}^T$ satisfying $\sum_{t=1}^T \alpha_t = 1$; quantile levels $\{ \omega_t \}_{t=1}^T \subset (0,1)$; number of clusters $m$; embedding dimension $r$; $K$-means approximation parameter $\varepsilon_* > 0$}

\textbf{Step 1 (bandwidth selection)}: Compute the squared distances $\norm{\bfx_i-\bfx_j}^2$ for all $1 \leq i < j \leq n$, and define the associated empirical cumulative distribution function
\begin{equation}
    G (h) = \frac{2}{n(n-1)}
    \sum\nolimits_{1\le i<j\le n} \bbone {\{ \norm{\bfx_i-\bfx_j}^2 \leq h \}}.
    \label{def:CDF-dist}
\end{equation}
Define the bandwidths ${h}_t$ as the empirical quantiles of $G$, 
\begin{equation}
    {h}_t = G^{-1} (\omega_t)
    = \inf \{ h: G(h) \geq \omega_t \}
    \qfor t \in \dbraks{T}.
    \label{eqn:bandwidth-selection}
\end{equation}

\textbf{Step 2 (kernel construction)}: Construct the multi-kernel matrix
\begin{equation*}
    \bfK = [K_{ij}]_{i,j=1}^n,
    \qwith
    K_{ij}
    = (1 - \delta_{ij}) \sum\nolimits_{t=1}^T \alpha_t
    \exp \pars[\big]{ - \norm{\bfx_i-\bfx_j}^2 / {h}_t }.
\end{equation*}

\textbf{Step 3 (spectral embedding)}: Compute the spectral decomposition
\begin{equation*}
    \bfK
    = \sum\nolimits_{\ell=1}^n
    \lambda_\ell \bfu_\ell \bfu_\ell^\top,
    \qwhere
    \lambda_1 \geq \cdots \geq \lambda_n.
\end{equation*} 
Define the spectral embedding $\{\fkx_i\}_{i=1}^n$ by
\begin{equation}
    \fkx_i = \braks[\big]{\sqrt{\lambda_1} \bfu_1 (i), \cdots, \sqrt{\lambda_{r}} \bfu_{r} (i)},
    \qfor i \in \dbraks{n}.
    \label{def:fkx-embedding}
\end{equation}

\textbf{Step 4 (clustering)}: Apply the $(1 + \varepsilon_*)$-approximate $K$-means algorithm with $m$ clusters to $\{\fkx_i\}_{i=1}^n$.

\KwOut{Estimated cluster membership $\hat{\pi} : \dbraks{n} \to \dbraks{m}$.}
\end{algorithm}

In the final step of Algorithm \ref{alg:multi-ksc}, we use the $(1+\varepsilon_*)$-approximate $K$-means procedure. Given a data set $\{\fkx_i\}_{i=1}^n \subset \bbr^r$, this procedure seeks a partition of the data into $m$ groups by approximately minimizing the within-cluster sum of squares. More precisely, it returns estimated cluster centers $\{\hat{\fka}_k\}_{k=1}^m \subset \bbr^r$ together with an estimated label map $\hat{\pi}: \dbraks{n} \to \dbraks{m}$ such that
\begin{equation*}
    \sum\nolimits_{i=1}^{n} \norm{\fkx_i - \hat{\fka}_{\hat{\pi} (i)}}^2
    \leq (1 + \varepsilon_*) \min\nolimits_{\{ \fka_k \}, \pi}
    \pars[\Big]{\sum\nolimits_{i=1}^{n} \norm{\fkx_i - {\fka}_{{\pi} (i)}}^2},
\end{equation*}
where the minimum is taken over all collections of centers $\{\fka_k\}_{k=1}^m \subset \bbr^r$ and all label maps $\pi:\dbraks{n}\to\dbraks{m}$. We adopt this approximate formulation since exact minimization of the $K$-means objective is computationally intractable in general, whereas efficient approximation algorithms are available. The approximate formulation is also sufficient for our theoretical purposes, since in the present setting the spectral embeddings $\{\fkx_i\}_{i=1}^n$ concentrate sharply around their population centers; see Section \ref{subsec:misclass-rate} for the corresponding argument. In the KSC setting, the points $\{\fkx_i\}_{i=1}^n$ are the spectral embeddings constructed from the multi-kernel matrix $\bfK$ or its associated normalized Laplacian $\bfL$. The objective is for the estimated label map $\hat{\pi}$ to recover the true cluster assignment $\bar{\pi}$, defined by
\begin{equation}
    \bar{\pi}: \dbraks{n} \to \dbraks{m},
    \qwith
    \bar{\pi}(i) = k
    \quad \Leftrightarrow \quad
    i \in \caC_k.
\end{equation}
As cluster labels are identifiable only up to a permutation, we measure the discrepancy between $\hat{\pi}$ and $\bar{\pi}$ by
\begin{equation}
    \caM (\hat{\pi}, \bar{\pi}) 
    = n^{-1} \min_{\tau} 
    \abs[\big]{ \{ i \in \dbraks{n}: \hat{\pi} (i) \neq (\tau \circ \bar{\pi}) (i) \} } ,
    \label{def:misclassification}
\end{equation}
where the minimum is taken over all permutations $\tau:\dbraks{m}\to\dbraks{m}$. We refer to $\caM(\hat{\pi},\bar{\pi})$ as the misclassification rate. The goal of this section is to establish theoretical guarantees for $\caM(\hat{\pi},\bar{\pi})$ when the estimated label map $\hat{\pi}$ is produced by Algorithm \ref{alg:multi-ksc}. Our analysis builds on the perturbation results of Section \ref{sec:main-res}. Since the bandwidths $\{h_t\}_{t=1}^T$ are now selected from the data, our perturbation results must be extended to kernel matrices with data-adaptive bandwidths. To this end, we impose several additional technical assumptions that ensure the selected bandwidths satisfy the required admissibility properties with high probability.

Consider the bandwidth-selection procedure \eqref{eqn:bandwidth-selection}. Recall that, in Assumption \ref{assump:multi-scale}, the deterministic counterparts $\theta_{k\ell}$ for the pairwise distances are grouped into $T_{\ref{assump:multi-scale}}$ levels characterized by $\{ \vartheta_s \}_{s=1}^{T_{\ref{assump:multi-scale}}}$. Since the pairwise distances concentrate around the corresponding counterparts $\theta_{k\ell}$, the empirical pairwise distances are expected to exhibit the same multi-scale structure. To this end, for each characteristic scale $\vartheta_s$, we let $N_s$ denote the number of pairwise distances whose deterministic proxy belongs to this scale,
\begin{equation*}
    N_s
    := \sum\nolimits_{k = 1}^m 
    \frac{n_k (n_k - 1)}{2}
    \bbone \curls[\big]{ \theta_{kk} \in
    [c_{\ref{assump:multi-scale}} \vartheta_s, 
    \vartheta_s / c_{\ref{assump:multi-scale}}] }
    + \sum\nolimits_{1 \leq k < \ell \leq m}
    n_k n_\ell
    \bbone \curls[\big]{ \theta_{k \ell} \in
    [c_{\ref{assump:multi-scale}} \vartheta_s, 
    \vartheta_s / c_{\ref{assump:multi-scale}}] }.
\end{equation*}
The total number of pairwise distances is
\begin{equation*}
    N := \sum\nolimits_{s=1}^{T_{\ref{assump:multi-scale}}} N_s
    = \frac{n (n-1)}{2}.
\end{equation*}
The purpose of the next assumption is to ensure that the bandwidths produced by \eqref{eqn:bandwidth-selection} behave like the admissible bandwidths introduced in Definition \ref{def:admissible-bandwidth}, even though they are now data-dependent. Intuitively, condition \eqref{cond:quantile-levels}, though formulated in a different way, essentially plays the same role as the second constraint in \eqref{cond:admissible-bandwidth}: each characteristic scale $\vartheta_s$ is captured by at least one bandwidth $h_t = G^{-1}(\omega_t)$.

\begin{assumption}[quantile levels]
\label{assump:quantile}
For each $s \in \dbraks{T_{\ref{assump:multi-scale}}}$, there exists at least one $t \in \dbraks{T}$ such that
\begin{equation}
    \omega_t \in
    \pars[\big]{{N_{< s}} / {N}, {(N_{< s} + N_s)} / {N}} ,
    \qwhere N_{< s} = N_1 + \cdots + N_{s-1}.
    \label{cond:quantile-levels}
\end{equation}
\end{assumption}

Under Assumption \ref{assump:quantile}, the bandwidths produced by the empirical quantile selection rule \eqref{eqn:bandwidth-selection} can be well approximated by a deterministic collection of admissible bandwidths.

\begin{lemma}[bandwidth admissibility]
\label{lemma:bandwidth-regularity}
Suppose Assumptions \ref{assump:multi-scale} and \ref{assump:quantile} hold. Let
\begin{equation*}
    \bar{G} (h)
    := \sum\nolimits_{k = 1}^m
    \frac{n_k (n_k-1) / 2}{N} \bbone { \{ \theta_{k k} \leq h \} } 
    + \sum\nolimits_{1 \leq k < \ell \leq m} 
    \frac{n_k n_\ell}{N} \bbone { \{ \theta_{k \ell} \leq h \} }.
\end{equation*}
Given quantile levels $\{\omega_t\}_{t=1}^T$, define the corresponding quantiles by
\begin{equation}
    \hbar_t^* 
    := \bar{G}^{-1} (\omega_t)
    = \inf \{ h: \bar{G}(h) \geq \omega_t \},
    \label{def:proxy-bandwidth}
\end{equation}
Then, $( \hbar_t^* )_{t = 1}^T$ forms a collection of admissible bandwidths. Let $h_t$ be defined as in \eqref{eqn:bandwidth-selection}. If, in addition, Assumption \ref{assump:concentration} holds and $\rho \leq n^{-c}$ for some constant $c > 0$, then for any fixed constants $\varepsilon \in (0, c/2)$ and $C > 0$, we have
\begin{equation}
    \bbp \curls[\big]{
    \abs{{h}_t - \hbar_t^*} / \hbar_t^*
    \leq n^\varepsilon \rho
    \, \text{ for all } \,
    t \in \dbraks{T} }
    \geq 1 - n^{-C}.
    \label{eqn:close-admissible}
\end{equation}    
\end{lemma}

Before extending the perturbation results of Section \ref{sec:main-res} to kernel matrices with data-adaptive bandwidths, we also modify the eigengap condition in Assumption \ref{assump:signal}. More precisely, since the eigengaps of $\barK$ depend on the bandwidths, we must specify the bandwidths at which they are evaluated. For the present purpose, it is sufficient to impose the eigengap condition on $\barK$ evaluated at the deterministic proxy bandwidths $( \hbar_t^* )_{t = 1}^T$ defined in \eqref{def:proxy-bandwidth}. Since several bandwidth choices are now involved, it is convenient to make explicit the bandwidth dependence of the multi-kernel matrix $\bfK$, its informative component $\barK$, and the associated spectral quantities. Accordingly, for a collection of bandwidths $\boldsymbol{h}=(h_t)_{t=1}^T$, we write $\bfK(\boldsymbol{h})$, $\barK(\boldsymbol{h})$, $\bfU(\boldsymbol{h})$, $\barU(\boldsymbol{h})$, and similarly for the remaining bandwidth-dependent quantities.

In addition, to simplify the presentation, we focus on the optimal eigengap regime $\Delta_r(\barLamb) \asymp n$. As discussed above, this regime can be expected in many situations under an appropriate choice of $r$. Importantly, it illustrates the main ideas of the data-adaptive extension in the cleanest setting. Our argument extends directly to more general eigengap regimes: one only needs to modify the eigengap assumption and replace the corresponding rates in the statements of the results. We omit these routine variants for clarity.

\begin{assumption}[eigengap]
\label{assump:gap-hbar}
Let $\boldsymbol{\hbar}^* = (\hbar_t^*)_{t=1}^T$ be the bandwidths defined in \eqref{def:proxy-bandwidth}. Suppose that
\begin{equation}
    \rho \leq n^{-c_{\ref{assump:gap-hbar}}}
    \qand
    \Delta_{r} (\barLamb (\boldsymbol{\hbar}^*)) 
    \geq c_{\ref{assump:gap-hbar}} n .
    \label{eqn:gap-adaptive}
\end{equation}
\end{assumption}

We now extend the entrywise eigenvector perturbation bounds in Theorem \ref{thm:Linfty-evec-bfK-barK} to kernel matrices with data-adaptive bandwidths. In principle, the eigenvalue perturbation bounds in Theorem \ref{thm:eval-bfK-barK} can be extended in the same way; we omit the corresponding statement for simplicity.

\begin{corollary}[$\ell_{2, \infty}$ eigenvector perturbation]
\label{coro:pertur-adaptive}
Let $\boldsymbol{h} = (h_t)_{t=1}^T$ be the data-adaptive bandwidths defined as in \eqref{eqn:bandwidth-selection}. Then, under Assumptions \ref{assump:multi-scale}--\ref{assump:concentration} and \ref{assump:tech-const}--\ref{assump:gap-hbar}, we have
\begin{subequations}
\begin{align}
    \norm{\bfU_r (\boldsymbol{h}) 
    \fkS_r^{\bfu} (\boldsymbol{h})
    - \barU_r (\boldsymbol{h}) }_{2,\infty}
    & \prec 
    \rho / \sqrt{n} + 1 / n^{3/2} ,
    \label{eqn:bfu-baru-adaptive} \\
    \norm{\bfU_r (\boldsymbol{h})
    \bfLamb_r(\boldsymbol{h})^{1/2}
    \fkS_r^{\bfu} (\boldsymbol{h})
    - \barU_r (\boldsymbol{h})
    \barLamb_r^{1/2} (\boldsymbol{h}) 
    }_{2,\infty}
    & \prec 
    \rho + 1 / n .
    \label{eqn:bfu-baru-adaptive-with-eigenvalue}
\end{align}
\end{subequations}
\end{corollary}

Here $\bfU_r(\boldsymbol{h})$ is viewed as a perturbation of $\barU_r(\boldsymbol{h})$. However, since now $\boldsymbol{h}$ is data-adaptive, the informative counterpart $\barU_r(\boldsymbol{h})$ is no longer deterministic, in contrast to the setting of \eqref{bound:Kev-without-eigs}. One could go one step further and replace $\barU_r(\boldsymbol{h})$ by $\barU_r(\boldsymbol{\hbar}^*)$, thereby obtaining a purely deterministic reference. We do not pursue this refinement here, because the bounds \eqref{eqn:bfu-baru-adaptive} and \eqref{eqn:bfu-baru-adaptive-with-eigenvalue} are already sufficient for the subsequent clustering analysis. Indeed, for any realization of $\boldsymbol{h}$, the informative component $\barK(\boldsymbol{h})$ retains the same blockwise constant structure, although the values associated with the blocks may now fluctuate with $\boldsymbol{h}$.

We can now present the analysis of the misclassification rate $\caM(\hat{\pi},\bar{\pi})$ defined in \eqref{def:misclassification}. Recall the spectral embeddings $\fkx_i$ introduced in \eqref{def:fkx-embedding}. The entrywise perturbation bound \eqref{eqn:bfu-baru-adaptive-with-eigenvalue} controls, up to a common orthogonal transformation, the deviation of these embeddings from the informative centers
\begin{equation}
    \bar{\fka}_k = \frac{1}{\sqrt{n_k}} 
    \braks[\big]{
    {\textstyle\sqrt{\barlamb_1}} \fkU_{k 1}, \cdots, {\textstyle\sqrt{\barlamb_{r}}} \fkU_{k r}},
    \label{def:fka-center}
\end{equation}
where we recall that $\fkU = [\fkU_{k\ell}]$ is the eigenvector matrix of the block-level matrix $\fkK$. In fact, in view of \eqref{eqn:baru-piecewise}, these centers satisfy $\bar{\fka}_k = (\barU_{r} \barLamb_{r}^{1/2})^\top \bfe_i$ for $i \in \caC_k$. To convert the resulting entrywise control into a guarantee for the $K$-means output, we also require a lower bound on the separation between distinct informative centers. This motivates the following assumption.

\begin{assumption}[separation]
\label{assump:between-cluster}
Suppose that there exists a constant $c_{\ref{assump:between-cluster}} > 0$ such that
\begin{equation}
    \bar{\fks} (\boldsymbol{\hbar}^*) 
    := \min\nolimits_{k \neq \ell} 
    \norm{\bar{\fka}_k (\boldsymbol{\hbar}^*) 
    - \bar{\fka}_\ell (\boldsymbol{\hbar}^*)}
    \geq n^{c_{\ref{assump:between-cluster}}} (\rho + 1/n) .
    \label{cond:separation-hstar}
\end{equation}
\end{assumption}

\begin{remark}
The separation condition \eqref{cond:separation-hstar} is not strong in the present setting. Indeed, the entries $\fkU_{k\ell}$ are typically of constant order. Moreover, under Assumption \ref{assump:gap-hbar}, the spiked eigenvalues $\barlamb_1 \geq \cdots \geq \barlamb_{r}$ are all of order $n$. It follows that the entries of the informative centers $\bar{\fka}_k$ are typically of constant order. Hence, it is natural to expect that distinct centers remain separated at a constant scale, which means that $\bar{\fks}(\boldsymbol{\hbar}^*) \gtrsim 1$. 

There is also a more concrete interpretation of the separation between the informative centers $\bar{\fka}_k$. Let $k,\ell \in \dbraks{m}$, and take $i \in \caC_k$ and $j \in \caC_\ell$. We can write
\begin{equation*}
    \norm{\bar{\fka}_k - \bar{\fka}_\ell}^2
    = \angles{\bfe_i - \bfe_j, \barU_{r} \barLamb_{r} \barU_{r}^\top (\bfe_i - \bfe_j)} 
    =: \angles{\bfe_i - \bfe_j, \barK_{r} (\bfe_i - \bfe_j)} 
    = (\barK_{r})_{ii} 
    + (\barK_{r})_{jj} 
    - 2 (\barK_{r})_{ij} ,    
\end{equation*}
where $\barK_{r} = \barU_{r} \barLamb_{r} \barU_{r}^\top$ represents the rank-$r$ truncation of the informative matrix $\barK$. It is not hard to see that $\barK_{r}$ inherits the blockwise constant structure of $\barK$. Hence, the (squared) separation $\norm{\bar{\fka}_k - \bar{\fka}_\ell}^2$ can be interpreted through the contrast between the corresponding within-block and between-block affinities encoded by the truncated informative matrix $\barK_{r}$. In particular, when $r = \rank(\barK)$, we have $\barK_{r}=\barK$, and therefore
\begin{align*}
    \norm{\bar{\fka}_k (\boldsymbol{\hbar}^*)
    - \bar{\fka}_\ell (\boldsymbol{\hbar}^*) }^2
    & = \bar{K}_{ii} (\boldsymbol{\hbar}^*) 
    + \bar{K}_{jj} (\boldsymbol{\hbar}^*) 
    - 2 \bar{K}_{ij} (\boldsymbol{\hbar}^*) \\
    & = \sum_{t \in \mathcal{S}_{k k} (0)} 
    \alpha_t \exp (-{\theta_{k k}} / {\hbar_t^*})
    + \sum_{t \in \mathcal{S}_{\ell \ell} (0)} 
    \alpha_t \exp ( -{\theta_{\ell \ell}} / {\hbar_t^*} )
    - 2 \sum_{t \in \mathcal{S}_{k \ell} (0)} 
    \alpha_t \exp (-{\theta_{k \ell}} / {\hbar_t^*}) .
\end{align*}
This representation also illustrates why it is reasonable to expect $\bar{\fks}(\boldsymbol{\hbar}^*) \gtrsim 1$.
\end{remark}

\begin{theorem}[exact recovery]
\label{thm:zero-misclass}
Let $C>0$ be arbitrary. Let $\hat{\pi}: \dbraks{n} \to \dbraks{m}$ be the label map produced by Algorithm \ref{alg:multi-ksc}. Under Assumptions \ref{assump:multi-scale}--\ref{assump:concentration} and \ref{assump:tech-const}--\ref{assump:between-cluster}, there exists $n_0(C)>0$ such that, for all $n\geq n_0(C)$,
\begin{equation}
    \bbp \{ \caM (\hat{\pi}, \bar{\pi}) = 0 \} \geq 1 - n^{-C}.
    \label{eqn:exact-recovery}
\end{equation}
\end{theorem}

\begin{remark}[probability tail]
Since $\caM(\hat{\pi},\bar{\pi}) \leq 1$ deterministically, Theorem \ref{thm:zero-misclass} immediately implies that $\bbe \caM(\hat{\pi},\bar{\pi}) \leq n^{-C}$ for all sufficiently large $n$. This polynomial tail probability in \eqref{eqn:exact-recovery} reflects the concentration assumption \eqref{bound:concentration}. Under a suitably strengthened concentration assumption with exponential tails, we expect that the present argument can be refined to yield an exponentially decaying bound for the exact-recovery failure probability. We do not pursue this extension here.

Such exponential behavior has been established in related non-kernel settings. In particular, \cite{lofflerOptimalitySpectralClustering2021} showed that a direct SVD-based spectral clustering procedure for isotropic Gaussian mixture models achieves an exponentially small misclassification rate. Their result further yields exact recovery under a sufficiently strong separation condition, with an explicit exponentially decaying failure probability.
\end{remark}

\begin{remark}[Normalized-Laplacian variant]
Algorithm \ref{alg:multi-ksc} describes the kernel version of the multi-KSC procedure, in which the spectral embeddings are constructed from the leading principal components of $\bfK$. The corresponding normalized-Laplacian variant is obtained by replacing $\bfK$ with $\bfL = n \bfD^{-1/2} \bfK \bfD^{-1/2}$ and computing the spectral decomposition as in \eqref{eqn:spectral-decomp-L}. In this case, the spectral embedding is defined by 
\begin{equation*}
    \fkx_i = \braks[\big]{\sqrt{\gamma_1} \bfv_1 (i), \cdots, \sqrt{\gamma_{r}} \bfv_{r} (i)}.
\end{equation*}
The bandwidth-selection procedure and the subsequent $K$-means step remain unchanged. The corresponding adaptations to the subsequent clustering analysis are straightforward and therefore omitted.
\end{remark}

We conclude this section with a simple numerical illustration of Algorithm \ref{alg:multi-ksc} in a multi-scale clustering setting. Let $n \in 6\bbn$, and consider a balanced six-component GMM in $\bbr^p$, with exactly $n/6$ observations from each cluster. For $i \in \caC_k$, we generate $\bfx_i = \bfmu_k + \bfA_k \bfz_i$, where $\bfz_i \sim \mathrm{N}(\bfzero,\bfI)$ are independent across $i \in \dbraks{n}$. The six clusters are divided into two scale groups, each containing three clusters. Fix $\kappa \in (0,1/2)$, and define $\beta_1 = 1/2 - \kappa$ and $\beta_2 = 1/2 + \kappa$, together with the diagonal matrices
\begin{equation*}
    \bfT_g
    = \diag ( 1^{-\beta_g},2^{-\beta_g}, \cdots, p^{-\beta_g} ) \in \bbr^{p \times p},
    \qfor g \in \{1,2\}.
\end{equation*}
Thus, the diagonal entries of $\bfT_1$ and $\bfT_2$ decay at different polynomial rates. Note that $\Tr \bfT_g \asymp p^{1 - \beta_g}$ and therefore ${\Tr \bfT_1} / {\Tr \bfT_2} \asymp p^{2\kappa}$. In particular, $\kappa$ controls the separation between the two distance scales. For each value of $p$, we first generate six independent random orthogonal matrices $\bfQ_1,\cdots,\bfQ_6$ and keep them fixed throughout the Monte Carlo experiment. We then set
\begin{equation*}
    \bfA_{3g-2} = \bfQ_{3g-2} \bfT_g^{1/2},
    \qquad
    \bfA_{3g-1} = \bfQ_{3g-1} \bfT_g^{1/2},
    \qquad
    \bfA_{3g} = 2 \bfQ_{3g} \bfT_g^{1/2},
\end{equation*}
for $g \in \{1,2\}$. Thus, all three clusters in group $g$ have the same asymptotic covariance scale, determined by $\Tr \bfT_g$. The third cluster differs from the other two only through a fixed-factor inflation of its covariance magnitude. To distinguish the first two clusters within each group, we define the cluster means by
\begin{equation*}
    \bfmu_{3g-2} = \bfb_g + \bfs_g,
    \qand
    \bfmu_{3g-1} = \bfmu_{3g} = \bfb_g,
\end{equation*}
where $\bfs_g$ and $\bfb_g$ denote the within-group and between-group offsets, respectively. Specifically,
\begin{equation*}
    \bfs_1 = (\Tr\bfT_1)^{1/2} \bfe_1,
    \qquad
    \bfs_2 = (\Tr\bfT_2 / p)^{1/2} {\bfone},
    \qquad
    \bfb_1 = (\Tr\bfT_1)^{1/2} \bfe_{p-1},
    \qquad
    \bfb_2 = (\Tr\bfT_2)^{1/2} \bfe_p.
\end{equation*}
Hence, within each scale group, the first cluster is separated from the other two by a mean shift whose squared norm is of the same order as the corresponding within-group distance scale. Under this setup, approximately $1/4$ of the squared pairwise distances are of the smaller order $p^{1/2-\kappa}$, corresponding to pairs of observations both drawn from the second scale group. The remaining $3/4$ are of the larger order $p^{1/2+\kappa}$.

We compare multi-KSC with an oracle-tuned single-KSC benchmark. For multi-KSC, we use the quadratically spaced quantile levels $\omega_t = {t}^2 / ({T+1})^2$. This choice places more candidate bandwidths in the lower quantile range while still spreading the remaining bandwidths across the empirical distance distribution. Here we consider $T \in \{3,4\}$ and use equal kernel weights $\alpha_t=1/T$. The single-KSC benchmark likewise uses a quantile-based bandwidth $h=G^{-1}(\omega)$. For each simulated data set, it selects the value of $\omega$ from $\{0.1,0.2,\ldots,0.9\}$ that yields the smallest misclassification rate. This benchmark is infeasible in practice, since it uses the true class labels to select the bandwidth separately in each Monte Carlo replication. It therefore provides a favorable reference for the single-bandwidth procedure.

\begin{figure}[htb]
    \centering
    \includegraphics[width=6.5in]{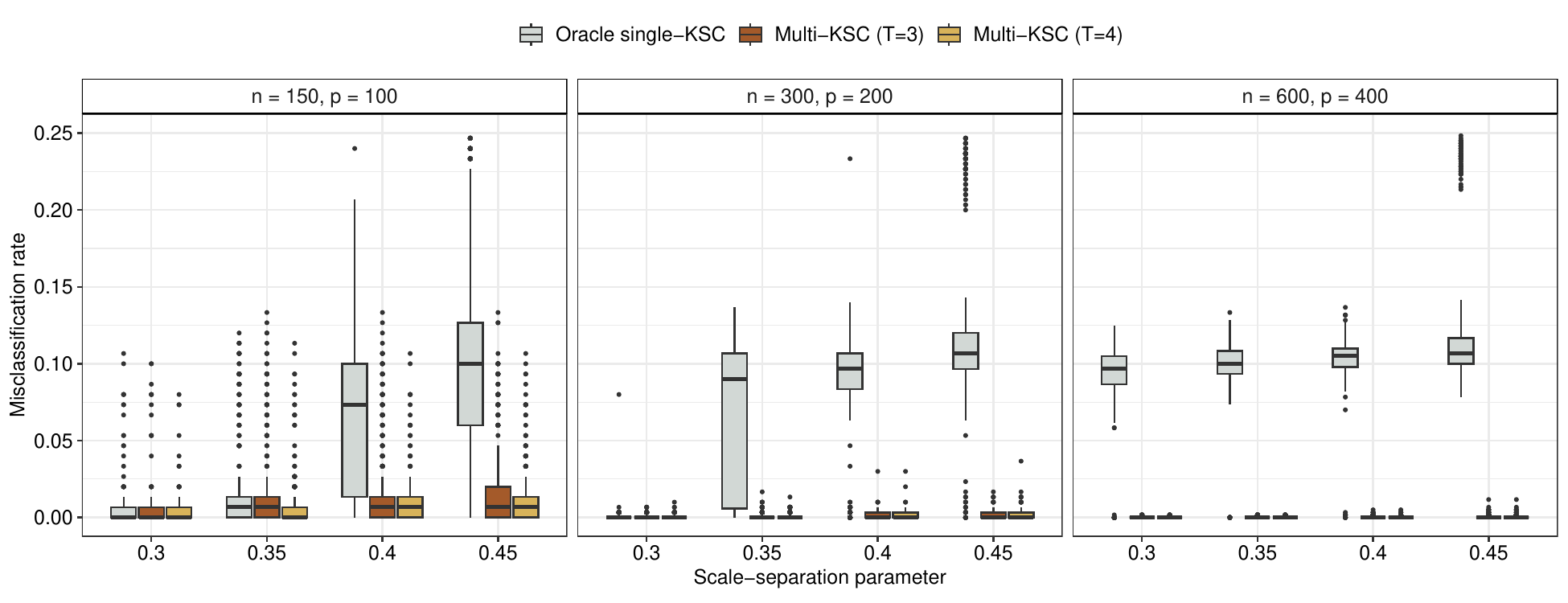}
    \caption{Misclassification rates over $500$ replications for the oracle-tuned single-KSC benchmark and multi-KSC with $T=3$ and $T=4$.}
    \label{fig:multi_vs_single}
\end{figure}

We consider $\kappa \in \{ 0.30,0.35,0.40,0.45 \}$ and $(n,p) \in \{ (150,100),(300,200),(600,400) \}$. Each configuration is repeated over $500$ independent Monte Carlo replications. The resulting misclassification rates are depicted in Figure \ref{fig:multi_vs_single}. In this example, both multi-KSC procedures have essentially zero median misclassification error across the displayed configurations. There are occasional nonzero outliers, especially in the smallest sample-size setting, but both $T=3$ and $T=4$ remain highly accurate overall even for the larger values of $\kappa$. Increasing the number of kernels from $T=3$ to $T=4$ does not yield a systematic improvement, suggesting that, in this example, the main benefit comes from covering the two relevant distance scales rather than from using a larger number of bandwidths. By contrast, the oracle single-KSC benchmark exhibits more substantial errors as $\kappa$ increases, with this deterioration becoming more pronounced at larger problem sizes. Since the single-bandwidth benchmark is selected using the true labels, this comparison illustrates the difficulty, in this example, of using one global bandwidth to represent both distance scales. Figure \ref{fig:multi_vs_single} also provides an empirical illustration of the quantile-based bandwidth selection rule \eqref{eqn:bandwidth-selection}. The multi-KSC bandwidths are selected solely from empirical quantiles of the observed pairwise squared distances, without using either the true labels or the population-scale parameters. Nevertheless, in this example, the resulting bandwidths capture the two relevant distance scales sufficiently well to yield accurate clustering.

To further illustrate how the quantile-based bandwidths capture complementary information from the two distance scales, Figure \ref{fig:ev_pairs} displays the corresponding spectral embeddings for one realization with $(n,p)=(300,200)$ and $\kappa=0.4$. We compare the multi-kernel matrix with $T=4$ to single-kernel matrices constructed using $\omega \in \{ 0.2, 0.7 \}$. The choice $\omega=0.2$ targets the smaller characteristic scale $p^{1/2-\kappa}$, whereas $\omega=0.7$ is for the larger scale $p^{1/2+\kappa}$. For each kernel matrix, we display the coordinate pairs $(\bfu_1,\bfu_3)$ and $(\bfu_2,\bfu_4)$ formed from its leading eigenvectors. The regions returned by $K$-means with $m=6$ are also overlaid.

\begin{figure}[htb]
    \centering
    \includegraphics[width=6in]{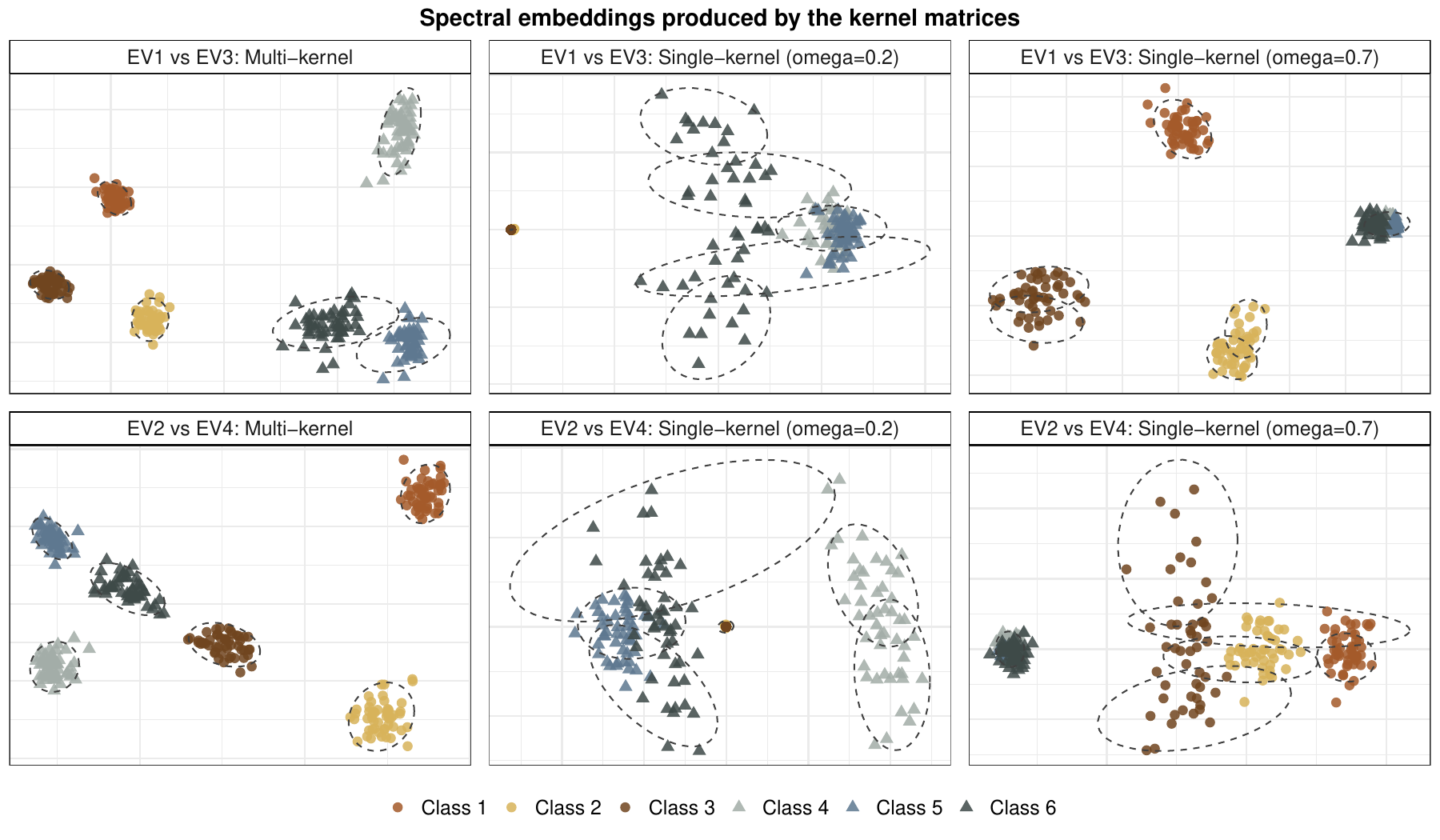}
    \caption{Comparison of spectral embeddings from the multi-kernel matrix ($T=4$) and single-kernel matrices ($\omega=0.2$ and $0.7$).}
    \label{fig:ev_pairs}
\end{figure}

The multi-kernel embedding in the left panel resolves the six classes jointly across the two displayed coordinate pairs: classes that are closer in one pair are separated in the other. In contrast, each single-bandwidth construction captures only part of the clustering structure. Depending on the selected bandwidth, some classes form diffuse or elongated clouds, while others overlap substantially, even when the remaining classes are well separated. For example, for $\omega=0.2$ in the middle panel, the three larger-scale clusters $\caC_1$, $\caC_2$, and $\caC_3$ concentrate in the same region and are difficult to distinguish. In this realization, $K$-means places these points in a single cluster and partitions the observations from $\caC_4$, $\caC_5$, and $\caC_6$ into the remaining five clusters. A corresponding phenomenon occurs for $\omega=0.7$ in the right panel, with the roles of the two scale groups reversed. This illustrates the mechanism behind the preceding misclassification results: different bandwidths emphasize different distance scales, whereas multi-KSC combines their complementary information before constructing the spectral embedding.

\begin{appendices}

\section{Perturbation analysis for multi-kernel matrices}
\label{sec:proof-kernel}

This section contains the proofs of our main spectral perturbation results, Theorems \ref{thm:eval-bfK-barK} and \ref{thm:Linfty-evec-bfK-barK}, for the multi-kernel matrix $\bfK$ defined in \eqref{def:multi-Kernel}. In Section \ref{subsec:finer-decomp}, we introduce an intermediate matrix $\ddotK$, which provides a more accurate approximation to the empirical kernel matrix $\bfK$ than the informative component $\barK$. We then establish Propositions \ref{prop:norm-K-barK-dotK}--\ref{prop:evec-two-stage}, which refine Theorems \ref{thm:eval-bfK-barK} and \ref{thm:Linfty-evec-bfK-barK} through a two-stage perturbation analysis: first, $\bfK$ is treated as a perturbation of $\ddotK$, and then $\ddotK$ is treated as a perturbation of $\barK$. The corresponding norm bounds for $\ddotK-\barK$ and $\bfK-\ddotK$ are established in Section \ref{subsec:norm-bound-residual}. The entrywise eigenvector perturbation bounds for $\ddotU_r \ddotfkS_r-\barU_r$ and $\bfU_r\fkS_r-\ddotU_r\ddotfkS_r$ are proved in Sections \ref{subsec:dominant-evec-entrywise} and \ref{subsec:subleading-evec-entrywise}, respectively. Combining these two bounds directly yields the bound for $\bfU_r\fkS_r-\barU_r$ in \eqref{bound:Kev-without-eigs}. Finally, the proof of \eqref{bound:Kev-with-eigs}, which additionally involves eigenvalue rescaling, is given in Section \ref{subsec:score-bound}.

\subsection{Finer decomposition of the multi-kernel matrix}
\label{subsec:finer-decomp}

In Section \ref{sec:main-res}, we introduce the decomposition of the multi-kernel matrix $\bfK = \barK + (\bfK-\barK)$. Here we further decompose the residual term $\bfK-\barK$ and establish finer control for its constituent parts. Recall that our guiding idea is to approximate the squared distances by the deterministic quantities $\theta_{k\ell}$ defined in \eqref{def:theta-kl}. Accordingly, for $i \in \caC_k$ and $j \in \caC_\ell$, we define the approximation error
\begin{equation}
    \xi_{ij} := \norm{\bfx_i-\bfx_j}^2 - \theta_{k \ell},
    \qfor
    i \in \caC_k, ~ j \in \caC_\ell.
    \label{def:xi-ij}
\end{equation}
Expanding the squared distance gives
\begin{align*}
    \norm{\bfx_i-\bfx_j}^2
    = & ~ \norm{\bfmu_k-\bfmu_\ell}^2
    + \norm{\bfx_i - \bfmu_k}^2 
    + \norm{\bfx_j - \bfmu_\ell}^2 \\
    & + 2 \angles{\bfx_i-\bfmu_k, \bfmu_k - \bfmu_\ell}
    + 2 \angles{\bfx_j-\bfmu_\ell, \bfmu_\ell - \bfmu_k}
    -2 \angles{\bfx_i-\bfmu_k, \bfx_j-\bfmu_\ell} .
\end{align*}
Using the data model \eqref{eqn:data-model}, we arrive at the decomposition
\begin{equation}
    \xi_{ij} = \zeta_i + \zeta_j + 2\Upsilon_{ij} - 2W_{ij},
    \qfor
    i \in \caC_k, ~ j \in \caC_\ell ,
    \label{eqn:xi-decomp}
\end{equation}
where the random fluctuations on the r.h.s. are defined as
\begin{equation}
    \zeta_i = \norm{\bfA_k\bfz_i}^2 - \Tr \bfSigma_k,
    \quad \quad
    \Upsilon_{ij} =  \angles{\bfmu_k - \bfmu_\ell, 
    \bfA_k\bfz_i - \bfA_\ell\bfz_j},
    \quad \quad
    W_{ij} = \angles{\bfA_k\bfz_i, \bfA_\ell\bfz_j} .   
    \label{def:zeta-Ups-W}
\end{equation}
By introducing
\begin{equation*}
    \bfXi = [\xi_{ij}]_{i,j=1}^n,
    \quad \quad
    \bfzeta = [\zeta_i]_{i=1}^n,
    \quad \quad
    \bfUps = [\Upsilon_{ij}]_{i,j=1}^n,
    \quad \quad
    \bfW = [W_{ij}]_{i,j=1}^n,
\end{equation*}
we can express the decomposition \eqref{eqn:xi-decomp} in matrix form as
\begin{equation*}
    \bfXi = \bfzeta \bfone^\top + \bfone \bfzeta^\top + 2\bfUps - 2\bfW.
\end{equation*}
The next lemma provides high-probability control of the random fluctuation terms defined above. Recall that the deterministic control parameters $\psi_{k\ell}$ were defined in \eqref{def:psi}. 

\begin{lemma}[concentration of pairwise distances]
\label{lemma:xi-control}
Under Assumption \ref{assump:concentration}, we have
\begin{equation}
    \abs{\zeta_i}
    + \abs{\Upsilon_{ij}} + 
    + (1-\delta_{ij}) \abs{W_{ij}}
    + (1-\delta_{ij}) \abs{\xi_{ij}}
    \prec \psi_{k \ell},
    \qfor
    i \in \caC_k, ~ j \in \caC_\ell.
    \label{eqn:bound-xi-zeta-Ups}
\end{equation}
\end{lemma}

\begin{proof}[Proof of Lemma \ref{lemma:xi-control}]
Let $i \in \caC_k$ and $j \in \caC_\ell$. By the Hanson-Wright type inequality \eqref{eqn:hanson-wright}, we have
\begin{equation*}
    \norm{\bfA_k\bfz_i}^2
    = \angles{\bfz_i, \bfOmega_{kk} \bfz_i}
    = \Tr \bfOmega_{kk} + \oprec{\norm{\bfOmega_{kk}}_\Fnorm} .
\end{equation*}
The bound $\xi_i = \oprec{\psi_{k \ell}}$ then follows from the identities $\Tr \bfOmega_{kk} = \Tr \bfSigma_k$ and $\norm{\bfOmega_{kk}}_\Fnorm = \norm{\bfSigma_k}_\Fnorm$. Next, to control $\Upsilon_{ij}$, we apply the Bernstein type inequality \eqref{eqn:bernstein} and obtain
\begin{equation*}
    \abs{\angles{\bfA_k \bfz_i, \bfmu_k - \bfmu_\ell}}
    \prec \norm{\bfA_k^{ \top} (\bfmu_k - \bfmu_\ell)}
    \leq \psi_{k \ell}
    \qand
    \abs{\angles{\bfA_\ell \bfz_j, \bfmu_k - \bfmu_\ell}}
    \prec \norm{\bfA_\ell^{ \top} (\bfmu_k - \bfmu_\ell)}    
    \leq \psi_{k \ell} .
\end{equation*}
For the term $W_{ij}$ with $i \neq j$, we first condition on $\bfz_j$, apply \eqref{eqn:bernstein}, and then invoke \eqref{eqn:hanson-wright} to deduce
\begin{equation*}
    \abs{W_{ij}}
    = \abs{\angles{\bfz_i, \bfOmega_{k \ell} \bfz_j}}
    \prec \norm{\bfOmega_{k \ell} \bfz_j} 
    = \angles{\bfz_j, \bfOmega_{\ell k} \bfOmega_{k \ell} \bfz_j}^{1/2} 
    \prec \norm{\bfOmega_{k \ell}}_\Fnorm 
    + \norm{\bfOmega_{k \ell} \bfOmega_{\ell k}}_\Fnorm^{1/2}
    \lesssim \norm{\bfOmega_{k \ell}}_\Fnorm ,    
\end{equation*}
where the last step uses the elementary inequality $\norm{\bfA \bfB}_\Fnorm \leq \norm{\bfA}_\Fnorm \norm{\bfB}_\Fnorm$. Moreover, by the Cauchy--Schwarz inequality, we have
\begin{equation*}
    \norm{\bfOmega_{k \ell}}_\Fnorm
    = ({\Tr \bfSigma_k \bfSigma_\ell})^{1/2}
    \leq \norm{\bfSigma_k}_\Fnorm^{1/2} \norm{\bfSigma_\ell}_\Fnorm^{1/2}
    \lesssim \norm{\bfSigma_k}_\Fnorm + \norm{\bfSigma_\ell}_\Fnorm
    \leq \psi_{k \ell} ,
\end{equation*}
which completes the control of $W_{ij}$. Collecting the above bounds for $\zeta_i$, $\Upsilon_{ij}$ and $W_{ij}$, we obtain the desired bound $\abs{\xi_{ij}} \prec \psi_{k \ell}$ and therefore complete the proof of Lemma \ref{lemma:xi-control}.
\end{proof}

We now introduce a new matrix $\ddotK$, which serves as an intermediate approximation between the empirical kernel matrix $\bfK$ and the purely deterministic counterpart $\barK$. In Theorems \ref{thm:eval-bfK-barK} and \ref{thm:Linfty-evec-bfK-barK}, we treated $\bfK$ directly as a perturbation of $\barK$. As mentioned, here we refine that analysis by inserting the intermediate matrix $\ddotK$, which allows the eigenvalue and eigenvector perturbation argument to be carried out in two stages: first comparing $\ddotK$ with $\barK$, and then comparing $\bfK$ with $\ddotK$. To construct $\ddotK$, we again expand the kernelized similarity $\exp (- \norm{\bfx_i - \bfx_j}^2 / \hbar_t)$ around $\theta_{k \ell} / \hbar_t$, but this time we retain the first-order terms except for the one involving $W_{ij}$. More precisely, the entries of $\ddotK = [\ddot{K}_{ij}]_{i,j=1}^n$ are defined by 
\begin{equation}
    \ddot{K}_{ij} 
    = \sum_{t \in \caS_{k \ell} (0)} \alpha_t \braks[\bigg]{
    f ( {\theta_{k \ell}} / {\hbar_t} )
    + f^\prime ( {\theta_{k \ell}} / {\hbar_t} ) 
    \cdot \frac{\zeta_i + \zeta_j + 2\Upsilon_{ij}}{\hbar_t} },
    \qfor
    i \in \caC_k, ~ j \in \caC_\ell.
    \label{def:dot-Kij}
\end{equation}
Here and throughout the sequel, we write $f(x)=\exp(-x)$ for the kernel profile. In this way, the matrix $\ddotK$ provides a more accurate approximation to $\bfK$ than $\barK$. To express the definition of $\ddotK$ in matrix form, we introduce the following derivative matrix, which is again blockwise constant,
\begin{equation}
    \bfQ := [Q_{ij}]_{i, j = 1}^n
    \qwith
    Q_{ij} := \sum_{t \in \caS_{k \ell} (0)}
    \frac{\alpha_t f^\prime ( {\theta_{k \ell}} / {\hbar_t} )}{\hbar_t},
    \qfor
    i \in \caC_k, ~ j \in \caC_\ell.
    \label{def:bfQ}
\end{equation}
Then, we have
\begin{equation}
    \ddotK 
    = \barK + \bfQ \circ (\bfzeta \bfone^\top + \bfone \bfzeta^\top + 2\bfUps).
    \label{eqn:diff-dotK-barK}
\end{equation}
By definition, we have $\rank (\bfQ) \leq m$ and $\rank (\bfUps) \leq 2m$, where the latter follows immediately from
\begin{equation*}
    \bfUps = \sum\nolimits_{\ell=1}^m
    (\bfups_\ell \bfone_{\caC_\ell}^\top
    + \bfone_{\caC_\ell} \bfups_\ell^\top)
    \qwhere 
    \bfups_\ell (i)
    := \angles{\bfmu_k - \bfmu_\ell, \bfA_k \bfz_i} 
    \qfor i \in \caC_k.
\end{equation*}
Therefore, by the elementary inequality $\rank (\bfA \circ \bfB) \leq \rank (\bfA) \rank (\bfB)$, we find
\begin{equation*}
    \rank (\ddotK)
    \leq \rank (\barK)
    +  \rank (\bfQ) [2 + \rank (\bfUps)]
    \leq m (2m + 3).
\end{equation*}
In particular, the matrix $\ddotK$ remains low-rank. We also note that, in the definitions of $\barK$ in \eqref{def:bar-Kij} and $\ddotK$ in \eqref{def:dot-Kij}, the diagonal entries are not set to zero. The next result control the spectral norm of both the residual terms $\ddotK - \barK$ and $\bfK - \ddotK$. Its proof is presented in Section \ref{subsec:norm-bound-residual}.

\begin{proposition}[two-stage norm bounds]
\label{prop:norm-K-barK-dotK}
Suppose that Assumptions \ref{assump:multi-scale}--\ref{assump:tech-const} hold. Let $\bfK$, $\ddotK$ and $\barK$ be kernel matrices constructed using an admissible collection of bandwidths $\{ \hbar_t \}_{t=1}^T$. Then,
\begin{subequations}
\begin{align}
    \norm{\ddotK - \barK} 
    & \prec n \rho,
    \label{bound:ddotK-barK} \\
    \norm{\bfK - \ddotK} 
    & \prec 1 + \sqrt{n} \phi + n \rho^2.
    \label{bound:bfK-ddotK}
\end{align}
\end{subequations}
\end{proposition}

Recall that $\phi \lesssim \sqrt{n}\rho$ by definition and that $\rho \ll 1$ under Assumption \ref{assump:signal}. Therefore, in the two-stage decomposition $\bfK-\barK = (\ddotK-\barK)+(\bfK-\ddotK)$, one may heuristically view $\ddotK-\barK$ as the leading structured component of the residual perturbation, with $\bfK-\ddotK$ treated as a further remainder term.

Denote the spectral decomposition of $\ddotK$ by
\begin{equation*}
    \ddotK 
    = \ddotU \ddotLamb \ddotU^\top,
    \qwhere \ddotLamb = \diag (\ddotlamb_1, \cdots, \ddotlamb_{m (2m + 3)})
    \qand
    \ddotU = [\ddotu_1, \cdots, \ddotu_{m (2m + 3)}].
\end{equation*}
Let $r \in \dbraks{m}$ be fixed. In parallel with \eqref{def:Lamb-U-r}, we define
\begin{equation*}
    \ddotLamb_r = \diag (\ddotlamb_1, \cdots, \ddotlamb_{r})
    \qand
    \ddotU_r = [\ddotu_1, \cdots, \ddotu_{r}].
\end{equation*}
The next result gives a refined bound for the spiked eigenvalues of $\ddotK$. Its proof is deferred to Section \ref{subsec:proof-refined-sqrtn}. This bound is sharper than the one obtained by directly combining \eqref{bound:ddotK-barK} with Weyl's inequality. We do not pursue an analogous refinement for $\norm{\bfLamb_r-\ddotLamb_r}$. Indeed, this quantity can be controlled directly by $\norm{\bfK-\ddotK}$ via Weyl's inequality, and the bound in \eqref{bound:bfK-ddotK} is already sufficiently sharp for our purposes. 

\begin{proposition}[eigenvalue perturbation]
\label{prop:dotlamb-barlamb}
Under the same setup as in Proposition \ref{prop:norm-K-barK-dotK},
\begin{equation}
    \norm{\ddotLamb_r - \barLamb_r}
    \prec \sqrt{n} \rho 
    + \frac{n^2 \rho^2}{\Delta_r (\barLamb)}.
    \label{bound:ddotLamb-barLamb}
\end{equation}
\end{proposition}

We note that Theorem \ref{thm:eval-bfK-barK} follows immediately from Propositions \ref{prop:norm-K-barK-dotK} and \ref{prop:dotlamb-barlamb}. Indeed, the first bound in \eqref{eqn:eigenvalue-perturb-K} follows from \eqref{bound:ddotK-barK} and \eqref{bound:bfK-ddotK} by the triangle inequality. For the second bound in \eqref{eqn:eigenvalue-perturb-K}, Weyl's inequality, together with \eqref{bound:bfK-ddotK} and \eqref{bound:ddotLamb-barLamb}, yields
\begin{equation*}
    \norm{\bfLamb_r - \barLamb_r}
    \leq \norm{\bfLamb_r - \ddotLamb_r} + \norm{\ddotLamb_r - \barLamb_r}
    \leq \norm{\bfK - \ddotK} + \norm{\ddotLamb_r - \barLamb_r}
    \leq 1 + \sqrt{n} \phi 
    + \frac{n^2 \rho^2}{\Delta_r (\barLamb)},
\end{equation*}
where we used $\sqrt{n} \rho \lesssim 1 + n \rho^2$ and $\Delta_r (\barLamb) \lesssim n$ in the last step.

We now turn to the eigenvector perturbation analysis. We decompose
\begin{equation*}
    \bfU_r\fkS_r^{\bfu} - \barU_r
    = \pars{\bfU_r\fkS_r^{\bfu} - \ddotU_r\ddotfkS_r^{\bfu}}
    + \pars{\ddotU_r \ddotfkS_r^{\bfu} - \barU_r},
\end{equation*}
where $\ddotfkS_r^{\bfu} = \sgn(\ddotU_r^\top \barU_r)$ is the orthogonal alignment matrix between $\ddotU_r$ and $\barU_r$. 

\begin{proposition}[two-stage $\ell_{2, \infty}$ eigenvector perturbation]
\label{prop:evec-two-stage}
Under the same setup as in Proposition \ref{prop:norm-K-barK-dotK},
\begin{subequations}
\begin{align}
    \norm{\ddotU_r \ddotfkS_r^{\bfu} - \barU_r}_{2,\infty}
    & \prec 
    \frac{\sqrt{n} \rho}{\Delta_r (\barLamb)}
    + \frac{n \rho}{\Delta_r (\barLamb)^2} 
    + \frac{n^{5/2} \rho^2}{\Delta_r (\barLamb)^3} ,
    \label{bound:ddotU-barU} \\
    \norm{\bfU_r \fkS_r^{\bfu} - \ddotU_r \ddotfkS_r^{\bfu}}_{2,\infty}
    & \prec 
    \frac{\sqrt{n} + n \phi + n^{3/2} \rho^2}{\Delta_r (\barLamb)^2} .
    \label{bound:bfU-ddotU}
\end{align}    
\end{subequations} 
\end{proposition}

The perturbation bounds \eqref{bound:ddotU-barU} and \eqref{bound:bfU-ddotU} are proved in Sections
\ref{subsec:dominant-evec-entrywise} and \ref{subsec:subleading-evec-entrywise}, respectively. In the optimal eigengap regime $\Delta_r(\barLamb)\asymp n$, these bounds simplify to
\begin{equation*}
    \norm{\ddotU_r \ddotfkS_r - \barU_r}_{2,\infty}
    \prec 
    {\rho} / {\sqrt{n}} 
    \qand
    \norm{\bfU_r\fkS_r-\ddotU_r\ddotfkS_r}_{2,\infty}
    \prec 
    1 / {n^{3/2}}
    + \phi / {n}
    + \rho^2 / \sqrt{n} .
\end{equation*}
These rates are consistent with the operator-norm bounds \eqref{bound:ddotK-barK} and \eqref{bound:bfK-ddotK}. Indeed, since $\norm{\barK} \asymp \norm{\barK}_{\Fnorm} \asymp n$ and $\norm{\barU_r}_{2,\infty} \asymp n^{-1/2}$, the first-stage perturbation bounds have the same relative scale:
\begin{equation*}
    {\norm{\ddotK - \barK}}
    \big / {\norm{\barK}}
    \prec \rho
    \qand
    {\norm{\ddotU_r \ddotfkS_r - \barU_r}_{2,\infty}}
    \big / {\norm{\barU_r}_{2,\infty}}
    \prec \rho.
\end{equation*}
Likewise, the second-stage bounds satisfy
\begin{equation*}
    {\norm{\bfK-\ddotK}} 
    \big / {\norm{\barK}}
    \prec 1 / {n}
    + \phi / \sqrt{n}
    + \rho^2
    \qand
    {\norm{\bfU_r\fkS_r-\ddotU_r\ddotfkS_r}_{2,\infty}}
    \big / {\norm{\barU_r}_{2,\infty}}
    \prec 1 / {n}
    + \phi / \sqrt{n}
    + \rho^2.
\end{equation*}
Thus, in the optimal eigengap regime, the rowwise eigenvector perturbation at each stage is of the same relative order as the corresponding matrix perturbation.

We conclude this section by recording several elementary properties of $\prec$ that will be used throughout the paper. Each of these properties follows directly from the definition.

\begin{lemma}[basic properties of $\prec$] 
\label{lemma:basic-property-prec}
The relation $\prec$ introduced in Definition \ref{def:stochastic-domination} satisfies the following properties.
\begin{enumerate}[label = (\roman*)]
    \item $X \prec Y$ and $Y \prec Z$ imply $X \prec Z$.
    \item If $X_1 \prec Y_1$ and $X_2 \prec Y_2$, then $X_1 + X_2 \prec Y_1 + Y_2$ and $X_1 X_2 \prec Y_1 Y_2$.
    \item Suppose that $X(s, t) \prec Y(s, t)$ uniformly in $s \in \mathcal{S}_n$ and $t \in \mathcal{T}_n$.
    If $|\mathcal{S}_n| \leq n^C$ for some constant $C$, then
    \begin{equation*}
        \sum\nolimits_{s \in \mathcal{S}_n} X(s,t) 
        \prec 
        \sum\nolimits_{s \in \mathcal{S}_n} Y(s,t)
        \qand
        \max_{s \in \mathcal{S}_n} X(s,t) 
        \prec 
        \max_{s \in \mathcal{S}_n} Y(s,t),
        \quad \text{ uniformly in } \quad
        t \in \mathcal{T}_n.
    \end{equation*}
    \item \label{item:cancel-prec} If $X \prec Y + n^{-\varepsilon} X$ for some fixed $\varepsilon > 0$, then $X \prec Y$.
    \item \label{item:compatibility-expectation} Let $\Psi$ be a deterministic control parameter satisfying $\Psi \geq n^{-C}$. Suppose that for all $\ell \in \bbn_+$ there is a constant $C_\ell$ such that the nonnegative random variable $X$ satisfies $\bbe X^\ell \leq n^{C_\ell}$. Then,
    \begin{equation*}
        X \prec \Psi
        \quad \Longleftrightarrow \quad
        \bbe X^\ell \prec \Psi^\ell
        ~ \text{ for every fixed } ~ \ell \in \bbn_+.
    \end{equation*}
\end{enumerate}
\end{lemma}

\subsection{Proof of Proposition \ref{prop:norm-K-barK-dotK}}
\label{subsec:norm-bound-residual}

\begin{proof}[Proof of \eqref{bound:ddotK-barK}]
Let $i \in \caC_k$ and $j \in \caC_\ell$. Recalling the definition of $\bar{K}_{ij}$ in \eqref{def:bar-Kij} and $\ddot{K}_{ij}$ in \eqref{def:dot-Kij}, we have
\begin{equation*}
    \ddot{K}_{ij} - \bar{K}_{ij}
    = \sum\nolimits_{t \in \caS_{k \ell} (0)} 
    \alpha_t f^\prime ( {\theta_{k \ell}} / {\hbar_t} ) 
    \cdot \frac{\zeta_i + \zeta_j + 2\Upsilon_{ij}}{\hbar_t}.
\end{equation*}
Note that we have $\min_{t \in \caS_{k \ell} (0)} \hbar_t \gtrsim \theta_{k \ell}$ by definition of $\caS_{k \ell} (0)$ in \eqref{def:caS}. Recall the definition of $\rho$ in \eqref{def:ctrl-para-rho-phi}. Applying Lemma \ref{lemma:xi-control}, together with the bound $\norm{f^\prime}_{\infty} \lesssim 1$, we obtain
\begin{equation}
    \abs{\ddot{K}_{ij} - \bar{K}_{ij}} 
    \prec {\psi_{k\ell}}/{\theta_{k \ell}}
    \leq \rho,
    \quad \text{ uniformly in } \quad
    i, j \in \dbraks{n},
    \label{bound-entry-Kdot-Kbar}
\end{equation}
It follows that
\begin{equation*}
    \norm{\ddotK - \barK} 
    \leq \norm{\ddotK - \barK}_{\Fnorm}
    = \pars[\Big]{\sum\nolimits_{i, j = 1}^n \abs{\ddot{K}_{ij} - \bar{K}_{ij}}^2}^{1/2}
    \prec n \rho,
\end{equation*}
which completes the proof of \eqref{bound:ddotK-barK}.
\end{proof}

\begin{proof}[Proof of \eqref{bound:bfK-ddotK}]
For $i \in \caC_k$ and $j \in \caC_\ell$ with $i \neq j$, we have the decomposition $\norm{\bfx_i - \bfx_j}^2 = \theta_{k \ell} + \xi_{ij}$. Applying Taylor expansion to $f (\norm{\bfx_i - \bfx_j}^2 / \hbar_t)$ around $\theta_{k\ell}/\hbar_t$ for each $t \in \caS_{k \ell}(0)$, we obtain
\begin{align}
\begin{split}
    K_{ij}
    = & ~ \sum\nolimits_{t \in \caS_{k \ell} (0)} 
    \alpha_t \braks[\Big]{f ( {\theta_{k \ell}} / {\hbar_t} ) 
    + f'( {\theta_{k\ell}} / {\hbar_t} )
    \cdot ({\xi_{ij}} / {\hbar_t}) 
    + f^{\prime \prime} ( {\eta_{t, ij}} / {\hbar_t} )
    \cdot ({\xi_{ij}} / {\hbar_t})^2 / 2 } \\
    & + \sum\nolimits_{t \in \caS_{k \ell} (1)}
    \alpha_t f \pars[\big]{ (\theta_{k \ell} + \xi_{ij}) / {\hbar_t} } ,        
\end{split} \label{eqn:K-taylor}
\end{align}
where each $\eta_{t, ij}$ lies between $\theta_{k \ell}$ and $\theta_{k \ell}+\xi_{ij}$. Recall the decomposition of $\xi_{ij}$ in \eqref{eqn:xi-decomp} and the definitions of $\ddotK_{ij}$ in \eqref{def:dot-Kij} and $W_{ij}$ in \eqref{def:zeta-Ups-W}. Let $\caH (\bfA) := \bfA - \diag(\bfA)$ denote the hollowing operator, which sets the diagonal entries of a square matrix to zero. Then the expansion
\eqref{eqn:K-taylor} can be written in matrix form as
\begin{equation}
    \bfK 
    = \caH (\ddotK)
    + \bfQ \circ \caH (\bfW) 
    + \bfR + \bfR^\prime,
    \label{eqn:diff-K-Kdot}
\end{equation}
where the entries of $\bfR = [R_{ij}]_{i,j=1}^n$ and $\bfR^\prime = [R_{ij}^\prime]_{i,j=1}^n$ are defined as follows: for $i \in \caC_k$ and $j \in \caC_\ell$,
\begin{subequations}
\begin{align}
    R_{ij} & := (1-\delta_{ij}) 
    \sum\nolimits_{t \in \caS_{k \ell} (0)} 
    \alpha_t f^{\prime \prime} ( {\eta_{ij}} / {\hbar_t} ) 
    \cdot (\xi_{ij} / \hbar_t)^2 / 2,
    \label{tmp:Rij} \\
    R_{ij}^\prime & := (1-\delta_{ij}) 
    \sum\nolimits_{t \in \caS_{k \ell} (1)}
    \alpha_t f \pars[\big]{(\theta_{k \ell} + \xi_{ij}) / {\hbar_t}}.
    \label{tmp:Rij-prime}
\end{align}    
\end{subequations}
Now, it follows from \eqref{eqn:diff-K-Kdot} that
\begin{equation*}
    \norm{\bfK - \ddotK} 
    \leq \norm{\bfR}
    + \norm{\bfR^\prime}
    + \norm{\bfQ \circ \caH (\bfW)}
    + \norm{\diag (\ddotK)}.
\end{equation*}

\shortpara{Control of $\mathbf{R}$}: Let $i \in \caC_k$ and $j \in \caC_\ell$ with $i \neq j$. By \eqref{eqn:bound-xi-zeta-Ups}, we have $\xi_{ij} = \oprec{\psi_{k\ell}}$. Together with the bounds $\norm{f''}_{\infty} \lesssim 1$ and $\min_{t \in \caS_{k \ell}(0)} \hbar_t \gtrsim \theta_{k \ell}$, this yields the entrywise estimate $\abs{R_{ij}} \prec \rho^2$. Consequently, 
\begin{equation*}
    \norm{\mathbf{R}} 
    \leq \norm{\mathbf{R}}_{\Fnorm}
    = \pars[\Big]{\sum\nolimits_{i, j = 1}^n \abs{R_{ij}}^2}^{1/2}
    \prec n \rho^2.
\end{equation*}

\shortpara{Control of $\mathbf{R}^\prime$}: Let $i \in \caC_k$ and $j \in \caC_\ell$ with $i \neq j$. Fix an arbitrarily large constant $C>0$. Since $\xi_{ij} = \oprec{\psi_{k\ell}}$, the definition of stochastic domination, together with the bound $\rho \lesssim n^{-c_{\ref{assump:signal}}}$ from Assumption \ref{assump:signal}, implies that there exists an event $\Xi$ with $\bbp(\Xi) \geq 1-n^{-C}$ such that, on $\Xi$,
\begin{equation*}
    \abs{\xi_{i j}}
    \leq n^{c_{\ref{assump:signal}}/2} \psi_{k \ell} 
    \leq \theta_{k \ell} / 2.
\end{equation*}
In particular, on the event $\Xi$, we have $\theta_{k\ell}+\xi_{ij} \geq \theta_{k\ell}/2$. Moreover, by the definition of $\caS_{k \ell} (1)$ in \eqref{def:caS}, we have $\max_{t \in \caS_{k \ell} (1)} \hbar_t \lesssim \theta_{k \ell} / n^{c_{\ref{assump:multi-scale}}}$. Consequently, the following holds on $\Xi$ for every $t \in \caS_{k \ell} (1)$,
\begin{equation*}
    (\theta_{k \ell} + \xi_{ij}) / {\hbar_t} 
    \geq \theta_{k \ell} / (2 \hbar_t) 
    \geq n^{c_{\ref{assump:multi-scale}} / 2}.
\end{equation*}
Using the lower bound $\rho \geq n^{-C_{\ref{assump:tech-const}}}$ from Assumption \ref{assump:tech-const}, it follows that, on the event $\Xi$ we have
\begin{equation*}
    \abs{R_{ij}^\prime} 
    \lesssim f (n^{c_{\ref{assump:multi-scale}} / 2})
    = \exp (- n^{c_{\ref{assump:multi-scale}} / 2}) \leq \rho^2.
\end{equation*}
Since $C>0$ was arbitrary, we conclude that $R_{ij}^\prime = \oprec{\rho^2}$ uniformly in $i,j \in \dbraks{n}$. In fact, the rate $\rho^2$ here can be replaced by $n^{-L}$ for any fixed large constant $L>0$; we use $\rho^2$ only to match the entrywise bound for $R_{ij}$ above. The same Frobenius norm argument as before then gives
\begin{equation*}
    \norm{\mathbf{R}^\prime} \prec n \rho^2.
\end{equation*}

\shortpara{Control of $\diag (\ddotK)$}: Taking the diagonal on both sides of \eqref{eqn:diff-dotK-barK}, and using $\diag(\bfUps)=0$, we obtain
\begin{equation*}
    \diag (\ddotK)
    = \diag (\barK) + 2 \bfQ \circ \diag (\bfzeta).
\end{equation*}
The first term can be controlled by $\norm{\diag (\barK)} = \bigO{1}$ as $\norm{f}_{\infty} \lesssim 1$. For the second term, recall the definition of $Q_{ij}$ in \eqref{def:bfQ}. Since $\norm{f'}_{\infty} \lesssim 1$, we have
\begin{equation}
    \abs{ Q_{ij} }
    \lesssim \min_{t \in \caS_{k \ell} (0)} (1 / \hbar_t)
    \asymp 1 / \theta_{k \ell},
    \qfor
    i \in \caC_k, ~ j\in\caC_\ell.
    \label{bound:upper-Q}
\end{equation}
Combining this with the bound on $\zeta_i$ from \eqref{eqn:bound-xi-zeta-Ups}, we get $\norm{\bfQ \circ \diag (\bfzeta)} = \max_{i \in \dbraks{n}} \abs{Q_{ii} \zeta_i}  \prec \rho$. Consequently,
\begin{equation*}
    \norm{\diag (\ddotK)} \prec 1 + \rho \lesssim 1.
\end{equation*}

\shortpara{Control of $\bfQ \circ \caH (\bfW)$}: This is established in the following lemma, whose proof is deferred to Section \ref{subsec:norm-QW}.

\begin{lemma}
\label{lemma:norm-QW}
Under the same setup as in Proposition \ref{prop:norm-K-barK-dotK}, we have
\begin{equation}
    \norm{ \bfQ \circ \caH(\bfW) }  
    \prec \sqrt{n} \phi.
    \label{bound:norm-Q-circ-W}
\end{equation}
\end{lemma}

\shortpara{Conclusion}: Combining the preceding estimates proves \eqref{bound:bfK-ddotK}.
\end{proof}

\subsection{Proof of \eqref{bound:ddotU-barU}}
\label{subsec:dominant-evec-entrywise}

This section is devoted to the proof of the entrywise eigenvector perturbation bound \eqref{bound:ddotU-barU}. The argument relies on the eigenvalue perturbation bound in Proposition \ref{prop:dotlamb-barlamb}, together with two auxiliary results, Lemmas \ref{lemma:u-dotK-barK-u} and \ref{lemma:inner-prod-dotU-barU}. The proofs of these three results are deferred to Section \ref{sec:tech-lemma}. We emphasize that they are established independently of \eqref{bound:ddotU-barU}; hence, the argument
below is non-circular. Before turning to the proof, we introduce several conventions to simplify the notation. First, we define the following cross-Gram matrices and orthogonal alignment matrices for comparing the eigenspaces associated with $\bfU_r$, $\ddotU_r$, and $\barU_r$:
\begin{equation}
    \fkQ_r \equiv \fkQ_r^\bfu = \bfU_r^\top \barU_r,
    \qquad
    \ddotfkQ_r \equiv \ddotfkQ_r^\bfu = \ddotU_r^\top \barU_r,
    \qquad
    \fkS_r \equiv \fkS_r^\bfu = \sgn(\fkQ_r),
    \qquad
    \ddotfkS_r \equiv \ddotfkS_r^\bfu = \sgn(\ddotfkQ_r).
    \label{def:fkH-fkS}
\end{equation}
We suppress the superscript $\bfu$ whenever no ambiguity can arise. We also abbreviate the eigengap as $\barDelta_{\lambda} \equiv \Delta_r(\barLamb)$. The following consequences of Assumption
\ref{assump:signal} will be used repeatedly throughout the proof,
\begin{equation*}
    \barDelta_{\lambda} \geq n^{c_{\ref{assump:signal}}},
    \qquad
    n \rho / \barDelta_{\lambda}
    \lesssim n^{-c_{\ref{assump:signal}}},
    \qquad
    n^3 \rho^2 / \barDelta_{\lambda}^3 
    \leq n^{-2 c_{\ref{assump:signal}}}.
\end{equation*}

\begin{proof}[Proof of \eqref{bound:ddotU-barU}]
We begin with the decomposition
\begin{equation}
    \ddotU_r \ddotfkS_r - \barU_r
    = \mathrm{I_{\ref{eqn:decomp-Udot-Ubar}}}
    + \mathrm{II_{\ref{eqn:decomp-Udot-Ubar}}}
    + \cdots
    + \mathrm{V_{\ref{eqn:decomp-Udot-Ubar}}},
    \label{eqn:decomp-Udot-Ubar}
\end{equation}
where the terms on the r.h.s. are defined as follows:
\begin{alignat*}{2}
    \mathrm{I_{\ref{eqn:decomp-Udot-Ubar}}}
    & = \ddotU_r (\ddotfkS_r - \ddotfkQ_r),
    \qquad \qquad &
    \mathrm{II_{\ref{eqn:decomp-Udot-Ubar}}}
    & = \ddotU_r (\ddotfkQ_r \ddotLamb_r - \ddotLamb_r \ddotfkQ_r) \ddotLamb_r^{-1}, \\
    \mathrm{III_{\ref{eqn:decomp-Udot-Ubar}}}
    & = \ddotK \ddotU_r \ddotfkQ_r (\ddotLamb_r^{-1} - \barLamb_r^{-1}),
    \qquad \qquad &
    \mathrm{IV_{\ref{eqn:decomp-Udot-Ubar}}}
    & = \ddotK (\ddotU_r \ddotfkQ_r - \barU_r) \barLamb_r^{-1}, \\
    \mathrm{V_{\ref{eqn:decomp-Udot-Ubar}}}
    & = (\ddotK - \barK) \barU_r \barLamb_r^{-1}. &&
\end{alignat*}
 Here, $\mathrm{I_{\ref{eqn:decomp-Udot-Ubar}}}$ is the error from replacing the sign matrix $\ddotfkS_r$ by the alignment matrix $\ddotfkQ_r$, while $\mathrm{II_{\ref{eqn:decomp-Udot-Ubar}}}$ is the error arising from interchanging $\ddotfkQ_r$ and $\ddotLamb_r$. The remaining three terms account for the successive replacement of the dotted matrices in $\ddotK(\ddotU_r \ddotfkQ_r) \ddotLamb_r^{-1}$ by their barred counterparts. We estimate these five terms separately.

\shortpara{Control of $\mathrm{I_{\ref{eqn:decomp-Udot-Ubar}}}$}: The control of $\mathrm{I_{\ref{eqn:decomp-Udot-Ubar}}}$ follows the argument in \cite[Lemma 2]{abbeEntrywiseEigenvectorAnalysis2020}. We reproduce the details in the present setting for completeness. Let $\barLamb_\perp = \diag (\barlamb_{r+1}, \cdots, \barlamb_{n})$, where, by convention, $\barlamb_i = 0$ for $i > \rank(\barK)$. Recall that $\barlamb_{r} - \barlamb_{r+1} \geq \barDelta_{\lambda}$. By Proposition \ref{prop:dotlamb-barlamb} and the bound $n \rho / \barDelta_{\lambda} \lesssim n^{-c_{\ref{assump:signal}}}$ from Assumption \ref{assump:signal}, we have
\begin{equation*}
    \norm{\ddotLamb_r - \barLamb_r}
    \prec \sqrt{n} \rho 
    + {n^2 \rho^2} / \barDelta_{\lambda}
    = \braks{
    n^{-1/2} (n \rho / \barDelta_{\lambda}) 
    + (n \rho / \barDelta_{\lambda})^2} 
    \cdot \barDelta_{\lambda}
    \lesssim n^{-c_{\ref{assump:signal}}} \barDelta_{\lambda}.
\end{equation*}
Also recall the bound \eqref{bound:ddotK-barK}. Fix $\varepsilon \in (0, c_{\ref{assump:signal}}/2]$ and $C > 0$. Then, there exists an high-probability event $\Xi_1$ with $\bbp(\Xi_1) \geq 1 - n^{-C}$ such that, on $\Xi_1$,
\begin{equation*}
    \dist (\ddotLamb_r, \barLamb_\perp) 
    \geq \barDelta_{\lambda} - \norm{\ddotLamb_r - \barLamb_r} \geq \barDelta_{\lambda} / 2
    \qand
    \norm{\ddotK - \barK} \leq n^{1+\varepsilon} \rho.
\end{equation*}
Applying the Davis--Kahan $\sin \Theta$ theorem \cite{davisRotationEigenvectorsPerturbation1970} to $\ddotK = \barK + (\ddotK - \barK)$, we obtain, on the event $\Xi_1$,
\begin{equation*}
    \norm{\ddotU_r \ddotU_r^\top - \barU_r \barU_r^\top}
    \leq \frac{\norm{\ddotK - \barK}}{\dist (\ddotLamb_r, \barLamb_\perp)}
    \leq 2 n^{\varepsilon} \cdot (n \rho / \barDelta_{\lambda})
    \leq 2 n^{-c_{\ref{assump:signal}} / 2}.
\end{equation*}
Since $\varepsilon \in (0, c_{\ref{assump:signal}}/2]$ and $C > 0$ are arbitrary, we obtain
\begin{equation}
    \norm{\ddotU_r \ddotU_r^\top - \barU_r \barU_r^\top} \prec n \rho / \barDelta_{\lambda}.
    \label{bound:sin-dotU-barU}
\end{equation}
Moreover, the preceding analysis implies that there exists a high-probability event on which the alignment matrix $\ddotfkQ_r$ and the sign matrix $\ddotfkS_r$ are both invertible, and
\begin{equation}
    \norm{\ddotfkQ_r} \leq 1,
    \quad \quad
    \norm{\ddotfkS_r} \leq 1,
    \quad \quad
    \norm{\ddotfkQ_r^{-1}} \leq 2,
    \quad \quad
    \norm{\ddotfkS_r^{-1}} \leq 1.
    \label{bound:elementary-S-H}
\end{equation}
In the sequel, we work on the aforementioned high-probability event so that the bounds in \eqref{bound:elementary-S-H} hold. The contribution of the complementary event can be handled straightforwardly using the definition of $\prec$.

Now, by \cite[Chapter I, Corollary 5.4 and Theorem 5.5]{stewart1990matrix}, we obtain
\begin{equation*}
    \norm{\ddotfkS_r - \ddotfkQ_r}
    \leq \norm{\ddotU_r \ddotU_r^\top - \barU_r \barU_r^\top}^2
    \prec (n \rho / \barDelta_{\lambda})^2,
\end{equation*}
Consequently,
\begin{equation}
    \norm{\mathrm{I_{\ref{eqn:decomp-Udot-Ubar}}}}_{2,\infty}
    = \norm{\ddotU_r (\ddotfkS_r - \ddotfkQ_r)}_{2,\infty}
    \leq \norm{\ddotU_r \ddotfkS_r}_{2,\infty}
    \cdot \norm{\ddotfkS_r^{-1}}
    \cdot \norm{\ddotfkS_r - \ddotfkQ_r}
    \prec (n \rho / \barDelta_{\lambda})^2
    \cdot \norm{\ddotU_r \ddotfkS_r}_{2,\infty}.
    \label{tmp:ctrl-term-1}
\end{equation}

\shortpara{Control of $\mathrm{II_{\ref{eqn:decomp-Udot-Ubar}}}$}: Using $\ddotU_r^\top \ddotK = \ddotLamb_r \ddotU_r^\top$ and $\barK \barU_r = \barU_r \barLamb_r$, we obtain
\begin{equation*}
    \ddotU_r^\top(\ddotK - \barK)\barU_r
    = \ddotLamb_r \ddotU_r^\top \barU_r
    - \ddotU_r^\top \barU_r \barLamb_r
    = \ddotLamb_r \ddotfkQ_r
    - \ddotfkQ_r \barLamb_r.
\end{equation*}
Hence, the commutator of $\ddotfkQ_r$ and $\ddotLamb_r$ admits the decomposition
\begin{equation}
    \ddotfkQ_r \ddotLamb_r - \ddotLamb_r \ddotfkQ_r
    = \ddotfkQ_r (\ddotLamb_r - \barLamb_r)
    - (\ddotU_r - \barU_r \ddotfkQ_r^{-1})^\top (\ddotK - \barK) \barU_r
    - (\barU_r \ddotfkQ_r^{-1})^\top (\ddotK - \barK) \barU_r.
    \label{tmp:interchange-Lamb-H}
\end{equation}
For the first term on the r.h.s. of \eqref{tmp:interchange-Lamb-H}, we combine \eqref{bound:elementary-S-H} and Proposition \ref{prop:dotlamb-barlamb} to obtain
\begin{equation*}
    \norm{\ddotfkQ_r (\ddotLamb_r - \barLamb_r)} 
    \prec \sqrt{n} \rho
    + n^2 \rho^2 / \barDelta_{\lambda}.
\end{equation*}
For the second term, we first note that \eqref{bound:sin-dotU-barU} implies
\begin{equation}
    \norm{\ddotU_r \ddotfkQ_r - \barU_r}
    = \norm{\ddotU_r \ddotU_r^\top \barU_r - \barU_r \barU_r^\top \barU_r}
    \leq \norm{\ddotU_r \ddotU_r^\top - \barU_r \barU_r^\top}
    \cdot \norm{\barU_r}
    \prec {n \rho} / {\barDelta_{\lambda}}.
    \label{bound:l2-dotU-barU}
\end{equation}
Together with Proposition \ref{prop:norm-K-barK-dotK}, this gives
\begin{equation*}
    \norm{(\ddotU_r - \barU_r \ddotfkQ_r^{-1})^\top 
    (\ddotK - \barK) \barU_r}
    \leq \norm{\ddotU_r \ddotfkQ_r - \barU_r}
    \cdot \norm{\ddotfkQ_r^{-1}}
    \cdot \norm{\ddotK - \barK}
    \cdot \norm{\barU_r}
    \prec n^2 \rho^2 / {\barDelta_{\lambda}}.
\end{equation*}
Finally, to control the last term on the r.h.s. of \eqref{tmp:interchange-Lamb-H}, we use the following technical estimate.

\begin{lemma} \label{lemma:u-dotK-barK-u}
Under the same setup as in Proposition \ref{prop:norm-K-barK-dotK},
\begin{equation}
    \norm{\barU_r^\top (\ddotK - \barK) \barU_r}
    \leq \norm{\barU^\top (\ddotK - \barK) \barU}
    \prec \sqrt{n}\rho.
    \label{eqn:u-dotK-barK-u}
\end{equation}
\end{lemma}

The proof of Lemma \ref{lemma:u-dotK-barK-u} is deferred to Section \ref{subsec:proof-using-Hanson-Wright}. Combining this estimate with \eqref{bound:elementary-S-H}, we obtain
\begin{equation*}
    \norm{(\barU_r \ddotfkQ_r^{-1})^\top (\ddotK - \barK) \barU_r}
    \leq \norm{\ddotfkQ_r^{-1}} 
    \cdot \norm{\barU_r^\top (\ddotK - \barK) \barU_r}
    \prec \sqrt{n} \rho.
\end{equation*}
Substituting these estimates into \eqref{tmp:interchange-Lamb-H}, we get $\norm{\ddotfkQ_r \ddotLamb_r - \ddotLamb_r \ddotfkQ_r} \prec \sqrt{n} \rho + n^2 \rho^2 / \barDelta_{\lambda}$, and consequently,
\begin{align}
\begin{split}
    \norm{\mathrm{II_{\ref{eqn:decomp-Udot-Ubar}}}}_{2,\infty}
    & = \norm{\ddotU_r (\ddotfkQ_r \ddotLamb_r - \ddotLamb_r \ddotfkQ_r) 
    \ddotLamb_r^{-1}}_{2,\infty} \\
    & \leq \norm{\ddotU_r \ddotfkS_r}_{2,\infty}
    \cdot \norm{\ddotfkS_r^{-1}}
    \cdot \norm{\ddotfkQ_r \ddotLamb_r - \ddotLamb_r \ddotfkQ_r}
    \cdot \norm{\ddotLamb_r^{-1}}
    \prec \pars{{\sqrt{n} \rho} / {\barDelta_{\lambda}}
    + {n^2 \rho^2} / {\barDelta_{\lambda}^2} }
    \cdot \norm{\ddotU_r\ddotfkS_r}_{2,\infty},
\end{split} \label{tmp:ctrl-term-2}
\end{align}
where we also used the bound $\norm{\ddotLamb_r^{-1}} \prec 1/\barDelta_{\lambda}$, which will be derived shortly in \eqref{bound:inv-dot-Lamb}.

\shortpara{Control of $\mathrm{III_{\ref{eqn:decomp-Udot-Ubar}}}$}: By Proposition \ref{prop:dotlamb-barlamb} and the assumption $\barlamb_r \geq \barDelta_{\lambda}$, we have
\begin{equation*}
    \norm{\ddotLamb_r^{-1} - \barLamb_r^{-1}}
    = \norm{\barLamb_r^{-1}
    (\ddotLamb_r - \barLamb_r) 
    \ddotLamb_r^{-1}}
    \leq \norm{\barLamb_r^{-1}}
    \cdot {\norm{\ddotLamb_r - \barLamb_r}}
    \cdot \norm{\ddotLamb_r^{-1}}
    \prec \pars{\sqrt{n} \rho / \barDelta_{\lambda}
    + n^2 \rho^2 / \barDelta_{\lambda}^2}
    \cdot \norm{\ddotLamb_r^{-1}}.
\end{equation*}
Since $n \rho / \barDelta_{\lambda} \lesssim n^{-c_{\ref{assump:signal}}}$, the prefactor on the r.h.s. is bounded by $n^{-c_{\ref{assump:signal}}}$. Consequently,
\begin{equation*}
    \norm{\ddotLamb_r^{-1}} 
    \leq \norm{\barLamb_r^{-1}} + \norm{\ddotLamb_r^{-1} - \barLamb_r^{-1}}
    \prec \norm{\barLamb_r^{-1}} + n^{-c_{\ref{assump:signal}}} \norm{\ddotLamb_r^{-1}}.
\end{equation*}
By Lemma \ref{lemma:basic-property-prec} \ref{item:cancel-prec}, it follows that
\begin{equation}
    \norm{\ddotLamb_r^{-1}} \prec 1/\barDelta_{\lambda}
    \quad \text{ and hence } \quad
    \norm{\ddotLamb_r^{-1} - \barLamb_r^{-1}} 
    \prec \sqrt{n} \rho / \barDelta_{\lambda}^2
    + n^2 \rho^2 / \barDelta_{\lambda}^3.
    \label{bound:inv-dot-Lamb}
\end{equation}
Next, combining $\abs{\bar{K}_{ij}} \lesssim 1$ with $\abs{\ddot{K}_{ij} - \bar{K}_{ij}} \prec \rho$ from \eqref{bound-entry-Kdot-Kbar}, we find $\abs{\ddot{K}_{ij}} \prec 1$. As a result,
\begin{equation*}
    \norm{\ddotK}_{2, \infty}
    = \max_{i \in \dbraks{n}} ~ \pars[\Big]{
    \sum\nolimits_{j=1}^n \abs{\ddot{K}_{ij}}^2}^{1/2}
    \prec \sqrt{n}.
\end{equation*}
Using this bound, together with \eqref{bound:inv-dot-Lamb}, we arrive at the control
\begin{equation}
    \norm{\mathrm{III_{\ref{eqn:decomp-Udot-Ubar}}}}_{2,\infty}
    = \norm{\ddotK \ddotU_r \ddotfkQ_r 
    (\ddotLamb_r^{-1} - \barLamb_r^{-1})}_{2, \infty}
    \leq \norm{\ddotK}_{2, \infty} 
    \cdot \norm{\ddotLamb_r^{-1} - \barLamb_r^{-1}}
    \prec n \rho / \barDelta_{\lambda}^2
    + n^{5/2} \rho^2 / \barDelta_{\lambda}^3.
    \label{tmp:ctrl-term-3}
\end{equation}

\shortpara{Control of $\mathrm{IV_{\ref{eqn:decomp-Udot-Ubar}}}$}: We further decompose
\begin{equation*}
    \mathrm{IV_{\ref{eqn:decomp-Udot-Ubar}}}
    = \ddotK (\ddotU_r \ddotfkQ_r - \barU_r) \barLamb_r^{-1}
    = \barK (\ddotU_r \ddotfkQ_r - \barU_r) \barLamb_r^{-1}
    + (\ddotK - \barK) (\ddotU_r \ddotfkQ_r - \barU_r) \barLamb_r^{-1}.
\end{equation*}

To control the first term on the r.h.s., we use the following auxiliary estimate.

\begin{lemma} \label{lemma:inner-prod-dotU-barU}
Under the same setup as in Proposition \ref{prop:norm-K-barK-dotK},
\begin{equation}
    \norm{\barK (\ddotU_r \ddotfkQ_r - \barU_r)}_{2, \infty}
    \prec {n \rho} / {\barDelta_{\lambda}}  
    + {n^{5/2} \rho^2} / {\barDelta_{\lambda}^2}.
    \label{eqn:inner-prod-dotU-barU}
\end{equation}
\end{lemma}

The proof of Lemma \ref{lemma:inner-prod-dotU-barU} is deferred to Section \ref{subsec:proof-refined-sqrtn}. Combining this estimate with $\barlamb_r \geq \barDelta_{\lambda}$, we obtain
\begin{equation*}
    \norm{\barK (\ddotU_r \ddotfkQ_r - \barU_r) \barLamb_r^{-1}}_{2, \infty}
    \leq \norm{\barK (\ddotU_r \ddotfkQ_r - \barU_r)}_{2, \infty}
    \cdot \norm{\barLamb_r^{-1}}
    \prec {n \rho} / {\barDelta_{\lambda}^2}  
    + {n^{5/2} \rho^2} / {\barDelta_{\lambda}^3}.
\end{equation*}
For the second term, the entrywise control \eqref{bound-entry-Kdot-Kbar} implies
\begin{equation}
    \norm{\ddotK - \barK}_{2, \infty} 
    = \max_{1 \leq i \leq n} \pars[\Big]{
    \sum\nolimits_{j=1}^n \abs{\ddot{K}_{ij} - \bar{K}_{ij}}^2}^{1/2}
    \prec \sqrt{n} \rho.
    \label{bound:max-row-dotK-barK}
\end{equation}
Combining this with the $\ell_2$ control \eqref{bound:l2-dotU-barU}, we obtain
\begin{equation*}
    \norm{(\ddotK - \barK) (\ddotU_r \ddotfkQ_r - \barU_r) \barLamb_r^{-1}}_{2, \infty}
    \leq \norm{\ddotK - \barK}_{2, \infty}
    \cdot \norm{\ddotU_r\ddotfkQ - \barU_r}
    \cdot \norm{\ddotLamb_r^{-1}}
    \prec {n^{3/2}} \rho^2 / {\barDelta_{\lambda}^2} .
\end{equation*}
Summarizing the two estimates above yields
\begin{equation}
    \norm{\mathrm{IV_{\ref{eqn:decomp-Udot-Ubar}}}}_{2,\infty}
    \prec 
    {n \rho} / {\barDelta_{\lambda}^2}  
    + {n^{3/2}\rho^2} / {\barDelta_{\lambda}^2}
    + {n^{5/2} \rho^2} / {\barDelta_{\lambda}^3}
    \lesssim 
    {n \rho} / {\barDelta_{\lambda}^2}  
    + {n^{5/2} \rho^2} / {\barDelta_{\lambda}^3},
    \label{tmp:ctrl-term-4}
\end{equation}
where in the last step we used the fact that $\barDelta_{\lambda} \leq \norm{\barK}_{\Fnorm} \lesssim n$.

\shortpara{Control of $\mathrm{V_{\ref{eqn:decomp-Udot-Ubar}}}$}: This term can be controlled directly using \eqref{bound:inv-dot-Lamb} and \eqref{bound:max-row-dotK-barK},
\begin{equation}
    \norm{\mathrm{V_{\ref{eqn:decomp-Udot-Ubar}}}}_{2,\infty}
    = \norm{(\ddotK - \barK) \barU_r \barLamb_r^{-1}}_{2, \infty}
    \leq \norm{\ddotK - \barK}_{2, \infty}
    \cdot \norm{\barLamb_r^{-1}}
    \prec {\sqrt{n} \rho} / {\barDelta_{\lambda}}.
    \label{tmp:ctrl-term-5}
\end{equation}

\shortpara{Conclusion}: Summarizing the bounds \eqref{tmp:ctrl-term-1}, \eqref{tmp:ctrl-term-2}, \eqref{tmp:ctrl-term-3}, \eqref{tmp:ctrl-term-4}, and \eqref{tmp:ctrl-term-5}, we obtain
\begin{equation*}
    \norm{\ddotU_r \ddotfkS_r- \barU_r}_{2,\infty}
    \prec \pars{{\sqrt{n} \rho} / {\barDelta_{\lambda}}
    + {n^2 \rho^2} / {\barDelta_{\lambda}^2} }
    \cdot \norm{\ddotU_r\ddotfkS_r}_{2,\infty}
    + \pars{{\sqrt{n} \rho} / {\barDelta_{\lambda}}
    + n \rho / \barDelta_{\lambda}^2
    + n^{5/2} \rho^2 / \barDelta_{\lambda}^3}.
\end{equation*}
Using $n \rho / \barDelta_{\lambda} \lesssim n^{-c_{\ref{assump:signal}}}$ again, we see that the prefactor of $\norm{\ddotU_r\ddotfkS_r}_{2,\infty}$ is bounded by $n^{-c_{\ref{assump:signal}}}$. For brevity, define
\begin{equation*}
    \Psi_{\mathrm{\ref{bound:ddotU-barU}}} = {{\sqrt{n} \rho} / {\barDelta_{\lambda}} + n \rho / \barDelta_{\lambda}^2 + n^{5/2} \rho^2 / \barDelta_{\lambda}^3}.
\end{equation*}
Then, by the triangle inequality, the preceding bound implies
\begin{equation*}
    \norm{\ddotU_r\ddotfkS_r}_{2,\infty}
    \leq \norm{\barU_r}_{2,\infty}
    + \norm{\ddotU_r\ddotfkS_r- \barU_r}_{2,\infty} 
    \prec n^{-c_{\ref{assump:signal}}} \norm{\ddotU_r\ddotfkS_r}_{2,\infty}
    + 1/\sqrt{n} 
    + \Psi_{\mathrm{\ref{bound:ddotU-barU}}}.
\end{equation*}
Here we used the representation of the eigenvectors $\baru_\ell$ in \eqref{eqn:baru-piecewise}, together with $n_k \asymp n$ from Assumption \ref{assump:tech-const}, to obtain $\norm{\barU_r}_{2,\infty} \asymp 1 / \sqrt{n}$. Now, it follows from Lemma \ref{lemma:basic-property-prec} \ref{item:cancel-prec} that
\begin{equation}
    \norm{\ddotU_r\ddotfkS_r}_{2,\infty}
    \prec 1/\sqrt{n} + \Psi_{\mathrm{\ref{bound:ddotU-barU}}}
    \lesssim 1/\sqrt{n},
    \label{bound:dotU-sqrt-n}
\end{equation}
where we also used $n^3 \rho^2 / \barDelta_{\lambda}^3 \leq n^{-2 c_{\ref{assump:signal}}}$ from Assumption \ref{assump:signal} to bound the deterministic parameter $\Psi_{\mathrm{\ref{bound:ddotU-barU}}}$ by $1/\sqrt{n}$. Substituting this estimate back into the previous bound yields
\begin{equation*}
    \norm{\ddotU_r \ddotfkS_r- \barU_r}_{2,\infty}
    \prec \pars{{\sqrt{n} \rho} / {\barDelta_{\lambda}}
    + {n^2 \rho^2} / {\barDelta_{\lambda}^2} } 
    (1/\sqrt{n} + \Psi_{\mathrm{\ref{bound:ddotU-barU}}}) 
    + \Psi_{\mathrm{\ref{bound:ddotU-barU}}}
    \lesssim \Psi_{\mathrm{\ref{bound:ddotU-barU}}},
\end{equation*}
where the last step again uses $n \rho / \barDelta_{\lambda} \leq n^{-c_{\ref{assump:signal}}}$ and $\barDelta_{\lambda} \lesssim n$. This proves the $\ell_{\infty}$ control \eqref{bound:ddotU-barU}.
\end{proof}

\subsection{Proof of \eqref{bound:bfU-ddotU}}
\label{subsec:subleading-evec-entrywise}

\begin{proof}[Proof of \eqref{bound:bfU-ddotU}]
We begin with the decomposition
\begin{equation}
    \bfU_r \fkS_r - \ddotU_r \ddotfkS_r
    = \mathrm{I_{\ref{eqn:decomp-U-Udot}}}
    + \mathrm{II_{\ref{eqn:decomp-U-Udot}}}
    + \cdots
    + \mathrm{VII_{\ref{eqn:decomp-U-Udot}}},
    \label{eqn:decomp-U-Udot}
\end{equation}
where the terms on the r.h.s. are defined by
\begin{alignat*}{2}
    \mathrm{I_{\ref{eqn:decomp-U-Udot}}}
    & = \bfU_r (\fkS_r - \fkQ_r),
    \qquad \qquad &
    \mathrm{II_{\ref{eqn:decomp-U-Udot}}}
    & = - \ddotU_r (\ddotfkS_r - \ddotfkQ_r), \\
    \mathrm{III_{\ref{eqn:decomp-U-Udot}}}
    & = \bfU_r (\fkQ_r \bfLamb_r 
    - \bfLamb_r \fkQ_r) \bfLamb_r^{-1},
    \qquad \qquad &
    \mathrm{IV_{\ref{eqn:decomp-U-Udot}}}
    & = - \ddotU_r ( \ddotfkQ_r \ddotLamb_r - \ddotLamb_r \ddotfkQ_r) \ddotLamb_r^{-1}, \\
    \mathrm{V_{\ref{eqn:decomp-U-Udot}}}
    & = \bfK \bfU_r \fkQ_r (\bfLamb_r^{-1} - \ddotLamb_r^{-1}),
    \qquad \qquad &
    \mathrm{VI_{\ref{eqn:decomp-U-Udot}}}
    & = \bfK (\bfU_r \fkQ_r - \ddotU_r \ddotfkQ_r) \ddotLamb_r^{-1}, \\
    \mathrm{VII_{\ref{eqn:decomp-U-Udot}}}
    & = (\bfK - \ddotK) \ddotU_r \ddotfkQ_r \ddotLamb_r^{-1}. &&
\end{alignat*}
This decomposition basically follows the same strategy as that used in \eqref{eqn:decomp-Udot-Ubar}. The first two terms, $\mathrm{I_{\ref{eqn:decomp-U-Udot}}}$ and $\mathrm{II_{\ref{eqn:decomp-U-Udot}}}$, capture the error incurred by replacing the sign matrices with the corresponding alignment matrices. The next two terms, $\mathrm{III_{\ref{eqn:decomp-U-Udot}}}$ and $\mathrm{IV_{\ref{eqn:decomp-U-Udot}}}$, account for the non-commutativity between the alignment matrices and the eigenvalue matrices. The remaining three terms arise from replacing the factors in $\bfK \bfU_r \fkQ_r \bfLamb_r^{-1}$ by their dotted counterparts. As in the proof of \eqref{bound:ddotU-barU}, we control these terms separately.

\shortpara{Control of $\mathrm{I_{\ref{eqn:decomp-U-Udot}}}$ and $\mathrm{II_{\ref{eqn:decomp-U-Udot}}}$}: We first note that $\mathrm{II_{\ref{eqn:decomp-U-Udot}}} = -\mathrm{I_{\ref{eqn:decomp-Udot-Ubar}}}$ has already been handled in \eqref{tmp:ctrl-term-1}. In addition, utilizing \eqref{bound:dotU-sqrt-n}, we may further bound $\norm{\ddotU_r\ddotfkS_r}_{2,\infty}$ appearing on the r.h.s. of \eqref{tmp:ctrl-term-1} by $1 / \sqrt{n}$. The treatment of $\mathrm{I_{\ref{eqn:decomp-U-Udot}}}$ is analogous. The only difference is that we now combine the Davis--Kahan theorem with the bound $\norm{\bfK - \barK} \prec 1 + n \rho$ from Proposition \ref{prop:norm-K-barK-dotK}. This yields
\begin{equation}
    \norm{\bfU_r \bfU_r^\top - \barU_r \barU_r^\top}
    \lesssim \frac{\norm{\bfK - \barK}}{\dist (\bfLamb_r, \barLamb_\perp)}
    \prec {(1 + n \rho)} / {\barDelta_{\lambda}},
    \label{bound-S-H}
\end{equation}
and consequently,
\begin{equation}
    \norm{\fkS_r - \fkQ_r} \leq \norm{\bfU_r \bfU_r^\top - \barU_r \barU_r^\top}^2 
    \prec {(1 + n^2 \rho^2)} / {\barDelta_{\lambda}^2}.
    \label{tmp:fkS-fkH-diff}
\end{equation}
In what follows, we work on a high-probability event on which \eqref{bound:elementary-S-H} holds and, in addition,
\begin{equation*}
    \norm{\fkQ_r} \leq 1,
    \quad \quad
    \norm{\fkS_r} \leq 1,
    \quad \quad
    \norm{(\fkQ_r)^{-1}} \leq 2,
    \quad \quad
    \norm{\fkS_r^{-1}} \leq 1.
\end{equation*}
To summarize, for the first two terms on the r.h.s. of \eqref{eqn:decomp-U-Udot}, we have
\begin{subequations} \label{tmp:sublead-term-1-2}
\begin{align}
    \norm{\mathrm{I_{\ref{eqn:decomp-U-Udot}}}}_{2,\infty}
    & \prec {(1 + n^2 \rho^2)} / {\barDelta_{\lambda}^2} 
    \cdot \norm{\bfU_r \fkS_r}_{2,\infty}, \\
    \norm{\mathrm{II_{\ref{eqn:decomp-U-Udot}}}}_{2,\infty}
    & \prec {n^{3/2} \rho^2} / {\barDelta_{\lambda}^2}.
\end{align}    
\end{subequations}


\shortpara{Control of $\mathrm{III_{\ref{eqn:decomp-U-Udot}}}$ and $\mathrm{IV_{\ref{eqn:decomp-U-Udot}}}$}: We first note that $\mathrm{IV_{\ref{eqn:decomp-U-Udot}}} = - \mathrm{II_{\ref{eqn:decomp-Udot-Ubar}}}$ has already been analyzed in \eqref{tmp:ctrl-term-2}. Combining this with \eqref{bound:dotU-sqrt-n} gives the desired bound for $\mathrm{IV_{\ref{eqn:decomp-U-Udot}}}$ in \eqref{tmp:sublead-term-4} below. It remains to control $\mathrm{III_{\ref{eqn:decomp-U-Udot}}}$. The argument
is similar to the one used for $\mathrm{II_{\ref{eqn:decomp-Udot-Ubar}}}$, but requires some additional care. We first observe that
\begin{equation}
    \bfU_r^\top (\bfK - \barK) \barU_r
    = \bfLamb_r \bfU_r^\top \barU_r
    - \bfU_r^\top \barU_r \barLamb_r
    = \bfLamb_r \fkQ_r
    - \fkQ_r \barLamb_r.
    \label{eqn:decomp-LambdaH-HbarLambda}
\end{equation}
Therefore, the commutator between $\fkQ_r$ and $\bfLamb_r$ can be decomposed as
\begin{align} \label{tmp:interchange-Lamb-H-2}
\begin{split}
    \fkQ_r \bfLamb_r - \bfLamb_r \fkQ_r
    = & ~ \fkQ_r (\bfLamb_r - \barLamb_r)
    - (\bfU_r - \barU_r \fkQ_r^{-1})^\top (\bfK - \barK) \barU_r \\
    & - (\barU_r \fkQ_r^{-1})^\top (\bfK - \ddotK) \barU_r 
    - (\barU_r \fkQ_r^{-1})^\top (\ddotK - \barK) \barU_r.
\end{split}
\end{align}
Except for the third term on the r.h.s. of \eqref{tmp:interchange-Lamb-H-2}, the remaining terms can be treated in the same way as the terms in \eqref{tmp:interchange-Lamb-H}, using the following counterparts of the perturbation bounds:
\begin{subequations}
\begin{align}
    \norm{\bfLamb_r - \barLamb_r} 
    & \leq \norm{\ddotLamb_r - \barLamb_r} 
    + \norm{\bfK - \ddotK}
    \prec 1 + \sqrt{n} \phi
    + n^2 \rho^2 / \barDelta_{\lambda} ,
    \label{tmp:bflamb-to-popu} \\
    \norm{\bfU_r \fkQ_r - \barU_r} 
    & \leq \norm{\bfU_r \bfU_r^\top - \barU_r \barU_r^\top}
    \prec (1 + n \rho) / \barDelta_{\lambda} .
    \label{tmp:bfu-to-popu}
\end{align}    
\end{subequations}
Here \eqref{tmp:bflamb-to-popu} follows from Propositions \ref{prop:norm-K-barK-dotK} and \ref{prop:dotlamb-barlamb}, together with the fact that $\sqrt{n} \rho \lesssim 1 + n \rho^2$. On the other hand, \eqref{tmp:bfu-to-popu} follows from \eqref{bound-S-H}. While for the third term on the r.h.s. of \eqref{tmp:interchange-Lamb-H-2}, we note that its operator norm is bounded directly by $\norm{\bfK - \ddotK}$. Hence,
\begin{align*}
    \norm{\fkQ_r \bfLamb_r - \bfLamb_r \fkQ_r} 
    & \lesssim \norm{\bfLamb_r - \barLamb_r} 
    + \norm{\bfU_r \fkQ_r - \barU_r}
    \cdot \norm{\bfK - \barK} 
    + \norm{\bfK - \ddotK} 
    + \norm{\barU_r^\top (\ddotK - \barK) \barU_r} \\
    & \prec 1 + \sqrt{n} \phi
    + n^2 \rho^2 / \barDelta_{\lambda} .
\end{align*}
Combining this estimate with $\norm{\bfLamb_r^{-1}} \prec 1 / \barDelta_{\lambda}$, which will be verified shortly below, yields the desired $\ell_\infty$ control for $\mathrm{III_{\ref{eqn:decomp-U-Udot}}}$. To summarize, we have obtained
\begin{subequations} \label{tmp:sublead-term-3-4}
\begin{align} 
    \norm{\mathrm{III_{\ref{eqn:decomp-U-Udot}}}}_{2,\infty}
    & \prec \pars[\big]{1 / \barDelta_{\lambda}
    + \sqrt{n} \phi / \barDelta_{\lambda}
    + n^2 \rho^2 / \barDelta_{\lambda}^2} 
    \cdot \norm{\bfU_r \fkS_r}_{2, \infty} , \\
    \norm{\mathrm{IV_{\ref{eqn:decomp-U-Udot}}}}_{2,\infty}
    & \prec {\rho} / {\barDelta_{\lambda}}
    + {n^{3/2} \rho^2} / {\barDelta_{\lambda}^2} .
    \label{tmp:sublead-term-4}
\end{align}
\end{subequations}

\shortpara{Control of $\mathrm{V_{\ref{eqn:decomp-U-Udot}}}$}: 
The analysis is similar to that leading to \eqref{tmp:ctrl-term-3}. First, we have $\norm{\bfLamb_r^{-1}} \prec 1 / \barDelta_{\lambda}$ by combining
\begin{equation*}
    \norm{\bfLamb_r^{-1} - \barLamb_r^{-1}}
    \leq \norm{\barLamb_r^{-1}}
    \cdot {\norm{\bfK - \barK}}
    \cdot \norm{\bfLamb_r^{-1}}
    \prec (1 + n \rho) / \barDelta_{\lambda}
    \cdot \norm{\bfLamb_r^{-1}}
    \lesssim n^{-c_{\ref{assump:signal}}} \norm{\bfLamb_r^{-1}},
\end{equation*}
with the triangle inequality. Consequently, by \eqref{bound:bfK-ddotK},
\begin{equation*}
    \norm{\bfLamb_r^{-1} - \ddotLamb_r^{-1}} 
    \prec \norm{\bfLamb_r - \ddotLamb_r} / \barDelta_{\lambda}^2
    \leq \norm{\bfK - \ddotK} / \barDelta_{\lambda}^2
    \prec 
    1 / \barDelta_{\lambda}^2
    + \sqrt{n} \phi / \barDelta_{\lambda}^2
    + n \rho^2 / \barDelta_{\lambda}^2.
\end{equation*}
Moreover, since $\norm{f}_{\infty} \lesssim 1$, we have the trivial bound $\norm{\bfK}_{2,\infty} \lesssim \sqrt{n}$. Therefore,
\begin{equation}
    \norm{\mathrm{V_{\ref{eqn:decomp-U-Udot}}}}_{2,\infty}
    = \norm{\bfK \bfU_r \fkQ_r (\bfLamb_r^{-1} - \ddotLamb_r^{-1})}_{2, \infty}
    \prec \sqrt{n} / \barDelta_{\lambda}^2
    + n \phi / \barDelta_{\lambda}^2
    + n^{3/2} \rho^2 / \barDelta_{\lambda}^2.
    \label{tmp:sublead-term-5}
\end{equation}

\shortpara{Control of $\mathrm{VI_{\ref{eqn:decomp-U-Udot}}}$}: Let $\ddotLamb_\perp = \diag (\ddotlamb_{r+1}, \cdots, \ddotlamb_{n})$. Applying the Davis--Kahan theorem again, we obtain
\begin{equation*}
    \norm{\bfU_r \fkQ_r - \ddotU_r \ddotfkQ_r}
    \leq \norm{\bfU_r\bfU_r^\top - \ddotU_r \ddotU_r^\top}
    \leq \frac{\norm{\bfK - \ddotK}}{\dist (\bfLamb_r, \ddotLamb_\perp)}
    \prec 
    1 / \barDelta_{\lambda}
    + \sqrt{n} \phi / \barDelta_{\lambda}
    + n \rho^2 / \barDelta_{\lambda},
\end{equation*}
where we used $1 / \dist (\bfLamb_r, \ddotLamb_\perp) \prec 1 / \barDelta$. This can be deduced from
\begin{equation*}
    \dist (\bfLamb_r, \ddotLamb_\perp) 
    \geq \barDelta_{\lambda} - \norm{\bfLamb_r - \barLamb_r}
    - \norm{\ddotLamb_r - \barLamb_r},
\end{equation*}
together with \eqref{tmp:bflamb-to-popu} and \eqref{bound:ddotLamb-barLamb}. Consequently, using $\norm{\bfK}_{2, \infty} \lesssim \sqrt{n}$ and $\norm{\ddotLamb_r^{-1}} \prec 1 / \barDelta$ again, we find
\begin{equation}
    \norm{\mathrm{VI_{\ref{eqn:decomp-U-Udot}}}}_{2,\infty}
    = \norm{\bfK (\bfU_r \fkQ_r - \ddotU_r \ddotfkQ_r) \ddotLamb_r^{-1}}_{2, \infty}
    \prec \sqrt{n} / \barDelta_{\lambda}^2
    + n \phi / \barDelta_{\lambda}^2
    + n^{3/2} \rho^2 / \barDelta_{\lambda}^2.
    \label{tmp:sublead-term-6}
\end{equation}

\shortpara{Control of $\mathrm{VII_{\ref{eqn:decomp-U-Udot}}}$}: Recall the decomposition \eqref{eqn:diff-K-Kdot}. We have
\begin{equation}
    \mathrm{VII_{\ref{eqn:decomp-U-Udot}}}
    = (\bfK - \ddotK) \ddotU_r \ddotfkQ_r \ddotLamb_r^{-1}
    = - \diag (\ddotK) \ddotU_r \ddotfkQ_r \ddotLamb_r^{-1}
    + \braks{\bfQ \circ \caH (\bfW)} \ddotU_r \ddotfkQ_r \ddotLamb_r^{-1}
    + (\bfR + \bfR') \ddotU_r \ddotfkQ_r \ddotLamb_r^{-1}.
     \label{tmp:term-VII-split}
\end{equation}
The first and third terms on the r.h.s. of \eqref{tmp:term-VII-split} are relatively straightforward to control. Indeed, by the definition of the $\ell_{2,\infty}$ norm, for any diagonal matrix $\bfD$ and conformable matrix $\bfA$, it holds that $\norm{\bfD \bfA}_{2,\infty} \leq \norm{\bfD} \cdot \norm{\bfA}_{2,\infty}$. Thus, for the first term, we have
\begin{align*}
    \norm{\diag (\ddotK) \ddotU_r \ddotfkQ_r \ddotLamb_r^{-1}}_{2,\infty}
    & \leq \norm{\diag (\ddotK)}
    \cdot \norm{\ddotU_r \ddotfkQ_r \ddotLamb_r^{-1}}_{2, \infty} \\
    & \leq \norm{\diag (\ddotK)}
    \cdot \norm{\ddotU_r \ddotfkS_r}_{2, \infty}
    \cdot \norm{\ddotfkS_r^{-1} \ddotfkQ_r \ddotLamb_r^{-1}}
    \prec 1 / (\sqrt{n} \barDelta_{\lambda}) .   
\end{align*}
Here, in the last step, we used $\norm{\diag(\ddotK)} \prec 1$ from Section \ref{subsec:norm-bound-residual}, together with \eqref{bound:elementary-S-H}, \eqref{bound:inv-dot-Lamb}, and \eqref{bound:dotU-sqrt-n}. For the third term on the r.h.s. of \eqref{tmp:term-VII-split}, the entrywise bound $\abs{R_{ij}}+\abs{R_{ij}'} \prec \rho^2$ from Section \ref{subsec:norm-bound-residual} implies
\begin{equation*}
    \norm{(\bfR + \bfR') \ddotU_r \ddotfkQ_r \ddotLamb_r^{-1}}_{2,\infty}
    \leq \norm{\bfR + \bfR'}_{2,\infty}
    \cdot \norm{\ddotU_r \ddotfkQ_r \ddotLamb_r^{-1}}
    \prec \sqrt{n} \rho^2 / \barDelta_{\lambda}.
\end{equation*}
It remains to control the second term on the r.h.s. of \eqref{tmp:term-VII-split}. We further split it as
\begin{equation}
    \braks{\bfQ \circ \caH (\bfW)} \ddotU_r \ddotfkQ_r \ddotLamb_r^{-1}
    = \braks{\bfQ \circ \caH (\bfW)} \barU_r \ddotLamb_r^{-1}
    + \braks{\bfQ \circ \caH (\bfW)} (\ddotU_r \ddotfkQ_r - \barU_r) \ddotLamb_r^{-1}.
    \label{tmp:QW-split}
\end{equation}
For the first term on the r.h.s., we use the following auxiliary estimate.

\begin{lemma}
\label{lemma:QW-barU-infnorm}
Under the same setup as in Proposition \ref{prop:norm-K-barK-dotK},
\begin{equation}
    \norm{\braks{\bfQ \circ \caH (\bfW)} \barU_r}_{2, \infty}
    \prec \rho.
\end{equation}
\end{lemma}

The proof of Lemma \ref{lemma:QW-barU-infnorm} is deferred to Section \ref{subsec:proof-using-Hanson-Wright}. Combining this estimate with \eqref{bound:inv-dot-Lamb}, we obtain
\begin{equation*}
    \norm{\braks{\bfQ \circ \caH (\bfW)} \barU_r \ddotLamb_r^{-1}}_{2, \infty}
    \leq \norm{\braks{\bfQ \circ \caH (\bfW)} \barU_r}_{2, \infty} \cdot \norm{\ddotLamb_r^{-1}}
    \prec \rho / \barDelta_{\lambda}.
\end{equation*}
The second term of \eqref{tmp:QW-split} is controlled more directly. Recall from \eqref{eqn:bound-xi-zeta-Ups} and \eqref{bound:upper-Q} that we have the entrywise control $\braks{\bfQ \circ \caH(\bfW)}_{ij} = \oprec{\rho}$. Together with the $\ell_2$ eigenvector perturbation bound \eqref{bound:l2-dotU-barU}, this yields
\begin{equation*}
    \norm{\braks{\bfQ \circ \caH (\bfW)} 
    (\ddotU_r \ddotfkQ_r - \barU_r)
    \ddotLamb_r^{-1}}_{2,\infty}
    \leq \norm{\bfQ \circ \caH (\bfW)}_{2,\infty}
    \cdot \norm{\ddotU_r\ddotfkQ_r - \barU_r}
    \cdot \norm{\ddotLamb_r^{-1}}
    \prec {n^{3/2} \rho^2} / {\barDelta_{\lambda}^2} .
\end{equation*}
Substituting the two preceding estimates into \eqref{tmp:QW-split}, we arrive at
\begin{equation*}
    \norm{\braks{\bfQ \circ \caH (\bfW)} \ddotU_r \ddotfkQ_r \ddotLamb_r^{-1}}_{2,\infty}
    \prec \rho / \barDelta_{\lambda} + {n^{3/2} \rho^2} / {\barDelta_{\lambda}^2} .
\end{equation*}
Now, summarizing the bounds for the three terms in \eqref{tmp:term-VII-split}, we conclude that
\begin{equation}
    \norm{\mathrm{VII_{\ref{eqn:decomp-U-Udot}}}}_{2, \infty}
    \prec 1/ (\sqrt{n} \barDelta_{\lambda})
    + {n^{3/2} \rho^2} / {\barDelta_{\lambda}^2} ,
    \label{tmp:sublead-term-7}
\end{equation}
where we used $\barDelta_{\lambda} \lesssim n$ and $\rho \lesssim \sqrt{n} \rho^2 + 1 / \sqrt{n}$ when merging the error bounds.

\shortpara{Conclusion}: Summarizing the estimates \eqref{tmp:sublead-term-1-2}, \eqref{tmp:sublead-term-3-4}, \eqref{tmp:sublead-term-5}, \eqref{tmp:sublead-term-6} and \eqref{tmp:sublead-term-7}, we obtain
\begin{equation*}
    \norm{\bfU_r \fkS_r - \ddotU_r \ddotfkS_r}_{2,\infty}
    \prec 
    \pars[\big]{1 / \barDelta_{\lambda}
    + \sqrt{n} \phi / \barDelta_{\lambda}
    + n^2 \rho^2 / \barDelta_{\lambda}^2} 
    \cdot \norm{\bfU_r \fkS_r}_{2, \infty}
    + \Psi_{\mathrm{\ref{bound:bfU-ddotU}}} , 
\end{equation*}
where, by Assumption \ref{assump:signal}, the prefactor before $\norm{\bfU_r \fkS_r}_{2,\infty}$ is bounded above by $n^{-c_{\ref{assump:signal}}}$. For brevity, we have introduced the deterministic control parameter
\begin{equation*}
    \Psi_{\mathrm{\ref{bound:bfU-ddotU}}} 
    = \pars{\sqrt{n} + n \phi + n^{3/2} \rho^2} / {\barDelta_{\lambda}^2}.
\end{equation*}
Now, a straightforward application of the triangle inequality and Lemma \ref{lemma:basic-property-prec} \ref{item:cancel-prec} yields
\begin{equation}
    \norm{\bfU_r \fkS_r}_{2,\infty}
    \prec \norm{\ddotU_r \ddotfkS_r}_{2, \infty}
    + \Psi_{\mathrm{\ref{bound:bfU-ddotU}}}
    \prec 1/\sqrt{n}
    + \Psi_{\mathrm{\ref{bound:bfU-ddotU}}}
    \lesssim 1/\sqrt{n}
    + \pars{\sqrt{n} + n \phi} / {\barDelta_{\lambda}^2},
    \label{bound:bfU-sqrt-n}
\end{equation}
where the second bound follows from \eqref{bound:dotU-sqrt-n} and Assumption \ref{assump:signal}. Substituting \eqref{bound:bfU-sqrt-n} into the preceding bound for $\norm{\bfU_r \fkS_r - \ddotU_r \ddotfkS_r}_{2,\infty}$ leads to
\begin{equation*}
    \norm{\bfU_r \fkS_r - \ddotU_r \ddotfkS_r}_{2,\infty}
    \prec \pars[\big]{1 / \barDelta_{\lambda}
    + \sqrt{n} \phi / \barDelta_{\lambda}
    + n^2 \rho^2 / \barDelta_{\lambda}^2} \cdot 
    (1 / \sqrt{n} + \Psi_{\mathrm{\ref{bound:bfU-ddotU}}})
    + \Psi_{\mathrm{\ref{bound:bfU-ddotU}}}
    \lesssim \Psi_{\mathrm{\ref{bound:bfU-ddotU}}},
\end{equation*}
which proves the desired bound \eqref{bound:bfU-ddotU}.
\end{proof}

\subsection{Proof of Theorem \ref{thm:Linfty-evec-bfK-barK}}
\label{subsec:score-bound}

\begin{proof}[Proof of Theorem \ref{thm:Linfty-evec-bfK-barK}]
The bound \eqref{bound:Kev-without-eigs} follows directly by combining the two estimates \eqref{bound:ddotU-barU} and \eqref{bound:bfU-ddotU} in Proposition \ref{prop:evec-two-stage}. Note that the term
$n \rho / \barDelta_{\lambda}^2$ can be absorbed since
\begin{equation*}
    n \rho / \barDelta_{\lambda}^2
    \lesssim (\sqrt{n} + n^{3/2} \rho^2) / \barDelta_{\lambda}^2
    \lesssim \sqrt{n} / \barDelta_{\lambda}^2
    + n^{5/2} \rho^2 / {\barDelta_{\lambda}^3}.
\end{equation*}

It remains to prove \eqref{bound:Kev-with-eigs}. We begin with the decomposition
\begin{equation}
    \bfU_r \bfLamb_r^{1/2} \fkS_r - \barU_r \barLamb_r^{1/2}
    =: \mathrm{I_{\ref{eqn:decomp-ULambda-barULambda}}} 
    + \mathrm{II_{\ref{eqn:decomp-ULambda-barULambda}}} 
    + \mathrm{III_{\ref{eqn:decomp-ULambda-barULambda}}},
    \label{eqn:decomp-ULambda-barULambda}
\end{equation}
where we introduced
\begin{align*}
    \mathrm{I_{\ref{eqn:decomp-ULambda-barULambda}}} 
    & := \bfU_r \fkS_r (\fkS_r^\top \bfLamb_r^{1/2} \fkS_r - \barLamb_r^{1/2}), \\
    \mathrm{II_{\ref{eqn:decomp-ULambda-barULambda}}} 
    & := (\bfU_r \fkS_r- \ddotK\barU_r\barLamb_r^{-1}) \barLamb_r^{1/2}, \\
    \mathrm{III_{\ref{eqn:decomp-ULambda-barULambda}}} 
    & := (\ddotK - \barK) \barU_r \barLamb_r^{-1/2}.
\end{align*}

\shortpara{Control of $\mathrm{I_{\ref{eqn:decomp-ULambda-barULambda}}}$}: 
We first bound $\norm{\fkS_r^\top \bfLamb_r^{1/2} \fkS_r - \barLamb_r^{1/2}}$. Note that $\fkS_r^\top \bfLamb_r^{1/2}\fkS_r = (\fkS_r^\top \bfLamb_r\fkS)^{1/2}$ since $\fkS_r$ is orthogonal. Therefore, by the perturbation bound for matrix square roots \cite[Lemma 2.1]{SCHMITT1992215},
\begin{equation}
    \norm{\fkS_r^\top \bfLamb_r^{1/2}\fkS_r - \barLamb_r^{1/2}}
    \leq
    \frac{\norm{\fkS_r^\top \bfLamb_r\fkS_r - \barLamb_r}}
    {\lambda_{\min} (\fkS_r^\top \bfLamb_r^{1/2}\fkS_r) 
    + \lambda_{\min} (\barLamb_r^{1/2})}
    \prec 
    {\norm{\bfLamb_r\fkS_r - \fkS_r\barLamb_r}}
    \big / {\barDelta_{\lambda}^{1/2}}.
    \label{tmp:SLambdaS-barLambda}
\end{equation}
Here, in the last step we used $\norm{\barLamb_r^{-1}} \lesssim 1 / \barDelta_{\lambda}$ and $\norm{\bfLamb_r^{-1}} \prec 1 / \barDelta_{\lambda}$. We proceed to control the norm on the r.h.s. of \eqref{tmp:SLambdaS-barLambda}. Specifically, we decompose
\begin{align*}
    \norm{\bfLamb_r\fkS_r - \fkS_r\barLamb_r} 
    & \leq \norm{\bfLamb_r (\fkS_r - \fkQ_r)} 
    + \norm{\bfLamb_r \fkQ_r - \fkQ_r \barLamb_r}
    + \norm{(\fkQ_r - \fkS_r)\barLamb_r} \\
    & = \norm{\bfLamb_r \fkQ_r - \fkQ_r \barLamb_r} 
    + \oprec[\big]{(n + n^3 \rho^2) / \barDelta_{\lambda}^2}  ,  
\end{align*}
where the second line follows from \eqref{tmp:fkS-fkH-diff}. Recall from \eqref{eqn:decomp-LambdaH-HbarLambda} that $\bfLamb_r \fkQ_r - \fkQ_r \barLamb_r = \bfU_r^\top (\bfK - \barK) \barU_r$. In particular, this matrix corresponds to the last three terms on the r.h.s. of \eqref{tmp:interchange-Lamb-H-2}, up to a sign. Therefore, by applying the estimates derived after \eqref{tmp:interchange-Lamb-H-2} to these terms, we obtain
\begin{equation*}
    \norm{\bfU_r^\top (\bfK - \barK) \barU_r}
    \prec 
    1 + \sqrt{n} \phi
    + n^2 \rho^2 / \barDelta_{\lambda}.
\end{equation*}
Combining the preceding estimates yields
\begin{equation*}
    \norm{\bfLamb_r\fkS_r - \fkS_r\barLamb_r}
    \prec 
    1 + \sqrt{n} \phi 
    + (n + n^3 \rho^2) / \barDelta_{\lambda}^2.
\end{equation*}
Recalling \eqref{bound:bfU-sqrt-n}, we therefore obtain
\begin{align}
\begin{split}
    \norm{\mathrm{I_{\ref{eqn:decomp-ULambda-barULambda}}}}_{2, \infty}
    &\leq 
    \norm{\bfU_r \fkS_r}_{2, \infty}
    \cdot \norm{\fkS_r^\top \bfLamb_r^{1/2}\fkS_r - \barLamb_r^{1/2}} \\
    & \prec
    \pars{1/\sqrt{n}
    + \sqrt{n}/\barDelta_{\lambda}^2 
    + n \phi / \barDelta_{\lambda}^2 }
    \pars{1/ \barDelta_{\lambda}^{1/2}
    + \sqrt{n} \phi / \barDelta_{\lambda}^{1/2}
    + n / \barDelta_{\lambda}^{5/2}
    + n^3 \rho^2 / \barDelta_{\lambda}^{5/2} } \\
    & \lesssim
    (1/ \sqrt{n} + \sqrt{n} \rho) / \barDelta_{\lambda}^{1/2}
    + (\sqrt{n} + n^{5/2} \rho^2) / \barDelta_{\lambda}^{5/2}
    + (n^{3/2} + n^{9/2} \rho^3) / \barDelta_{\lambda}^{9/2} \\
    & \lesssim
    (1/ \sqrt{n} + \sqrt{n} \rho) / \barDelta_{\lambda}^{1/2}
    + n^3 \rho^2 /\barDelta_{\lambda}^3
    + n^{3/2} / \barDelta_{\lambda}^{9/2},
\end{split} \label{tmp:ctrl-term-I-Lamb}
\end{align}
where, in passing to the third line, we used $\phi \lesssim \sqrt{n}\rho$ together with the elementary estimates
\begin{equation*}
    \sqrt{n} (1 + n \rho + n^{2} \rho^2 )
    \lesssim \sqrt{n} + n^{5/2} \rho^2
    \qand
    n^{3/2} (1 + n \rho + n^{2} \rho^2 + n^{3} \rho^3)
    \lesssim n^{3/2} (1 + n^{3} \rho^3).
\end{equation*}
Finally, in the last step, we used $n^3 \rho^2 /\barDelta_{\lambda}^3 \ll 1$ from Assumption \ref{assump:signal}.

\shortpara{Control of $\mathrm{II_{\ref{eqn:decomp-ULambda-barULambda}}}$}: The control of this term relies on the following two estimates, which follow as byproducts of the proof of Proposition \ref{prop:evec-two-stage}:
\begin{subequations}
\begin{align}
    \norm{\ddotU_r \ddotfkS_r- \ddotK\barU_r\barLamb_r^{-1}}_{2,\infty}
    & \prec n \rho / \barDelta_{\lambda}^2
    + n^{5/2} \rho^2 / \barDelta_{\lambda}^3,
    \label{tmp:dotUS} \\
    \norm{\bfU_r \fkS_r- \ddotK\barU_r\barLamb_r^{-1}}_{2,\infty}
    & \prec ({\sqrt{n} + n \phi}) / {\barDelta_{\lambda}^2}
    + n^{5/2} \rho^2 / \barDelta_{\lambda}^3.
    \label{tmp:bfUS}
\end{align}    
\end{subequations}
In fact, we have
\begin{equation*}
    \ddotU_r \ddotfkS_r- \ddotK\barU_r\barLamb_r^{-1}
    = \mathrm{I_{\ref{eqn:decomp-Udot-Ubar}}}
    + \cdots 
    + \mathrm{IV_{\ref{eqn:decomp-Udot-Ubar}}} ,
\end{equation*}
so \eqref{tmp:dotUS} follows by combining the estimates \eqref{tmp:ctrl-term-1}, \eqref{tmp:ctrl-term-2}, \eqref{tmp:ctrl-term-3}, \eqref{tmp:ctrl-term-4} with \eqref{bound:dotU-sqrt-n}. The second estimate \eqref{tmp:bfUS} then follows from \eqref{tmp:dotUS} together with \eqref{bound:bfU-ddotU}. Recalling $\barlamb_1 \lesssim n$, we therefore get
\begin{equation}
    \norm{\mathrm{II_{\ref{eqn:decomp-ULambda-barULambda}}}}_{2, \infty}
    \leq \norm{\bfU_r \fkS_r- \ddotK\barU_r\barLamb_r^{-1}}_{2,\infty}
    \cdot \norm{\barLamb_r^{1/2}}
    \prec ({n + n^{3/2} \phi}) / {\barDelta_{\lambda}^2}
    + n^{3} \rho^2 / \barDelta_{\lambda}^3.
    \label{tmp:ctrl-term-II-Lamb}
\end{equation}

\shortpara{Control of $\mathrm{III_{\ref{eqn:decomp-ULambda-barULambda}}}$}: Note that $\mathrm{V_{\ref{eqn:decomp-Udot-Ubar}}} = \mathrm{III_{\ref{eqn:decomp-ULambda-barULambda}}} \barLamb_r^{1/2}$ was already analyzed in \eqref{tmp:ctrl-term-5}. Hence, by the same argument,
\begin{equation}
    \norm{(\ddotK - \barK) \barU_r \barLamb_r^{-1/2}}_{2, \infty}
    \leq \norm{\ddotK - \barK}_{2, \infty}
    \cdot \norm{\barLamb_r^{-1/2}}
    \prec {\sqrt{n} \rho} / {\barDelta_{\lambda}^{1/2}}.
    \label{tmp:ctrl-term-III-Lamb}
\end{equation}

\shortpara{Conclusion}: Summarizing the controls \eqref{tmp:ctrl-term-I-Lamb}, \eqref{tmp:ctrl-term-II-Lamb} and \eqref{tmp:ctrl-term-III-Lamb}, we arrive at
\begin{equation*}
    \norm{\bfU_r \bfLamb_r^{1/2}\fkS_r - \barU_r\barLamb_r^{1/2}}_{2,\infty}
    \prec
    {\sqrt{n} \rho} / {\barDelta_{\lambda}^{1/2}}
    + ({n + n^{3/2} \phi}) / {\barDelta_{\lambda}^2}
    + n^3 \rho^2 / \barDelta_{\lambda}^3
    + n^{3/2} / \barDelta_{\lambda}^{9/2}.
\end{equation*}
Here we absorbed the rate $1 / (\sqrt{n} \barDelta_{\lambda}^{1/2})$ into $n / {\barDelta_{\lambda}^2}$, using $\barDelta_{\lambda} \lesssim n$. This proves \eqref{bound:Kev-with-eigs}.
\end{proof}

\section{Technical lemmas}
\label{sec:tech-lemma}

\subsection{Proof of Lemma \ref{lemma:norm-QW}}
\label{subsec:norm-QW}

This section is devoted to the proof of Lemma \ref{lemma:norm-QW}, which gives a high-probability bound on $\norm{\bfQ \circ \caH(\bfW)}$. The main difficulty is that the quadratic-form estimate \eqref{eqn:hanson-wright} is assumed only in the sense of stochastic domination, and is therefore weaker than the standard Hanson--Wright inequality. To make this distinction explicit, recall that if the entries $\{ \bfz_i(\alpha) \}$ are independent sub-Gaussian random variables, then
\begin{equation*}
    \bbp \curls[\big]{
    \abs{\angles{\bfz_i, \bfB \bfz_i} - \Tr \bfB} \geq t}
    \leq
    2 \exp \pars[\bigg]{-c \min \curls[\bigg]{
    \frac{t^2}{K^4 \norm{\bfB}_{\Fnorm}^2},
    \frac{t}{K^2 \norm{\bfB}} }},
    \quad \forall t \geq 0,
\end{equation*}
where $c > 0$ is an absolute constant and $K := \max_{\alpha} \norm{\bfz_i(\alpha)}_{\psi_2}$; see, for example, \cite[Theorem 6.2.1]{vershyninHighdimensionalProbabilityIntroduction2018}. If such an exponential-tail estimate were available, \eqref{bound:norm-Q-circ-W} could be obtained by a standard $\varepsilon$-net argument. In the present setting, however, \eqref{eqn:hanson-wright} provides only polynomial tail control, albeit of arbitrarily high order. This is insufficient for taking a union bound over an exponentially large net. To handle this issue, we follow the idea of \cite{chafaiConvergenceExtremalEigenvalues2018,srivastavaCovarianceEstimationDistributions2013}, which combines the resolvent method with a rank-one update argument to control the extreme eigenvalues of sample covariance matrices of the form $(1/n)\sum_{i=1}^n \bfz_i\bfz_i^\top$, where the random vectors $\bfz_i$ are i.i.d. and isotropic. Importantly, the arguments in these works require only weak moment assumptions on $\bfz_i$ and do not rely on exponential-tail estimates for quadratic forms.

Our proof of Lemma \ref{lemma:norm-QW} proceeds in two steps. First, we use a simple decoupling argument to reduce the problem of bounding $\bfQ \circ \caH(\bfW)$ to that of bounding a random matrix of the form $\sum_{i=1}^n \bfy_i \bfy_i^\top$ where the vectors $\bfy_i$ may have different covariance profiles. Second, we adapt the argument of \cite{chafaiConvergenceExtremalEigenvalues2018,srivastavaCovarianceEstimationDistributions2013} to control the spectral norm of this random matrix. These two steps are summarized in Lemmas \ref{lemma:decoupling} and \ref{lemma:norm-cov-type}, respectively.

\begin{lemma}[decoupling]
\label{lemma:decoupling}
Let $\bfw_{i} = \bfA_{\bar{\pi} (i)} \bfz_i$, and let $\{ \tilde{\bfw}_i \}_{i=1}^n$ be an independent copy of $\{ \bfw_i \}_{i=1}^n$. Define $\bfH, \tilde{\bfH}$ by
\begin{equation}
    H_{i j} = [\bfQ \circ \caH(\bfW)]_{ij} = (1 - \delta_{ij}) \cdot Q_{ij}
    \angles{\bfw_i, \bfw_j}
    \qand
    \tilde{H}_{i j} = (1 - \delta_{ij}) \cdot Q_{ij}
    \angles{\bfw_i, \tilde{\bfw}_j}.
    \label{tmp:hat-H}
\end{equation}
Let $\Psi \geq n^{-C}$ be a deterministic control parameter for some constant $C > 0$. Then
\begin{equation*}
    \norm{\tilde{\bfH}} \prec \Psi
    \quad \Longrightarrow \quad
    \norm{\bfH} \prec \Psi.
\end{equation*}
\end{lemma}

\begin{proof}[Proof of Lemma \ref{lemma:decoupling}]
The proof here follows the decoupling argument in \cite[Theorem 6.1.1]{vershyninHighdimensionalProbabilityIntroduction2018}. Let $\{ b_i \}_{i=1}^n$ be independent Bernoulli random variables with $\bbp \{ b_i = 1 \} = \bbp \{ b_i = 0 \} = 1/2$, independent of the random vectors $\{ {\bfw}_i \}_{i=1}^n$ and $\{ \tilde{\bfw}_i \}_{i=1}^n$. Define $\mathbf{F}, \tilde{\mathbf{F}} \in \bbr^{n \times n}$ by
\begin{equation*}
    F_{ij} := (1 - \delta_{ij}) \cdot Q_{ij} b_i (1 - b_j) \angles{\bfw_i, \bfw_j}
    \qand
    \tilde{F}_{ij} := (1 - \delta_{ij}) \cdot Q_{ij} b_i (1 - b_j) \angles{\bfw_i, \tilde{\bfw}_j}.
\end{equation*}
Let $\bbe_b$ denote expectation with respect to $\{b_i\}_{i=1}^n$, and let $\bbe_w$ denote expectation with respect to $\{\bfw_i\}_{i=1}^n$ and $\{\tilde{\bfw}_i\}_{i=1}^n$. Since $\bbe_b [b_i(1-b_j)] = 1 / 4$ for $i \neq j$, we have $\bfH = 4 \bbe_{b} [\bfF]$. Therefore, for any fixed $\ell \in \bbn_+$, Jensen's inequality and the convexity of $\bfA \mapsto \norm{\bfA}^\ell$ yield
\begin{equation}
    \bbe_w \norm{\bfH}^\ell
    = \bbe_w \norm{4 \bbe_b [\mathbf{F}]}^\ell
    \leq 4^\ell \bbe_b \bbe_w \norm{\mathbf{F}}^\ell
    = 4^\ell \bbe_b \bbe_w \norm{\tilde{\mathbf{F}}}^\ell
    \leq 4^\ell \bbe_w \norm{\tilde{\bfH}}^\ell.
    \label{tmp:moment-comparison}
\end{equation}
For the last two steps, we used that, conditional on $\{ b_i \}_{i=1}^n$, the matrices $\mathbf{F}$ and $\tilde{\mathbf{F}}$ have the same distribution. Moreover, for fixed realization of $\{ b_i \}_{i=1}^n$, the matrix $\tilde{\mathbf{F}}$ could be obtained from $\tilde{\bfH}$ by setting certain rows and columns to zero, and hence $\norm{\tilde{\mathbf{F}}} \leq \norm{\tilde{\bfH}}$.

By Lemma \ref{lemma:basic-property-prec} \ref{item:compatibility-expectation}, we know that $\norm{\tilde{\bfH}}\prec \Psi$ implies $\bbe \norm{\tilde{\bfH}}^\ell \prec \Psi^\ell$ for every fixed $\ell \in \bbn_+$. Together with \eqref{tmp:moment-comparison}, this gives $\bbe \norm{\bfH}^\ell \prec \Psi^\ell$ for every fixed $\ell \in \bbn_+$. Applying Lemma \ref{lemma:basic-property-prec} \ref{item:compatibility-expectation} again yields $\norm{\bfH}\prec \Psi$.
\end{proof}

\begin{lemma}
\label{lemma:norm-cov-type}
Let $\bfz_i \in \bbr^{d_i}$, $i \in \dbraks{n}$, be independent random vectors with $\bbe \bfz_i = \mathbf{0}$ and $\cov(\bfz_i)=\bfI$. Assume further that the following quadratic-form estimate holds uniformly in $i \in \dbraks{n}$ and deterministic matrices $\bfB\in\bbr^{d_i\times d_i}$:
\begin{equation}
    \abs{\angles{\bfz_i, \bfB \bfz_i} - \Tr \bfB}
    \prec \norm{\bfB}_{\Fnorm}.
    \label{eqn:quad-concentration}
\end{equation}
Let $\bfA_i \in \bbr^{p \times d_i}$, $i \in \dbraks{n}$, be deterministic matrices, and define
\begin{equation*}
    \bfS_n := \sum\nolimits_{i = 1}^n \bfy_{i} \bfy_{i}^\top,
    \qwhere
    \bfy_{i} := \bfA_i \bfz_i.
\end{equation*}
Then, we have
\begin{equation}
    \norm{\bfS_n} 
    \prec \max_{i \in \dbraks{n}} 
    ~ \norm{\bfA_i}_{\Fnorm}^2
    + \sum\nolimits_{i=1}^n \norm{\bfA_i}^2.
    \label{eqn:bound-cov-type}
\end{equation}    
\end{lemma}

Before giving the proof, we make a few comments on the argument. Compared with \cite{chafaiConvergenceExtremalEigenvalues2018,srivastavaCovarianceEstimationDistributions2013}, our proof is considerably simpler thanks to the following two aspects: the desired bound \eqref{eqn:bound-cov-type} tolerates an $n^\varepsilon$ loss, and the quadratic-form estimate \eqref{eqn:quad-concentration} provides polynomial tail control of arbitrarily high order. The main additional complication is that the random vectors $\bfy_i$ are allowed to be anisotropic and inhomogeneous, whereas the works \cite{chafaiConvergenceExtremalEigenvalues2018,srivastavaCovarianceEstimationDistributions2013} focus on i.i.d. isotropic vectors. We therefore modify the definition of feasible upper shifts introduced in these works to accommodate the anisotropic covariance profiles of $\bfy_i$.

\begin{proof}[Proof of Lemma \ref{lemma:norm-cov-type}]
Let us write
\begin{equation*}
    \bfSigma_i = \cov (\bfy_i) = \bfA_i \bfA_i^\top, 
    \quad i \in \dbraks{n}
    \qand
    \bfS_{\ell} := \sum\nolimits_{i = 1}^{\ell} \bfy_{i} \bfy_{i}^\top,
    \quad \ell = 0, 1, \cdots, n.
\end{equation*}
For a fixed realization of the random vectors $\{\bfy_i\}_{i=1}^n$, we say that $\{u_\ell\}_{\ell=0}^n$ forms a sequence of feasible upper soft edges for $\{\bfS_\ell\}_{\ell=0}^n$ if $u_0>0$ and, for every $\ell\in\dbraks{n}$,
\begin{equation}
    u_{\ell} > \lambda_{\max} (\bfS_{\ell})
    \qand
    \Tr [(u_{\ell} - \bfS_{\ell})^{-1} \bfSigma_i] 
    \leq \Tr [(u_{\ell-1} - \bfS_{\ell - 1})^{-1} \bfSigma_i],
    \quad \forall i \in \dbraks{n}.
    \label{eqn:feasible-upper}
\end{equation}
In \cite{chafaiConvergenceExtremalEigenvalues2018,srivastavaCovarianceEstimationDistributions2013}, the increments $u_\ell - u_{\ell-1}$ are referred to as feasible upper shifts. Fix arbitrary (small) $\varepsilon > 0$ and (large) $C > 0$. We next construct a deterministic candidate sequence which forms a sequence of feasible upper soft edges with probability at least high probability. More specifically, let
\begin{equation}
    u_0 := n^{\varepsilon} 
    \max_{i \in \dbraks{n}} \Tr \bfSigma_i,
    \qand
    u_\ell := u_{\ell-1} + n^{\varepsilon} \norm{\bfSigma_\ell},
    \qfor \ell \in \dbraks{n}.
    \label{def:upper-construct}
\end{equation}
Our goal is to prove that there exists an event $\Xi$, with $\bbp(\Xi)\geq 1-n^{-C}$, such that for every realization in $\Xi$, the sequence $\{u_\ell\}_{\ell=0}^n$ defined in \eqref{def:upper-construct} forms a sequence of feasible upper soft edges for $\{\bfS_\ell\}_{\ell=0}^n$. 

We prove this by induction. Note that $\bfS_0 = \mathbf{0}$. Hence, the definition of $u_0$ yields
\begin{equation}
    \Tr [(u_{0} - \bfS_{0})^{-1} \bfSigma_i] 
    \leq n^{-\varepsilon},
    \quad \forall i \in \dbraks{n}.
    \label{tmp:upper0-initial}
\end{equation}
Suppose that, for some $\ell\in\dbraks{n}$, we have constructed an event $\Xi_{\dbraks{\ell-1}}$ on which the feasibility condition \eqref{eqn:feasible-upper} holds for all indices $1 \leq \ell' \leq \ell-1$. Together with \eqref{tmp:upper0-initial}, this implies that on $\Xi_{\dbraks{\ell-1}}$,
\begin{equation*}
    u_{\ell-1} > \lambda_{\max} (\bfS_{\ell-1})
    \qand
    \Tr [(u_{\ell-1} - \bfS_{\ell-1})^{-1} \bfSigma_i] 
    \leq n^{-\varepsilon},
    \quad \forall i \in \dbraks{n}.
    \label{tmp:induction-hypo}
\end{equation*}
We now show how to extend the feasibility condition from index $\ell-1$ to index $\ell$. On the event $\Xi_{\dbraks{\ell-1}}$, the matrix $u_\ell-\bfS_{\ell-1}$ is positive definite as $u_\ell \geq u_{\ell-1}$. Define
\begin{equation*}
    \fkq_1 (u_{\ell})
    := \angles{\bfy_{\ell}, (u_{\ell} - \bfS_{\ell-1})^{-1} \bfy_{\ell}}
    \qand
    \fkq_{2, i} (u_{\ell})
    := \frac{\angles{\bfy_\ell, (u_{\ell} - \bfS_{\ell-1})^{-1} \bfSigma_i 
    (u_{\ell} - \bfS_{\ell-1})^{-1} \bfy_\ell}}
    {\Tr [(u_{\ell-1} - \bfS_{\ell-1})^{-1} \bfSigma_i] 
    - \Tr [(u_{\ell} - \bfS_{\ell-1})^{-1} \bfSigma_i]}.
\end{equation*}
Following \cite[Lemma 3.3]{srivastavaCovarianceEstimationDistributions2013}, we claim that a sufficient condition for \eqref{eqn:feasible-upper} to hold at the index $\ell$ is
\begin{equation}
    \fkq_1 (u_{\ell}) < 1
    \qand
    \fkq_{2, i} (u_{\ell}) \leq 1 - \fkq_1 (u_{\ell}),
    \quad \forall i \in \dbraks{n}.
    \label{tmp:fkq-sufficient}
\end{equation}
For completeness, we verify this claim. The first condition of \eqref{tmp:fkq-sufficient} gives $\fkq_1 (u_{\ell}) = \norm{(u_{\ell} - \bfS_{\ell-1})^{-1/2} \bfy_{\ell}}^2 < 1$, which implies the positive definiteness of $\bfI - (u_{\ell} - \bfS_{\ell-1})^{-1/2} \bfy_{\ell} \bfy_{\ell}^\top (u_{\ell} - \bfS_{\ell-1})^{-1/2}$. Conjugating by $(u_\ell-\bfS_{\ell-1})^{1/2}$ yields the positive definiteness of $u_{\ell} - \bfS_{\ell} = (u_{\ell} - \bfS_{\ell-1}) - \bfy_{\ell} \bfy_{\ell}^\top$. Hence, $u_{\ell} > \lambda_{\max} (\bfS_{\ell})$. It remains to check the trace inequality in \eqref{eqn:feasible-upper}. By the Sherman-Morrison formula, 
\begin{align*}
    \Tr [(u_{\ell} - \bfS_{\ell})^{-1} \bfSigma_i] 
    & = \Tr [(u_{\ell} - \bfS_{\ell-1} - \bfy_{\ell} \bfy_{\ell}^\top)^{-1} \bfSigma_i] \\
    & = \Tr [(u_{\ell} - \bfS_{\ell-1})^{-1} \bfSigma_i] 
    + \frac{\angles{\bfy_\ell, (u_{\ell} - \bfS_{\ell-1})^{-1} \bfSigma_i 
    (u_{\ell} - \bfS_{\ell-1})^{-1} \bfy_\ell}}
    {1 - \angles{\bfy_{\ell}, (u_{\ell} - \bfS_{\ell-1})^{-1} \bfy_{\ell}}}.
\end{align*}
Now it becomes evident that \eqref{tmp:fkq-sufficient} implies the trace inequality of \eqref{eqn:feasible-upper}. 

Since $\bfz_\ell$ satisfies the quadratic-form concentration estimate
\eqref{eqn:quad-concentration}, it holds that
\begin{equation*}
    \sup\nolimits_{\bfB \in \bbr^{p \times p}} 
    \bbp \curls[\big]{\abs{\angles{\bfy_\ell, \bfB \bfy_\ell} - \Tr (\bfB \bfSigma_\ell)}
    \geq n^{\varepsilon} \norm{\bfA_{\ell}^\top \bfB \bfA_{\ell}}_{\Fnorm} / 2 }
    \leq n^{-(C+3)}.
\end{equation*}
Note that if $\bfB$ is nonnegative definite, then $\norm{\bfA_{\ell}^\top \bfB \bfA_{\ell}}_{\Fnorm} \leq \Tr ( \bfA_{\ell}^\top \bfB \bfA_{\ell} ) = \Tr ( \bfB \bfSigma_{\ell} )$. Conditioning on $\bfS_{\ell-1}$, the resolvent $(u_{\ell} - \bfS_{\ell-1})^{-1}$ is deterministic with respect to the randomness of $\bfy_\ell$. Therefore, there exists an event $\Xi_\ell$ with $\bbp(\Xi_\ell)\geq 1-n^{-(C+1)}$ such that, on $\Xi_\ell$, the following estimates hold simultaneously for all $i \in \dbraks{n}$:
\begin{align*}
    \angles{\bfy_{\ell}, (u_{\ell} - \bfS_{\ell-1})^{-1} \bfy_{\ell}} 
    & \leq 
    n^{\varepsilon} \Tr [(u_{\ell} - \bfS_{\ell - 1})^{-1} \bfSigma_\ell] / 2, \\
    \angles{\bfy_\ell, (u_{\ell} - \bfS_{\ell-1})^{-1} \bfSigma_i 
    (u_{\ell} - \bfS_{\ell-1})^{-1} \bfy_\ell}
    & \leq 
    n^{\varepsilon} \Tr [(u_{\ell} - \bfS_{\ell - 1})^{-1} \bfSigma_i
    (u_{\ell} - \bfS_{\ell - 1})^{-1} \bfSigma_\ell] / 2.
\end{align*}
In particular, on the event $\Xi_{\dbraks{\ell-1}} \cap \Xi_\ell$, using the monotonicity of $u \mapsto \Tr [(u - \bfS_{\ell - 1})^{-1} \bfSigma_\ell]$ for $u > \lambda_{\max}(\bfS_{\ell-1})$, together with the induction hypothesis \eqref{tmp:induction-hypo}, we obtain
\begin{equation*}
    \fkq_1 (u_{\ell}) 
    \leq n^{\varepsilon} \Tr [(u_{\ell} - \bfS_{\ell - 1})^{-1} \bfSigma_\ell] / 2
    \leq n^{\varepsilon} \Tr [(u_{\ell-1} - \bfS_{\ell - 1})^{-1} \bfSigma_\ell] / 2
    \leq 1/2.
\end{equation*}
Likewise, using the resolvent identity and monotonicity in the spectral parameter,
\begin{equation*}
    \Tr [(u_{\ell-1} - \bfS_{\ell-1})^{-1} \bfSigma_i] 
    - \Tr [(u_{\ell} - \bfS_{\ell-1})^{-1} \bfSigma_i]
    \geq (u_\ell - u_{\ell-1}) 
    \Tr [(u_{\ell} - \bfS_{\ell-1})^{-2} \bfSigma_i] .
\end{equation*}
Therefore, using $u_\ell - u_{\ell-1} = n^{\varepsilon} \norm{\bfSigma_\ell}$, we have, on $\Xi_{\dbraks{\ell-1}} \cap \Xi_\ell$,
\begin{equation*}
    \fkq_{2, i} (u_{\ell}) 
    \leq \frac{n^{\varepsilon}}{2 (u_\ell - u_{\ell-1})} 
    \cdot
    \frac{\Tr [(u_{\ell} - \bfS_{\ell - 1})^{-1} \bfSigma_i
    (u_{\ell} - \bfS_{\ell - 1})^{-1} \bfSigma_\ell]}
    {\Tr [(u_{\ell} - \bfS_{\ell - 1})^{-1} \bfSigma_i
    (u_{\ell} - \bfS_{\ell - 1})^{-1}]}
    \leq \frac{n^{\varepsilon} \norm{\bfSigma_\ell}}
    {2 (u_\ell - u_{\ell-1})}
    = 1/2.
\end{equation*}
In the last inequality, we used the elementary inequality $\Tr(\bfA\bfB)\leq \norm{\bfB}\Tr\bfA$ when $\bfA$ and $\bfB$ are nonnegative definite. Hence, on the event $\Xi_{\dbraks{\ell-1}}\cap\Xi_\ell$, the sufficient condition \eqref{tmp:fkq-sufficient} holds, and consequently the feasibility condition \eqref{eqn:feasible-upper} is satisfied at the index $\ell$.

Now let $\Xi_\ell$, $\ell\in\dbraks{n}$, be the events constructed above, and set $\Xi = \bigcap_{\ell=1}^n \Xi_{\ell}$. Since $\bbp (\Xi_\ell) \geq 1 - n^{-(C+1)}$ for each $\ell \in \dbraks{n}$, the union bound gives $\bbp (\Xi) \geq 1 - n^{-C}$. By construction, on the event $\Xi$, the sequence $\{ u_{\ell} \}_{\ell = 0}^n$ defined in \eqref{def:upper-construct} forms a sequence of feasible upper soft edges for $\{\bfS_\ell\}_{\ell=0}^n$. Consequently, it holds that
\begin{equation*}
    \bbp \curls[\Big]{ \norm{\bfS_n} \geq n^\varepsilon \pars[\Big]{
    \max_{i \in \dbraks{n}} \Tr \bfSigma_i
    + \sum\nolimits_{i=1}^n \norm{\bfSigma_i} } }
    = \bbp \{ \norm{\bfS_n} \geq u_n \} \leq 1 - \bbp (\Xi) \leq n^{-C}.
\end{equation*}
By the arbitrariness of $\varepsilon>0$ and $C>0$, this concludes the proof.
\end{proof}

Now, equipped with Lemmas \ref{lemma:decoupling} and \ref{lemma:norm-cov-type}, we are ready to control $\norm{\bfQ \circ \caH(\bfW)}$.

\begin{proof}[Proof of Lemma \ref{lemma:norm-QW}]
Let $\{ \tilde{\bfz}_i \}_{i=1}^n$ be an independent copy of $\{ {\bfz}_i \}_{i=1}^n$. Let $\tilde{\bfw}_i = \bfA_{\bar{\pi} (i)} \tilde{\bfz}_i$ and define $\tilde{\bfH}$ as in \eqref{tmp:hat-H}. By Lemma \ref{lemma:decoupling}, in order to prove \eqref{bound:norm-Q-circ-W}, it suffices to show that $\norm{\tilde{\bfH}}^2 = \norm{\tilde{\bfH} \tilde{\bfH}^\top} \prec n \phi^2$. Write
\begin{equation*}
    \tilde{\bfH} = [\bfh_1, \cdots, \bfh_n],
    \qwhere 
    \bfh_i (j) = \tilde{H}_{ji} = (1 - \delta_{ij}) \cdot Q_{ij} \angles{\bfw_j, \tilde{\bfw}_i}.
\end{equation*}
Then $\tilde{\bfH} \tilde{\bfH}^\top = \sum\nolimits_{i = 1}^n \bfh_i \bfh_i^\top$, where the vectors $\bfh_i$ are independent when conditioning on the randomness of $\{ \bfz_j \}_{j=1}^n$. Moreover, we can write $\bfh_i = \bfB_i^\top \tilde{\bfz}_i$ by introducing
\begin{equation*}
    \bfB_i = [\bfb_{i1}, \cdots, \bfb_{in}],
    \qwhere
    \bfb_{ij} 
    = (1 - \delta_{ij}) Q_{ij} \bfA_{\bar{\pi} (i)}^\top \bfw_j
    = (1 - \delta_{ij}) Q_{ij} \bfOmega_{\bar{\pi} (i) \bar{\pi} (j)} \bfz_j.
\end{equation*}
Applying Lemma \ref{lemma:norm-cov-type} to $\tilde{\bfH}$ conditionally on $\{\bfz_j\}_{j=1}^n$, we obtain
\begin{equation}
    \norm{\tilde{\bfH} \tilde{\bfH}^\top} 
    \prec \max_{i \in \dbraks{n}} 
    ~ \norm{\bfB_i}_{\Fnorm}^2
    + \sum\nolimits_{i=1}^n \norm{\bfB_i}^2.
    \label{tmp:hatH-bound}
\end{equation}
We now estimate the two terms on the r.h.s. of \eqref{tmp:hatH-bound}. First, by the definition of $\bfB_i$,
\begin{equation*}
    \norm{\bfB_i}_{\Fnorm}^2
    = \sum\nolimits_{j=1}^{n} \norm{\bfb_{ij}}^2
    = \sum\nolimits_{j \not= i} \abs{Q_{ij}}^2 
    \angles{\bfz_j, \bfOmega_{\bar{\pi} (j) \bar{\pi} (i)}
    \bfOmega_{\bar{\pi} (i) \bar{\pi} (j)} \bfz_j}
    \prec \sum\nolimits_{j \not= i} \abs{Q_{ij}}^2 
    \norm{\bfOmega_{\bar{\pi} (i) \bar{\pi} (j)}}_{\Fnorm}^2
    \lesssim n \phi^2,
\end{equation*}
where the last inequality follows from the upper bound on $Q_{ij}$ in \eqref{bound:upper-Q} and the definition of the control parameter $\phi$ in \eqref{def:ctrl-para-rho-phi}. We next bound the operator norm of $\bfB_i$. Applying Lemma \ref{lemma:norm-cov-type} again yields
\begin{equation*}
    \norm{\bfB_i}^2 
    = \norm{\bfB_i \bfB_i^\top}
    = \norm[\Big]{\sum\nolimits_{j \not= i} \bfb_{ij} \bfb_{ij}^\top}
    \prec \max_{j \not= i} \curls[\big]{\abs{Q_{ij}}^2 \norm{\bfOmega_{\bar{\pi} (i) \bar{\pi} (j)}}_{\Fnorm}^2}
    + \sum\nolimits_{j \not= i} \abs{Q_{ij}}^2 \norm{\bfOmega_{\bar{\pi} (i) \bar{\pi} (j)}}^2
    \lesssim \phi^2.
\end{equation*}
Substituting these two estimates into \eqref{tmp:hatH-bound} yields $\norm{\tilde{\bfH} \tilde{\bfH}^\top} \prec n \phi^2$. This completes the proof.
\end{proof}

\subsection{Proof of Lemmas \ref{lemma:u-dotK-barK-u} and \ref{lemma:QW-barU-infnorm}}
\label{subsec:proof-using-Hanson-Wright}

This section is devoted to the proof of Lemmas \ref{lemma:u-dotK-barK-u} and \ref{lemma:QW-barU-infnorm}. We first record an auxiliary concentration result, which shows that the Bernstein-type bound \eqref{eqn:bernstein} and the Hanson--Wright-type bound \eqref{eqn:hanson-wright} remain valid after concatenating the random vectors from Assumption \ref{assump:concentration}.

\begin{lemma}[concatenation]
\label{lemma:long-concentration}
Let $\{ \bfz_i \}_{i=1}^n$ be independent isotropic random vectors satisfying Assumption \ref{assump:concentration}. Let $\bfz$ denote their concatenation, namely $\bfz^\top = [\bfz_1^\top, \cdots, \bfz_n^\top]$, and let $N = \sum_{k=1}^K n_k d_k$ be the dimension of $\bfz$. Then
\begin{subequations}
\begin{alignat}{2}
    \abs{\angles{\bfa, \bfz}}
    &\prec \norm{\bfa},
    & \qquad & \text{ for all deterministic } \qquad 
    \bfa \in \bbr^{N},
    \label{eqn:long-bernstein} \\
    \abs{\angles{\bfz, \bfB \bfz} - \Tr \bfB}
    & \prec \norm{\bfB}_{\Fnorm},
    & \qquad & \text{ for all deterministic } \qquad
    \bfB \in \bbr^{N \times N}.
    \label{eqn:long-hanson-wright}
\end{alignat}
\end{subequations}
\end{lemma}

Let us remark that the proof of Lemma \ref{lemma:long-concentration} below is a straightforward adaptation of the argument used in the proof of \cite[Theorem 7.7]{erdosDynamicalApproachRandom2017}. The main difference is that we consider the concatenation of independent random vectors, whereas \cite[Theorem 7.7]{erdosDynamicalApproachRandom2017} treats the corresponding problem for independent random variables. We also note that the proof actually yields Lemma \ref{lemma:long-concentration} in a slightly more general setting: each vector being concatenated may have its own dimension, not necessarily belonging to the finite set $\{d_1,\ldots,d_K\}$ in Assumption \ref{assump:concentration}. Moreover, the number of vectors being concatenated may exceed $n$, provided it remains bounded by $n^C$ for some fixed constant $C>0$.

\begin{proof}[Proof of \eqref{eqn:long-bernstein}]
By homogeneity, it suffices to consider the case $\norm{\bfa}=1$. Let us decompose the deterministic vector $\bfa$ as $\bfa^\top = [\bfa_1^\top, \cdots, \bfa_n^\top]$, such that $\bfa_i$ has the same dimension as $\bfz_i$. Now, $\angles{\bfa, \bfz} = \sum\nolimits_{i=1}^n \angles{\bfa_i, \bfz_i}$. By the Marcinkiewicz--Zygmund inequality; see, for example, \cite[Section 10.3]{chowProbabilityTheoryIndependence2012}, for each fixed $\ell \in \bbn_+$,
\begin{equation*}
    \bbe \abs{\angles{\bfa, \bfz}}^\ell
    = \bbe \abs[\Big]{\sum\nolimits_{i=1}^n \angles{\bfa_i, \bfz_i}}^\ell
    \leq C_{\ell} \bbe \pars[\Big]{\sum\nolimits_{i=1}^n \abs{\angles{\bfa_i, \bfz_i}}^2}^{\ell/2},
\end{equation*}
where $C_{\ell}>0$ depends only on $\ell$. By Assumption \ref{assump:concentration},
\begin{equation*}
    \sum\nolimits_{i=1}^n \abs{\angles{\bfa_i, \bfz_i}}^2
    \prec \sum\nolimits_{i=1}^n \norm{\bfa_i}^2
    = \norm{\bfa}^2 = 1.
\end{equation*}
Hence, by Lemma \ref{lemma:basic-property-prec} \ref{item:compatibility-expectation}, we have
\begin{equation*}
    \bbe \pars[\Big]{\sum\nolimits_{i=1}^n \abs{\angles{\bfa_i, \bfz_i}}^2}^{\ell/2}
    \prec 1.
\end{equation*}
Consequently, we have $\bbe \abs{\angles{\bfa, \bfz}}^\ell \prec 1$ for each fixed $\ell \in \bbn_+$. Since $\ell$ is arbitrary, using the reverse implication in Lemma \ref{lemma:basic-property-prec} \ref{item:compatibility-expectation}, we conclude the proof of \eqref{eqn:long-bernstein}.
\end{proof}

\begin{proof}[Proof of \eqref{eqn:long-hanson-wright}]
We begin by partitioning the matrix $\bfB$ into blocks $\{\bfB_{ij}\}_{i,j=1}^n$, conformally with the decomposition of $\bfz$, so that $\angles{\bfz, \bfB \bfz} = \sum\nolimits_{i, j = 1}^n \angles{\bfz_i, \bfB_{ij} \bfz_j}$. Now,
\begin{equation}
    \angles{\bfz, \bfB \bfz} - \Tr \bfB
    = \sum\nolimits_{i = 1}^n 
    \pars[\big]{\angles{\bfz_i, \bfB_{ii} \bfz_i} - \Tr \bfB_{ii}}
    + \sum\nolimits_{i \neq j}
    \angles{\bfz_i, \bfB_{ij} \bfz_j} .
    \label{tmp:quad-diag-offdiag}
\end{equation}
We control the two terms on the r.h.s. separately.

For the diagonal part, the argument is similar to that in the proof of \eqref{eqn:long-bernstein}. Using the Marcinkiewicz--Zygmund inequality together with Lemma \ref{lemma:basic-property-prec} \ref{item:compatibility-expectation}, we obtain, for each fixed $\ell \in \bbn_+$,
\begin{equation*}
    \bbe \abs[\Big]{\sum\nolimits_{i=1}^n \pars[\big]{
    \angles{\bfz_i, \bfB_{ii}\bfz_i} - \Tr \bfB_{ii}} }^\ell
    \leq C_{\ell} \bbe \pars[\Big]{\sum\nolimits_{i=1}^n 
    \abs{\angles{\bfz_i, \bfB_{ii}\bfz_i} - \Tr \bfB_{ii}}^2}^{\ell/2} 
    \prec \pars[\big]{\sum\nolimits_{i=1}^n \norm{\bfB_{ii}}_{\Fnorm}^2}^{\ell/2}
    \leq \norm{\bfB}_{\Fnorm}^\ell .
\end{equation*}
Since $\ell \in \bbn_+$ is arbitrary, Lemma \ref{lemma:basic-property-prec} \ref{item:compatibility-expectation} yields
\begin{equation}
    \abs[\Big]{\sum\nolimits_{i=1}^n \pars[\big]{
    \angles{\bfz_i, \bfB_{ii}\bfz_i} - \Tr \bfB_{ii}} }
    \prec \norm{\bfB}_{\Fnorm} .
    \label{tmp:quad-diag-part}
\end{equation}

It remains to control the off-diagonal part in \eqref{tmp:quad-diag-offdiag}. We use a decoupling argument similar to that in the proof of Lemma \ref{lemma:decoupling}. We first record a simple auxiliary estimate. Let $\{ \tilde{\bfz}_i \}_{i=1}^n$ be an independent copy of $\{ \bfz_i \}_{i=1}^n$, and let $\bfD \in \bbr^{N \times N}$ be deterministic. Conditioning on $\bfz$ and applying \eqref{eqn:long-bernstein} to $\tilde{\bfz}$, we obtain
\begin{equation*}
    \abs{\angles{\bfz, \bfD \tilde{\bfz}}}
    \prec \norm{\bfD^\top \bfz} .
\end{equation*}
Writing $\bfD = [\mathbf{d}_1, \cdots, \mathbf{d}_N]$, where $\mathbf{d}_r$ denotes the $r$-th column of $\bfD$, another application of \eqref{eqn:long-bernstein} gives
\begin{equation*}
    \norm{\bfD^\top \bfz}^2
    = \sum\nolimits_{r=1}^N \abs{\angles{\mathbf{d}_r, \bfz}}^2
    \prec \sum\nolimits_{r=1}^N \norm{\mathbf{d}_r}^2
    = \norm{\bfD}_{\Fnorm}^2 .
\end{equation*}
Hence, we have
\begin{equation}
    \abs{ \angles{\bfz, \bfD \tilde{\bfz}} }
    \prec \norm{\bfD}_{\Fnorm},
    \qquad \text{ for all deterministic } \qquad
    \bfD \in \bbr^{N \times N}.
    \label{tmp:decouple-quad}
\end{equation}

We now return to the off-diagonal part of \eqref{tmp:quad-diag-offdiag}. For each fixed pair $(i,j)$ with $i \neq j$, we have
\begin{equation*}
    1 
    = \frac{1}{2^{n-2}} 
    \sum\nolimits_{\bbj_1 \cup \bbj_2 = \dbraks{n}} 
    \bbone\{i \in \bbj_1 \} \bbone\{ j \in \bbj_2 \},
\end{equation*}
where the sum is over all ordered partitions of $\dbraks{n}$ into two nonempty disjoint subsets $\bbj_1$ and $\bbj_2$. Moreover, the number of such ordered partitions is $\sum\nolimits_{\bbj_1 \cup \bbj_2 = \dbraks{n}} 1 = 2^n - 2$. Now, we can write
\begin{align*}
    \sum\nolimits_{i \neq j} \angles{\bfz_i, \bfB_{ij} \bfz_j} 
    & = \sum\nolimits_{i \neq j} \angles{\bfz_i, \bfB_{ij} \bfz_j} 
    \pars[\Big]{ \frac{1}{2^{n-2}} 
    \sum\nolimits_{\bbj_1 \cup \bbj_2 = \dbraks{n}} 
    \bbone\{ i \in \bbj_1 \} \bbone\{ j \in \bbj_2 \} } \\
    & = \frac{1}{2^{n-2}} \sum\nolimits_{\bbj_1 \cup \bbj_2 = \dbraks{n}} 
    \sum\nolimits_{i \in \bbj_1 , j \in \bbj_2} 
    \angles{\bfz_i, \bfB_{ij} \bfz_j} .
\end{align*}
Fix $\varepsilon>0$ and $\ell \in \bbn_+$. For each ordered partition $\bbj_1 \cup \bbj_2 = \dbraks{n}$, the two collections $\{ \bfz_i \}_{i \in \bbj_1}$ and $\{ \bfz_j \}_{j \in \bbj_2}$ are independent. Thus, \eqref{tmp:decouple-quad}, together with Lemma \ref{lemma:basic-property-prec} \ref{item:compatibility-expectation}, implies that, for all sufficiently large $n$,
\begin{equation*}
    \bbe \abs[\Big]{ \sum\nolimits_{i \in \bbj_1 , j \in \bbj_2} 
    \angles{\bfz_i, \bfB_{ij} \bfz_j} }^\ell
    \leq n^\varepsilon \pars[\Big]{
    \sum\nolimits_{i \in \bbj_1 , j \in \bbj_2} \norm{\bfB_{ij}}_{\Fnorm}^2 }^{\ell/2}
    \leq n^\varepsilon \norm{\bfB}_{\Fnorm}^\ell .
\end{equation*}
By Minkowski's inequality, it follows that
\begin{equation*}
    \pars[\Big]{ \bbe \abs[\Big]{\sum\nolimits_{i \neq j} 
    \angles{\bfz_i, \bfB_{ij} \bfz_j}}^\ell }^{1 / \ell}
    \leq \frac{1}{2^{n-2}} 
    \sum\nolimits_{\bbj_1 \cup \bbj_2 = \dbraks{n}} 
    \pars[\Big]{ \bbe \abs[\Big]{ \sum\nolimits_{i \in \bbj_1 , j \in \bbj_2} 
    \angles{\bfz_i, \bfB_{ij} \bfz_j} }^\ell }^{1 / \ell}
    \leq 4 n^{\varepsilon / \ell} \norm{\bfB}_{\Fnorm} .
\end{equation*}
Since $\varepsilon>0$ and $\ell \in \bbn_+$ are arbitrary, Lemma \ref{lemma:basic-property-prec} \ref{item:compatibility-expectation} implies that
\begin{equation}
    \abs[\Big]{\sum\nolimits_{i \neq j} 
    \angles{\bfz_i, \bfB_{ij} \bfz_j}} 
    \prec \norm{\bfB}_{\Fnorm} .
    \label{tmp:quad-offdiag-part}
\end{equation}
Combining \eqref{tmp:quad-diag-part} and \eqref{tmp:quad-offdiag-part} with the decomposition \eqref{tmp:quad-diag-offdiag} completes the proof of \eqref{eqn:long-hanson-wright}.
\end{proof}

Equipped with Lemma \ref{lemma:long-concentration}, we are now ready to prove Lemmas \ref{lemma:u-dotK-barK-u} and \ref{lemma:QW-barU-infnorm}.

\begin{proof}[Proof of Lemma \ref{lemma:u-dotK-barK-u}]
Note that the first inequality in \eqref{eqn:u-dotK-barK-u} follows directly from the fact that the operator norm of a submatrix is bounded by that of the full matrix. For the second inequality, since $m$ is fixed, it suffices to establish entrywise control of $\barU^\top (\ddotK - \barK) \barU \in \bbr^{m \times m}$. More precisely, it is enough to show that
\begin{equation*}
    \abs{ \angles{\baru_k, (\ddotK - \barK) \baru_{\ell}} } 
    \prec \sqrt{n} \rho,
    \qfor
    k, \ell \in \dbraks{m}.
\end{equation*}
We prove the following slightly more general estimate: for any deterministic $\bfs_1, \bfs_2 \in \bbr^n$, 
\begin{equation}
    \abs{ \angles{\bfs_1, (\ddotK - \barK) \bfs_2} } 
    \prec \sqrt{n} \rho,
    \quad \text{ if } \quad
    \norm{\bfs_1}_{\infty} \vee \norm{\bfs_2}_{\infty} \lesssim 1 / \sqrt{n}.
    \label{eqn:average-dotK-barK}
\end{equation}
Indeed, by the representation of the eigenvectors $\baru_k$ in \eqref{eqn:baru-piecewise}, we have $\norm{\baru_k}_{\infty} \vee \norm{\baru_{\ell}}_{\infty} \lesssim n^{-1/2}$. Thus \eqref{eqn:average-dotK-barK} implies the desired entrywise bound.

It remains to prove \eqref{eqn:average-dotK-barK}. To this end, we use the decomposition of $\ddotK - \barK$ given in \eqref{eqn:diff-dotK-barK} to get
\begin{align} \label{tmp:dotK-barK-kl}
\begin{split}
    \angles{\bfs_1, (\ddotK - \barK) \bfs_2}
    & = \angles{\bfs_1, 
    \braks{\bfQ \circ (\bfzeta \bfone^\top + \bfone \bfzeta^\top + 2\bfUps)} \bfs_2} \\
    & = \angles{\bfs_1, (\bfQ \circ \bfzeta \bfone^\top) \bfs_2}
    + \angles{\bfs_1, (\bfQ \circ \bfone \bfzeta^\top) \bfs_2}
    + 2 \angles{\bfs_1, ({\bfQ \circ \bfUps}) \bfs_2}.  
\end{split} 
\end{align}
We control the terms on the r.h.s. of \eqref{tmp:dotK-barK-kl} separately. Recall the definition of $\zeta_i$ in \eqref{def:zeta-Ups-W}. Write $\bfs_a=(s_{ai})_{i=1}^n$, and let $\bfz$ denote the concatenation of $\{ \bfz_i \}_{i=1}^n$, as in Lemma \ref{lemma:long-concentration}. Then
\begin{align*}
    \angles{\bfs_1, (\bfQ \circ \bfzeta \bfone^\top) \bfs_2}
    & = \sum\nolimits_{i,j=1}^n 
    Q_{ij} \zeta_i s_{1i} s_{2j} \\
    & = \sum\nolimits_{i,j=1}^n 
    Q_{ij} s_{1i} s_{2j} 
    \braks[\big]{ \angles{\bfz_i, \bfOmega_{\bar{\pi}(i)\bar{\pi}(i)} \bfz_i} 
    - \Tr \bfOmega_{\bar{\pi}(i)\bar{\pi}(i)} } 
    =: \angles{\bfz, \bfB \bfz} - \Tr \bfB.
\end{align*}
where $\bfB = \bfB_{11} \oplus \cdots \oplus \bfB_{nn}$ is block diagonal, with diagonal blocks
\begin{equation*}
    \bfB_{ii}
    = {s_{1i}} \pars[\Big]{ \sum\nolimits_{j=1}^n Q_{ij} s_{2j}} 
    \bfOmega_{\bar{\pi}(i) \bar{\pi}(i)}.
\end{equation*}
To apply the Hanson-Wright estimate \eqref{eqn:long-hanson-wright}, we need an estimate for $\norm{\bfB}_{\Fnorm}$. By \eqref{bound:upper-Q}, together with the relation $2\theta_{k\ell} \geq \theta_{kk} + \theta_{\ell\ell}$ which follows directly from the definition, we have $\abs{Q_{ij}} \lesssim 1/\theta_{\bar{\pi}(i)\bar{\pi}(i)}$. Hence,
\begin{align*}
    \norm{\bfB}_{\Fnorm}^2 
    = \sum\nolimits_{i=1}^n \norm{\bfB_{ii}}_{\Fnorm}^2 
    & = \sum\nolimits_{i=1}^n \abs{s_{1i}}^2 
    \braks[\Big]{ \sum\nolimits_{j=1}^n Q_{ij} s_{2j}}^2 
    \norm{\bfOmega_{\bar{\pi}(i) \bar{\pi}(i)}}_{\Fnorm}^2 \\
    & \lesssim n \sum\nolimits_{i=1}^n \abs{s_{1i}}^2 
    \pars[\big]{\norm{\bfSigma_{\bar{\pi}(i)}}_{\Fnorm} / \theta_{\bar{\pi} (i) \bar{\pi} (i)}}^2
    \leq n \rho^2 \sum\nolimits_{i=1}^n \abs{s_{1i}}^2 
    \lesssim n \rho^2,
\end{align*}
where we used $\norm{\bfs_1} \lesssim 1$ and $\norm{\bfs_2}_\infty \lesssim 1 / \sqrt{n}$. Applying \eqref{eqn:long-hanson-wright} therefore yields the desired bound for the first term on the r.h.s. of \eqref{tmp:dotK-barK-kl}. The second term is treated in exactly the same way. Consequently,
\begin{equation}
    \abs{\angles{\bfs_1, (\bfQ \circ \bfzeta \bfone^\top) \bfs_2}}
    + \abs{\angles{\bfs_1, (\bfQ \circ \bfone \bfzeta^\top) \bfs_2}}
    \prec \sqrt{n} \rho.
    \label{tmp:contri-from-zeta}
\end{equation}
It remains to control the third term in \eqref{tmp:dotK-barK-kl}. Recall the definition of $\bfUps$ in \eqref{def:zeta-Ups-W}. We further decompose
\begin{align*}
    \angles{\bfs_1, ({\bfQ \circ \bfUps}) \bfs_2}
    & = \sum\nolimits_{i,j=1}^n 
    Q_{ij} \Upsilon_{ij} s_{1i} s_{2j} \\
    & = \sum\nolimits_{i,j=1}^n 
    Q_{ij} s_{1i} s_{2j}
    \angles{\bfmu_{\bar{\pi} (i)} - \bfmu_{\bar{\pi}(j)}, \bfA_{\bar{\pi} (i)} \bfz_i} 
    - \sum\nolimits_{i,j=1}^n  
    Q_{ij} s_{1i} s_{2j} 
    \angles{\bfmu_{\bar{\pi} (i)} - \bfmu_{\bar{\pi}(j)}, \bfA_{\bar{\pi} (j)} \bfz_j}.
\end{align*}
The two sums are handled in the same way, so we only consider the first one. This sum can be expressed as $\angles{\bfa,\bfz}$, where $\bfa^\top = [\bfa_1^\top,\ldots,\bfa_n^\top]$ is a deterministic vector with
\begin{equation*}
    \bfa_i := s_{1i} 
    \sum\nolimits_{j=1}^n Q_{ij} s_{2j} 
    \bfA_{\bar{\pi}(i)}^\top 
    \pars[\big]{\bfmu_{\bar{\pi} (i)}  - \bfmu_{\bar{\pi}(j)}}.
\end{equation*}
We next estimate $\norm{\bfa}$. By \eqref{bound:upper-Q}, the definition \eqref{def:psi}, and the bound $\norm{\bfs_2}_{\infty} \lesssim n^{-1/2}$, we have
\begin{align*}
    \norm{\bfa_i} \lesssim \abs{s_{1i}}
    \sum\nolimits_{j=1}^n \abs{s_{2j}} 
    \cdot {\norm[\big]{\bfA_{\bar{\pi}(i)}^\top 
    \pars[\big]{\bfmu_{\bar{\pi} (i)}  - \bfmu_{\bar{\pi}(j)}} }}
    \big / {\theta_{\bar{\pi} (i) \bar{\pi} (j)}}
    \lesssim \sqrt{n} \rho \abs{s_{1i}}.
\end{align*}
Consequently,
\begin{equation*}
    \norm{\bfa}^2
    = \sum\nolimits_{i=1}^n \norm{\bfa_i}^2 
    \lesssim n\rho^2 \norm{\bfs_1}^2
    = n\rho^2.
\end{equation*}
Now, applying \eqref{eqn:long-bernstein} yields $\abs{\angles{\bfa, \bfz}} \prec \sqrt{n}\rho$. The second sum admits the same bound. Therefore,
\begin{equation}
    \abs{\angles{\bfs_1, ({\bfQ \circ \bfUps}) \bfs_2}} \prec \sqrt{n} \rho.
    \label{tmp:contri-from-Ups}
\end{equation}
Combining the estimates \eqref{tmp:contri-from-zeta} and \eqref{tmp:contri-from-Ups} with \eqref{tmp:dotK-barK-kl} yields the desired entrywise bound. We thus complete the proof of Lemma \ref{lemma:u-dotK-barK-u}.
\end{proof}

\begin{proof}[Proof of Lemma \ref{lemma:QW-barU-infnorm}]
Recall that $\braks{\bfQ \circ \caH(\bfW)} \barU_r$ is an $n \times r$ matrix. Since $r$ is fixed, we have
\begin{equation*}
    \norm{\braks{\bfQ \circ \caH (\bfW)} \barU_r}_{2, \infty}
    \lesssim \max\nolimits_{k \in \dbraks{r}}
    \norm{\braks{\bfQ \circ \caH(\bfW)} \baru_k}_{\infty}.
\end{equation*}
Thus it suffices to control the $\ell_\infty$ norm of each column. In fact, we prove the following more general estimate: for any deterministic vector $\bfs = (s_j)_{j=1}^n \in \bbr^n$ satisfying $\norm{\bfs} \lesssim 1$,
\begin{equation}
    \norm{\braks{\bfQ \circ \caH(\bfW)} \bfs}_{\infty} \prec \rho.
    \label{eqn:general-QW-s}
\end{equation}
Fix $i \in \dbraks{n}$. The $i$-th entry of $\braks{\bfQ \circ \caH(\bfW)} \bfs$ is 
\begin{align*}
    \angles{\bfe_i, \braks{\bfQ \circ \caH (\bfW)} \bfs}
    = \sum\nolimits_{j \neq i} Q_{ij} W_{ij} s_{j}
    = \sum\nolimits_{j \neq i} Q_{ij} s_{j}
    \angles{\bfz_i, \bfOmega_{\bar{\pi}(i) \bar{\pi}(j)} \bfz_j}
    =: \angles{\bfz_i, \bfa_i},
\end{align*}
where we introduced
\begin{equation*}
    \bfa_i = \sum\nolimits_{j \neq i} 
    Q_{ij} s_{j}  
    \bfOmega_{\bar{\pi}(i) \bar{\pi}(j)} \bfz_j
    = \bfB_{(-i)} \bfz_{(-i)}.
\end{equation*}
Here, $\bfz_{(-i)}$ denotes the concatenation of the vectors $\bfz_j$ with $j \neq i$, and $\bfB_{(-i)}$ is the corresponding block row matrix whose $j$-th block is $Q_{ij} \bfs(j) \bfOmega_{\bar{\pi}(i) \bar{\pi}(j)}$. We first bound $\norm{\bfa_i}$. By the Hanson--Wright estimate \eqref{eqn:hanson-wright}, together with the elementary inequality $\norm{\bfA}_{\Fnorm} \leq \Tr \bfA$ for positive semidefinite $\bfA$, we have
\begin{align*}
    \norm{\bfa_i}^2
    \prec \norm{\bfB_{(-i)}}_{\Fnorm}^2
    & = \sum\nolimits_{j \not= i} 
    \norm{Q_{ij} s_{j}  
    \bfOmega_{\bar{\pi}(i) \bar{\pi}(j)}}^2_{\Fnorm} \\
    & \lesssim \sum\nolimits_{j \not= i} 
    \abs{s_{j}}^2
    \pars[\big]{
    \norm{\bfSigma_{\bar{\pi}(i) \bar{\pi}(i)}}^2_{\Fnorm}
    + \norm{\bfSigma_{\bar{\pi}(j) \bar{\pi}(j)}}^2_{\Fnorm} }
    \big / \theta_{\bar{\pi} (i) \bar{\pi} (j)}^2
    \lesssim \rho^2     
    \sum\nolimits_{j \not= i} 
    \abs{s_{j}}^2
    \leq \rho^2.
\end{align*}
As $\bfa_i$ is independent from $\bfz_i$, we may apply the Bernstein estimate \eqref{eqn:long-bernstein} conditional on $\bfa_i$ to obtain
\begin{equation*}
    \abs{\angles{\bfe_i, \braks{\bfQ \circ \caH (\bfW)} \bfs}}
    = \abs{\angles{\bfz_i, \bfa_i}}
    \prec \norm{\bfa_i} \prec \rho
\end{equation*} 
This proves the desired $\ell_\infty$ estimate \eqref{eqn:general-QW-s}, and hence completes the proof of Lemma \ref{lemma:QW-barU-infnorm}.
\end{proof}

\subsection{Proof of Proposition \ref{prop:dotlamb-barlamb} and Lemma \ref{lemma:inner-prod-dotU-barU}}
\label{subsec:proof-refined-sqrtn}

This section is devoted to the proof of Proposition \ref{prop:dotlamb-barlamb} and Lemma \ref{lemma:inner-prod-dotU-barU}. These two results provide perturbation bounds for the spectral components of $\ddotK$ relative to the unperturbed matrix $\barK$. Recall that we have already established the a priori estimate $\norm{\ddotK - \barK} \prec n \rho$ in \eqref{bound:ddotK-barK}. As mentioned above, Proposition \ref{prop:dotlamb-barlamb} gives a finer control of the spiked eigenvalues of $\ddotK$, improving by a factor of order $\sqrt{n}$ the bound obtained by directly combining \eqref{bound:ddotK-barK} with Weyl's inequality. Similarly, Lemma \ref{lemma:inner-prod-dotU-barU} gives a refined estimate for the spiked eigenvectors $\ddotU_r$. In fact, recall the $\ell_2$ estimate \eqref{bound:l2-dotU-barU}, which follows from \eqref{bound:ddotK-barK} and the Davis--Kahan theorem. A direct application of this estimate only gives
\begin{equation*}
    \norm{\barK (\ddotU_r \ddotfkQ_r - \barU_r)}_{2, \infty}
    \leq
    \norm{\barK}_{2, \infty}
    \cdot \norm{\ddotU_r \ddotfkQ_r - \barU_r}
    \prec
    \sqrt{n} \cdot ({n \rho} / {\barDelta_r}).
\end{equation*}
Therefore, Lemma \ref{lemma:inner-prod-dotU-barU} improves this estimate by removing the factor $\sqrt{n}$ on the r.h.s. The mechanism behind these two refinements is similar to that in Lemma \ref{lemma:u-dotK-barK-u}: instead of relying on the operator-norm control of $\ddotK-\barK$, which needs to cover the worst case for the direction of eigenvectors, one uses the delocalization of $\barU_r$ to exploit fluctuation averaging in the associated bilinear forms. The main complication here is that, we are now concerning spectral components of $\ddotK$, rather than the perturbation $\ddotK-\barK$ itself. Thus, before applying the available fluctuation estimates, we need to express these spectral quantities in terms of the perturbation. This step is accomplished through the resolvent method combined with the contour integral technique. 

To this end, we introduce the resolvents associated with $\ddotK$ and $\ddotK-\barK$. For $z \in \bbc$, set
\begin{equation*}
    \ddotG (z)
    = (\ddotK - z)^{-1},
    \qand
    \bfR (z)
    = (\ddotK - \barK - z)^{-1}.
\end{equation*}
More precisely, $\ddotG(z)$ is defined for $z \notin \operatorname{spec}(\ddotK)$, whereas $\bfR(z)$ is defined for $z \notin \operatorname{spec}(\ddotK-\barK)$.

\begin{proof}[Proof of Proposition \ref{prop:dotlamb-barlamb}]
Recall from Assumption \ref{assump:signal} that $\barDelta_r \gtrsim n^{1 + c_{\ref{assump:signal}}} \rho$. Fix arbitrary constants $\varepsilon \in (0, (c_{\ref{assump:signal}} \vee 1) / 2)$ and $C > 0$. By \eqref{bound:ddotK-barK}, there exists a high-probability event $\Xi_1$ with $\bbp(\Xi_1) \geq 1-n^{-C}$ such that, on $\Xi_1$,
\begin{equation}
    \abs{\ddotlamb_k - \barlamb_k} 
    \leq \norm{\ddotK - \barK}
    \leq n^{1+\varepsilon/2} \rho,
    \qfor k \in \dbraks{r}.
    \label{tmp:a-priori-dotlamb}
\end{equation}
Since $\barlamb_r \geq \barDelta_r \gtrsim n^{1+2\varepsilon}\rho$, it follows that, on $\Xi_1$,
\begin{equation}
    \ddotlamb_k 
    \geq \barlamb_r - \barDelta_r / 4 =: \barlamb_{r-},
    \qfor k \in \dbraks{r}.
    \label{tmp:lower-bound-dotlamb}
\end{equation}
Next, let us introduce
\begin{equation}
    \delta (x) := 
    \sqrt{n} \rho 
    + (n \rho)^2 / x,
    \qfor x > 0.
    \label{def:func-delta}
\end{equation}
We also introduce the union of intervals
\begin{equation*}
    \mathcal{I} := \bigcup\nolimits_{k = 1}^r \mathcal{I}_k
    \qwhere
    \mathcal{I}_k := \braks[\big]{
    \barlamb_k - n^{2 \varepsilon} \delta(\barlamb_k), 
    ~ \barlamb_k + n^{2 \varepsilon} \delta(\barlamb_k)}.
\end{equation*}
Using $\barDelta_r \gtrsim n^{1+2\varepsilon}\rho$ again, it is not difficult to check
\begin{equation}
    x \gtrsim \barDelta_r \gtrsim n^{4 \varepsilon} \delta (x),
    \qfor 
    x \geq \barlamb_{r-}.
    \label{tmp:order-delta}
\end{equation}
In particular, we have $\mathcal{I} \subset [\barlamb_{r-},\infty)$ by construction. We first show that, with high probability, $\ddotlamb_k \in \mathcal{I}$ for every $k \in \dbraks{r}$. Equivalently, we show that $[\barlamb_{r-},\infty)\setminus \mathcal{I}$ is a forbidden region for $\ddotlamb_k$; namely,
\begin{equation}
    \ddotlamb_k \notin 
    [\barlamb_{r-}, \infty) \backslash \mathcal{I},
    \qfor k \in \dbraks{r}.
    \label{tmp:forbidden-region}
\end{equation}

For what follows, we always work on the high-probability event $\Xi_1$. Using \eqref{tmp:a-priori-dotlamb} and \eqref{tmp:lower-bound-dotlamb}, we find that $\ddotlamb_k \notin \operatorname{spec}(\ddotK-\barK)$ for all $k \in \dbraks{r}$. On the other hand, for any $x \notin \operatorname{spec}(\ddotK-\barK)$, we have
\begin{equation*}
    \det (\ddotK - x) = 0
    \quad \Longleftrightarrow \quad
    \det \braks[\big]{ \barU \barLamb \barU^\top 
    + (\ddotK - \barK - x) } = 0 
    \quad \Longleftrightarrow \quad
    \det \braks[\big]{ \barLamb^{-1} + \barU^\top \bfR (x) \barU } = 0 .   
\end{equation*}
Here the first equivalence uses the spectral decomposition of $\barK$ in \eqref{eqn:spec-decomp-barK}, while the second follows from $\det(\ddotK-\barK-x)\neq 0$ and Sylvester's determinant identity $\det (\bfI + \mathbf{A} \mathbf{B}) = \det (\bfI + \mathbf{B} \mathbf{A})$. Therefore, to prove \eqref{tmp:forbidden-region}, it suffices to show that
\begin{equation}
    \det \braks[\big]{ \barU^\top \bfR (x) \barU + \barLamb^{-1} } \not= 0,
    \qfor
    x \in [\barlamb_{r-}, \infty) \backslash \mathcal{I}.
    \label{tmp:det-nonsingular}
\end{equation}
Before proving \eqref{tmp:det-nonsingular}, we first record several elementary estimates for the resolvent $\bfR(z)$. The resolvent $\bfR(z)$ is well-defined whenever $\abs{z} \geq \barlamb_{r-}$. Indeed, on the event $\Xi_1$, we have $\dist (z, \operatorname{spec} (\ddotK - \barK)) \geq \abs{z} / 2$ and hence
\begin{equation}
    \norm{\bfR(z)} \leq 2 / \abs{z}, 
    \qfor 
    \abs{z} \geq \barlamb_{r-}.
    \label{tmp:norm-resolvent}
\end{equation}
Using the resolvent expansion, we obtain
\begin{equation}
    \bfR(z)
    = - \bfI / z
    - (\ddotK - \barK) / z^2
    + (\ddotK - \barK)\bfR(z)(\ddotK - \barK) / z^2  .  
    \label{tmp:resolvent-expand-R}
\end{equation}
We control the last term on the r.h.s. using \eqref{tmp:a-priori-dotlamb} together with \eqref{tmp:norm-resolvent}, and control the second term using Lemma \ref{lemma:u-dotK-barK-u}. It turns out that there exists another high-probability event $\Xi_2$, with $\bbp(\Xi_2) \geq 1 - n^{-C}$, such that on $\Xi := \Xi_1 \cap \Xi_2$, we have
\begin{equation}
    \barU^\top \bfR(z) \barU
    = - \bfI / z
    + \bigO[\big]{ n^\varepsilon
    \delta (\abs{z}) / \abs{z}^2 },
    \qfor 
    \abs{z} \geq \barlamb_{r-}.
    \label{tmp:URU-without-first}
\end{equation}
Since the matrices in \eqref{tmp:URU-without-first} have fixed dimension, the big $O$ notation may be interpreted either entrywise or in operator norm. 

We now return to the proof of \eqref{tmp:det-nonsingular}. Let $x \in [\barlamb_{r-},\infty) \setminus \mathcal{I}$. By \eqref{tmp:URU-without-first}, we have, on the event $\Xi$,
\begin{equation*}
    \barLamb^{-1} + \barU^\top \bfR (x) \barU 
    = \barLamb^{-1} - \bfI / x 
    + \bigO[\big]{n^{\varepsilon} \delta (x) / x^2} .
\end{equation*}
Therefore, to prove \eqref{tmp:det-nonsingular}, it suffices to show that, on the event $\Xi$,
\begin{equation}
    \min_{k \in \dbraks{m}} 
    \abs{1 / \barlamb_k - 1 / x} \gtrsim n^{2 \varepsilon} \delta (x) / x^2
    \qfor
    x \in [\barlamb_{r-}, \infty) \backslash \mathcal{I}.
    \label{tmp:lower-reciprocal-diff}
\end{equation}
We prove \eqref{tmp:lower-reciprocal-diff} by considering three cases.
\begin{enumerate}[label = (\roman*)]
    \item Let $k \in \dbraks{m} \setminus \dbraks{r}$. Then $x - \barlamb_k \geq \barDelta_r/2$, and therefore
    \begin{equation*}
        \abs{1 / \barlamb_k - 1 / x} 
        \gtrsim \barDelta_r / (\barlamb_k x)
        \geq \barDelta_r / x^2
        \gtrsim n^{4 \varepsilon} \delta (x) / x^2.
    \end{equation*}
    \item Let $k \in \dbraks{r}$ and suppose that $\barlamb_k \geq 2 x$. Then 
    \begin{equation*}
        \abs{1 / \barlamb_k - 1 / x} 
        \geq 1 / (2 x) 
        \gtrsim \barDelta_r / x^2
        \gtrsim n^{4 \varepsilon} \delta (x) / x^2 .
    \end{equation*}
    \item Let $k \in \dbraks{r}$ and suppose that $\barlamb_k < 2 x$. Since $x \notin \mathcal{I}_k$, the definition of $\mathcal{I}_k$ gives
    \begin{equation*}
        \abs{1 / \barlamb_k - 1 / x} 
        \gtrsim \abs{x - \barlamb_k} / x^2
        \geq n^{2 \varepsilon} \delta (\barlamb_k) / x^2
        \gtrsim n^{2 \varepsilon} \delta (x) / x^2,
    \end{equation*}
    where the last step follows directly from the definition of $\delta$ and the relation $\barlamb_k < 2x$.
\end{enumerate}
This proves \eqref{tmp:lower-reciprocal-diff}, and hence \eqref{tmp:det-nonsingular}. Consequently, on the high-probability event $\Xi_1\cap\Xi_2$, none of the first $r$ eigenvalues of $\ddotK$ lies in the forbidden region $[\barlamb_{r-},\infty) \setminus \mathcal{I}$. Equivalently, $\ddotlamb_k \in \mathcal{I}$ for $k \in \dbraks{r}$.


The preceding argument shows that each of the first $r$ eigenvalues of $\ddotK$ lies in the union $\mathcal{I}$, but it does not identify the corresponding interval. We next refine this localization by showing that $\ddotlamb_k$ lies in $\mathcal{I}_k$, up to possible overlaps with neighboring intervals. To this end, introduce the disks
\begin{equation*}
    \bbb_k := \{ z \in \bbc: \abs{z - \barlamb_k} \leq n^{2 \varepsilon} \delta(\barlamb_k) \},
    \qfor k \in \dbraks{r}.
\end{equation*}
Thus $\mathbb{B}_k \cap \bbr = \mathcal{I}_k$. Fix $k\in\dbraks{r}$, and let $\mathcal{L}_k \subset \dbraks{r}$ denote the set of indices whose intervals intersect $\mathcal{I}_k$:
\begin{equation*}
    \mathcal{L}_{k} 
    := \{ \ell \in  \dbraks{r}: \mathcal{I}_{\ell} \cap \mathcal{I}_k \not= \varnothing \}.
\end{equation*}
The set $\mathcal{L}_k$ is a consecutive block of indices containing $k$. Indeed, according to \eqref{tmp:order-delta}, the map $\lambda \mapsto \lambda + n^{2 \varepsilon} \delta (\lambda)$ is increasing for $\lambda \geq \barlamb_{r-}$, so overlaps can occur only between neighboring intervals in the ordered sequence of eigenvalues. Now consider the contour
\begin{equation*}
    \Gamma_k = \partial \pars[\Big]{
    \bigcup\nolimits_{\ell \in \mathcal{L}_{k}} 
    \bbb_\ell }.
\end{equation*}
By construction, the contour $\Gamma_k$ encloses precisely the eigenvalues $\{ \barlamb_\ell \}_{\ell \in \mathcal{L}_{k}}$, and it does not enclose any other eigenvalue of $\barK$; that is, it excludes $\{ \barlamb_\ell \}_{\ell \in \dbraks{m} \backslash \mathcal{L}_{k}}$. Now recall that the relevant eigenvalues of $\barK$ and $\ddotK$ can be characterized as zeros of
\begin{equation*}
    \bar{\varphi} (z) := \det \braks[\big]{\barLamb^{-1} - \bfI / z}.
    \qand
    \ddot{\varphi} (z) := \det \braks[\big]{\barLamb^{-1} + \barU^\top \bfR (z) \barU}.
\end{equation*}
The contour $\Gamma_k$ encloses exactly $|\mathcal{L}_k|$ zeros of $\bar{\varphi}$. We claim that it encloses the same number of zeros of $\ddot{\varphi}$. By Rouch\'{e}'s theorem, it suffices to prove that
\begin{equation*}
    \abs{\bar{\varphi} (z)}
    > \abs{\ddot{\varphi} (z) - \bar{\varphi} (z)},
    \qfor
    z \in \Gamma_k.
\end{equation*}
Applying the expansion \eqref{tmp:URU-without-first} in the definition of $\ddot{\varphi}(z)$, together with the multilinear expansion of the determinant, it is enough to show that,
\begin{equation*}
    \prod\nolimits_{\ell=1}^m 
    \abs{1 / \barlamb_{\ell} - 1 / z} 
    \gtrsim n^{\varepsilon} 
    \sum\nolimits_{\abs{\mathcal{A}} < m}
    \pars[\big]{n^{\varepsilon} \delta (\abs{z}) / \abs{z}^2}^{m - \abs{\mathcal{A}}}
    \prod\nolimits_{\ell \in \mathcal{A}}
    \abs{1 / \barlamb_{\ell} - 1 / z},
    \qfor
    z \in \Gamma_k.
\end{equation*}
where the sum runs over all subsets $\mathcal{A} \subset \dbraks{m}$ with $\abs{\mathcal{A}} < m$. In particular, this follows once we have
\begin{equation*}
    \min\nolimits_{\ell \in \dbraks{m}} 
    \abs{1 / \barlamb_{\ell} - 1 / z} \gtrsim n^{2 \varepsilon} \delta (\abs{z}) / \abs{z}^2,
    \qfor
    z \in \Gamma_k
\end{equation*}
This estimate is proved on the event $\Xi$ by the same argument as that used for \eqref{tmp:lower-reciprocal-diff}, and we omit the details. Therefore, on the high-probability event $\Xi$, Rouch\'{e}'s theorem implies that, for each $k \in \dbraks{r}$, the contour $\Gamma_k$ encloses exactly $\abs{\mathcal{L}_k}$ zeros of $\ddot{\varphi}$. Consequently, on the event $\Xi$,
\begin{equation} 
    \ddotlamb_k \in \bigcup\nolimits_{\ell \in \mathcal{L}_k} \mathcal{I}_{\ell}.
    \label{tmp:dotlamb-location}
\end{equation}
It remains to estimate the size of the union on the r.h.s. Fix
$\ell \in \mathcal{L}_k$. We consider two cases.
\begin{enumerate}[label = (\roman*)]
    \item Suppose first that $\ell>k$, so that $\barlamb_\ell \leq \barlamb_k$. Since $\mathcal{I}_\ell \cap \mathcal{I}_k \neq \varnothing$, the construction of the intervals gives
    \begin{equation*}
        \barlamb_{\ell} + n^{2 \varepsilon} \delta (\barlamb_{\ell})
        \geq \barlamb_k - n^{2 \varepsilon} \delta (\barlamb_k).
    \end{equation*}
    Combining \eqref{tmp:order-delta} with the preceding display yields
    \begin{equation*}
        \barlamb_{\ell} \geq \barlamb_k 
        - n^{2 \varepsilon} \delta (\barlamb_{\ell})
        - n^{2 \varepsilon} \delta (\barlamb_k)
        \geq \barlamb_k - n^{-\varepsilon} (\barlamb_k + \barlamb_\ell)
        \geq \barlamb_k / 2.
    \end{equation*}
    Consequently, using the monotonicity of $\delta (x)$, the left endpoint of $\mathcal{I}_\ell$ satisfies
    \begin{equation*}
        \barlamb_{\ell} - n^{2 \varepsilon} \delta (\barlamb_{\ell}) 
        \geq \barlamb_k - n^{2 \varepsilon} \delta (\barlamb_k) - 2 n^{2 \varepsilon} \delta (\barlamb_{\ell}) 
        \geq \barlamb_k - n^{4 \varepsilon} \delta (\barlamb_k).
    \end{equation*}
    \item Suppose next that $\ell<k$, so that $\barlamb_\ell \geq \barlamb_k$. Again using $\mathcal{I}_\ell \cap \mathcal{I}_k \neq \varnothing$, we have
    \begin{equation*}
        \barlamb_{\ell} + n^{2 \varepsilon} \delta (\barlamb_{\ell})
        \leq \barlamb_k + n^{2 \varepsilon} \delta (\barlamb_k) 
        + 2 n^{2 \varepsilon} \delta (\barlamb_{\ell})
        \leq \barlamb_k + n^{4 \varepsilon} \delta (\barlamb_k).
    \end{equation*}
\end{enumerate}
Combining the two cases, we obtain
\begin{equation*}
    \bigcup\nolimits_{\ell \in \mathcal{L}_k} \mathcal{I}_{\ell}
    \subset
    \braks[\big]{\barlamb_k - n^{4 \varepsilon} \delta (\barlamb_k), 
    \barlamb_k + n^{4 \varepsilon} \delta (\barlamb_k)}.
\end{equation*}
Together with \eqref{tmp:dotlamb-location}, this yields
\begin{equation*}
    \bbp \curls[\big]{ \abs{\ddotlamb_k - \barlamb_k} 
    \leq n^{4 \varepsilon} \delta (\barlamb_k)} \geq 1 - 2 n^{-C}.
\end{equation*}
Since $\varepsilon$ and $C$ are arbitrary, this proves Proposition \ref{prop:dotlamb-barlamb}. In fact, the argument gives a slightly stronger form of the result, since $\barlamb_k$ may be much larger than $\barDelta_r$.
\end{proof}

\begin{proof}[Proof of Lemma \ref{lemma:inner-prod-dotU-barU}]
We use a contour integral argument. Fix constants $\varepsilon \in (0, (c_{\ref{assump:signal}} \wedge 1) / 2)$ and $C > 0$ as in the proof of Lemma \ref{prop:dotlamb-barlamb}, and let the high-probability event $\Xi$ be defined as there. In what follows, we work on $\Xi$, so the argument is deterministic. Consider the contour
\begin{equation*}
    \Gamma := \partial \braks[\Big]{
    \bigcup\nolimits_{k=1}^r \{ z \in \bbc: \abs{z - \barlamb_k} \leq \barDelta_r / 4 \} }.
\end{equation*}
By the analysis in the proof of Lemma \ref{prop:dotlamb-barlamb}, it is not difficult to check that the contour $\Gamma$ encloses the first $r$ eigenvalues $\{ \ddotlamb_k \}_{k=1}^r$ of the matrix $\ddotK$, and excludes all remaining eigenvalues $\{ \ddotlamb_k \}_{k \in \dbraks{n} \backslash \dbraks{r}}$. Recalling the spectral decomposition of $\ddotK$, Cauchy's integral formula therefore gives
\begin{equation*}
    \ddotU_r \ddotU_r^\top
    = - \frac{1}{2\bar{\pi} \rmi} \sum_{i=1}^n \oint_{\Gamma} 
    \frac{\ddotu_i \ddotu_i^\top}{\ddotlamb_i - z} \rmd z
    = - \frac{1}{2\bar{\pi} \rmi} \oint_{\Gamma} \ddotG (z) \rmd z.
\end{equation*}
Fix arbitrary $i \in \dbraks{n}$ and $k \in \dbraks{r}$. Then the $(i,k)$-th entry of $\barK \ddotU_r \ddotfkQ_r$ can be written as
\begin{equation} 
    \angles{\bfe_i, \barK \ddotU_r \ddotfkQ_r \bfe_k}
    = \angles{\barK \bfe_i, \ddotU_r \ddotU_r^\top \baru_k}
    = - \frac{1}{2\bar{\pi} \rmi} \oint_{\Gamma}
    \angles{\barK \bfe_i, \ddotG(z) \baru_k} \rmd z.
    \label{barKm_hatb_hatu}
\end{equation}
We next expand the resolvent $\ddotG(z)$ appearing on the r.h.s. Using the spectral decomposition $\barK = \barU \barLamb \barU^\top$ and the Woodbury matrix identity, we obtain
\begin{equation*}
    \ddotG(z)
    = (\barU \barLamb \barU^\top + \ddotK - \barK - z)^{-1} 
    = \bfR(z) - \bfR(z) \barU \braks[\big]{\barLamb^{-1} + \barU^\top \bfR(z) \barU}^{-1}\barU^\top \bfR(z).    
\end{equation*}
Moreover, by \eqref{tmp:a-priori-dotlamb} and $\barlamb_{r-} \gtrsim \barDelta_r \gtrsim n^{1+2\varepsilon}\rho$, the contour $\Gamma$ does not enclose any eigenvalue of $\ddotK - \barK$. Hence
\begin{equation*}
    \frac{1}{2\bar{\pi} \rmi} \oint_{\Gamma} \bfR (z) \rmd z = 0.
\end{equation*}
In particular, we arrive at the Green function representation
\begin{align}
\begin{split}
    \angles{\bfe_i, \barK \ddotU_r \ddotfkQ_r \bfe_k}
    & = \frac{1}{2\bar{\pi} \rmi} \oint_{\Gamma}
    \angles[\big]{\barU^\top \bfR(z) \barK \bfe_i, 
    \braks[\big]{\barLamb^{-1} + \barU^\top \bfR(z) \barU}^{-1}
    \barU^\top \bfR(z) \baru_k} \rmd z \\
    & =: \frac{1}{2\bar{\pi} \rmi} \oint_{\Gamma}
    \angles{\bfa (z), \mathbf{M} (z) \bfb(z) } \rmd z.    
\end{split} \label{tmp:green-func-representation}
\end{align}

We next analyze the vectors $\bfa(z), \bfb(z) \in \bbc^m$ and the matrix $\mathbf{M}(z) \in \bbr^{m \times m}$ for $z \in \Gamma$. The goal is to separate, for each quantity, a leading contribution from a smaller remainder. More precisely, we split
\begin{equation*}
    \mathbf{M} (z) = \mathbf{M}_0 (z) + \mathbf{M}_1 (z) ,
    \qquad
    \bfa (z) = \bfa_0 (z) + \bfa_1 (z) ,
    \qquad
    \bfb (z) = \bfb_0 (z) + \bfb_1 (z) ,
\end{equation*}
where the leading terms and remainders satisfy
\begin{subequations} \label{tmp:split-MBA}
\begin{alignat}{3}
    \mathbf{M}_0 (z)
    & = {z\barLamb} (z-\barLamb)^{-1}
    = \bigO{\abs{z}^2 / \barDelta_r},
    &\qquad \qquad
    \mathbf{M}_1 (z)
    & = \bigO[\big]{n^\varepsilon \abs{z}^2 \delta (\abs{z}) / \barDelta_r^2} 
    \label{tmp:split-bfM} \\
    \bfa_0 (z)
    & = - \barU^\top \barK \bfe_i / z
    = \bigO{n^{1/2} / \abs{z}},
    &\qquad \qquad
    \bfa_1 (z)
    & = \bigO[\big]{ n^{1/2 + \varepsilon} \delta (\abs{z}) / \abs{z}^2 } 
    \label{tmp:split-bfa} \\
    \bfb_0 (z)
    & = - \bfe_k / z
    = \bigO{1 / \abs{z}},
    &\qquad \qquad
    \bfb_1 (z)
    & = \bigO[\big]{ n^{\varepsilon} \delta (\abs{z}) / \abs{z}^2 }
    \label{tmp:split-bfb}.
\end{alignat} 
\end{subequations}
Since $\abs{z} \geq \barlamb_{r-}$ for $z \in \Gamma$, the estimate \eqref{tmp:order-delta} implies that the upper bound for $\mathbf{M}_1(z)$ is smaller than the corresponding upper bound displayed for $\mathbf{M}_0(z)$. In particular, $\mathbf{M}(z)$ is controlled by the upper bound displayed for $\mathbf{M}_0(z)$. The same observation applies to $\bfa(z)$ and $\bfb(z)$.

We now derive the bounds in \eqref{tmp:split-MBA}. We begin with $\mathbf{M}(z)$. Applying the resolvent identity gives
\begin{align}
\begin{split}
    \mathbf{M} (z)
    & = \braks[\big]{ \barLamb^{-1} - \bfI / z
    + \barU^\top \bfR(z) \barU + \bfI / z }^{-1} \\
    & = \pars[\big]{\barLamb^{-1} - \bfI / z}^{-1}
    + \mathbf{M} (z)
    \braks[\big]{ \barU^\top \bfR(z) \barU + \bfI / z }
    \pars[\big]{\barLamb^{-1} - \bfI / z}^{-1}    
    =: \mathbf{M}_0 (z) + \mathbf{M}_1 (z).
\end{split} \label{tmp:resolvent-expand-invLamb}
\end{align}
Here
\begin{equation*}
    \norm{\mathbf{M}_0 (z)}
    = \norm{{z\barLamb} (z-\barLamb)^{-1}}
    \leq \abs{z} + \abs{z}^2 \norm{(z-\barLamb)^{-1}}
    \lesssim \abs{z} + \abs{z}^2 / \barDelta_r,
    \qfor
    z \in \Gamma.
\end{equation*}
Indeed, by the construction of $\Gamma$, we have $\abs{z - \barlamb_k} \geq \barDelta_r / 4$ for every $k \in \dbraks{r}$. On the other hand, since $\barlamb_r - \barlamb_{r+1} \geq \barDelta_r$, we also have $\abs{z - \barlamb_k} \geq 3\barDelta_r / 4$ for every $k \in \dbraks{m} \backslash \dbraks{r}$. Hence $\norm{(z-\barLamb)^{-1}} \lesssim 1 / \barDelta_r$. Since $\abs{z} \gtrsim \barDelta_r$ on $\Gamma$, the first bound in \eqref{tmp:split-bfM} follows. Next, taking norms in \eqref{tmp:resolvent-expand-invLamb} and using the bound on $\mathbf{M}_0(z)$ together with \eqref{tmp:URU-without-first}, we obtain
\begin{equation*}
    \norm{\mathbf{M} (z)}
    \lesssim \abs{z}^2 / \barDelta_r
    + \pars[\big]{ n^\varepsilon \delta (\abs{z}) / \barDelta_r }
    \cdot \norm{\mathbf{M} (z)},
    \qfor 
    z \in \Gamma.
\end{equation*}
By \eqref{tmp:order-delta}, the prefactor of $\norm{\mathbf{M}(z)}$ on the r.h.s. is bounded by $n^{-2\varepsilon}$. Therefore, 
\begin{equation*}
    \mathbf{M} (z) = \bigO{\abs{z}^2 / \barDelta_r}.
\end{equation*}
Substituting this bound back into the definition of $\mathbf{M}_1(z)$ in \eqref{tmp:resolvent-expand-invLamb}, we get
\begin{equation*}
    \mathbf{M}_1 (z)
    = \mathbf{M} (z)
    \braks[\big]{ \barU^\top \bfR(z) \barU + \bfI / z }
    \mathbf{M}_0 (z)
    = \bigO[\big]{n^\varepsilon \abs{z}^2 \delta (\abs{z}) / \barDelta_r^2},
    \qfor 
    z \in \Gamma.
\end{equation*}
This proves the second estimate in \eqref{tmp:split-bfM}. Next, \eqref{tmp:split-bfa} and \eqref{tmp:split-bfb} can be obtained in the same way as \eqref{tmp:URU-without-first}, by applying the resolvent expansion of $\bfR(z)$ in \eqref{tmp:resolvent-expand-R}. Indeed, \eqref{tmp:split-bfb} is simply the columnwise version of \eqref{tmp:URU-without-first}, specialized to the $k$-th column of $\barU^\top \bfR(z) \barU$. Hence, it remains only to discuss $\barU^\top \bfR(z) \barK \bfe_i$. The argument is basically the same, except that $\baru_k$ is replaced by $\barK \bfe_i$. Therefore, we use the more general estimate \eqref{eqn:average-dotK-barK}, mentioned in the proof of Lemma \ref{lemma:u-dotK-barK-u}, to obtain
\begin{equation*}
    \angles{\baru_\ell, (\ddotK - \barK) \barK \bfe_i} 
    \prec \sqrt{n} \rho \cdot \norm{\barK \bfe_i}
    \prec n \rho.
\end{equation*}
Here we used the elementary estimates $\norm{\barK \bfe_i}_{\infty} \asymp 1$ and $\norm{\barK \bfe_i} \asymp \sqrt{n}$, which follow directly from the definition of $\barK$ in \eqref{def:bar-Kij}. Hence the condition required for \eqref{eqn:average-dotK-barK} is satisfied after a simple normalization. This verifies the remaining estimates in \eqref{tmp:split-MBA}.

Applying the estimates in \eqref{tmp:split-MBA} to the integrand in \eqref{tmp:green-func-representation}, we obtain
\begin{equation*}
    \angles{\bfe_i, \barK \ddotU_r \ddotfkQ_r \bfe_k}
    = \frac{1}{2\bar{\pi} \rmi} \oint_{\Gamma}
    \angles{\bfa_0 (z), \mathbf{M}_0 (z) \bfb_0 (z) } \rmd z
    + \frac{1}{2\bar{\pi} \rmi} \oint_{\Gamma}
    \bigO[\big]{ n^{1/2 + \varepsilon} \delta (\abs{z}) / \barDelta_r^2 }
    \rmd z.
\end{equation*}
Since $\abs{z} \gtrsim \barDelta_r$ on $\Gamma$ and the length of $\Gamma$ is $\bigO{\barDelta_r}$, the error term on the r.h.s. is bounded by $\bigO[\big]{ n^{1/2+\varepsilon} \delta(\barDelta_r) / \barDelta_r }$. It remains to evaluate the leading term. By the definitions of $\bfa_0(z)$, $\mathbf{M}_0(z)$, and $\bfb_0(z)$,
\begin{equation*}
    \frac{1}{2\bar{\pi} \rmi} \oint_{\Gamma}
    \angles{\bfa_0 (z), \mathbf{M}_0 (z) \bfb_0 (z) } \rmd z
    = \frac{1}{2\bar{\pi} \rmi} \oint_{\Gamma}
    \frac{\barlamb_k \angles{\bfe_i, \barK \barU \bfe_k }}
    {z (z - \barlamb_k)} \rmd z
    = \angles{\bfe_i, \barK \barU \bfe_k },
\end{equation*}
where we used that $\Gamma$ encloses $\barlamb_k$. Combining the preceding estimates, we have shown that, on the event $\Xi$,
\begin{equation*}
    \abs{\angles{\bfe_i, (\barK \ddotU_r \ddotfkQ_r - \barK \barU) \bfe_k}}
    \leq n^{\varepsilon} ( n \rho / \barDelta_r
    + n^{5/2} \rho^2 / \barDelta_r^2 ).
\end{equation*}
This completes the proof of Lemma \ref{lemma:inner-prod-dotU-barU}.
\end{proof}

\section{Perturbation analysis for normalized Laplacian matrices}
\label{sec:proof-laplacian}

In analogy with the kernel setting, we introduce the following intermediate approximation to the normalized Laplacian matrix:
\begin{equation}
    \ddotL := n \ddotD^{-1/2}\ddotK\ddotD^{-1/2} 
    \qwhere
    \ddotD = \diag (\ddot{D}_{1}, \cdots, \ddot{D}_{n})
    \qwith
    \ddot{D}_{i} = \sum\nolimits_{j=1}^n \ddot{K}_{ij}.
    \label{def:dotL-dotD}
\end{equation}
Denote the spectral decomposition of $\ddotL$ by
\begin{equation*}
    \ddotL = \ddotV \ddotGamma \ddotV^\top
    = \sum\nolimits_{i=1}^n \ddotgamma_i \ddotv_i \ddotv_i^\top.
\end{equation*}
Fix $r\in\dbraks{m}$. As before, we focus on the leading $r$ eigenvalues and their associated eigenvectors, denoted by
\begin{equation*}
    \ddotGamma_r = \diag (\ddotgamma_{1}, \cdots, \ddotgamma_{r})
    \qand
    \ddotV_r = [\ddotv_{1}, \cdots, \ddotv_{r}].
\end{equation*}
We next establish perturbation results for the normalized Laplacian matrices that parallel their kernel counterparts in Propositions \ref{prop:norm-K-barK-dotK}--\ref{prop:evec-two-stage}. The resulting convergence rates are directly analogous to those in the kernel setting, with the eigengap $\Delta_r(\barLamb)$ replaced by its normalized-Laplacian counterpart $\Delta_r(\barGamma)$.

\begin{proposition}[two-stage norm bounds] 
\label{prop:norm-L-barL-hatL}
Suppose that Assumptions \ref{assump:multi-scale}, \ref{assump:concentration}, \ref{assump:eigengap-Laplacian} and \ref{assump:tech-const} hold. Let $\bfL$, $\ddotL$ and $\barL$ be normalized Laplacian matrices constructed using an admissible collection of bandwidths $\{ \hbar_t \}_{t=1}^T$. Then,
\begin{subequations}
\begin{align}
    \norm{\ddotL - \barL} 
    & \prec 
    n \rho, 
    \label{bound:L-dominant} \\
    \norm{\bfL - \ddotL} 
    & \prec
    1 + \sqrt{n}\phi + n\rho^2.
    \label{bound:L-subleading}
\end{align}
\end{subequations}
\end{proposition}

\begin{proposition}[eigenvalue perturbation] 
\label{prop:spike-L}
Under the same setup as in Proposition \ref{prop:norm-L-barL-hatL},
\begin{equation}
    \norm{\ddotGamma_r - \barGamma_r}    
    \prec 
    \sqrt{n}\rho + \frac{n^2\rho^2}{\Delta_r (\barGamma)}.
\end{equation}
\end{proposition}

For the eigenvector perturbation analysis, we use the two-stage decomposition
\begin{equation*}
    \bfV_r\fkS_{r}^{\bfv} - \barV_r
    = \pars{\bfV_r\fkS_{r}^{\bfv}-\ddotV_r\ddotfkS_{r}^{\bfv}}
    + \pars{\ddotV_r\ddotfkS_{r}^{\bfv}- \barV_r},
\end{equation*}
where $\ddotfkS_r^{\bfv} = \sgn(\ddotV_r^\top \barV_r)$ is the orthogonal alignment matrix between $\ddotV_r$ and $\barV_r$.

\begin{proposition}[two-stage $\ell_{2, \infty}$ eigenvector perturbation]
\label{prop:L-eigv-two-step}
Under the same setup as in Proposition \ref{prop:norm-L-barL-hatL},
\begin{subequations}
\begin{align}
    \norm{\ddotV_r \ddotfkS_{r}^{\bfv} - \barV_r}_{2,\infty}
    & \prec 
    \frac{\sqrt{n} \rho}{\Delta_r (\barGamma)}
    + \frac{n \rho}{\Delta_r (\barGamma)^2} 
    + \frac{n^{5/2} \rho^2}{\Delta_r (\barGamma)^3},
    \label{bound:entrywise-ev-L-dominant} \\
    \norm{\bfV_r \fkS_{r}^{\bfv} -\ddotV_r\ddotfkS_{r}^{\bfv}}_{2,\infty}
    & \prec 
    \frac{\sqrt{n} + n \phi + n^{3/2} \rho^2}{\Delta_r (\barGamma)^2}.
    \label{bound:entrywise-ev-L-subleading} 
\end{align}    
\end{subequations}
\end{proposition}

In the remainder of this section, we adapt the arguments of Section \ref{sec:proof-kernel} to establish Propositions \ref{prop:norm-L-barL-hatL}--\ref{prop:L-eigv-two-step}.

\subsection{Proof of Proposition \ref{prop:norm-L-barL-hatL}}

We first prove the following perturbation bounds for the degree matrices.

\begin{lemma}[degree matrices]
\label{lemma:degree}
Under the same setup as in Proposition \ref{prop:norm-L-barL-hatL}, 
\begin{equation}
    \norm{\ddotD - \barD} \prec n \rho,
    \qquad
    \norm{\bfD - \ddotD} \prec 1 + n \rho^2,
    \qquad
    \norm{\bfD - \barD} \prec 1 + n \rho.
    \label{eqn:degree-diff}
\end{equation}
\end{lemma}

\begin{proof}[Proof of Lemma \ref{lemma:degree}]
It suffices to prove the first two estimates in \eqref{eqn:degree-diff}; the third then follows from the triangle inequality. The first estimate is an immediate consequence of the entrywise bound in \eqref{bound-entry-Kdot-Kbar},
\begin{equation*}
    \abs{\ddot{D}_{i} - \bar{D}_{i}}
    \leq \sum\nolimits_{j=1}^n \abs{\ddot{K}_{ij} - \bar{K}_{ij}} 
    \prec n \rho.
\end{equation*}
We next prove the estimate for $\norm{\bfD - \ddotD}$. Using the decomposition \eqref{eqn:diff-K-Kdot}, we write
\begin{equation*}
    D_{i} - \ddot{D}_{i}
    = \sum\nolimits_{j=1}^n (K_{ij} - \ddot{K}_{ij})
    = - \ddot{K}_{ii}
    + \sum\nolimits_{j \neq i} (Q_{ij} W_{ij} + R_{ij} + R_{ij}' ).
\end{equation*}
By the entrywise bounds $\abs{R_{ij}}+\abs{R_{ij}'} \prec \rho^2$ and $\abs{\ddot{K}_{ii}} \prec 1$, established in the proof of \eqref{bound:bfK-ddotK}, we obtain
\begin{equation*}
    \abs{D_{i} - \ddot{D}_{i}} \prec 1 + n \rho^2 + 
    \abs[\Big]{\sum\nolimits_{j \neq i} Q_{ij} W_{ij}}.
\end{equation*}
It remains to control the last term. By definition, $W_{ij} = \angles{\bfA_{\bar{\pi} (i)} \bfz_i, \bfA_{\bar{\pi} (j)} \bfz_j}$, and hence
\begin{equation*}
    \sum\nolimits_{j \neq i} Q_{ij} W_{ij} 
    = \sum\nolimits_{j \neq i} Q_{ij} \angles{\bfz_i, \bfOmega_{\bar{\pi}(i) \bar{\pi}(j)} \bfz_j}
    =: \angles{\bfz_i, \bfa_i}
    \qwhere
    \bfa_i
    = \sum\nolimits_{j \neq i} 
    Q_{ij} 
    \bfOmega_{\bar{\pi}(i) \bar{\pi}(j)} \bfz_j.
\end{equation*}
Here, by the Hanson--Wright estimate \eqref{eqn:long-hanson-wright} and the bound $\abs{Q_{ij}} \lesssim 1/\theta_{\bar{\pi}(i)\bar{\pi}(j)}$, we have
\begin{equation*}
    \norm{\bfa_i}^2
    \prec \sum\nolimits_{j \not= i} 
    \abs{Q_{ij}}^2
    \norm{\bfOmega_{\bar{\pi}(i) \bar{\pi}(j)}}^2_\Fnorm
    \lesssim n \rho^2.
\end{equation*}
Note that $\bfa_i$ is independent of $\bfz_i$. Hence, applying the Bernstein estimate \eqref{eqn:bernstein}, conditionally on $\bfa_i$, yields
\begin{equation*}
    \abs[\Big]{\sum\nolimits_{j \neq i} Q_{ij} W_{ij}}
    = \abs{ \angles{\bfz_i, \bfa_i} }
    \prec \norm{\bfa_i}
    \prec \sqrt{n} \rho.
\end{equation*}
Combining the preceding estimates, and using the elementary inequality $\sqrt{n}\rho \lesssim 1+n\rho^2$, we arrive at the second estimate in \eqref{eqn:degree-diff}, and thus the proof is complete.
\end{proof}

Equipped with Lemma \ref{lemma:degree}, we are now prepared to prove Proposition \ref{prop:norm-L-barL-hatL}.

\begin{proof}[Proof of Proposition \ref{prop:norm-L-barL-hatL}]
By Assumption \ref{assump:multi-scale}, we have $\caS_{k\ell}(0) \neq \varnothing$. Also recall $\alpha_t \asymp 1$ from Assumption \ref{assump:tech-const}. Hence, we have $\barK_{ij} \gtrsim 1$ uniformly in $i,j$. Together with Lemma \ref{lemma:degree}, this implies that, with high probability,
\begin{equation*}
    \min\nolimits_{i \in \dbraks{n}} \bar{D}_{i}
    \asymp \min\nolimits_{i \in \dbraks{n}} \ddot{D}_{i}
    \asymp \min\nolimits_{i \in \dbraks{n}} {D}_{i}
    \asymp n.
\end{equation*}
Consequently,
\begin{equation}
    \norm{\barD^{-1}} \prec {1}/{n},
    \qquad
    \norm{\ddotD^{-1}} \prec {1}/{n},
    \qquad
    \norm{\bfD^{-1}} \prec {1}/{n}.
    \label{tmp:lower-D}
\end{equation}

We first prove \eqref{bound:L-dominant}. Note that we have the following decomposition,
\begin{align} \label{eqn:decomp-ddotL-barL}
\begin{split}
    \ddotL - \barL
    & = n \ddotD^{-1/2} \ddotK \ddotD^{-1/2} 
    - n \barD^{-1/2} \barK \barD^{-1/2}\\
    & = n (\ddotD^{-1/2} - \barD^{-1/2}) \ddotK \ddotD^{-1/2}
    + n \barD^{-1/2} \ddotK (\ddotD^{-1/2} - \barD^{-1/2})
    + n \barD^{-1/2} (\ddotK - \barK) \barD^{-1/2}.
\end{split}
\end{align}
Using the first estimate in \eqref{eqn:degree-diff} together with \eqref{tmp:lower-D}, we obtain
\begin{equation}
    \norm{\ddotD^{-1/2} - \barD^{-1/2}}
    = \max_{i \in \dbraks{n}} 
    \frac{\abs{\ddot{D}_{i} - \bar{D}_{i}}}
    {\sqrt{\ddot{D}_{i} \bar{D}_{i}}
    \pars[\big]{\sqrt{\ddot{D}_{i}} + \sqrt{\bar{D}_{i}}} }
    \prec \rho / \sqrt{n}.
    \label{tmp:1-over-sqrt-dotD}
\end{equation}
It follows that the first two terms on the r.h.s. of \eqref{eqn:decomp-ddotL-barL} can be controlled using
\begin{equation*}
    \norm{(\ddotD^{-1/2} - \barD^{-1/2}) \ddotK \ddotD^{-1/2}}
    \prec \rho
    \qand
    \norm{\barD^{-1/2} \ddotK (\ddotD^{-1/2} - \barD^{-1/2})}
    \prec \rho .   
\end{equation*}
For the last term, \eqref{bound:ddotK-barK} and \eqref{tmp:lower-D} lead to
\begin{equation*}
    \norm{\barD^{-1/2} (\ddotK - \barK) \barD^{-1/2}}
    \leq \norm{\barD^{-1/2}} 
    \cdot \norm{\ddotK - \barK} 
    \cdot \norm{\barD^{-1/2}}
    \prec \rho.
\end{equation*}
Combining these estimates proves \eqref{bound:L-dominant}. We next prove \eqref{bound:L-subleading}. Similarly,
\begin{equation}
    \bfL - \ddotL
    = n (\bfD^{-1/2} - \ddotD^{-1/2}) \bfK \bfD^{-1/2}
    + n \ddotD^{-1/2} \bfK (\bfD^{-1/2} - \ddotD^{-1/2})
    + n \ddotD^{-1/2} (\bfK - \ddotK) \ddotD^{-1/2} .   
    \label{eqn:decomp-bfL-ddotL}
\end{equation}
Combining the second estimate in \eqref{eqn:degree-diff} with \eqref{tmp:lower-D}, we obtain
\begin{equation}
    \norm{\bfD^{-1/2} - \ddotD^{-1/2}}
    = \max_{i \in \dbraks{n}} \frac{\abs{D_{i} - \ddot{D}_{i}}}
    {\sqrt{D_{i} \ddot{D}_{i}}
    \pars[\big]{\sqrt{D_{i}} + \sqrt{\ddot{D}_{i}}} }
    \prec 1 / n^{3/2}
    + \rho^2 / \sqrt{n}.
    \label{tmp:1-over-sqrt-bfD}
\end{equation}
Therefore, the first two terms on the r.h.s. of \eqref{eqn:decomp-bfL-ddotL} can be controlled using
\begin{equation*}
    \norm{(\bfD^{-1/2} - \ddotD^{-1/2}) \bfK \bfD^{-1/2}}
    \prec {1}/{n} + \rho^2
    \qand
    \norm{\ddotD^{-1/2} \bfK (\bfD^{-1/2} - \ddotD^{-1/2})}
    \prec {1}/{n} + \rho^2.
\end{equation*}
For the last term, we invoke \eqref{bound:bfK-ddotK} to get
\begin{equation*}
    \norm{\ddotD^{-1/2} (\bfK - \ddotK) \ddotD^{-1/2}}
    \leq \norm{\ddotD^{-1/2}} \cdot \norm{\bfK - \ddotK} \cdot \norm{\ddotD^{-1/2}}
    \prec {1}/{n} + {\phi}/{\sqrt{n}} + \rho^2.
\end{equation*}
Summarizing the preceding estimates proves \eqref{bound:L-subleading}.
\end{proof}

\subsection{Proof of Proposition \ref{prop:spike-L}}
\label{subsec:spike-L}


The proof of Proposition \ref{prop:spike-L} largely parallels that of Proposition \ref{prop:dotlamb-barlamb}. The main difference is that, in the normalized Laplacian setting, we need the following analogue of Lemma \ref{lemma:u-dotK-barK-u}.

\begin{lemma} 
\label{lemma:v-hatL-barL-v}
Under the same setup as in Proposition \ref{prop:norm-L-barL-hatL}, 
\begin{equation}
    \norm{\barV_r^\top (\ddotL - \barL) \barV_r}
    \leq \norm{\barV^\top (\ddotL - \barL) \barV}
    \prec \sqrt{n}\rho + n\rho^2.
    \label{bound:v-hatL-barL-v}
\end{equation}    
\end{lemma}

Before proving Lemma \ref{lemma:v-hatL-barL-v}, we comment on a minor technical point arising in the adaptation of the proof of Proposition \ref{prop:dotlamb-barlamb} to its normalized-Laplacian counterpart, Proposition \ref{prop:spike-L}. Compared with \eqref{eqn:u-dotK-barK-u}, the bound \eqref{bound:v-hatL-barL-v} contains an additional term of order $n \rho^2$. Nevertheless, this term is harmless for the subsequent argument. Indeed, after accounting for the degree normalization, using \eqref{bound:v-hatL-barL-v} in the proof of Proposition \ref{prop:spike-L} has the same effect as replacing \eqref{eqn:u-dotK-barK-u} by the weaker estimate
\begin{equation}
    \norm{\barU^\top (\ddotK-\barK)\barU}
    \prec
    \sqrt{n}\rho+n\rho^2.
    \label{tmp:loose}
\end{equation}
In the proof of Proposition \ref{prop:dotlamb-barlamb}, the estimate \eqref{eqn:u-dotK-barK-u} is used only in the derivation of \eqref{tmp:URU-without-first}, where it yields
\begin{equation*} 
    \norm{\barU^\top (\ddotK - \barK) \barU / z^2} 
    \prec \sqrt{n} \rho / \abs{z}^2 
    \lesssim \delta(\abs{z}) / \abs{z}^2 .
\end{equation*}
Replacing \eqref{eqn:u-dotK-barK-u} by \eqref{tmp:loose} does not alter
the resulting bound. Indeed, by the definition of $\delta(x)$ in
\eqref{def:func-delta} and the fact that the spectral parameters
considered in the argument satisfy $\abs{z}\lesssim n$, we have
\begin{equation*} 
    \norm{\barU^\top (\ddotK - \barK) \barU / z^2} 
    \prec (\sqrt{n} \rho + n \rho^2) / \abs{z}^2 
    \lesssim \sqrt{n} \rho / \abs{z}^2 + n^2 \rho^2 / \abs{z}^3 
    = \delta(\abs{z}) / \abs{z}^2 .
\end{equation*}
Consequently, the proof of Proposition \ref{prop:dotlamb-barlamb} extends to Proposition \ref{prop:spike-L} with only this minor modification.

\begin{proof}[Proof of Lemma \ref{lemma:v-hatL-barL-v}]
Using \eqref{tmp:lower-D} and \eqref{tmp:1-over-sqrt-dotD}, we have
\begin{equation*}
    \ddotD^{-1/2} - \barD^{-1/2}
    = - (\ddotD - \barD) \ddotD^{-1/2} \barD^{-1/2} (\ddotD^{1/2} + \barD^{1/2})^{-1}
    = - (\ddotD - \barD) \barD^{-3/2} / 2
    + \oprec{\rho^2 / \sqrt{n}}
\end{equation*}
where the error is interpreted in terms of operator norm. Also recall \eqref{bound:ddotK-barK}. We combine the preceding expansion with the decomposition \eqref{eqn:decomp-ddotL-barL}. Replacing every occurrence of $\ddotD$ outside the difference $\ddotD-\barD$ by its barred counterpart $\barD$ yields
\begin{equation*}
    \ddotL - \barL
    = - \frac{n}{2} (\ddotD - \barD) \barD^{-3/2} \barK \barD^{-1/2}
    - \frac{n}{2} \barD^{-1/2} \barK \barD^{-3/2} (\ddotD - \barD)
    + n \barD^{-1/2} (\ddotK - \barK) \barD^{-1/2} 
    + \oprec{n \rho^2}.
\end{equation*}
For $k \in \dbraks{m}$, define the deterministic vectors
\begin{equation*}
    \barv_{*k} := \barD^{-1/2} \barv_k
    \qand
    \barv_{**k} := \barD^{-3/2} \barK \barD^{-1/2} \barv_k.
\end{equation*}
Then, for $k, \ell \in \dbraks{m}$, the $(k,\ell)$-th entry of $\barV^\top(\ddotL-\barL)\barV$ satisfies
\begin{equation}
    \angles{\barv_k, (\ddotL - \barL) \barv_\ell}
    = - \frac{n}{2} \angles{\barv_k, (\ddotD - \barD) \barv_{**\ell}} 
    - \frac{n}{2} \angles{\barv_{**k}, (\ddotD - \barD) \barv_\ell} 
    + n \angles{\barv_{*k}, (\ddotK - \barK) \barv_{*\ell}} 
    + \oprec{n \rho^2}.
    \label{eqn:decomp-u-ddotL-barL-u}
\end{equation}
Recall the piecewise constant structure of $\barv_k$ illustrated in \eqref{eqn:barv-piecewise}. As a result, we have
\begin{equation}
    \norm{\barv_k}_{\infty} 
    \asymp 1 / \sqrt{n},
    \qquad
    \norm{\barv_{*k}}_{\infty} 
    \lesssim 1 / n,
    \qquad
    \norm{\barv_{**k}}_{\infty} 
    \asymp 1 / n^{3/2}.
    \label{tmp:vk-infnity}
\end{equation}
We now estimate the three terms on the r.h.s. of \eqref{eqn:decomp-u-ddotL-barL-u}. The last term is controlled by combining the second bound in \eqref{tmp:vk-infnity} with the general averaging estimate \eqref{eqn:average-dotK-barK}, obtained in the proof of Lemma \ref{lemma:u-dotK-barK-u}. Specifically,
\begin{equation*}
    \angles{\barv_{*k}, (\ddotK - \barK) \barv_{*\ell}} = \oprec{\rho / \sqrt{n}}.
\end{equation*}
For the first term, we use the definition of the degree matrices and again invoke \eqref{eqn:average-dotK-barK}:
\begin{align*}
    \angles{\barv_k, (\ddotD - \barD) \barv_{**\ell}}
    & = \sum\nolimits_{i=1}^n (\ddot{D}_{i}- \bar{D}_{i}) 
    \angles{\barv_k, \bfe_i}
    \angles{\barv_{**\ell}, \bfe_i} \\
    & = \sum\nolimits_{i,j=1}^n (\ddot{K}_{ij}- \bar{K}_{ij}) 
    \angles{\barv_k, \bfe_i}
    \angles{\barv_{**\ell}, \bfe_i}
    = \angles{\barv_k \circ \barv_{**\ell}, (\ddotK - \barK) \bfone}     
    = \oprec{ {\rho}/{\sqrt{n}} } ,
\end{align*}
where we used $\norm{\barv_k \circ \barv_{**\ell}}_{\infty} \lesssim 1 / n^2$, which follows from \eqref{tmp:vk-infnity}. The second term in \eqref{eqn:decomp-u-ddotL-barL-u} is treated in the same way. This completes the proof. 

We finally note that, as in \eqref{eqn:average-dotK-barK}, the above argument extends directly to general deterministic vectors whose $\ell_\infty$-norms are controlled by $1/\sqrt{n}$. More precisely,
\begin{equation}
    \abs{ \angles{\bfs_1, (\ddotL - \barL) \bfs_2} }
    \prec \sqrt{n}\rho + n\rho^2,
    \quad \text{ if } \quad
    \norm{\bfs_1}_{\infty} \vee \norm{\bfs_2}_{\infty} \lesssim 1 / \sqrt{n}.
    \label{eqn:average-dotL-barL}
\end{equation}
The required modifications are straightforward and therefore omitted.
\end{proof}

\subsection{Proof of Proposition \ref{prop:L-eigv-two-step} and Theorem \ref{thm:Linfty-evec-bfL-barL}}

\begin{proof}[Proof of Proposition \ref{prop:L-eigv-two-step}]
The proof here closely parallels that of Proposition \ref{prop:evec-two-stage}, which is given in Sections \ref{subsec:dominant-evec-entrywise} and \ref{subsec:subleading-evec-entrywise}. We therefore describe only the necessary modifications and omit the arguments that remain unchanged. The discussion is divided into two parts, corresponding to the bounds \eqref{bound:entrywise-ev-L-dominant} and \eqref{bound:entrywise-ev-L-subleading}, respectively. To simplify the notation, we introduce the abbreviations
\begin{equation*}
    \barDelta_\gamma \equiv \Delta_r (\barGamma),
    \qquad
    \fkQ_r \equiv \bfV_r^\top \barV_r,
    \qquad
    \ddotfkQ_r \equiv \ddotV_r^\top \barV_r,
    \qquad
    \fkS_r \equiv \sgn(\fkQ_r),
    \qquad
    \ddotfkS_r \equiv \sgn(\ddotfkQ_r).
\end{equation*}

\shortpara{Adaptation for \eqref{bound:entrywise-ev-L-dominant}}:
The bound \eqref{bound:entrywise-ev-L-dominant} follows from the arguments in Section \ref{subsec:dominant-evec-entrywise}, with the corresponding normalized Laplacian objects in place of the kernel-matrix objects. The only point requiring comment is that the two auxiliary inputs, Lemmas \ref{lemma:u-dotK-barK-u} and \ref{lemma:inner-prod-dotU-barU}, together with the row-wise control of $\norm{\ddotK-\barK}_{2,\infty}$ in \eqref{bound:max-row-dotK-barK}, must be replaced by their normalized Laplacian analogues. The analogue of Lemma \ref{lemma:u-dotK-barK-u} is Lemma \ref{lemma:v-hatL-barL-v}, proved in Section \ref{subsec:spike-L}. The analogue of Lemma \ref{lemma:inner-prod-dotU-barU} is the following.

\begin{lemma} 
\label{lemma:inner-prod-dotL-barL}
Under the same setup as in Proposition \ref{prop:norm-L-barL-hatL}, 
\begin{equation}
    \norm{\barL (\ddotV_r \ddotfkQ_{r} - \barV_r)}_{2, \infty}
    \prec n\rho/\barDelta_\gamma + n^{5/2}\rho^2/\barDelta_\gamma^2s.
\end{equation}    
\end{lemma}

This lemma can be proved by following the proof of Lemma \ref{lemma:inner-prod-dotU-barU}. The only additional point is the same as in the proof of Proposition \ref{prop:spike-L}: the extra $n \rho^2$ term arising from Lemma \ref{lemma:v-hatL-barL-v}, or more generally from \eqref{eqn:average-dotL-barL}, does not affect the overall argument. For reference, we remark that the relevant estimates of this type are used in the proof of Lemma \ref{lemma:inner-prod-dotU-barU} only when deriving \eqref{tmp:split-bfa} and \eqref{tmp:split-bfb}.

It remains to provide a bound on $\norm{\ddotL-\barL}_{2,\infty}$ replacing \eqref{bound:max-row-dotK-barK}. We first establish an entrywise estimate for $\ddot{L}_{ij} - \bar{L}_{ij}$. By applying \eqref{tmp:lower-D} and \eqref{tmp:1-over-sqrt-dotD} to the decomposition \eqref{eqn:decomp-ddotL-barL}, we obtain
\begin{equation*}
    \ddot{L}_{ij} - \bar{L}_{ij}
    = n \pars[\big]{\ddot{D}_{i}^{-1/2} - \bar{D}_{i}^{-1/2}} \ddot{K}_{ij} \ddot{D}_j^{-1/2}
    + n \bar{D}_i^{-1/2} \ddot{K}_{ij} \pars[\big]{\ddot{D}_{j}^{-1/2} - \bar{D}_{j}^{-1/2}}
    + n \bar{D}_i^{-1/2} (\ddot{K}_{ij} - \bar{K}_{ij}) \bar{D}_j^{-1/2}
    = \oprec{\rho}.
\end{equation*}
Consequently, 
\begin{equation}
    \norm{\ddotL - \barL}_{2, \infty} 
    = \max_{1 \leq i \leq n} \pars[\Big]{
    \sum\nolimits_{j=1}^n \abs{\ddot{L}_{ij} - \bar{L}_{ij}}^2}^{1/2}
    \prec \sqrt{n} \rho.
\end{equation}
With these replacements in hand, the proof of \eqref{bound:entrywise-ev-L-dominant} follows exactly as in Section \ref{subsec:dominant-evec-entrywise}.

\shortpara{Adaptation for \eqref{bound:entrywise-ev-L-subleading}}: 
The proof of \eqref{bound:entrywise-ev-L-subleading} is largely analogous to that of \eqref{bound:bfU-ddotU}, presented in Section \ref{subsec:subleading-evec-entrywise}. As in \eqref{eqn:decomp-U-Udot}, we decompose
\begin{equation}
    \bfV_r \fkS_r - \ddotV_r \ddotfkS_r
    = \mathrm{I_{\ref{eqn:V-Vdot}}}
    + \mathrm{II_{\ref{eqn:V-Vdot}}}
    + \cdots
    + \mathrm{VII_{\ref{eqn:V-Vdot}}},
    \label{eqn:V-Vdot}
\end{equation}
where the terms on the r.h.s. are defined as in \eqref{eqn:decomp-U-Udot}, with the kernel-matrix objects replaced by their normalized Laplacian counterparts. The first six terms can be treated in the same way as in Section \ref{subsec:subleading-evec-entrywise}. Thus, it remains to discuss the last term,
\begin{equation*}
    \mathrm{VII_{\ref{eqn:V-Vdot}}}
    = (\bfL - \ddotL) \ddotV_r \ddotfkQ_r \ddotGamma_r^{-1}.
\end{equation*}
In the kernel-matrix case, the treatment of $\mathrm{VII_{\ref{eqn:decomp-U-Udot}}}$ relies on the specific structure of $\bfK-\ddotK$. In the present normalized Laplacian setting, the difference $\bfL-\ddotL$ has a more involved form, and hence requires a separate discussion. Using the decomposition \eqref{eqn:decomp-bfL-ddotL}, we write $\mathrm{VII_{\ref{eqn:V-Vdot}}} = \mathrm{VII^1_{\ref{eqn:V-Vdot}}} + \cdots + \mathrm{VII^4_{\ref{eqn:V-Vdot}}}$, where
\begin{align*}
    \mathrm{VII^1_{\ref{eqn:V-Vdot}}}
    & = n (\bfD^{-1/2} - \ddotD^{-1/2}) \bfK \bfD^{-1/2} 
    \ddotV_r \ddotfkQ_{r} \ddotGamma_r^{-1}, \\
    \mathrm{VII^2_{\ref{eqn:V-Vdot}}}
    & = n \ddotD^{-1/2} \bfK (\bfD^{-1/2} - \ddotD^{-1/2}) 
    \ddotV_r \ddotfkQ_{r} \ddotGamma_r^{-1}, \\
    \mathrm{VII^3_{\ref{eqn:V-Vdot}}}
    & = n \ddotD^{-1/2} (\bfK - \ddotK) (\ddotD^{-1/2} - \barD^{-1/2})
    \ddotV_r \ddotfkQ_{r} \ddotGamma_r^{-1}, \\
    \mathrm{VII^4_{\ref{eqn:V-Vdot}}}
    & = n \ddotD^{-1/2} (\bfK - \ddotK) \barD^{-1/2}
    \ddotV_r \ddotfkQ_{r} \ddotGamma_r^{-1}. 
\end{align*}
Following the same argument as in Section \ref{subsec:dominant-evec-entrywise}, we obtain the normalized Laplacian analogues of \eqref{bound:elementary-S-H}, \eqref{bound:l2-dotU-barU}, and \eqref{bound:inv-dot-Lamb}:
\begin{equation*}
    \norm{\ddotfkQ_r} + \norm{\ddotfkS_r}
    + \norm{\ddotfkQ_r^{-1}} + \norm{\ddotfkS_r^{-1}} \prec 1,
    \qquad
    \norm{\ddotV_r \ddotfkQ_r - \barV_r}
    \prec {n \rho} / {\barDelta_{\gamma}},
    \qquad
    \norm{\ddotGamma_r^{-1}} \prec 1 / \barDelta_\gamma.
\end{equation*}
Since the degree matrices are diagonal, these estimates, together with \eqref{tmp:lower-D}, \eqref{tmp:1-over-sqrt-bfD}, and the trivial bound $\norm{\bfK}_{2,\infty}\lesssim \sqrt n$, imply
\begin{align*}
    \norm{\mathrm{VII^1_{\ref{eqn:V-Vdot}}}}_{2, \infty}
    & \leq
    n \norm{\bfD^{-1/2} - \ddotD^{-1/2}}
    \cdot \norm{\bfK}_{2, \infty}
    \cdot \norm{\bfD^{-1/2}\ddotV_r \ddotfkQ_{r} \ddotGamma_r^{-1}}
    \prec 1/(\sqrt{n}\barDelta_\gamma) + \sqrt{n}\rho^2/\barDelta_\gamma, \\
    \norm{\mathrm{VII^2_{\ref{eqn:V-Vdot}}}}_{2, \infty}
    & \leq
    n \norm{\ddotD^{-1/2}}
    \cdot \norm{\bfK}_{2, \infty}
    \cdot \norm{(\bfD^{-1/2} - \ddotD^{-1/2})\ddotV_r \ddotfkQ_{r} \ddotGamma_r^{-1}}
    \prec 1/(\sqrt{n}\barDelta_\gamma) + \sqrt{n}\rho^2/\barDelta_\gamma.
\end{align*}
We next estimate the third term. The analysis in Section \ref{subsec:norm-bound-residual} implies
\begin{equation*}
    \norm{\bfK - \ddotK}_{2,\infty}
    \leq \norm{\diag (\ddotK)}_{2,\infty}
    + \norm{\bfR + \bfR' + \bfQ \circ \caH (\bfW)}_{2,\infty}
    \prec 1 + \sqrt{n} \rho .
\end{equation*}
Combining this estimate with \eqref{tmp:1-over-sqrt-dotD}, we obtain
\begin{equation*}
    \norm{\mathrm{VII^3_{\ref{eqn:V-Vdot}}}}_{2, \infty}
    \leq
    n \norm{\ddotD^{-1/2}}
    \cdot \norm{\bfK - \ddotK}_{2,\infty}
    \cdot \norm{(\ddotD^{-1/2} - \barD^{-1/2}) \ddotV_r \ddotfkQ_{r} \ddotGamma_r^{-1}}
    \prec \rho/\barDelta_\gamma + \sqrt{n}\rho^2/\barDelta_\gamma.
\end{equation*}
It remains to control the fourth term. This part is analogous to the argument leading to \eqref{tmp:sublead-term-7}. Specifically,
\begin{equation*}
    \norm{\mathrm{VII^4_{\ref{eqn:V-Vdot}}}}_{2, \infty}
    \leq 
    n \norm{\ddotD^{-1/2}}
    \cdot \norm{(\bfK - \ddotK) \barD^{-1/2} \ddotV_r \ddotfkQ_{r}}_{2,\infty}
    \cdot \norm{\ddotGamma_r^{-1}}
    \prec 
    n \norm{(\bfK - \ddotK) \barD^{-1/2} \ddotV_r}_{2,\infty} 
    / (\sqrt{n}\barDelta_\gamma).
\end{equation*}
For the matrix on the r.h.s., we further decompose
\begin{align*}
    (\bfK - \ddotK) \barD^{-1/2} \ddotV_r \ddotfkQ_{r} 
    = & ~ \braks{\bfQ \circ \caH (\bfW)} \barD^{-1/2} \barV_r
    + \braks{\bfQ \circ \caH (\bfW)} \barD^{-1/2} (\ddotV_r \ddotfkQ_{r}  - \barV_r) \\
    & - \diag (\ddotK) \barD^{-1/2} \ddotV_r \ddotfkQ_{r} 
    + (\bfR + \bfR') \barD^{-1/2} \ddotV_r \ddotfkQ_{r}.     
\end{align*}
Except for the first term, the other three terms are treated exactly as in the derivation of \eqref{tmp:sublead-term-7}. Hence
\begin{equation*}
    \norm{(\bfK - \ddotK) \barD^{-1/2} \ddotV_r}_{2, \infty}
    \leq \norm{\braks{\bfQ \circ \caH (\bfW)} \barD^{-1/2} \barV_r}_{2, \infty}
    + (n \rho^2 / \barDelta_\gamma
    + 1 / n 
    + \rho^2).
\end{equation*}
Note that Lemma \ref{lemma:QW-barU-infnorm} is not directly applicable when controlling this remaining term. However, as noted in the proof of Lemma \ref{lemma:QW-barU-infnorm}, we actually have the more general estimate \eqref{eqn:general-QW-s}, which implies
\begin{equation}
    \norm{\braks{\bfQ \circ \caH (\bfW)} 
    \barD^{-1/2} \barV_r}_{2, \infty}
    \prec \rho / \sqrt{n} ,
\end{equation}
and therefore
\begin{equation*}
    \norm{\mathrm{VII^4_{\ref{eqn:V-Vdot}}}}_{2, \infty}
    \prec
    1/(\sqrt{n}\barDelta_\gamma) + n^{3/2}\rho^2/\barDelta_\gamma^2.
\end{equation*}

Summarizing the preceding estimates, we arrive at
\begin{equation*}
    \norm{\mathrm{VII_{\ref{eqn:V-Vdot}}}}_{2, \infty}
    \prec
    1/(\sqrt{n}\barDelta_\gamma) + n^{3/2}\rho^2/\barDelta_\gamma^2.
\end{equation*}
Thus, although $\bfL-\ddotL$ has a more complicated structure than $\bfK-\ddotK$, the resulting bound matches its kernel-matrix analogue \eqref{tmp:sublead-term-7}. With this replacement in hand, the proof of \eqref{bound:entrywise-ev-L-subleading} follows from the proof of \eqref{bound:bfU-ddotU}, after replacing the kernel-matrix objects by their normalized Laplacian counterparts.
\end{proof}

\begin{proof}[Proof of Theorem \ref{thm:Linfty-evec-bfL-barL}]
The bound \eqref{eqn:rowwise-L-without-eval} follows directly from Proposition \ref{prop:L-eigv-two-step} and the triangle inequality. The bound \eqref{eqn:rowwise-L-with-eval}, which additionally incorporates the eigenvalue rescaling, follows from an argument entirely analogous to that in Section \ref{subsec:score-bound}. As before, it suffices to replace the kernel-matrix quantities there by their normalized-Laplacian counterparts. We therefore omit the details for simplicity.
\end{proof}

\section{Analysis of the clustering algorithm}
\label{sec:proof-misclass}

This section contains the technical proofs of the theoretical results stated in Section \ref{sec:KSC-algorithm}. In Section \ref{subsec:adaptive-bandwidths}, we prove Lemma \ref{lemma:bandwidth-regularity} and Corollary \ref{coro:pertur-adaptive}, which characterize the properties of the data-adaptive bandwidths and the associated kernel matrices. In Section \ref{subsec:misclass-rate}, we prove Theorem \ref{thm:zero-misclass}. The proof of this exact recovery result relies on a technical result from \cite{abbeLpTheoryPCA2022}.

\subsection{Data-adaptive bandwidths}
\label{subsec:adaptive-bandwidths}

\begin{proof}[Proof of Lemma \ref{lemma:bandwidth-regularity}]
We first verify the admissibility of the deterministic bandwidths $( \hbar_t^* )_{t=1}^T$. Since $\omega_t \in (0,1)$, each quantile $\hbar_t^* = \bar{G}^{-1}(\omega_t)$ coincides with one of the deterministic proxies $\theta_{k\ell}$. Therefore, by Assumption \ref{assump:multi-scale}, the first condition in \eqref{cond:admissible-bandwidth} is satisfied. Next, fix $s \in \dbraks{T_{\ref{assump:multi-scale}}}$. By Assumption \ref{assump:quantile}, there exists some $t \in \dbraks{T}$ such that $\omega_t \in \pars[\big]{{N_{< s}} / {N}, {(N_{< s} + N_s)} / {N}}$. Recalling the definition of $\bar{G}$ and using Assumption \ref{assump:multi-scale}, we then obtain
\begin{equation*}
    \hbar_t^* = \bar{G}^{-1} (\omega_t)
    \in \{ \theta_{k \ell}: \theta_{k \ell} \in
    [c_{\ref{assump:multi-scale}} \vartheta_s, 
    \vartheta_s / c_{\ref{assump:multi-scale}}] \}.
\end{equation*}
Hence, the second condition in \eqref{cond:admissible-bandwidth} also holds. This proves the admissibility of $( \hbar_t^* )_{t=1}^T$.

It remains to prove \eqref{eqn:close-admissible}. Fix arbitrary $\varepsilon \in (0,c/2)$ and $C>0$. By Lemma \ref{lemma:xi-control}, there exists an event $\Xi$ satisfying $\bbp(\Xi) \geq 1-n^{-C}$ such that, on $\Xi$, the following holds for every $i \in \caC_k$ and $j \in \caC_\ell$ with $i \neq j$,
\begin{equation}
    \norm{\bfx_i - \bfx_j}^2 
    \in \braks[\big]{
    \theta_{k \ell} - n^{\varepsilon/2} \psi_{k \ell},
    \theta_{k \ell} + n^{\varepsilon/2} \psi_{k \ell} }
    \subset
    \braks[\big]{
    (1 - n^{\varepsilon/2} \rho) \theta_{k \ell},
    (1 + n^{\varepsilon/2} \rho) \theta_{k \ell} }.
    \label{eqn:distant-event}
\end{equation}
In the remainder of the proof, we work on the event $\Xi$. Set $\eta := n^{\varepsilon/2}\rho$. We claim that, for all $t \in \dbraks{T}$,
\begin{equation}
    (1-2\eta)\hbar_t^* \leq h_t \leq (1+\eta)\hbar_t^*.
    \label{tmp:ht-upper-lower}
\end{equation}
Fix $t \in \dbraks{T}$. For every $(k,\ell)$ satisfying $\theta_{k\ell} \leq \hbar_t^*$, the upper bound in \eqref{eqn:distant-event} implies that $\norm{\bfx_i-\bfx_j}^2 \leq (1+\eta)\hbar_t^*$ for all $i \in \caC_k$ and $j \in \caC_\ell$ with $i\neq j$. Thus, all pairwise distances whose associated deterministic proxies do not exceed $\hbar_t^*$ are counted by $G((1+\eta)\hbar_t^*)$. Consequently, we have
\begin{equation*}
    G((1+\eta)\hbar_t^*)\geq \bar{G}(\hbar_t^*)\geq \omega_t
    \quad \text{ and thus } \quad
    h_t \leq (1+\eta)\hbar_t^*.
\end{equation*}
Conversely, suppose that $\norm{\bfx_i-\bfx_j}^2 \leq (1-2\eta) \hbar_t^*$. Then the corresponding deterministic proxy must satisfy $\theta_{k\ell} < \hbar_t^*$, since otherwise
\begin{equation*}
    \norm{\bfx_i-\bfx_j}^2
    \geq (1-\eta)\theta_{k\ell}
    \geq (1-\eta)\hbar_t^*
    > (1-2\eta)\hbar_t^* .
\end{equation*}
which leads to a contradiction. In particular, only pairwise distances associated with deterministic proxies strictly smaller than $\hbar_t^*$ can contribute to $G((1-2\eta)\hbar_t^*)$. Hence, we have
\begin{equation*}
    G((1-2\eta)\hbar_t^*)\leq \bar{G}(\hbar_t^*-)<\omega_t
    \quad \text{ and thus } \quad
    h_t \geq (1-2\eta) \hbar_t^*.
\end{equation*}
Combining the two bounds, we obtain \eqref{tmp:ht-upper-lower} and therefore $\abs{h_t-\hbar_t^*} / \hbar_t^* \leq 2\eta = 2 n^{\varepsilon}\rho$. This proves \eqref{eqn:close-admissible}.
\end{proof}

\begin{proof}[Proof of Corollary \ref{coro:pertur-adaptive}]
Throughout the proof, let $\Xi$ denote the high-probability event introduced in the proof of Lemma \ref{lemma:bandwidth-regularity}. On $\Xi$, the data-adaptive bandwidth vector $\boldsymbol{h} = (h_t)_{t=1}^T$ satisfies
\begin{equation}
    \bbh \equiv \bbh (\varepsilon)
    := \prod\nolimits_{t=1}^T 
    \braks[\big]{ (1 - n^\varepsilon \rho) \hbar_t^*,
    (1 + n^\varepsilon \rho) \hbar_t^* } .
    \label{def:bbH}
\end{equation}
Hereafter, we choose $\varepsilon \in (0, c_{\ref{assump:gap-hbar}}/2)$. Under Assumption \ref{assump:gap-hbar}, this choice ensures that $n^\varepsilon\rho \ll 1$. Rather than analyzing $\bfK(\boldsymbol{h})$ directly for the random, data-adaptive bandwidth vector $\boldsymbol{h}$, we establish estimates uniformly over the family of kernel matrices $\{\bfK(\boldsymbol{\hbar}): \boldsymbol{\hbar} \in \bbh\}$ indexed by deterministic bandwidth vectors $\boldsymbol{\hbar}$. The desired conclusions for $\bfK(\boldsymbol{h})$ then follow by restricting these uniform estimates to the event $\Xi$, on which $\boldsymbol{h} \in \bbh$. We only prove \eqref{eqn:bfu-baru-adaptive}; the proof of \eqref{eqn:bfu-baru-adaptive-with-eigenvalue}, based on \eqref{bound:Kev-with-eigs}, is entirely analogous.

We first establish a deterministic perturbation estimate for the informative eigenvalues. We have
\begin{equation}
    \max\nolimits_{k \in \dbraks{m}} 
    \abs{\barlamb_k (\boldsymbol{\hbar})
    - \barlamb_k (\boldsymbol{\hbar}^*)}
    \leq 
    \norm{\barK (\boldsymbol{\hbar}) - \barK (\boldsymbol{\hbar}^*)} 
    \leq 
    \norm{\barK (\boldsymbol{\hbar}) - \barK (\boldsymbol{\hbar}^*)}_{\Fnorm}
    \lesssim n^{1+\varepsilon} \rho,
    \qfor \boldsymbol{\hbar} \in \bbh.
    \label{tmp:pertur-barK-bbH}
\end{equation}
To see this, note that $\caS_{k\ell}(0)$ and $\caS_{k\ell}(1)$ remain unchanged as $\boldsymbol{\hbar}$ varies over $\bbh$. For $t\in\caS_{k\ell}(0)$, we have
\begin{equation*}
    \abs[\big]{\exp (- \theta_{k\ell}/\hbar_t)
    - \exp (- \theta_{k\ell}/\hbar_t^*)}
    \leq \theta_{k\ell} \abs{1 / \hbar_t - 1 / \hbar_t^*}
    \lesssim n^\varepsilon \rho.
\end{equation*}
Here, we used the definition of $\bbh$ and the fact that $\theta_{k\ell} \lesssim \hbar_t^*$ for $t\in\caS_{k\ell}(0)$. On the other hand, for $t\in\caS_{k\ell}(1)$,
\begin{equation*}
    \abs[\big]{\exp (- \theta_{k\ell}/\hbar_t)
    - \exp (- \theta_{k\ell}/\hbar_t^*)}
    \leq \exp (- \theta_{k\ell}/\hbar_t)
    + \exp (- \theta_{k\ell}/\hbar_t^*)
    \lesssim \exp (- n^{c_{\ref{assump:multi-scale}} / 2})
    \lesssim n^\varepsilon \rho.
\end{equation*}
Combining the two preceding bounds leads to \eqref{tmp:pertur-barK-bbH}. Now, combining \eqref{tmp:pertur-barK-bbH} with the eigengap condition \eqref{eqn:gap-adaptive}, and recalling that $n^{\varepsilon} \rho \ll 1$, we obtain
\begin{equation}
    \Delta_{r} (\barLamb (\boldsymbol{\hbar}^*)) 
    \geq c_{\ref{assump:gap-hbar}} n / 2,
    \qfor \boldsymbol{\hbar} \in \bbh.
    \label{tmp:claim-1}
\end{equation}
We next control the difference between the sample and informative eigenvalues. On $\Xi$, we have
\begin{equation}
    \max\nolimits_{k \in \dbraks{m}} 
    \abs{\lambda_k (\boldsymbol{\hbar})
    - \barlamb_k (\boldsymbol{\hbar})}
    \leq 
    \norm{\bfK (\boldsymbol{\hbar}) - \barK (\boldsymbol{\hbar})} 
    \leq 
    \norm{\bfK (\boldsymbol{\hbar}) - \barK (\boldsymbol{\hbar})}_{\Fnorm}
    \lesssim n^{\varepsilon} (1 + n \rho),
    \label{tmp:pertur-diffK-bbH}
\end{equation}
uniformly over $\boldsymbol{\hbar}\in\bbh$. This estimate can be deduced by repeating the argument in Section \ref{subsec:norm-bound-residual}, while using the uniform distance concentration estimate \eqref{eqn:distant-event}. We omit the details. Combining \eqref{tmp:pertur-diffK-bbH} with \eqref{tmp:claim-1} yields
\begin{equation}
    \Delta_{r} (\bfLamb (\boldsymbol{\hbar}^*)) 
    = [\lambda_{r} (\boldsymbol{\hbar})
    - \lambda_{r + 1} (\boldsymbol{\hbar})]
    \wedge 
    \lambda_{r} (\boldsymbol{\hbar})
    \geq c_{\ref{assump:gap-hbar}} n / 4,
    \qfor \boldsymbol{\hbar} \in \bbh.
    \label{tmp:claim-2}
\end{equation}

We now start the derivation of \eqref{eqn:bfu-baru-adaptive}. Fix arbitrary $\varepsilon_0 > 0$ and $C_0>0$. Retain the high-probability event $\Xi$ defined above, and choose its probability parameter such that $\bbp(\Xi)\geq 1-n^{-C_0}$. Recall that $\boldsymbol{h} \in \bbh$ on $\Xi$. It therefore suffices to show that
\begin{subequations}
\begin{equation}
    \bbp \curls[\big] { 
    \sup\nolimits_{\boldsymbol{\hbar} \in \bbh}  
    \norm{\bfU_r (\boldsymbol{\hbar}) 
    \fkS_r^{\bfu} (\boldsymbol{\hbar})
    - \barU_r (\boldsymbol{\hbar}) }_{2,\infty}
    \leq 
    n^{\varepsilon_0} (\rho / \sqrt{n} + 1 / n^{3/2}) }
    \geq 1 - 2 n^{-C_0}.
    \label{tmp:sup-inside}
\end{equation}
By \eqref{tmp:claim-1}, the eigengap condition $\Delta_{r} (\barLamb (\boldsymbol{\hbar})) \gtrsim n$ holds uniformly over $\boldsymbol{\hbar}\in\bbh$. Inspection of the proof of \eqref{bound:Kev-without-eigs} shows that the same argument applies to every deterministic $\boldsymbol{\hbar}\in\bbh$. In particular,
\begin{equation}
    \inf\nolimits_{\boldsymbol{\hbar} \in \bbh}
    \bbp \curls[\big] {   
    \norm{\bfU_r (\boldsymbol{\hbar}) 
    \fkS_r^{\bfu} (\boldsymbol{\hbar})
    - \barU_r (\boldsymbol{\hbar}) }_{2,\infty}
    \leq 
    n^{\varepsilon_0/2} (\rho / \sqrt{n} + 1 / n^{3/2}) }
    \geq 1 - 2 n^{-2 C_1}.
    \label{tmp:sup-outside}
\end{equation}    
\end{subequations}
where $C_1>0$ is a sufficiently large constant to be specified below. To upgrade the pointwise estimate \eqref{tmp:sup-outside} to the uniform estimate \eqref{tmp:sup-inside}, we use an $\varepsilon$-net argument. The key technical input is the following local continuity estimate, whose proof is deferred to the end of this section.

\begin{lemma}[local continuity]
\label{lemma:Lipschitz-norm}
There exists a fixed constant $D_0>0$ such that, on the event $\Xi$, for all bandwidth vectors $\boldsymbol{\hbar},\boldsymbol{\hbar}^\prime \in \bbh$ satisfying $\norm{\boldsymbol{\hbar} - \boldsymbol{\hbar}^\prime} \leq n^{-(D_0+1)}$,
\begin{equation}
    \abs[\big]{ \norm{\bfU_r (\boldsymbol{\hbar}) 
    \fkS_r^{\bfu} (\boldsymbol{\hbar})
    - \barU_r (\boldsymbol{\hbar}) }_{2,\infty}
    - \norm{\bfU_r (\boldsymbol{\hbar}^\prime) 
    \fkS_r^{\bfu} (\boldsymbol{\hbar}^\prime)
    - \barU_r (\boldsymbol{\hbar}^\prime) }_{2,\infty} }
    \leq n^{D_0} \norm{\boldsymbol{\hbar} 
    - \boldsymbol{\hbar}^\prime}.
    \label{tmp:Lipschitz-continuity}
\end{equation}
\end{lemma}

Here, the distance between two bandwidth vectors is measured in the usual Euclidean norm, $\norm{\boldsymbol{\hbar} - \boldsymbol{\hbar}^\prime}^2 = \sum_{t=1}^T \abs{\hbar_t - \hbar_t^\prime}^2$. Assuming Lemma \ref{lemma:Lipschitz-norm}, let $\mathbb{A}$ be an $n^{-(D_0+2)}$-net of $\bbh$ with respect to the Euclidean norm. As $n^{- C_{\ref{assump:tech-const}}} \lesssim \vartheta_s \lesssim n^{C_{\ref{assump:tech-const}}}$, the hypercube $\bbh$ is a product of $T$ intervals whose lengths have at most polynomial order in $n$. Therefore, there exists a fixed constant $D_1>0$ such that $\abs{\mathbb{A}} \leq n^{D_0 + D_1}$. Now choose $C_1 > C_0 + D_0 + D_1$. Applying \eqref{tmp:sup-outside} to each $\boldsymbol{\hbar}\in\mathbb{A}$ and taking a union bound yields
\begin{equation}
    \bbp \curls[\big]{
    \sup\nolimits_{\boldsymbol{\hbar} \in \mathbb{A}} 
    \norm{\bfU_r (\boldsymbol{\hbar}) 
    \fkS_r^{\bfu} (\boldsymbol{\hbar})
    - \barU_r (\boldsymbol{\hbar}) }_{2,\infty}
    \leq 
    n^{\varepsilon_0 / 2} (\rho / \sqrt{n} + 1 / n^{3/2})}
    \geq 
    1 - 2 n^{-C_0}.
    \label{tmp:bound-on-net}
\end{equation}
Now suppose that the event in \eqref{tmp:bound-on-net} holds, and fix an arbitrary $\boldsymbol{\hbar} \in \bbh$. By the definition of $\mathbb{A}$, there exists some $\boldsymbol{\hbar}' \in \mathbb{A}$ such that $\norm{\boldsymbol{\hbar} - \boldsymbol{\hbar}^\prime} \leq n^{-(D_0 + 2)}$. Consequently, Lemma
\ref{lemma:Lipschitz-norm} implies that
\begin{align*}
    \norm{\bfU_r (\boldsymbol{\hbar}) 
    \fkS_r^{\bfu} (\boldsymbol{\hbar})
    - \barU_r (\boldsymbol{\hbar}) }_{2,\infty}
    & \leq \norm{\bfU_r (\boldsymbol{\hbar}^\prime) 
    \fkS_r^{\bfu} (\boldsymbol{\hbar}^\prime)
    - \barU_r (\boldsymbol{\hbar}^\prime) }_{2,\infty}
    + n^{D_0} \norm{\boldsymbol{\hbar} 
    - \boldsymbol{\hbar}^\prime} \\
    & \leq n^{\varepsilon_0 / 2} (\rho / \sqrt{n} + 1 / n^{3/2})
    + 1 / n^{2}
    \leq n^{\varepsilon_0} (\rho / \sqrt{n} + 1 / n^{3/2}).    
\end{align*}
This proves \eqref{tmp:sup-inside} and complete the proof of \eqref{eqn:bfu-baru-adaptive}.
\end{proof}

It remains to prove Lemma \ref{lemma:Lipschitz-norm}. We emphasize that the purpose of this lemma is not to obtain a sharp continuity estimate for the entrywise perturbation error in \eqref{tmp:Lipschitz-continuity}. Rather, for the $\varepsilon$-net argument used in the proof of Corollary \ref{coro:pertur-adaptive}, it suffices to establish a local Lipschitz bound whose constant grows at most polynomially in $n$. Accordingly, we only seek a crude bound below.

\begin{proof}[Proof of Lemma \ref{lemma:Lipschitz-norm}]
Throughout the proof, we work on the event $\Xi$. First, we have
\begin{equation*}
    \norm{\bfK (\boldsymbol{\hbar}) - \bfK (\boldsymbol{\hbar}^\prime)}
    \leq \norm{\bfK (\boldsymbol{\hbar}) - \bfK (\boldsymbol{\hbar}^\prime)}_{\Fnorm}
    \lesssim n^{3 C_{\ref{assump:tech-const}} + 1} 
    \norm{\boldsymbol{\hbar} - \boldsymbol{\hbar}^\prime},
\end{equation*}
which can be deduced from
\begin{equation*}
    \abs[\big]{\exp (- \norm{\bfx_i - \bfx_j}^2 / \hbar_t)
    - \exp (- \norm{\bfx_i - \bfx_j}^2 / \hbar_t^\prime)}
    \lesssim \norm{\bfx_i - \bfx_j}^2
    \cdot \abs{1/\hbar_t - 1/\hbar_t^\prime}
    \lesssim n^{3 C_{\ref{assump:tech-const}}} \abs{\hbar_t - \hbar_t^\prime},
\end{equation*}  
where we used \eqref{eqn:distant-event} and the technical assumption $n^{- C_{\ref{assump:tech-const}}} \lesssim \vartheta_s\lesssim n^{C_{\ref{assump:tech-const}}}$. On the other hand, Combining \eqref{eqn:gap-adaptive} with \eqref{tmp:pertur-barK-bbH} and \eqref{tmp:pertur-diffK-bbH}, we have $\dist ( \bfLamb_r (\boldsymbol{\hbar}), \bfLamb_\perp (\boldsymbol{\hbar}^\prime) ) \gtrsim n$. Therefore, the Davis--Kahan $\sin\Theta$ theorem gives
\begin{subequations}
\begin{equation}
    \norm{\bfU_r (\boldsymbol{\hbar}) \bfU_r (\boldsymbol{\hbar})^\top 
    - \bfU_r (\boldsymbol{\hbar}^\prime) \bfU_r (\boldsymbol{\hbar}^\prime)^\top}
    \leq \frac{\norm{\bfK (\boldsymbol{\hbar}) - \bfK (\boldsymbol{\hbar}^\prime)}}
    {\dist ( \bfLamb_r (\boldsymbol{\hbar}), 
    \bfLamb_\perp (\boldsymbol{\hbar}^\prime) )}
    \lesssim n^{3 C_{\ref{assump:tech-const}}} 
    \norm{\boldsymbol{\hbar} - \boldsymbol{\hbar}^\prime}.
    \label{tmp:lip-bfU}
\end{equation}
A similar argument gives
\begin{equation}
    \norm{\barU_r (\boldsymbol{\hbar}) \barU_r (\boldsymbol{\hbar})^\top 
    - \barU_r (\boldsymbol{\hbar}^\prime) \barU_r (\boldsymbol{\hbar}^\prime)^\top}
    \lesssim n^{3 C_{\ref{assump:tech-const}}} 
    \norm{\boldsymbol{\hbar} - \boldsymbol{\hbar}^\prime}.
    \label{tmp:lip-barU}
\end{equation}    
\end{subequations}
We now choose the constant $D_0$ sufficiently large so that $D_0 \geq 3 C_{\ref{assump:tech-const}}$. Under the assumption $\norm{\boldsymbol{\hbar} - \boldsymbol{\hbar}^\prime} \leq n^{-(D_0 + 1)}$, the r.h.s. of \eqref{tmp:lip-bfU} and \eqref{tmp:lip-bfU} are controlled by $O(n^{-1})$. 

We are now ready to prove \eqref{tmp:Lipschitz-continuity}. For notational simplicity, write
\begin{equation*}
    \bfU_1 \equiv \bfU_r (\boldsymbol{\hbar}),
    \quad
    \barU_1 \equiv \barU_r (\boldsymbol{\hbar}),
    \quad
    \bfU_2 \equiv \bfU_r (\boldsymbol{\hbar}^\prime),
    \quad
    \barU_2 \equiv \barU_r (\boldsymbol{\hbar}^\prime).
\end{equation*}
Since the $\ell_{2,\infty}$ norm is invariant under right
multiplication by an orthogonal matrix, we have
\begin{align}
\begin{split}
    & ~ \abs[\big]{\norm{\bfU_1 \sgn (\bfU_1^\top \barU_1) - \barU_1}_{2, \infty}
    - \norm{\bfU_2 \sgn (\bfU_2^\top \barU_2) - \barU_2}_{2, \infty}} \\
    \leq & ~ \norm[\big]{ [\bfU_1 \sgn (\bfU_1^\top \barU_1) - \barU_1] 
    \sgn (\barU_1^\top \barU_2)
    - [{\bfU_2 \sgn (\bfU_2^\top \barU_2) - \barU_2}] }_{2, \infty}
    = \norm{ \mathrm{I_{\ref{eqn:U-Lipschitz}}}
    + \cdots
    + \mathrm{IV_{\ref{eqn:U-Lipschitz}}} }_{2, \infty},
\end{split} \label{eqn:U-Lipschitz}
\end{align}
where the terms on the r.h.s. are defined as
\begin{align*}
    \mathrm{I_{\ref{eqn:U-Lipschitz}}} 
    & := \bfU_1 \sgn \pars[\big]{\bfU_1^\top 
    [\barU_1 \sgn (\barU_1^\top \barU_2) - \barU_2]}, \\
    \mathrm{II_{\ref{eqn:U-Lipschitz}}} 
    & := \bfU_1 \sgn \pars[\big]{
    [\bfU_1 - \bfU_2 \sgn (\bfU_2^\top \bfU_1)]^\top \barU_2}, \\
    \mathrm{III_{\ref{eqn:U-Lipschitz}}} 
    & := [\bfU_1 \sgn (\bfU_1^\top \bfU_2) - \bfU_2] 
    \sgn (\bfU_2^\top \barU_2), \\
    \mathrm{IV_{\ref{eqn:U-Lipschitz}}} 
    & := - [\barU_1 \sgn (\barU_1^\top \barU_2) - \barU_2].
\end{align*}
The last two terms can be controlled directly by \eqref{tmp:lip-bfU} and \eqref{tmp:lip-barU},
\begin{align*}
    \norm{\mathrm{III_{\ref{eqn:U-Lipschitz}}}}_{2, \infty}
    & \leq \norm{\bfU_1 \sgn (\bfU_1^\top \bfU_2) - \bfU_2}
    \lesssim \norm{\bfU_1 \bfU_1^\top - \bfU_2 \bfU_2^\top}
    \lesssim n^{3 C_{\ref{assump:tech-const}}} 
    \norm{\boldsymbol{\hbar} - \boldsymbol{\hbar}^\prime} ,\\
    \norm{\mathrm{IV_{\ref{eqn:U-Lipschitz}}}}_{2, \infty}
    & \leq \norm{\barU_1 \sgn (\barU_1^\top \barU_2) - \barU_2}
    \lesssim \norm{\barU_1 \barU_1^\top - \barU_2 \barU_2^\top}
    \lesssim n^{3 C_{\ref{assump:tech-const}}} 
    \norm{\boldsymbol{\hbar} - \boldsymbol{\hbar}^\prime}.     
\end{align*}
To control the first two terms, we use the local Lipschitz continuity of the polar factor:
\begin{equation*}
    \norm{\sgn (\fkB_1) - \sgn (\fkB_2)}
    \lesssim \frac{1}{\sigma_{\min} (\fkB_1) + \sigma_{\min} (\fkB_2)}
    \norm{\fkB_1 - \fkB_2}.
\end{equation*}
Applying this bound together with \eqref{tmp:lip-bfU} and \eqref{tmp:lip-barU}, we obtain
\begin{align*}
    \norm{\mathrm{I_{\ref{eqn:U-Lipschitz}}}}_{2, \infty}
    & \leq \norm{\barU_1 \sgn (\barU_1^\top \barU_2) - \barU_2}
    \lesssim n^{3 C_{\ref{assump:tech-const}}} 
    \norm{\boldsymbol{\hbar} - \boldsymbol{\hbar}^\prime},\\
    \norm{\mathrm{I
    I_{\ref{eqn:U-Lipschitz}}}}_{2, \infty}
    & \leq \norm{\bfU_1 - \bfU_2 \sgn (\bfU_2^\top \bfU_1)}
    \lesssim n^{3 C_{\ref{assump:tech-const}}} 
    \norm{\boldsymbol{\hbar} - \boldsymbol{\hbar}^\prime} .    
\end{align*}
To summarize, the l.h.s. of \eqref{eqn:U-Lipschitz} can be controlled by $n^{3 C_{\ref{assump:tech-const}}} \norm{\boldsymbol{\hbar} - \boldsymbol{\hbar}^\prime}$. This concludes the proof.
\end{proof}

\subsection{Exact recovery}
\label{subsec:misclass-rate}

Our control of the misclassification rate $\caM(\pi,\hat{\pi})$ is inspired by the argument in \cite{abbeLpTheoryPCA2022}. In particular, we rely on \cite[Lemma D.1]{abbeLpTheoryPCA2022}. For completeness, we state a version of this lemma in our notation and include its proof. Let $\{ \fkx_i \}_{i=1}^n \subset \bbr^r$ be a dataset, where $r$ is fixed. For a labeling map $\pi : \dbraks{n} \to \dbraks{m}$ and a collection of centers $\{ \fka_k \}_{k=1}^m \subset \bbr^r$, define the associated $K$-means loss by
\begin{equation}
    \caL ( \pi, \{ \fka_k \} )
    = \sum\nolimits_{i=1}^{n} \norm{\fkx_i - \fka_{\bar{\pi} (i)}}^2.
    \label{def:loss-Kmeans}
\end{equation}

\begin{lemma} 
\label{lemma:misclassification}
Let $m \geq 2$, and suppose that we are given two collections of centers $\{ \bar{\fka}_k \}_{k=1}^m, \{ \hat{\fka}_k \}_{k=1}^m  \subseteq \bbr^r$ together with two labeling maps $\bar{\pi}, \hat{\pi}: \dbraks{n} \to \dbraks{m}$. Define the minimum separation between the reference centers and the smallest reference cluster size by
\begin{equation*}
    \bar{\fks} = \min\nolimits_{k \neq \ell} \, \norm{\bar{\fka}_k - \bar{\fka}_{\ell}}
    \qand
    n_{\min} = \min\nolimits_{k \in \dbraks{m}} \, \abs{\bar{\pi}^{-1} (k)}.
    \label{tmp:fks-bar}
\end{equation*}
Suppose that, for some $B > 0$ and $\eta \in (0,\bar{\fks}/2)$, the following two conditions hold:
\begin{enumerate}[label = (\roman*)]
\begin{subequations}
    \item The clustering $(\hat{\pi}, \{\hat{\fka}_k\})$ is near-optimal, in the sense that
    \begin{equation}
        \caL ( \hat{\pi}, \{ \hat{\fka}_k \} ) 
        \leq B \caL ( \bar{\pi}, \{ \bar{\fka}_k \} ) 
        \qand
        \hat{\pi} (i) \in \argmin\nolimits_{k \in \dbraks{m}} 
        \, \norm{\fkx_i - \hat{\fka}_k},
        \quad
        \forall i \in \dbraks{n}.
        \label{eqn:near-opt}
    \end{equation}
    \item The reference clustering satisfies the low-noise condition
    \begin{equation}
        \caL ( \bar{\pi}, \{ \bar{\fka}_k \} )  
        \leq \frac{\eta^2 n_{\min}}{(1 + \sqrt{B})^2 m} .
        \label{eqn:low-noise}
    \end{equation}
\end{subequations}
\end{enumerate}
Then there exists a permutation $\tau : \dbraks{m} \to \dbraks{m}$ such that
\begin{equation}
    \curls{ i \in \dbraks{n} : 
    \hat{\pi} (i) \neq (\tau \circ \bar{\pi}) (i)}
    \subset 
    \curls{ i \in \dbraks{n} : 
    \norm{ \fkx_i - \bar{\fka}_{\bar{\pi} (i)} } \geq \bar{\fks} / 2 - \eta }.
    \label{eqn:inclusion-misclass}
\end{equation}
\end{lemma}

Let us emphasize that Lemma \ref{lemma:misclassification} does not require $(\{ \bar{\fka}_k \}, \bar{\pi})$ to minimize the $K$-means loss. Rather, the only condition imposed on this reference clustering is the low-noise condition \eqref{eqn:low-noise}.

\begin{proof}[Proof of Lemma \ref{lemma:misclassification}]
For notational convenience, define
\begin{equation*}
    \fkX = [\fkx_1, \cdots, \fkx_n]^\top 
    \in \bbr^{n \times r},
    \qquad
    \bar{\fkA} = [\bar{\fka}_1, \cdots, \bar{\fka}_m]^\top 
    \in \bbr^{m \times r},
    \qquad
    \bar{\bfPi} = [\bfe_{\bar{\pi} (1)}, \cdots, \bfe_{\bar{\pi} (n)}]^\top 
    \in \bbr^{n \times m}.
\end{equation*}
Define $\hat{\fkA}$ and $\hat{\bfPi}$ analogously from $\{\hat{\fka}_k\}_{k=1}^m$ and $\hat{\pi}$, respectively. Thus, the $i$th row of $\bar{\bfPi}\bar{\fkA}$ is $\bar{\fka}_{\bar{\pi}(i)}$, while the $i$th row of $\hat{\bfPi}\hat{\fkA}$ is $\hat{\fka}_{\hat{\pi}(i)}$. By the triangle inequality and the definition of the $K$-means loss in \eqref{def:loss-Kmeans},
\begin{equation*}
    \norm{\bar{\bfPi} \bar{\fkA} - \hat{\bfPi} \hat{\fkA}}_\Fnorm
    \leq 
    \norm{\fkX - \bar{\bfPi}\bar{\fkA}}_\Fnorm 
    + \norm{\fkX - \hat{\bfPi} \hat{\fkA}}_\Fnorm
    = \sqrt{\caL ( \bar{\pi}, \{ \bar{\fka}_k \} )}
    + \sqrt{\caL ( \hat{\pi}, \{ \hat{\fka}_k \} )}
    \leq (1+\sqrt{B}) \sqrt{\caL ( \bar{\pi}, \{ \bar{\fka}_k \} )} ,   
\end{equation*}
where the last step uses the near-optimality condition \eqref{eqn:near-opt}. Hence, by the low-noise condition \eqref{eqn:low-noise},
\begin{equation}
    \norm{\bar{\bfPi} \bar{\fkA} - \hat{\bfPi} \hat{\fkA}}_\Fnorm^2
    \leq
    (1+\sqrt{B})^2 \caL ( \bar{\pi}, \{ \bar{\fka}_k \} )
    \leq {\eta^2 n_{\min}} / {m}
    < {\bar{\fks}^2 n_{\min}} / ({4m}).
    \label{eqn:YM-YMstar}
\end{equation}

We next align the estimated centers $\{ \hat{\fka}_k \}_{k=1}^m$ with the reference centers $\{ \bar{\fka}_k \}_{k=1}^m$. For $k,\ell \in \dbraks{m}$, let
\begin{equation*}
    \fkn_{k\ell} 
    = \abs{ \bar{\pi}^{-1} (k) \cap \hat{\pi}^{-1} (\ell) }
    = \abs{ \{i \in \dbraks{n}: \bar{\pi} (i) = k \text{ and } \hat{\pi} (i) = \ell \}}.
\end{equation*}
Thus, $\fkn_{k\ell}$ represents the number of data points assigned to the $k$-th reference cluster and the $\ell$-th estimated cluster. For every $k\in\dbraks{m}$, define $\tau(k) := \arg\max_{\ell \in \dbraks{m}} \fkn_{k\ell}$, where ties are resolved by choosing the smallest index. It is not hard to see that $\fkn_{k,\tau(k)} \geq {n_{\min}} / {m}$ for every $k \in \dbraks{m}$. 

We claim that $\tau$ is a permutation of $\dbraks{m}$. Since $\tau$ maps the finite set $\dbraks{m}$ into itself, it suffices to show that $\tau$ is injective. Suppose, to the contrary, that there exist distinct $k,\ell\in\dbraks{m}$ and some $v\in\dbraks{m}$ such that $\tau(k)=\tau(\ell)=v$. Then, by the triangle inequality and the definition of $\bar{\fks}$,
\begin{equation*}
    \norm{\bar{\fka}_k - \hat{\fka}_v}^2
    + \norm{\hat{\fka}_v - \bar{\fka}_{\ell}}^2
    \geq \pars[\big]{ \norm{\bar{\fka}_k - \hat{\fka}_v} 
    + \norm{\hat{\fka}_v - \bar{\fka}_{\ell}} }^2 / 2
    \geq \norm{\bar{\fka}_k - \bar{\fka}_{\ell}}^2 / 2
    \geq \bar{\fks}^2/2.
\end{equation*}
It follows that
\begin{equation*}
    \norm{\bar{\bfPi} \bar{\fkA} - \hat{\bfPi} \hat{\fkA}}_\Fnorm^2
    =
    \sum\nolimits_{k',\ell' \in \dbraks{m}} \fkn_{k'\ell'} 
    \norm{\bar{\fka}_{k'} - \hat{\fka}_{\ell'}}_2^2
    \geq
    \fkn_{k v} \norm{\bar{\fka}_k - \hat{\fka}_v}^2
    + \fkn_{\ell v} \norm{\bar{\fka}_{\ell} - \hat{\fka}_v}^2
    \geq \bar{\fks}^2 n_{\min} / (2 m),
\end{equation*}
which contradicts \eqref{eqn:YM-YMstar}. Hence, $\tau$ is a permutation. Next, we use \eqref{eqn:YM-YMstar} again to get
\begin{equation*}
    \sum\nolimits_{k=1}^m
    \norm{\bar{\fka}_k-\hat{\fka}_{\tau(k)}}^2
    \leq
    \frac{m}{n_{\min}}
    \sum\nolimits_{k=1}^m
    \fkn_{k,\tau(k)}
    \norm{\bar{\fka}_k-\hat{\fka}_{\tau(k)}}^2
    \leq
    \frac{m}{n_{\min}}
    \norm{\bar{\bfPi}\bar{\fkA}-\hat{\bfPi}\hat{\fkA}}_\Fnorm^2
    \leq \eta^2.
\end{equation*}
In particular, the estimated centers and the reference centers are algined in the sense that
\begin{equation}
    \max\nolimits_{k \in \dbraks{m}}\norm{\bar{\fka}_{k} - \hat{\fka}_{\tau(k)}} 
    \leq \eta.
    \label{tmp:center-alignment}
\end{equation}

It remains to show that a point sufficiently close to its reference center must receive the corresponding estimated label. Fix $i\in\dbraks{n}$, and write $k:=\bar{\pi}(i)$. For every $\ell\neq k$, it follows from \eqref{tmp:center-alignment} that
\begin{align*}
    \norm{\fkx_i - \hat{\fka}_{\tau(\ell)}}
    - \norm{\fkx_i - \hat{\fka}_{\tau(k)}}
    & \geq 
    \norm{\fkx_i - \bar{\fka}_{\ell}} - 
    \norm{\fkx_i - \bar{\fka}_{k} } -2 \eta \\
    & \geq
    \pars[\big]{\norm{\bar{\fka}_{\ell} - \bar{\fka}_{k} }
    - \norm{\fkx_i - \bar{\fka}_{k} }}
    - \norm{\fkx_i - \bar{\fka}_{k} } -2 \eta 
    \geq
    \bar{\fks} - 2(\norm{\fkx_i - \bar{\fka}_{\bar{\pi}(i)} } + \eta).
\end{align*}
Consequently, whenever $\norm{\fkx_i - \bar{\fka}_{\bar{\pi} (i)} } < \bar{\fks} / 2 - \eta$, we have
\begin{equation*}
    \norm{\fkx_i - \hat{\fka}_{\tau(\ell)}} > \norm{\fkx_i - \hat{\fka}_{\tau(\bar{\pi} (i))}},
    \qfor
    \ell \neq \bar{\pi} (i).
\end{equation*}
In other words, $\hat{\fka}_{\tau(\bar{\pi}(i))}$ is the unique nearest estimated center to $\fkx_i$. By the nearest-center property in \eqref{eqn:near-opt}, we have $\hat{\pi}(i) = \tau(\bar{\pi}(i))$. This proves \eqref{eqn:inclusion-misclass}.
\end{proof}

Apart from Lemma \ref{lemma:misclassification}, we also need the following result, which shows that the separation condition in Assumption \ref{assump:between-cluster} continues to hold uniformly over $\boldsymbol{\hbar} \in \bbh$, after possibly decreasing the separation constant.

\begin{lemma}[separation]
\label{lemma:separation}
Let $\bbh \equiv \bbh(\varepsilon)$ be defined as in \eqref{def:bbH}, where $0 < \varepsilon < (c_{\ref{assump:gap-hbar}} \wedge c_{\ref{assump:between-cluster}})/2$. Then, under the same setup as in Theorem \ref{thm:zero-misclass},
\begin{equation}
    \bar{\fks} (\boldsymbol{\hbar}) 
    = \min\nolimits_{k \neq \ell} 
    \norm{\bar{\fka}_k (\boldsymbol{\hbar}) 
    - \bar{\fka}_\ell (\boldsymbol{\hbar})}
    \geq n^{c_{\ref{assump:between-cluster}} / 2} 
    (\rho + 1/n),
    \qfor \boldsymbol{\hbar} \in \bbh.
    \label{claim:separation-uniform}
\end{equation}
\end{lemma}

\begin{proof}[Proof of Lemma \ref{lemma:separation}]
Without loss of generality, we take the embedding dimension in \eqref{def:fka-center} to be $r = m$. Thus, the population centers can be written as
\begin{equation*}
    \bar{\fka}_k (\boldsymbol{\hbar}) 
    = \barLamb (\boldsymbol{\hbar})^{1/2} \fkU (\boldsymbol{\hbar})^\top \bfe_k / \sqrt{n_k}
    \in \bbr^{m},
\end{equation*}
We prove \eqref{claim:separation-uniform} by combining the separation condition \eqref{cond:separation-hstar} at $\boldsymbol{\hbar}^*$ with a standard perturbation analysis for $\fkU(\boldsymbol{\hbar})$ and $\barLamb(\boldsymbol{\hbar})$. Fix $\boldsymbol{\hbar} \in \bbh$. Define the alignment matrix between
$\fkU(\boldsymbol{\hbar})$ and $\fkU(\boldsymbol{\hbar}^*)$ by
\begin{equation*}
    \bar{\fkQ} \equiv \bar{\fkQ} (\boldsymbol{\hbar}, \boldsymbol{\hbar}^*) 
    = \fkU (\boldsymbol{\hbar})^\top \fkU (\boldsymbol{\hbar}^*)
    \qand
    \bar{\fkS} = \sgn (\bar{\fkQ}).
\end{equation*}
Since $\bar{\fkS}$ is orthogonal, for any $k \neq \ell$,
\begin{align*}
    \norm{\bar{\fka}_k (\boldsymbol{\hbar}) 
    - \bar{\fka}_\ell (\boldsymbol{\hbar})}
    & = \norm{\bar{\fkS}^\top \bar{\fka}_k (\boldsymbol{\hbar}) 
    - \bar{\fkS}^\top \bar{\fka}_\ell (\boldsymbol{\hbar})} \\
    & \geq \norm{\bar{\fka}_k (\boldsymbol{\hbar}^*) 
    - \bar{\fka}_\ell (\boldsymbol{\hbar}^*)}
    - \norm{\bar{\fkS}^\top 
    \bar{\fka}_k (\boldsymbol{\hbar}) 
    - \bar{\fka}_k (\boldsymbol{\hbar}^*)}
    - \norm{\bar{\fkS}^\top 
    \bar{\fka}_\ell (\boldsymbol{\hbar}) 
    - \bar{\fka}_\ell (\boldsymbol{\hbar}^*)}.
\end{align*}
It therefore suffices to control the two center perturbation terms on the r.h.s.

Recall that $\fkK = \fkU \bar{\Lambda} \fkU^\top$. As in \eqref{tmp:pertur-barK-bbH}, an entrywise analysis of $\fkK(\boldsymbol{\hbar})$ gives
\begin{equation}
    \norm{\barLamb (\boldsymbol{\hbar}) 
    - \barLamb (\boldsymbol{\hbar}^*)}
    \leq \norm{\fkK (\boldsymbol{\hbar}) 
    - \fkK (\boldsymbol{\hbar}^*)}
    \leq \norm{\fkK (\boldsymbol{\hbar}) 
    - \fkK (\boldsymbol{\hbar}^*)}_{\Fnorm}
    \lesssim n^{1 + \varepsilon} \rho,
    \label{tmp:fkK-norm-bound}
\end{equation}
uniformly over $\boldsymbol{\hbar} \in \bbh$. Now, the Davis--Kahan $\sin\Theta$ theorem, together with \eqref{tmp:fkK-norm-bound} and the first condition in the eigengap assumption \eqref{eqn:gap-adaptive}, yields
\begin{equation*}
    \norm{\fkU (\boldsymbol{\hbar}) \fkU (\boldsymbol{\hbar})^\top
    - \fkU (\boldsymbol{\hbar}^*) \fkU (\boldsymbol{\hbar}^*)^\top}
    \lesssim \frac{\norm{\fkK (\boldsymbol{\hbar}) 
    - \fkK (\boldsymbol{\hbar}^*)}}
    {\dist ( \barLamb (\boldsymbol{\hbar}), 
    \barLamb_\perp (\boldsymbol{\hbar}^*) )}
    \lesssim n^{\varepsilon} \rho.
\end{equation*}
Consequently,
\begin{equation}
    \norm{\bar{\fkQ} - \bar{\fkS}}
    \lesssim n^{2 \varepsilon} \rho^2
    \qand
    \norm{\fkU (\boldsymbol{\hbar}) \bar{\fkS}
    - \fkU (\boldsymbol{\hbar}^*)}
    \lesssim n^{\varepsilon} \rho.
    \label{tmp:perturb-fkU}
\end{equation}
Moreover, by the boundedness of the kernel function, \eqref{tmp:fkK-norm-bound}, and \eqref{eqn:gap-adaptive}, we have 
\begin{equation*}
    \norm{\barLamb (\boldsymbol{\hbar})} \asymp n
    \qand
    \norm{\barLamb (\boldsymbol{\hbar})^{-1}} \asymp 1/n ,   
\end{equation*}
uniformly over $\boldsymbol{\hbar} \in \bbh$. Hence,
\begin{align*}
    \norm{\bar{\fkS}^\top
    \bar{\fka}_k (\boldsymbol{\hbar}) 
    - \bar{\fka}_k (\boldsymbol{\hbar}^*)}
    & \leq \frac{1}{\sqrt{n_k}} 
    \norm{\fkU (\boldsymbol{\hbar})
    \barLamb (\boldsymbol{\hbar})^{1/2} 
    \bar{\fkS}
    - \fkU (\boldsymbol{\hbar}^*)
    \barLamb (\boldsymbol{\hbar}^*)^{1/2}} \\
    & \leq \frac{1}{\sqrt{n_k}} 
    \norm[\big]{\fkU (\boldsymbol{\hbar})
    \bar{\fkS}
    [\bar{\fkS}^\top
    \barLamb (\boldsymbol{\hbar})^{1/2}
    \bar{\fkS}
    - \barLamb (\boldsymbol{\hbar}^*)^{1/2}]}
    +  \frac{1}{\sqrt{n_k}} 
    \norm[\big]{[\fkU (\boldsymbol{\hbar})
    \bar{\fkS}
    - \fkU (\boldsymbol{\hbar}^*)]
    \barLamb (\boldsymbol{\hbar}^*)^{1/2}} \\
    & \lesssim \frac{1}{\sqrt{n}}
    \norm{\bar{\fkS}^\top
    \barLamb (\boldsymbol{\hbar})^{1/2}
    \bar{\fkS}
    - \barLamb (\boldsymbol{\hbar}^*)^{1/2}}
    + n^{\varepsilon} \rho.
\end{align*}
It remains to control the first term above, which is similar to the treatment of $\mathrm{I_{\ref{eqn:decomp-ULambda-barULambda}}}$. Since the eigenvalues of $\barLamb(\boldsymbol{\hbar})$ and $\barLamb(\boldsymbol{\hbar}^*)$ are both of order $n$, the standard square-root perturbation bound gives
\begin{align*}
    \frac{1}{\sqrt{n}}
    \norm{\bar{\fkS}^\top
    \barLamb (\boldsymbol{\hbar})^{1/2}
    \bar{\fkS}
    - \barLamb (\boldsymbol{\hbar}^*)^{1/2}}
    & \lesssim 
    \frac{1}{n}
    \norm{\barLamb (\boldsymbol{\hbar}) \bar{\fkS}
    - \bar{\fkS} \barLamb (\boldsymbol{\hbar}^*)} \\
    & \leq \frac{1}{n} 
    \norm{\barLamb (\boldsymbol{\hbar}) 
    (\bar{\fkS} - \bar{\fkQ})} 
    + \frac{1}{n} 
    \norm{\fkU (\boldsymbol{\hbar})^\top 
    [\fkK (\boldsymbol{\hbar}) - \fkK (\boldsymbol{\hbar}^*)] 
    \fkU (\boldsymbol{\hbar}^*)}
    + \frac{1}{n} 
    \norm{(\bar{\fkQ} - \bar{\fkS})
    \barLamb (\boldsymbol{\hbar}^*)} \\
    & \lesssim 
    n^{\varepsilon} \rho,
\end{align*}
where the last inequality follows from \eqref{tmp:fkK-norm-bound} and \eqref{tmp:perturb-fkU}. Therefore, we have shown that,
\begin{equation*}
    \norm{\bar{\fkS}^\top \bar{\fka}_k (\boldsymbol{\hbar})
    - \bar{\fka}_k (\boldsymbol{\hbar}^*)} 
    \leq n^{2 \varepsilon} \rho,
\end{equation*}
uniformly over $k \in \dbraks{m}$ and $\boldsymbol{\hbar} \in \bbh$. Combining this bound with \eqref{cond:separation-hstar}, we obtain
\begin{equation*}
    \min\nolimits_{k \neq \ell} 
    \norm{\bar{\fka}_k (\boldsymbol{\hbar}) 
    - \bar{\fka}_\ell (\boldsymbol{\hbar})}
    \geq 
    \min\nolimits_{k \neq \ell} 
    \norm{\bar{\fka}_k (\boldsymbol{\hbar}^*) 
    - \bar{\fka}_\ell (\boldsymbol{\hbar}^*)}
    - 2 n^{2 \varepsilon} \rho
    \geq n^{c_{\ref{assump:between-cluster}}} 
    (\rho + 1/n)
    - 2 n^{2 \varepsilon} \rho
    \geq n^{c_{\ref{assump:between-cluster}} / 2} 
    (\rho + 1/n),
\end{equation*}
for all sufficiently large $n$. Here we used $\varepsilon < c_{\ref{assump:between-cluster}}$, and hence $n^{2 \varepsilon} \rho \ll n^{c_{\ref{assump:between-cluster}}}$. This proves the claim.
\end{proof}

Now, equipped with Lemmas \ref{lemma:misclassification} and \ref{lemma:separation}, we are ready to prove Theorem \ref{thm:zero-misclass}.

\begin{proof}[Proof of Theorem \ref{thm:zero-misclass}]
To apply Lemma \ref{lemma:misclassification}, we take the embeddings and reference centers to be
\begin{equation*}
    \fkx_i = (\bfU_r \bfLamb_r^{1/2} \fkS_r)^\top \bfe_i
    \qand
    \bar{\fka}_k = (\barU_r \barLamb_r^{1/2})^\top \bfe_i
    \qfor i \in \caC_k.
\end{equation*}
The definition of $\bar{\fka}_k$ agrees with \eqref{def:fka-center}. On the other hand, the embeddings $\fkx_i$ differ from those in \eqref{def:fkx-embedding} by the common orthogonal transformation $\fkS_r^{\bfu}$. Since the $K$-means loss \eqref{def:loss-Kmeans} is invariant under a common orthogonal transformation of all embedded points and cluster centers, this modification does not affect the labels produced by the $(1+\varepsilon_*)$-approximate $K$-means procedure. We therefore work with the above rotated embeddings throughout the proof. Also note that both the empirical embedding $\fkx_i$ and the informative centers $\bar{\fka}_k$ depend on the bandwidth vector. As in the proof of Corollary \ref{coro:pertur-adaptive}, it is enough to establish the desired result uniformly over all deterministic bandwidth vectors $\boldsymbol{\hbar} \in \bbh$. More precisely, we prove
\begin{equation}
    \bbp \curls{ \sup\nolimits_{\boldsymbol{\hbar} \in \bbh}
    \caM (\hat{\pi} (\boldsymbol{\hbar}), \bar{\pi}) = 0 } 
    \geq 1 - n^{-C}.
    \label{tmp:misclass-uniform-h}
\end{equation}
Once \eqref{tmp:misclass-uniform-h} is established, Lemma \ref{lemma:bandwidth-regularity} allows us to pass from deterministic $\boldsymbol{\hbar} \in \bbh$ to the data-adaptive bandwidth vector $\boldsymbol{h}$ selected by the procedure \eqref{eqn:bandwidth-selection}.

We first verify the two conditions \eqref{eqn:near-opt} and \eqref{eqn:low-noise} required by Lemma \ref{lemma:misclassification}. By the definition of the $(1+\varepsilon_*)$-approximate $K$-means algorithm,
\begin{equation*}
    \caL ( \hat{\pi}, \{ \hat{\fka}_k \} )
    \leq (1 + \varepsilon_*) \cdot \min\nolimits_{\{ \fka_k \}, \pi}
    \caL ( {\pi}, \{ {\fka}_k \} ) 
    \leq (1 + \varepsilon_*) 
    \caL ( \bar{\pi}, \{ \bar{\fka}_k \} ) .
\end{equation*}
Thus, the near-optimality condition \eqref{eqn:near-opt} holds with $B = 1 + \varepsilon_*$. 

Next, let $\varepsilon_0>0$ be a sufficiently small constant. As in the proof of Corollary \ref{coro:pertur-adaptive}, the argument leading to \eqref{tmp:sup-inside} yields the following uniform estimate with high probability:
\begin{equation}
    \sup\nolimits_{\boldsymbol{\hbar} \in \bbh}  
    \norm{\bfU_r (\boldsymbol{\hbar}) 
    \bfLamb_r (\boldsymbol{\hbar})^{1/2}
    \fkS_r^{\bfu} (\boldsymbol{\hbar})
    - \barU_r (\boldsymbol{\hbar}) 
    \barLamb_r (\boldsymbol{\hbar})^{1/2} }_{2,\infty}
    \leq 
    n^{\varepsilon_0} (\rho + 1 / n).
    \label{tmp:uniform-with-eigs}
\end{equation}
For the remainder of the proof, we work on this event. To simplify the notation, we suppress the dependence on $\boldsymbol{\hbar}$, with the understanding that all estimates below hold uniformly over $\boldsymbol{\hbar} \in \bbh$. Choose $\varepsilon_0 = c_{\ref{assump:between-cluster}} / 4$. Then, by \eqref{tmp:uniform-with-eigs} and Lemma \ref{lemma:separation}, we have
\begin{equation}
    \max\nolimits_{i \in \dbraks{n}} \norm{\fkx_i - \bar{\fka}_{\bar{\pi} (i)}}
    \leq n^{c_{\ref{assump:between-cluster}} / 4} (\rho + 1 / n)
    \leq \bar{\fks} / n^{c_{\ref{assump:between-cluster}} / 4}.
    \label{tmp:row-wise-deviation}
\end{equation}
In particular, since $n_{\min}\asymp n$ and $m\asymp 1$, we have
\begin{equation*}
    \caL ( \bar{\pi}, \{ \bar{\fka}_k \} )
    = \sum\nolimits_{i=1}^{n} \norm{\fkx_i - \bar{\fka}_{\bar{\pi} (i)}}^2
    \leq \frac{\bar{\fks}^2 n}{n^{c_{\ref{assump:between-cluster}} / 2}}
    \leq \frac{(\bar{\fks} / 4)^2 n_{\min}}{(1 + \sqrt{B})^2 m},
\end{equation*}
for all sufficiently large $n$. Hence, the low-noise condition \eqref{eqn:low-noise} is satisfied with $\eta = \bar{\fks}/4$.

We may now apply Lemma \ref{lemma:misclassification}. Let $\tau:\dbraks{m}\to\dbraks{m}$ be the permutation given by that lemma. Then
\begin{equation*}
    \caM (\hat{\pi}, \bar{\pi})
    \leq \frac{1}{n} \abs[\big]{ \curls{ i \in \dbraks{n} : 
    \hat{\pi} (i) \neq (\tau \circ \bar{\pi}) (i)} } \\
    \leq \frac{1}{n} \abs[\big]{ \curls{ i \in \dbraks{n} : 
    \norm{ \fkx_i - \bar{\fka}_{\bar{\pi} (i)} } \geq \bar{\fks} / 4} }  .  
\end{equation*}
However, \eqref{tmp:row-wise-deviation} ensures that $\max_{i \in \dbraks{n}} \norm{ \fkx_i - \bar{\fka}_{\bar{\pi} (i)} } < \bar{\fks} / 4$ for all sufficiently large $n$. Therefore,
\begin{equation*}
    \caM (\hat{\pi} (\boldsymbol{\hbar}), \bar{\pi}) = 0,
\end{equation*}
uniformly over $\boldsymbol{\hbar} \in \bbh$. This proves \eqref{tmp:misclass-uniform-h}; the conclusion of Theorem \ref{thm:zero-misclass} then follows from Lemma \ref{lemma:bandwidth-regularity}.
\end{proof}

\end{appendices}


\printbibliography


\end{document}